\documentclass[10pt,twocolumn,letterpaper]{article}

\usepackage[pagenumbers]{cvpr} 

\usepackage{graphicx}
\usepackage{amsmath}
\usepackage{amssymb}
\usepackage{booktabs}

\usepackage[pagebackref,breaklinks,colorlinks]{hyperref}

\usepackage[capitalize]{cleveref}
\crefname{section}{Sec.}{Secs.}
\Crefname{section}{Section}{Sections}
\Crefname{table}{Table}{Tables}
\crefname{table}{Tab.}{Tabs.}

\usepackage{subcaption}
\usepackage{amsmath}
\usepackage{amssymb}
\usepackage{amsthm}
\newtheorem{theorem}{Theorem}
\newtheorem{lemma}[theorem]{Lemma}
\newtheorem{definition}[theorem]{Definition}
\newtheorem{proposition}[theorem]{Proposition}
\newtheorem{corollary}[theorem]{Corollary}
\DeclareMathOperator*{\E}{\mathbb{E}}
\usepackage{multirow}
\usepackage{hhline}

\usepackage{alphalph}
\def\cvprPaperID{9771} 
\def\confName{CVPR}
\def\confYear{2023}

\begin{document}

\graphicspath{{./figures/schemes/}{./figures/results/}{./figures/images/}{./figures/results/}{./figures/images/davis/}{./figures/images/crvd/}}

\title{Patch-Craft Self-Supervised Training for Correlated Image Denoising}

\author{Gregory Vaksman and Michael Elad\\
CS Department - The Technion\\
Haifa, Israel\\
{\tt\small grishav@campus.technion.ac.il, elad@cs.technion.ac.il}
}
\maketitle

\begin{abstract}
   Supervised neural networks are known to achieve excellent results in various image restoration tasks. However, such training requires datasets composed of pairs of corrupted images and their corresponding ground truth targets. Unfortunately, such data is not available in many applications. For the task of image denoising in which the noise statistics is unknown, several self-supervised training methods have been proposed for overcoming this difficulty. Some of these require knowledge of the noise model, while others assume that the contaminating noise is uncorrelated, both assumptions are too limiting for many practical needs. This work proposes a novel self-supervised training technique suitable for the removal of unknown correlated noise. The proposed approach neither requires knowledge of the noise model nor access to ground truth targets. The input to our algorithm consists of easily captured bursts of noisy shots. Our algorithm constructs artificial patch-craft images from these bursts by patch matching and stitching, and the obtained crafted images are used as targets for the training. Our method does not require registration of the images within the burst. We evaluate the proposed framework through extensive experiments with synthetic and real image noise.
\end{abstract}

\section{Introduction}
\label{sec:intro}

Supervised neural networks have proven themselves powerful, achieving impressive results in solving image restoration problems~(e.g.,~\cite{zhang2017beyond,zhang2018ffdnet,Mao2016ImageRU,tai2017memnet,liu2018multi,liu2018non,zhang2020residual,Liang2021SwinIRIR}). In the commonly deployed supervised training for such  tasks, one needs a dataset consisting of pairs of corrupted and ground truth images. The degraded images are fed to the network input, while the ground truth counterparts are used as guiding targets. When the degradation model is known and easy to implement, one can construct such a dataset by applying the degradation to clean images. However, a problem arises when the degradation model is unknown. In such cases, while it is relatively easy to acquire distorted images, obtaining their ground truth counterparts can be challenging. For this reason, there is a need for self-supervised methods that use corrupted images only in the training phase. More on these methods is detailed in Section~\ref{sec:related_work}.

In this work we focus on the problem of image denoising with an unknown noise model. More specifically, we assume that the noise is additive, zero mean, but not necessarily Gaussian, and one that could be cross-channel and short-range spatially correlated\footnote{By short-term we refer to the case in which the auto-correlation function decays fast, implying that only nearby noise pixels may be highly correlated. The correlation range we consider is governed by the patch size in our algorithm - see Section~\ref{sec:framework} for more details.}. We additionally assume that the noise is (mostly) independent of the image and nearly homogeneous, i.e., having low to moderate spatially variant statistics. Examples of such noise could be Gaussian correlated noise or real image noise in digital cameras. Several recent papers propose  methods for self-supervised training under similar such  challenging conditions. However, they all assume an uncorrelated noise or a noise with a known model, thus limiting their coverage of the need posed.

This work proposes a novel self-supervised training framework for addressing the problem of image denoising of an unknown correlated noise. The proposed algorithm gets as input bursts of shots, where each frame in the burst captures nearly the same scene, up to moderate movements of the camera and objects. Such sequences of images are easily captured in many digital cameras. Our algorithm uses one image from the burst as the input, utilizing the rest of the frames of the same burst for constructing (noisy) targets for training. For creating these target images, we harness the concept of patch-craft frames introduced in PaCNet~\cite{Vaksman2021PatchCV}. Similar to PaCNet, we split the input shot into fully overlapping patches. For each patch, we find its nearest neighbor within the rest of the burst images. Note that, unlike PaCNet, we strictly omit the input shot from the neighbor search. We proceed by building $m$ patch-craft images by stitching the found neighbor patches, where $m$ is the patch size, and use these frames as denoising targets. The above can be easily extended by using more than one nearest neighbor per patch, this way enriching dramatically the number of patch-craft frames and their diversity.

The proposed technique for creating artificial target images is sensitive to the possibility of getting statistical dependency between the input and the target noise. 
To combat this flaw, we propose a method for statistical analysis of the target noise. This analysis suggests simple actions that reduce dependency between the target noise and the denoiser's input, leading to a significant boost in performance. We evaluate the proposed framework through extensive experiments with synthetic and real image noise, showing that the proposed framework  outperforms leading self-supervised methods. To summarize, the contributions of this work are the following:
\begin{itemize}
    \item We propose a novel self-supervised framework for training an image denoiser, where the noise may be  cross-channel and short-range spatially correlated. Our approach relies simply on the availability of bursts of noisy images; the ground truth is unavailable, and the noise model is unknown. 
    \item We suggest a method for statistical analysis of the target noise that leads to a boost in performance.
    \item We demonstrate superior denoising performance compared to leading alternative self-supervised denoising methods.
\end{itemize}

\section{Related Work}
\label{sec:related_work}
This paper focuses on denoising of images when the noise model is unknown. However, there are various levels in this lack of knowledge, and accordingly, different levels of corresponding solutions. The most simple case is when the noise is known to be zero-mean Gaussian i.i.d (independent and identically distributed), and the only unknown parameter is the standard deviation $\sigma$. In this case, training a single model for handling a range of $\sigma$ values can be an efficient and elegant solution~\cite{zhang2017beyond,Lefkimmiatis2018UniversalDN,Vaksman2020LIDIALL,Mohan2020RobustAI}. This approach is known as \emph{blind denoising}~\cite{Mohan2020RobustAI}. Unfortunately, such a network is likely to perform very poorly when applied to images contaminated by a correlative or a non-Gaussian noise.

An approach known as  \emph{Noise2noise}~\cite{Lehtinen2018Noise2NoiseLI} assumes that ground truth images are not available, but the training dataset consists of pairs of noisy images created by adding independent noise realizations to the same clean image. Noise2noise suggests to train on such image pairs, both being noisy but with different and independent noise realizations. This method has been shown to be quite effective, however, the missing ingredient is the lack of an accessible way for acquiring such perfectly aligned noisy pairs, rendering this method as challenging in real circumstances. 

A more common assumption is that the noise model is known, but ground truth images are not available. Several works~(e.g.,~\cite{Xu2020NoisyasCleanLS,Moran2020Noisier2NoiseLT,Pang2021RecorruptedtoRecorruptedUD}) utilize the idea of Noise2noise~\cite{Lehtinen2018Noise2NoiseLI} for handling these cases as well. They propose, each in its own way, to create noisy image pairs. For instance, the work in \cite{Xu2020NoisyasCleanLS,Moran2020Noisier2NoiseLT} suggests adding independent realizations of synthetic noise that follows the known model to the input images, then training a network using the noisier images as inputs and the original noisy ones as targets. Alternatively,~\cite{Pang2021RecorruptedtoRecorruptedUD} adds two different realizations of synthetic noise to the input images for training the denoising network.

A different idea proposed in~\cite{Prakash2021FullyUD,prakash2021interpretable} is harnessing a variational auto-encoder (VAE)~\cite{Kingma2014AutoEncodingVB} to solve the denoising task. These techniques assume that the noise distribution is known, either as a formula or as a histogram. These algorithms construct a VAE that gets noisy images at the input and produces reconstructed ones at the output. In the training stage, they maximize the log-likelihood probability of the noisy image given the reconstructed one when the probability is calculated using the provided noise distribution formula or the histogram. In the inference stage, they use the VAE to generate many candidate outputs and then obtain the reconstructed image by approximating the MMSE or MAP estimate.  

Another self-supervised technique suitable for these assumptions was introduced in~\cite{Vaksman2020LIDIALL} for lightweight architectures, and extended in~\cite{mohan2021adaptive} for more general networks. This technique, referred to as \emph{noise resampling}, suggests the following: First, train an initial denoiser somehow and apply it to a set of corrupted images to obtain initial reconstructions. Then, for each reconstructed image, create its noisy counterpart by adding a new synthetic noise. Finally, retrain the network using pairs of re-corrupted and reconstructed images. When the noise model is unknown, noise-to-noise and noise resampling methods can be applied by assuming Gaussianity and estimating the parameter~$\sigma$. However, such a strategy may lead to inferior performance when the noise is correlated or strongly deviates from Gaussianity.

A less strict assumption is that the noise model is unknown, but the noise is spatially independent. For such a case, several papers in recent literature have proposed to train networks that utilize the same noisy image both as input and output while applying various regularizations. For brevity of our discussion, we shall refer to these as \emph{image2itself} techniques. For instance, Noise2void~\cite{krull2019noise2void,krull2020probabilistic} suggests using a blind-spot architecture in which the receptive field of each processed pixel excludes the pixel itself. Such a strategy constrains the network by avoiding to learn the trivial identity operation. The work reported in~\cite{batson2019noise2self,xie2020noise2same,Quan2020Self2SelfWD,Laine2019HighQualitySD} and other papers take this idea forward by suggesting more sophisticated blind-spot methods, sometimes combining them with additional regularization terms. 

Blind-spotting is not the only regularization idea for these circumstances. For example, Neighbor2neighbor~\cite{Huang2021Neighbor2NeighborSD} proposes generating training pairs by random sub-sampling the same noisy image. The sub-sampling is conducted such that corresponding pixels of the same image pair are neighbors in the sampled image, thus having a very similar appearance.
Alternatively, the work reported in~\cite{Soltanayev2018TrainingDL} uses a regularizer based on the SURE~\cite{stein1981estimation} estimator as a replacement for the supervised targets. 

Unfortunately, all these image-to-itself methods strongly rely on a spatial independence property of the noise, and therefore are doomed to overfit when the contaminating noise is correlated. In such a case, the network may confuse the noise for content, impairing the denoising performance. Worth noting is the exception in which HDN~\cite{prakash2021interpretable} shows an ability to recover microscopy images from structured noise. However, as natural images are considerably more diverse than microscopy ones, their approach may find a challenge when applied to general denoising tasks.

\section{Proposed Framework -- Preliminaries}
\label{sec:framework}
This paper proposes a self-supervised framework for training denoisers using bursts of noisy images. Captured objects do not have to be static, yet we recommend avoiding bursts that contain sharp movements or severe lighting changes. One can obtain such sequences using burst mode in digital cameras or recording short videos. 

We start by introducing some notations. Denote a noisy burst as $\left\{\mathbf{y}_1, \ldots, \mathbf{y}_M\right\}$, where $M$ is the burst length and $\mathbf{y}_i$ is the $i$'th image. The clean image and the input noise corresponding to $\mathbf{y}_i$ are denoted by $\mathbf{x}_i$ and $\mathbf{z}_i$, respectively, and thus ${\mathbf{z}_i = \mathbf{y}_i - \mathbf{x}_i}$. The symbol $f(\cdot)$ stands for a denoiser and $\mathbf{\hat{x}}_i$ for the reconstructed counterpart of $\mathbf{y}_i$, i.e., ${\mathbf{\hat{x}}_i = f\left(\mathbf{y}_i\right)}$. Finally, we denote the artificial target image corresponding to $\mathbf{y}_i$  by $\mathbf{\tilde{x}}_i$, and the target noise by $\mathbf{w}_i$, ${\mathbf{w}_i = \mathbf{\tilde{x}}_i - \mathbf{x}_i}$. 
The proposed framework is shown schematically in Figure~\ref{fig:framework}. 
\begin{figure*}
    \centering
	\begin{subfigure}{\textwidth}
	    \captionsetup{justification=centering}
		\includegraphics[width=\textwidth]{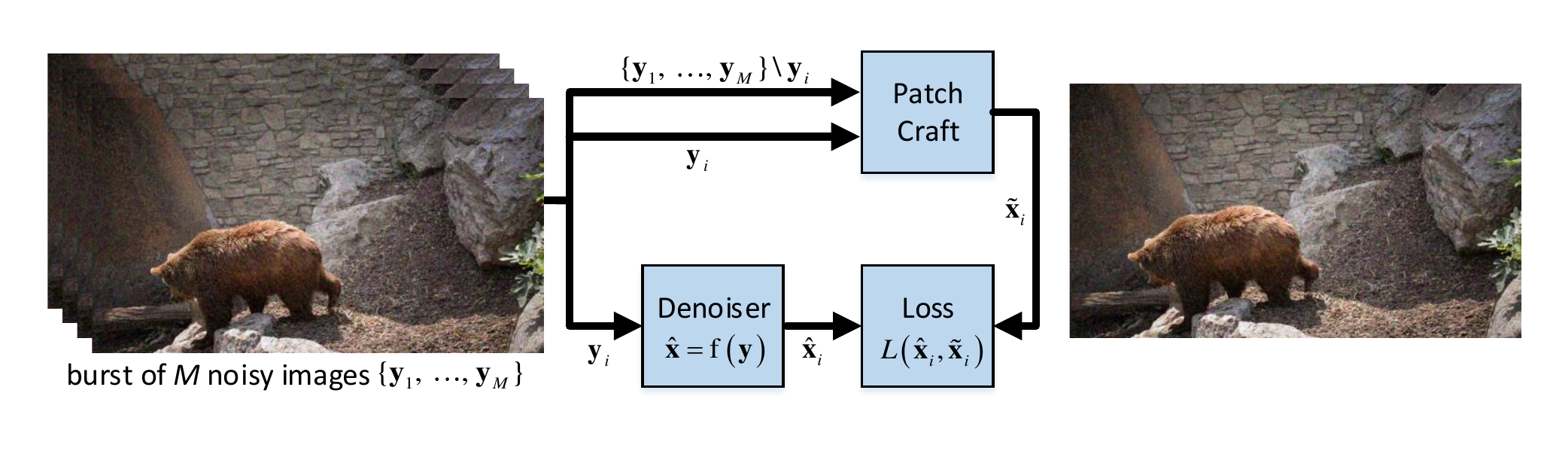}
	\end{subfigure}
	\caption{The proposed self-supervised training framework based on bursts of images and patch-craft created target images.}
	\label{fig:framework}
\end{figure*}
Our algorithm harnesses the concept of \emph{patch-craft} frames introduced in ~\cite{Vaksman2021PatchCV}. Each input burst $\left\{\mathbf{y}_1, \ldots, \mathbf{y}_M\right\}$ is split into two subsets: The first is $\mathbf{y}_i$, consisting of a single shot, and the the second, $\Gamma$, containing the rest of the images, $\Gamma = \left\{\mathbf{y}_1, \ldots, \mathbf{y}_M\right\} \backslash \mathbf{y}_i$. The image $\mathbf{y}_i$ is used as a denoiser's input, while the set $\Gamma$ is fed to the patch-craft block for creating an artificial target, $\mathbf{\tilde{x}}_i$. 

The patch-craft block operates as follows. We start from splitting the input shot $\mathbf{y}_i$ to fully overlapping patches of size $n \times n$, boundary pixels handled by mirror padding. As a result, we get $n^2$ sets of non-overlapping patches that cover the full support of the image, $\{\Upsilon_{k,l}\}_{k,l=0}^{n-1}$, to which we refer by their offsets from the left upper corner. Two examples of such sets are shown in Figure~\ref{fig:patch_set}. The offsets vary from $\left(0, 0\right)$, i.e., no offset, to $\left(n - 1, n - 1\right)$.
\begin{figure}
    \centering
	\begin{subfigure}{0.23\textwidth}
	    \captionsetup{justification=centering}
		\includegraphics[width=\textwidth]{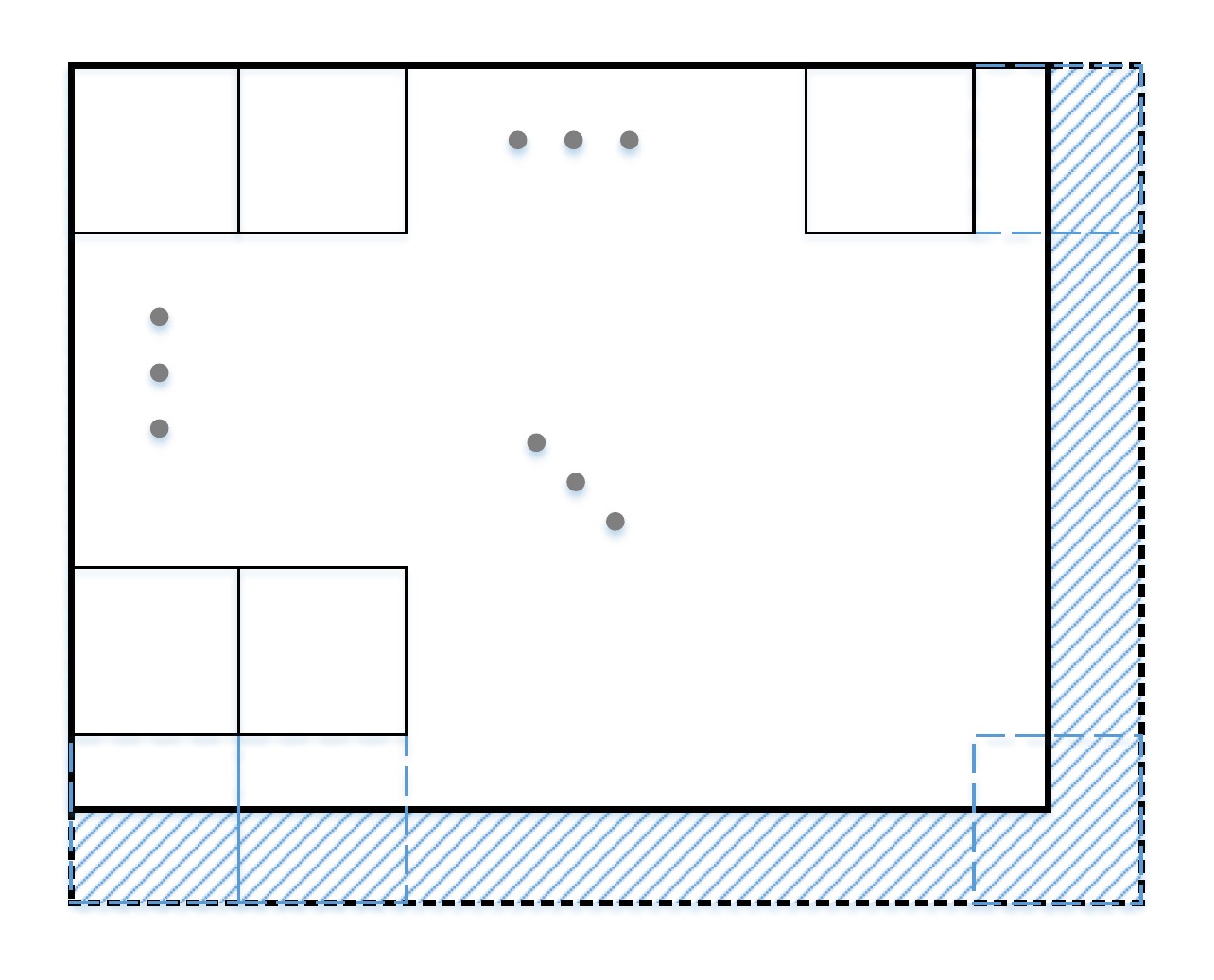}
		\caption{$\Upsilon_{0, 0}$}
		\label{fig:patch_set:00}
	\end{subfigure}
	\begin{subfigure}{0.23\textwidth}
	    \captionsetup{justification=centering}
		\includegraphics[width=\textwidth]{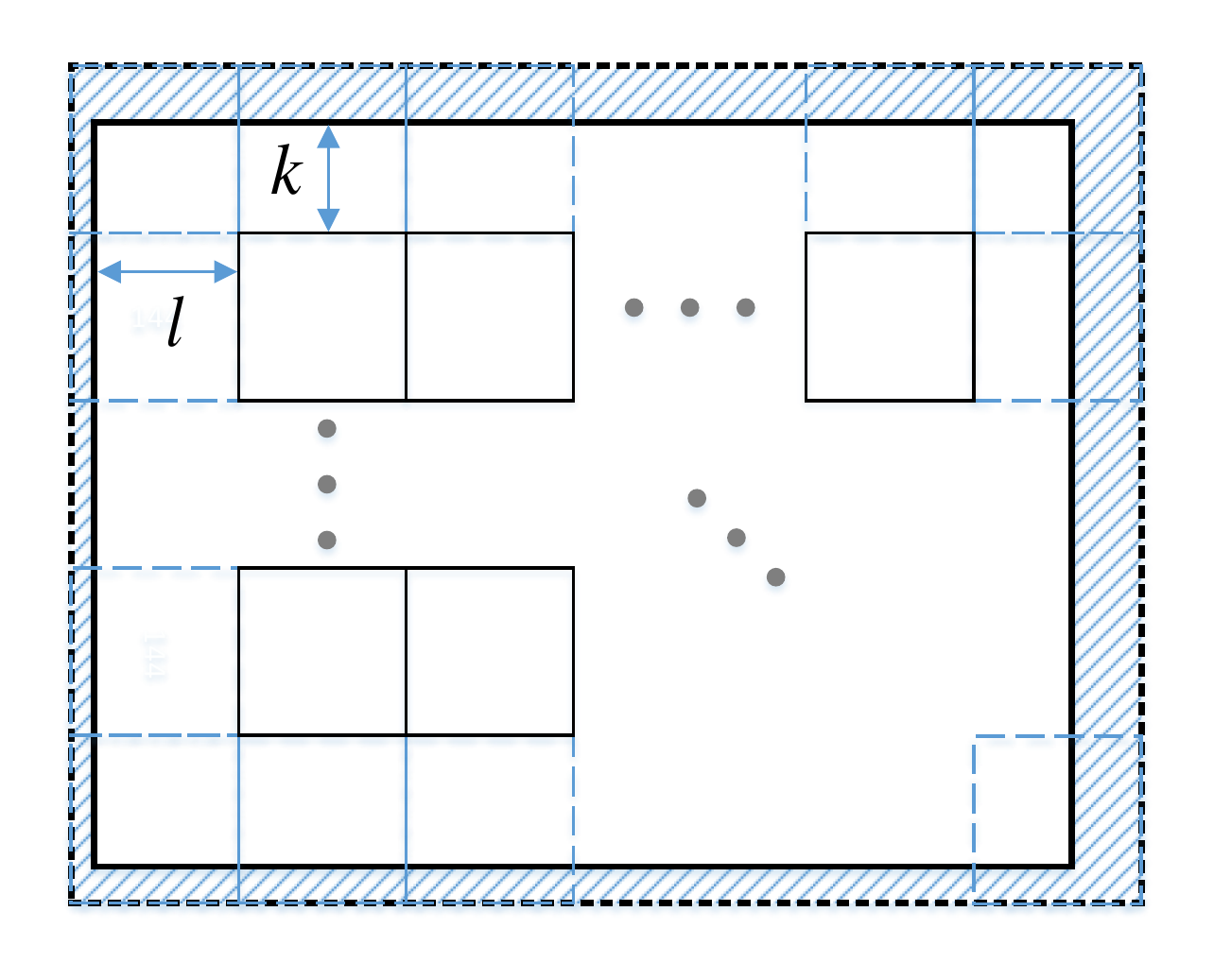}
		\caption{$\Upsilon_{k,l}$}
		\label{fig:patch_set:vh}
	\end{subfigure}
	\caption{Examples of sets of non-overlapping patches. The solid part of the rectangle represents the input image support, while the dashed part stands for the boundary effects. Figure~\ref{fig:patch_set:00} shows the case of an offset $(0, 0)$, while Figure~\ref{fig:patch_set:vh} refers to an offset $\left(k, l\right)$.}
	\label{fig:patch_set}
\end{figure}
Each of the images $\{\Upsilon_{k,l}\}$ can be converted to a patch-craft image $\{\mathbf{\tilde{y}}_{k,l}\}$ by replacing each patch in  $\Upsilon_{k,l}$ with its nearest neighbor from the set $\Gamma$ and cutting out pixels corresponding to the padding. For finding the neighbors, we use an $L_2$ distance while the search in each image of $\Gamma$ is restricted by a bounding box of size $B \times B$ centered at the patch location. Finally, at any iteration of the training, we randomly choose one of the $n^2$ available patch-craft images $\{\mathbf{\tilde{y}}_{k,l}\}$ to be a target $\mathbf{\tilde{x}}_i$. Figure~\ref{fig:pc_im} shows an example of a noisy image and one of the corresponding patch-craft images.
\begin{figure}
	\begin{subfigure}{0.23\textwidth}
	    \captionsetup{justification=centering}
		\includegraphics[width=\textwidth]{example_sg10.0_k3_s7_f38_im_n}
		\caption{Noisy image}
		\label{fig:pc_im:n}
	\end{subfigure}
	\begin{subfigure}{0.23\textwidth}
	    \captionsetup{justification=centering}
		\includegraphics[width=\textwidth]{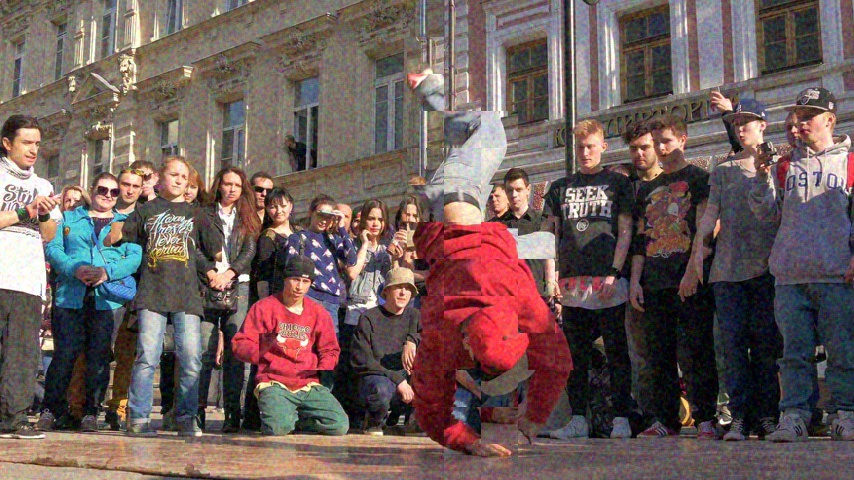}
		\caption{Patch-craft image}
		\label{fig:pc_im:pc}
	\end{subfigure}
	\caption{An example of a noisy shot and one of the corresponding patch-craft images.}
	\label{fig:pc_im}
\end{figure}

Few notes are in order: Since the neighbor patches come from different locations, all $n^2$ patch-craft images constructed from the same burst are similar but not identical. Each holds additional information enriching the training process. Furthermore, the proposed technique can be extended by finding $k$ nearest neighbors per patch and choosing one of them (e.g. randomly) in the patch-craft construction. This extension may increase the number of possible patch-craft images, substantially enriching their diversity.

\section{Proposed Framework - Analysis}
\label{sec:stat_analysis}
Consider a denoiser $f_{\mathbf{\theta}}(\cdot)$ parameterized by $\mathbf{\theta}$ that gets a noisy image $\mathbf{y}$ and produces it's reconstructed image $\mathbf{\hat{x}}$, i.e., $\mathbf{\hat{x}} = f_{\mathbf{\theta}}(\mathbf{y})$. The desired (ground truth) image is denoted by $\mathbf{x}$ and the input noise is $\mathbf{z}$, thus ${\mathbf{y}=\mathbf{x} + \mathbf{z}}$. In this section we discuss a training procedure in which instead of clean targets $\mathbf{x}$, one uses noisy ones $\mathbf{\tilde{x}}$ contaminated by a target noise $\mathbf{w}$, $\mathbf{w} = \mathbf{\tilde{x}} - \mathbf{x}$. 
Generally, training a denoiser is sensitive to the dependency between the input image $\mathbf{y}$ and the target noise $\mathbf{w}$. To illustrate this vulnerability, imagine the extreme case where $\mathbf{w}$ is equal to $\mathbf{z}$ for all pairs in the dataset. In such a case, the input and target images would be identical, and the denoiser would learn the useless identity operation. This example aligns with the intuition that 
lowering such dependency may result in better training and eventual denoising performance. In this section we propose a statistical analysis of the target noise and suggest a simple way to reduce it's dependency with the input noise. 

\subsection{Ideal Case}
\label{sec:ideal_case}
Results reported in Noise2Noise~\cite{Lehtinen2018Noise2NoiseLI} suggest that if the target noise is independent of the network's input, using these targets for training a denoiser may be almost as effective as using the ground truth images. We formalize this  hereafter, and start with few supporting notations. 

Let ${l = \frac{1}{2}\left\|f_{\mathbf{\theta}}(\mathbf{y}) - \mathbf{x}\right\|_2^2}$ be the supervised $L_2$ loss of the single image pair $\left\{f_{\mathbf{\theta}}\left(\mathbf{y}\right), \mathbf{x}\right\}$. Denote by $\nabla_{\mathbf{\theta}} l$ the gradient of $l$ w.r.t. $\mathbf{\theta}$, ${\nabla_{\mathbf{\theta}} l = \nabla _\theta^T f_{\mathbf{\theta}}(\mathbf{y})\left(f_{\mathbf{\theta}}(\mathbf{y}) - \mathbf{x}\right)}$ . Then the supervised MSE loss and its full gradient are given by ${L = \E\left[l\right]}$  and ${\nabla_{\mathbf{\theta}} L = \nabla_{\mathbf{\theta}}\left( \E\left[l\right]\right) = \E\left[\nabla_{\mathbf{\theta}} l\right]}$. Similarly, we denote the self-supervised $L_2$ loss of $\left\{f_{\mathbf{\theta}}\left(\mathbf{y}\right), \mathbf{\tilde{x}}\right\}$ by ${\tilde{l} = \frac{1}{2}\left\|f_{\mathbf{\theta}}(\mathbf{y}) - \mathbf{\tilde{x}}\right\|_2^2}$. Correspondingly, the gradient of $\tilde{l}$ is denoted by $\nabla_{\mathbf{\theta}} \tilde{l}$, being ${\nabla_{\mathbf{\theta}} \tilde{l} = \nabla_\theta^T f_{\mathbf{\theta}}(\mathbf{y})\left(f_{\mathbf{\theta}}(\mathbf{y}) - \mathbf{\tilde{x}}\right)}$.

\begin{lemma}
\label{lem:sgd}
If the target noise $\mathbf{w}$ is independent of the image $\mathbf{x}$ and noise $\mathbf{z}$, and admits a zero-mean ${\E\left[\mathbf{w}\right] = \mathbf{0}}$, then $\nabla_{\mathbf{\theta}} \tilde{l}$ is an unbiased estimator of $\nabla_{\mathbf{\theta}} L$, i.e., ${\E\left[\nabla_{\mathbf{\theta}} \tilde{l}\right] = \nabla_{\mathbf{\theta}} L}$.
\end{lemma}
The proof of the Lemma is given in appendix~\ref{app:proof_sgd}. The implication is that under appropriate assumptions, self-supervised training with noisy targets is equivalent to a variation of the SGD~\cite{Robbins2007ASA} algorithm with the regular supervised MSE loss. Thus, all guarantees and intuitions that are valid for a supervised training with SGD and an MSE loss are also correct for training with noisy targets. 

Returning to the proposed scheme, a conclusion from Lemma~\ref{lem:sgd} is that statistical independence between the input image $\mathbf{y}_i$ and the target noise $\mathbf{w}_i$ is desirable. Therefore, as a first and simple step for reducing this dependency, we omit $\mathbf{y}_i$ from the set $\Gamma$. A further method for reducing this dependency is discussed next.

\subsection{Dependency Reduction}
\label{sec:dep_red}
As mentioned above, the training procedure can be sensitive to a statistical dependency between the target noise and the network's input, and thus we seek ways to reduce it. We bring in this section the main points of the proposed method, leaving formal proofs and other details to appendix~\ref{app:dep_red}.

As the proposed method may seem counter-intuitive at first glance, we start by building the reader's intuition gradually. Let us discuss the two most common types of dependencies that may be introduced by patch matching: (I)~\emph{overfitting input noise} and (II)~\emph{underfitting ground truth images}. Dependency of type~(I) refers to cases when patch matching does a ``too good'' job, bringing target noise $\mathbf{w}$ that mimics the input noise, $\mathbf{z}$. This dependency is characterized by a positive correlation between $\mathbf{w}$ and $\mathbf{z}$. 

As for type~(II) dependency, it happens when the patch-craft, $\mathbf{\tilde{x}}$, and ground truth, $\mathbf{x}$, images tend to be dissimilar. It is less intuitive, but this dependency is manifested in a negative correlation between $\mathbf{w}$ and $\mathbf{x}$. Here is brief explanation of this phenomenon: Consider the following scalar covariances computed over pairs of images: $\sigma_{\mathbf{\tilde{x}},\mathbf{\tilde{x}}}$ $\sigma_{\mathbf{x},\mathbf{\tilde{x}}}$, $\sigma_{\mathbf{x},\mathbf{w}}$, and $\sigma_{\mathbf{x},\mathbf{x}}$. Clearly, ${\sigma_{\mathbf{x},\mathbf{\tilde{x}}} = \sigma_{\mathbf{x},\mathbf{x}} + \sigma_{\mathbf{x},\mathbf{w}}}$. 
Assuming that $\sigma_{\mathbf{\tilde{x}},\mathbf{\tilde{x}}}\approx \sigma_{\mathbf{x},\mathbf{x}}$, the dissimilarity between $\mathbf{\tilde{x}}$ and $\mathbf{x}$ reduces the value of $\sigma_{\mathbf{x},\mathbf{\tilde{x}}}$, which means that $\sigma_{\mathbf{x},\mathbf{w}}$ is necessarily negative (see more on this phenomenon in appendix~\ref{app:dep_red}).

Let us look at an empirical covariance between $\mathbf{y}$ and $\mathbf{r}$, denoted by $s_{\mathbf{y}, \mathbf{r}}$, where $\mathbf{r} = \mathbf{\tilde{x}} - \mathbf{y}$. This covariance is a scalar obtained for each possible image pair $\{\mathbf{y},\mathbf{r}\}$. By assessing many such pairs, we get a histogram of these covariance values, which we analyze next. Observe that these covariance values are accessible, easily computed from the data we have. Here are few facts regarding $s_{\mathbf{y}, \mathbf{r}}$:
\begin{itemize}
    \item If $\mathbf{x}$, $\mathbf{z}$, and $\mathbf{w}$ are mutually independent, $s_{\mathbf{y}, \mathbf{r}}$ converges in distribution to a Gaussian centered at $-\sigma_{z}^2$ and thus
    ${\E\left[s_{\mathbf{y}, \mathbf{r}}\right] = -\sigma_{z}^2}$.
    \item Type~(I) dependency implies that $\mathbf{z}$ and $\mathbf{w}$ are heavily correlated, thus  ${\E\left[s_{\mathbf{y}, \mathbf{r}}\right] > -\sigma_{z}^2}$. However, for large enough patch-sizes, and when discarding $\mathbf{y}_i$ from the set $\Gamma$, this behavior is expected to be rare and can be disregarded. 
    \item We have seen that type~(II) dependency leads to negative values of $\sigma_{\mathbf{x},\mathbf{w}}$. Thus, we get that ${\E\left[s_{\mathbf{y}, \mathbf{r}}\right] < -\sigma_{z}^2}$. 
\end{itemize}
A formal proof of these statements is given in appendix~\ref{app:dep_red}. 

Let us now return to the $s_{\mathbf{y}, \mathbf{r}}$ histogram, while assuming that the dependency of type~(I) is low. Figure~\ref{fig:s_yr_hist} presents two examples of such histograms for two types of noise - more details on these noise realizations is given in the next section. One can easily spot the expression of type~(II) dependencies in both -- the longer left tail. Note that the histograms are cropped, and the tail is longer than shown in the figures (especially in Figure~\ref{fig:s_yr_hist:davis}). As expected, due to this dependency, the histogram mean is shifted left relative to its peak. To reduce this dependency, we cut the left tail by excluding from the training set all image pairs for which $s_{\mathbf{y}, \mathbf{r}} < s_{min}$. The threshold $s_{min}$ is set such that the mean of the resulting histogram coincides with its peak. As shown in Figure~\ref{fig:valid_psnr}, the dependency reduction substantially boosts denoising performance.

\begin{figure}
    \centering
	\begin{subfigure}{0.23\textwidth}
	    \captionsetup{justification=centering}
		\includegraphics[width=\textwidth]{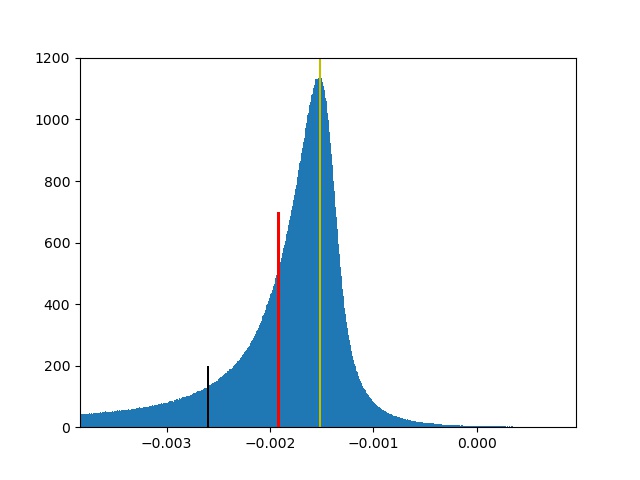}
		\caption{Correlated Gaussian noise}
		\label{fig:s_yr_hist:davis}
	\end{subfigure}
	\begin{subfigure}{0.23\textwidth}
	    \captionsetup{justification=centering}
		\includegraphics[width=\textwidth]{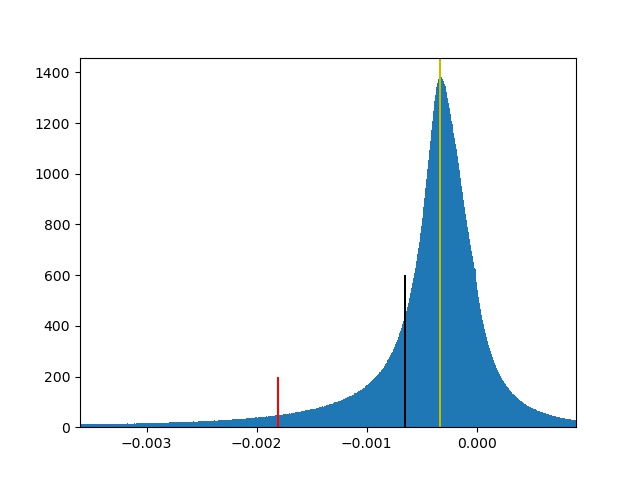}
		\caption{Real image noise}
		\label{fig:s_yr_hist:crvd}
	\end{subfigure}
	\caption{Examples of $s_{\mathbf{y}, \mathbf{r}}$ histograms in experiments with correlated Gaussian and real image noise. The yellow bar is located at the histogram peak, while the black bar shows the location of the mean. The red bar indicates the location of $s_{min}$.}
	\label{fig:s_yr_hist}
\end{figure}

\begin{figure}
    \centering
	\begin{subfigure}{0.23\textwidth}
	    \captionsetup{justification=centering}
		\includegraphics[width=\textwidth]{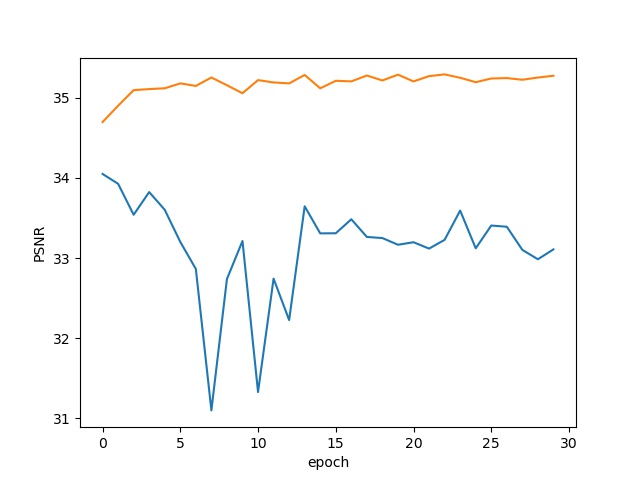}
		\caption{Correlated Gaussian noise}
		\label{fig:valid_psnr:davis}
	\end{subfigure}
	\begin{subfigure}{0.23\textwidth}
	    \captionsetup{justification=centering}
		\includegraphics[width=\textwidth]{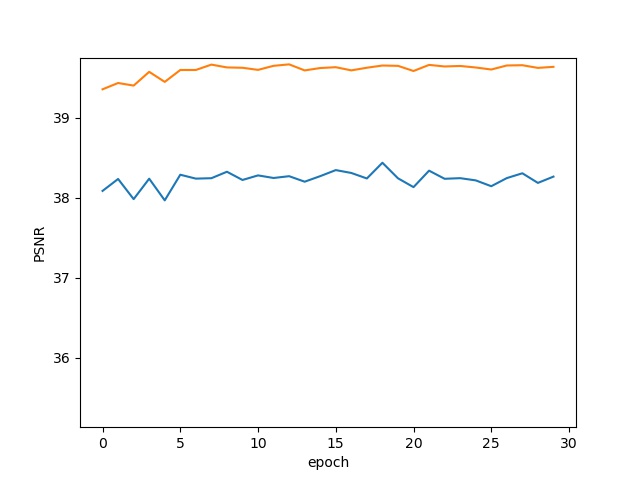}
		\caption{Real image noise}
		\label{fig:valid_psnr:crvd}
	\end{subfigure}
	\caption{Validation PSNR before and after dependency reduction. The blue line shows validation PSNR vs. epoch number during training on the full dataset, while the orange line indicates the PSNR after excluding image pairs for which $s_{\mathbf{y}, \mathbf{r}} < s_{min}$.}
	\label{fig:valid_psnr}
\end{figure}

\section{Experimental Results}
We turn to report the denoising performance of the proposed  framework and its comparison with leading self-supervised methods. Our framework is referred to as Patch-Craft (PC). We consider two experiments, one with correlated Gaussian noise and the second with real-world noise. In both experiments, we train networks in an adaptation manner~\cite{Vaksman2020LIDIALL}, beginning with bias-free networks~\cite{Mohan2020RobustAI} pre-trained for blind i.i.d. Gaussian denoising task, and retraining them using the proposed patch-craft method. We consider two different architectures for our scheme: DnCNN~\cite{zhang2017beyond} and U-Net~\cite{ronneberger2015u}. The first is similar to the one used in~\cite{kadkhodaie2021stochastic}, while U-Net is taken from~\cite{Mohan2020RobustAI}. We use bias-free versions of both networks. We denote by PC-DnCNN and PC-UNet the networks retrained using the PC framework, where B-DnCNN and B-UNet stand for their initial versions trained for blind i.i.d. Gaussian denoising. 

For comparison, we choose three latest state-of-the-art (SoTA) self-supervised training methods: Recorrupted-to-recorrupted (R2R)~\cite{Pang2021RecorruptedtoRecorruptedUD}, Neighbor2Neighbor (N2N)~\cite{Huang2021Neighbor2NeighborSD}, and Blind2Unblind (B2U)~\cite{wang2022blind2unblind}. In addition, we show a comparison with  BM3D~\cite{dabov2007image}, which gets as input parameter the standard deviation, $\sigma$, of the noise. Since the noise is not i.i.d. Gaussian, the actual standard deviation is not necessarily the optimal parameter for BM3D. Thus we apply BM3D with two configurations: a BM3D with the actual $\sigma$ of the noise, and a grid search to find the best performing parameter. We call this configuration oracle BM3D (O-BM3D).

Our algorithm requires bursts of images for training. Therefore, we use datasets containing short video sequences in all experiments. Our analysis suggests that the patch size, $n$, should be big. Moreover, it should grow with the standard deviation of the noise, $\sigma$, and correlation range. Following this, we increase the value of $n$ accordingly. 
For quantitative evaluation, we choose the commonly used PSNR and SSIM metrics. For more technical details regarding the training, we refer the reader to appendices~\ref{app:training_details}~and~\ref{app:additional_results}.

\subsection{Correlated Gaussian Denoising}
We start with additive correlated Gaussian noise, using the DAVIS dataset~\cite{pont20172017} at 480p resolution. We train the networks on 90 bursts of length 7 frames, each taken from a different training video sequence at an arbitrary location. In each training burst, the middle frame is used as the network input, whereas the rest 6 are utilized for building the patch-craft targets. For the test, we use frames taken from 30 test video sequences. From each sequence we take 3 nonconsecutive frames at arbitrary locations. Each of the obtained 90 test frames is denoised as a single image. 

The correlated noise is created by convolving an i.i.d. Gaussian noise with a rectangular flat kernel of size $k \times k$. The competing methods are trained using the code packages and parameters supplied by the authors. For R2R, we choose $\alpha = 2$ among the three options listed in the original paper (0.5, 2, 20) since it leads to the best denoising results.

Table~\ref{tab:corr_g_res} summarizes the denoising performance for various $\sigma$ and $k$ values. Figure~\ref{fig:davis_22_1_s20_k3_crvd_9_3_iso25600} and Figures~\ref{fig:davis_20_1_s10_k3_11_0_s10_k4},~\ref{fig:davis_2_0_s15_k4_8_1_s20_k4},~and~\ref{fig:davis_4_2_s15_k3_19_0_s5_k4} in appendix~\ref{app:additional_results} show  visual comparisons between the denoised images.
\begin{table*}
    \centering
    \begin{tabular}{|c|c|c|c|c|c|c|c|c|c|c|c|c|}
         \hline
        $\sigma$ & $k$ & Noisy & R2R & N2N & B2U & BM3D & O-BM3D & B-DnCNN & B-UNet & PC-UNet & PC-DnCNN \\
         \hline
         \multirow{6}{*}{5} & \multirow{2}{*}{2} & 34.15 & 38.88 & 35.20 & 29.30 & 38.28 & \textcolor{red}{39.69} & 37.83 & 36.73 & 39.27 & \textcolor{blue}{39.57} \\
         & & 0.852 & 0.960 & 0.886 & 0.720 & 0.951 & \textcolor{red}{0.969} & 0.945 & 0.923 & \textcolor{blue}{0.967} & \textcolor{red}{0.969} \\
         \hhline{|~|-|-|-|-|-|-|-|-|-|-|-|}
         & \multirow{2}{*}{3} & 34.15 & 37.29 & 34.64 & 28.74 & 36.50 & 38.19 & 36.02 & 35.25 & \textcolor{blue}{38.67} & \textcolor{red}{38.81} \\
         & & 0.859 & 0.943 & 0.879 & 0.719 & 0.926 & 0.957 & 0.916 & 0.896 & \textcolor{blue}{0.964} & \textcolor{red}{0.965} \\
         \hhline{|~|-|-|-|-|-|-|-|-|-|-|-|}
         & \multirow{2}{*}{4} & 34.16 & 36.22 & 34.48 & 30.46 & 35.83 & 37.13 & 35.33 & 34.83 & \textcolor{blue}{38.06} & \textcolor{red}{38.31} \\
         & & 0.868 & 0.930 & 0.885 & 0.765 & 0.920 & 0.948 & 0.908 & 0.894 & \textcolor{blue}{0.961} & \textcolor{red}{0.964} \\
         \hline
         \multirow{6}{*}{10} & \multirow{2}{*}{2} & 28.13 & 34.55 & 29.55 & 23.65 & 33.25 & 35.37 & 32.82 & 31.57 & \textcolor{blue}{35.89} & \textcolor{red}{36.10} \\
         & & 0.639 & 0.902 & 0.707 & 0.441 & 0.867 & 0.927 & 0.850 & 0.799 & \textcolor{blue}{0.937} & \textcolor{red}{0.939} \\
         \hhline{|~|-|-|-|-|-|-|-|-|-|-|-|}
         & \multirow{2}{*}{3} & 28.13 & 32.5 & 28.85 & 23.85 & 30.96 & 33.56 & 30.44 & 29.64 & \textcolor{blue}{35.16} & \textcolor{red}{35.32} \\
         & & 0.653 & 0.849 & 0.693 & 0.454 & 0.796 & 0.897 & 0.774 & 0.736 & \textcolor{blue}{0.932} & \textcolor{red}{0.934} \\
         \hhline{|~|-|-|-|-|-|-|-|-|-|-|-|}
         & \multirow{2}{*}{4} & 28.13 & 31.21 & 28.67 & 23.45 & 30.16 & 32.3 & 29.58 & 29.07 & \textcolor{blue}{34.69} & \textcolor{red}{34.79} \\
         & & 0.670 & 0.818 & 0.705 & 0.433 & 0.782 & 0.872 & 0.756 & 0.730 & \textcolor{blue}{0.931} & \textcolor{red}{0.932} \\
         \hline
         \multirow{6}{*}{15} & \multirow{2}{*}{2} & 24.61 & 31.81 & 26.27 & 22.31 & 30.28 & 32.99 & 29.72 & 28.56 & \textcolor{blue}{33.77} & \textcolor{red}{33.96} \\
         & & 0.489 & 0.828 & 0.567 & 0.374 & 0.776 & 0.886 & 0.747 & 0.688 & \textcolor{blue}{0.907} & \textcolor{red}{0.909} \\
         \hhline{|~|-|-|-|-|-|-|-|-|-|-|-|}
         & \multirow{2}{*}{3} & 24.61 & 29.59 & 25.43 & 24.44 & 27.71 & 31.11 & 27.12 & 26.41 & \textcolor{blue}{32.98} & \textcolor{red}{33.16} \\
         & & 0.503 & 0.747 & 0.547 & 0.491 & 0.671 & 0.842 & 0.645 & 0.606 & \textcolor{blue}{0.900} & \textcolor{red}{0.902} \\
         \hhline{|~|-|-|-|-|-|-|-|-|-|-|-|}
         & \multirow{2}{*}{4} & 24.61 & 28.26 & 25.26 & 22.26 & 26.83 & 29.79 & 26.22 & 25.75 & \textcolor{blue}{32.4} & \textcolor{red}{32.57} \\
         & & 0.521 & 0.709 & 0.562 & 0.398 & 0.653 & 0.806 & 0.623 & 0.596 & \textcolor{blue}{0.897} & \textcolor{red}{0.899} \\
         \hline
         \multirow{6}{*}{20} & \multirow{2}{*}{2} & 22.11 & 30.1 & 23.82 & 20.53 & 28.17 & 31.41 & 27.48 & 26.41 & \textcolor{blue}{32.28} & \textcolor{red}{32.43} \\
         & & 0.387 & 0.765 & 0.46 & 0.304 & 0.691 & 0.851 & 0.655 & 0.594 & \textcolor{blue}{0.876} & \textcolor{red}{0.879} \\
          \hhline{|~|-|-|-|-|-|-|-|-|-|-|-|}
         & \multirow{2}{*}{3} & 22.11 & 27.57 & 23.02 & 7.74 & 25.39 & 29.49 & 24.76 & 24.14 & \textcolor{blue}{31.44} & \textcolor{red}{31.63} \\
         & & 0.400 & 0.655 & 0.443 & 0.088 & 0.568 & 0.796 & 0.541 & 0.508 & \textcolor{blue}{0.869} & \textcolor{red}{0.872} \\
          \hhline{|~|-|-|-|-|-|-|-|-|-|-|-|}
         & \multirow{2}{*}{4} & 22.11 & 25.93 & 22.91 & 22.33 & 24.48 & 28.15 & 23.84 & 23.41 & \textcolor{blue}{30.78} & \textcolor{red}{30.97} \\
         & & 0.417 & 0.599 & 0.461 & 0.417 & 0.548 & 0.752 & 0.52 & 0.496 & \textcolor{blue}{0.863} & \textcolor{red}{0.866} \\
         \hline
         \multicolumn{2}{|c|}{\multirow{2}{*}{Average}} & 27.25 & 31.99 & 28.18 & 23.26 & 30.65 & 33.27 & 30.10 & 29.31 & \textcolor{blue}{34.62} & \textcolor{red}{34.80}\\
         \multicolumn{2}{|c|}{} & 0.605 & 0.809 & 0.650 & 0.467 & 0.762 & 0.875 & 0.740 & 0.706 & \textcolor{blue}{0.917} & \textcolor{red}{0.919} \\
         \hline
    \end{tabular}
    \caption{Denoising performance with correlated Gaussian noise. PC-UNet and PC-DnCNN are trained using the proposed patch-craft framework. The best PSNR and SSIM results are marked Red. The second-best results are marked blue.}
    \label{tab:corr_g_res}
\end{table*}
\begin{table*}
    \centering
    \begin{tabular}{|c|c|c|c|c|c|c|c|c|c|c|c|c|}
         \hline
         ISO & $\sigma$ & Noisy & R2R & N2N & B2U & BM3D & O-BM3D & B-UNet & B-DnCNN & PC-UNet & PC-DnCNN \\
         \hline
         \multirow{2}{*}{1600} & \multirow{2}{*}{3.3} & 37.67 & 39.58 & 37.71 & 36.88 & 38.61 & 41.12 & 37.71 & 37.71 & \textcolor{blue}{41.25} & \textcolor{red}{41.33} \\
         & & 0.925 & 0.962 & 0.925 & 0.915 & 0.946 & \textcolor{blue}{0.979} & 0.926 & 0.926 & \textcolor{red}{0.981} & \textcolor{red}{0.981} \\
         \hline
         \multirow{2}{*}{3200} & \multirow{2}{*}{4.5} & 35.03 & 37.18 & 35.10 & 4.94 & 36.08 & 38.99 & 35.08 & 35.10 & \textcolor{blue}{39.50} & \textcolor{red}{39.64} \\
         & & 0.874 & 0.937 & 0.876 & 0.011 & 0.910 & 0.969 & 0.876 & 0.877 & \textcolor{red}{0.975} & \textcolor{blue}{0.974} \\
         \hline
         \multirow{2}{*}{6400} & \multirow{2}{*}{6.3} & 32.10 & 34.67 & 32.19 & 4.98 & 33.20 & 36.57 & 32.14 & 32.17 & \textcolor{blue}{37.18} & \textcolor{red}{37.32} \\
         & & 0.794 & 0.892 & 0.798 & 0.002 & 0.851 & 0.956 & 0.795 & 0.797 & \textcolor{blue}{0.964} & \textcolor{red}{0.965} \\
         \hline
         \multirow{2}{*}{12800} & \multirow{2}{*}{8.8} & 29.25 & 31.71 & 29.33 & 23.79 & 30.46 & 34.30 & 29.28 & 29.33 & \textcolor{blue}{34.89} & \textcolor{red}{35.15} \\
         & & 0.690 & 0.833 & 0.695 & 0.451 & 0.771 & 0.939 & 0.691 & 0.694 & \textcolor{blue}{0.951} & \textcolor{red}{0.952} \\
         \hline
         \multirow{2}{*}{25600} & \multirow{2}{*}{13.1} & 25.77 & 28.26 & 25.86 & 20.11 & 27.10 & 31.35 & 25.80 & 25.85 & \textcolor{red}{32.45} & \textcolor{blue}{32.38} \\
         & & 0.506 & 0.701 & 0.512 & 0.326 & 0.620 & 0.910 & 0.507 & 0.511 & \textcolor{red}{0.933} & \textcolor{blue}{0.932} \\
         \hline
         \multicolumn{2}{|c|}{\multirow{2}{*}{Average}} & 31.96 & 34.28 & 32.04 & 18.14 & 33.09 & 36.47 & 32.00 & 32.03 & \textcolor{blue}{37.05} & \textcolor{red}{37.16} \\
         \multicolumn{2}{|c|}{} & 0.758 & 0.865 & 0.761 & 0.341 & 0.820 & \textcolor{blue}{0.951} & 0.759 & 0.761 & \textcolor{red}{0.961} & \textcolor{red}{0.961} \\
         \hline

    \end{tabular}
    \caption{Denoising performance with real-world image noise. PC-UNet and PC-DnCNN are trained with the proposed patch-craft framework ($\sigma$ is the STD using ground truth images). The best PSNR and SSIM results are marked red, and the second-best marked blue.}
    \label{tab:real_im_n_res}
\end{table*}
As can be seen from the table and figures, the current SoTA self-supervised methods with which we compare face difficulties
\footnote{B2U~\cite{wang2022blind2unblind} training sometimes loses stability, getting extremely low PSNR/SSIM on images contaminated with spatially correlated noise.} 
in train networks when the noise is correlated, when the difficulty increases with the correlation range and the intensity of the noise. The classical, signal processing oriented, O-BM3D method achieves relatively high PSNR (typically 1-3 dB below networks trained using our framework). However, as can be seen from the figures, in the case of moderate to severe noise, the visual quality of the O-BM3D outputs leaves much to be desired since the method tends to produce blurred images or leave a noticeable amount of low-frequency noise unfiltered. Not to mention that finding the optimal parameter $\sigma$ for BM3D when the ground truth targets are unavailable may not be easy.
\begin{figure*}[htbp]
    \centering
	\begin{subfigure}{0.18\textwidth}
	    \captionsetup{justification=centering}
		\includegraphics[width=\textwidth]{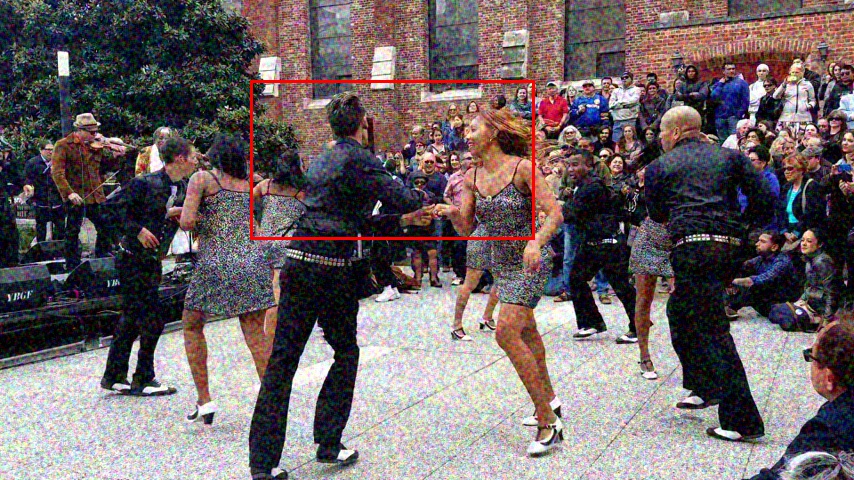}
		\caption{Noisy \\ 22.13 / 0.623}
		\label{fig:davis_22_1_s20_k3:noisy_rect}
	\end{subfigure}
	\begin{subfigure}{0.18\textwidth}
	    \captionsetup{justification=centering}
		\includegraphics[width=\textwidth]{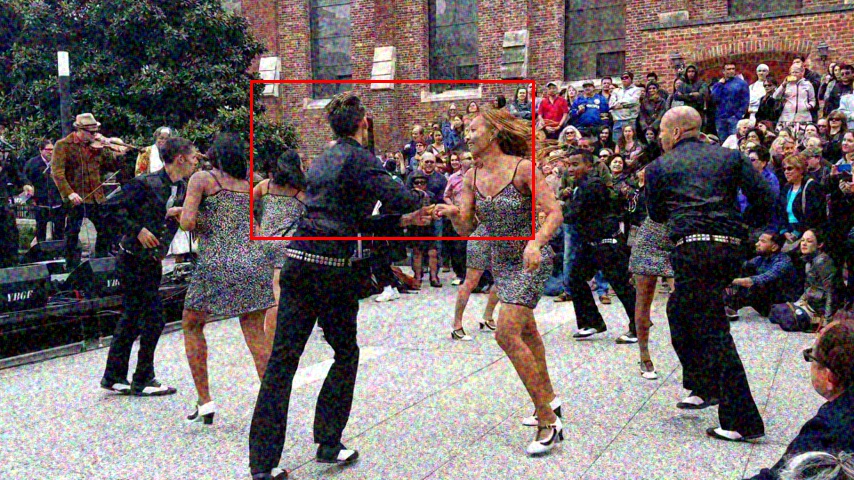}
		\caption{N2N \\ 23.05 / 0.657}
		\label{fig:davis_22_1_s20_k3:n2n_rect}
	\end{subfigure}
	\begin{subfigure}{0.18\textwidth}
	    \captionsetup{justification=centering}
		\includegraphics[width=\textwidth]{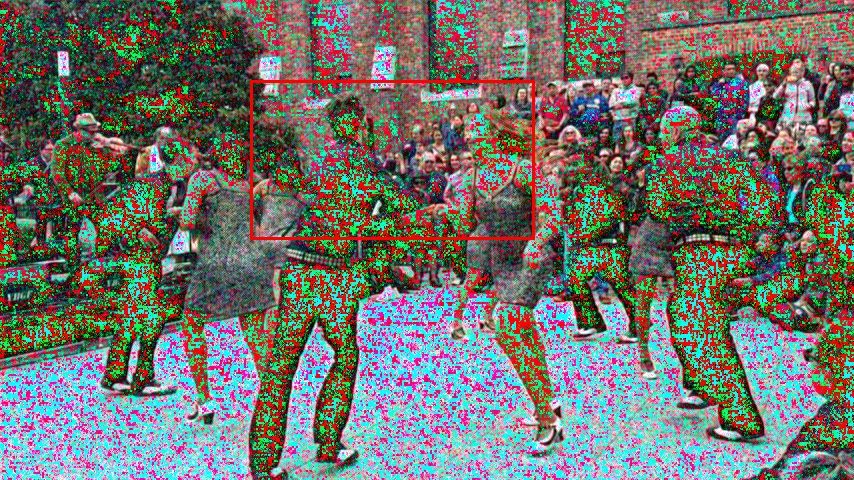}
		\caption{B2U \\ 9.55 / 0.242}
		\label{fig:davis_22_1_s20_k3:b2u_rect}
	\end{subfigure}
	\begin{subfigure}{0.18\textwidth}
	    \captionsetup{justification=centering}
		\includegraphics[width=\textwidth]{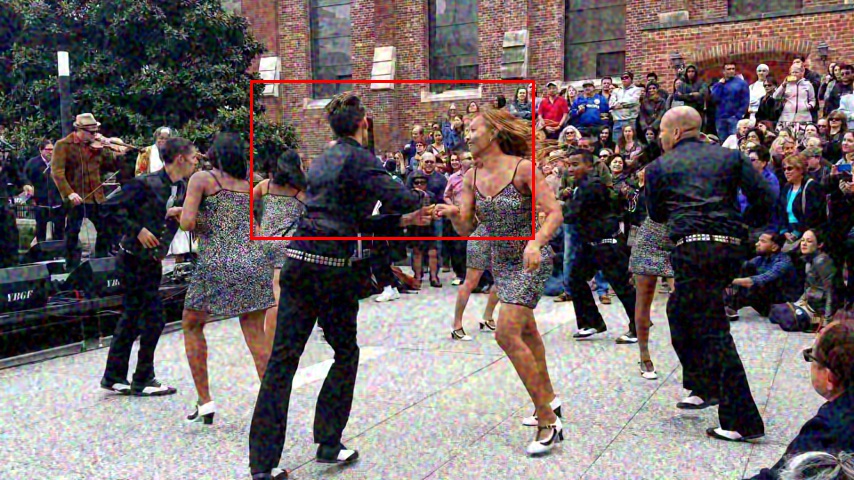}
		\caption{BM3D \\ 24.70 / 0.711}
		\label{fig:davis_22_1_s20_k3:bm3d_rect}
	\end{subfigure}
	\begin{subfigure}{0.18\textwidth}
	    \captionsetup{justification=centering}
		\includegraphics[width=\textwidth]{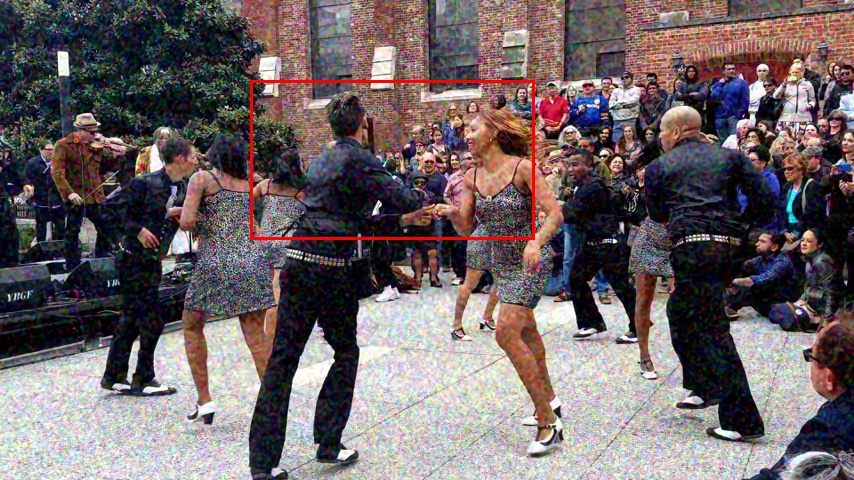}
		\caption{B-DnCNN \\ 24.37 / 0.720}
		\label{fig:davis_22_1_s20_k3:b_dncnn_rect}
	\end{subfigure}
	\begin{subfigure}{0.18\textwidth}
	    \captionsetup{justification=centering}
		\includegraphics[width=\textwidth]{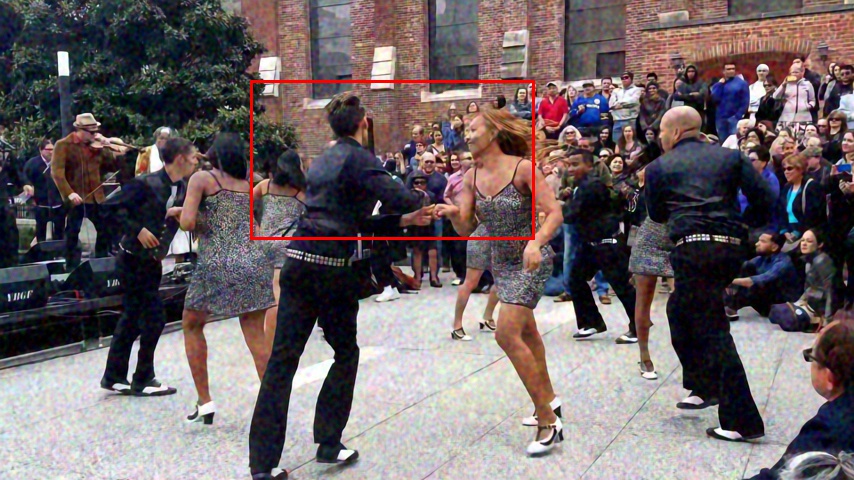}
		\caption{R2R \\ 25.53 / 0.712}
		\label{fig:davis_22_1_s20_k3:r2r_rect}
	\end{subfigure}
	\begin{subfigure}{0.18\textwidth}
	    \captionsetup{justification=centering}
		\includegraphics[width=\textwidth]{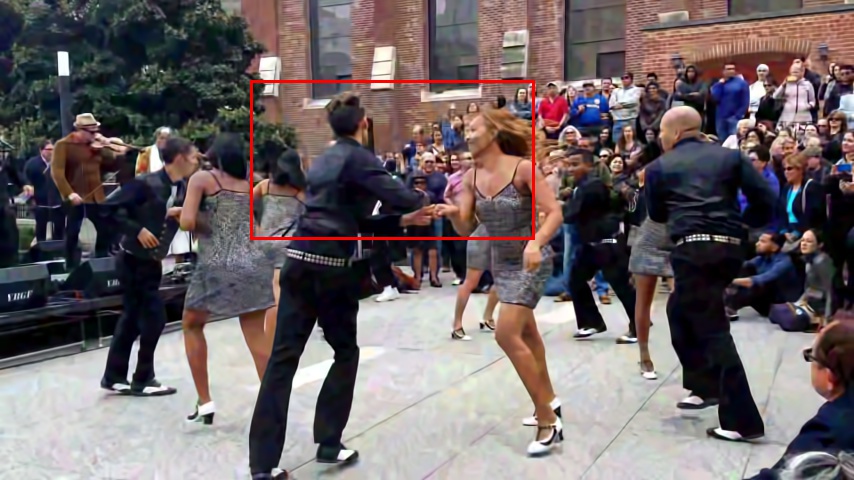}
		\caption{BM3D-O \\ 25.13 / 0.703}
		\label{fig:davis_22_1_s20_k3:bm3d_opt_rect}
	\end{subfigure}
	\begin{subfigure}{0.18\textwidth}
	    \captionsetup{justification=centering}
		\includegraphics[width=\textwidth]{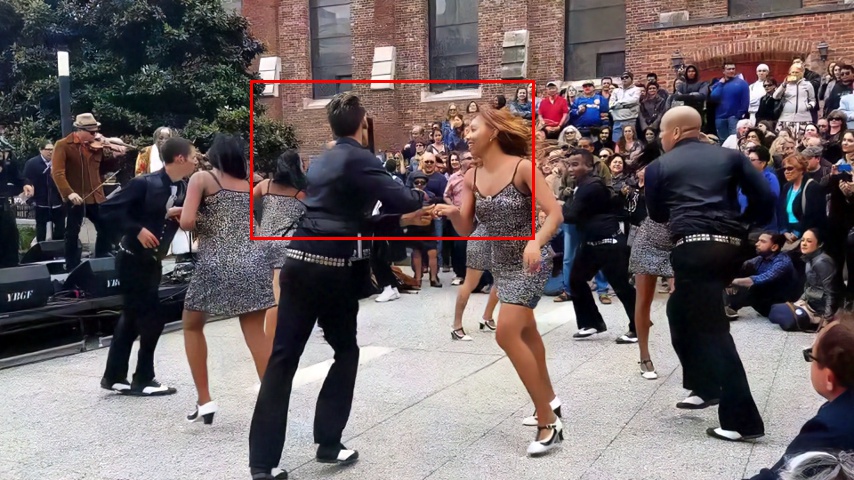}
		\caption{PC-UNet \\ 28.80 / 0.880}
		\label{fig:davis_22_1_s20_k3:pc_unet_rect}
	\end{subfigure}
	\begin{subfigure}{0.18\textwidth}
	    \captionsetup{justification=centering}
		\includegraphics[width=\textwidth]{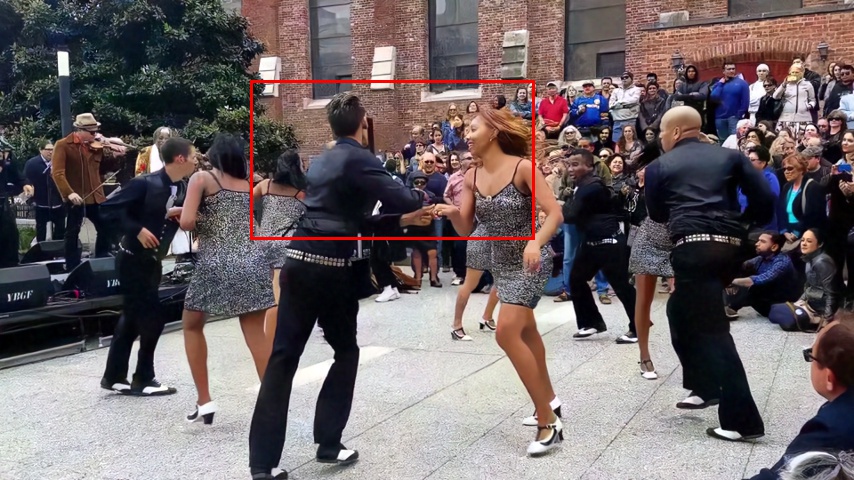}
		\caption{PC-DnCNN \\ 29.01 / 0.884}
		\label{fig:davis_22_1_s20_k3:pc_dncnn_rect}
	\end{subfigure}
	\begin{subfigure}{0.18\textwidth}
	    \captionsetup{justification=centering}
		\includegraphics[width=\textwidth]{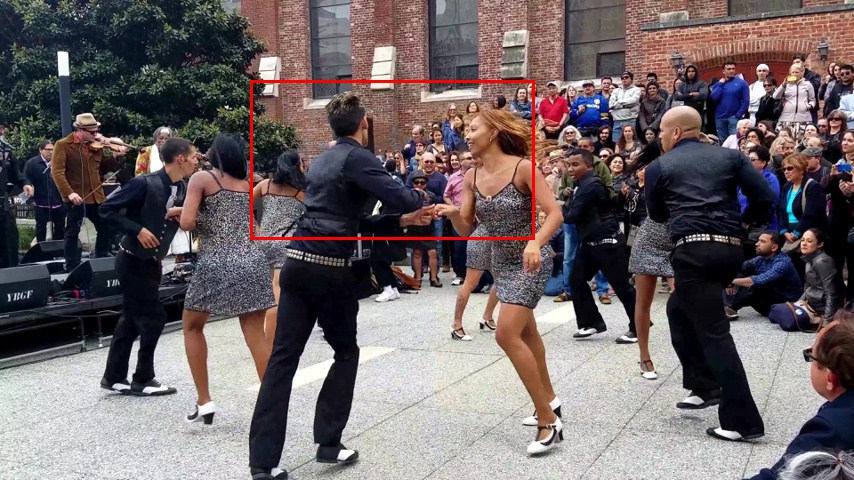}
		\caption{Clean \newline }
		\label{fig:davis_22_1_s20_k3:clean_rect}
	\end{subfigure}
	\begin{subfigure}{0.18\textwidth}
	    \captionsetup{justification=centering}
		\includegraphics[width=\textwidth]{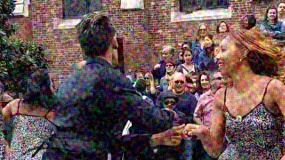}
		\caption{Noisy}
		\label{fig:davis_22_1_s20_k3:noisy_crop}
	\end{subfigure}
	\begin{subfigure}{0.18\textwidth}
	    \captionsetup{justification=centering}
		\includegraphics[width=\textwidth]{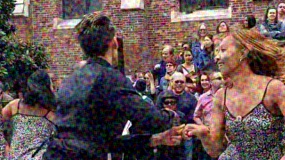}
		\caption{N2N}
		\label{fig:davis_22_1_s20_k3:n2n_crop}
	\end{subfigure}
	\begin{subfigure}{0.18\textwidth}
	    \captionsetup{justification=centering}
		\includegraphics[width=\textwidth]{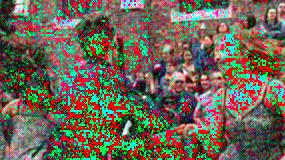}
		\caption{B2U}
		\label{fig:davis_22_1_s20_k3:b2u_crop}
	\end{subfigure}
	\begin{subfigure}{0.18\textwidth}
	    \captionsetup{justification=centering}
		\includegraphics[width=\textwidth]{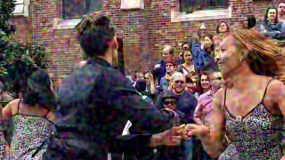}
		\caption{BM3D}
		\label{fig:davis_22_1_s20_k3:bm3d_crop}
	\end{subfigure}
	\begin{subfigure}{0.18\textwidth}
	    \captionsetup{justification=centering}
		\includegraphics[width=\textwidth]{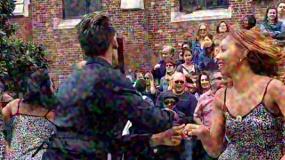}
		\caption{B-DnCNN}
		\label{fig:davis_22_1_s20_k3:b_dncnn_crop}
	\end{subfigure}
	\begin{subfigure}{0.18\textwidth}
	    \captionsetup{justification=centering}
		\includegraphics[width=\textwidth]{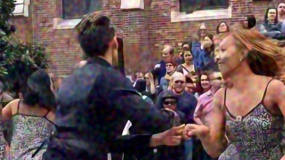}
		\caption{R2R}
		\label{fig:davis_22_1_s20_k3:r2r_crop}
	\end{subfigure}
	\begin{subfigure}{0.18\textwidth}
	    \captionsetup{justification=centering}
		\includegraphics[width=\textwidth]{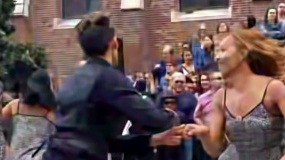}
		\caption{BM3D-O}
		\label{fig:davis_22_1_s20_k3:bm3d_opt_crop}
	\end{subfigure}
	\begin{subfigure}{0.18\textwidth}
	    \captionsetup{justification=centering}
		\includegraphics[width=\textwidth]{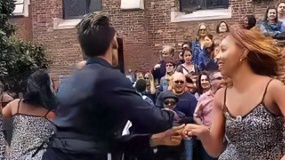}
		\caption{PC-UNet}
		\label{fig:davis_22_1_s20_k3:pc_unet_crop}
	\end{subfigure}
	\begin{subfigure}{0.18\textwidth}
	    \captionsetup{justification=centering}
		\includegraphics[width=\textwidth]{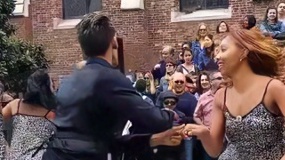}
		\caption{PC-DnCNN}
		\label{fig:davis_22_1_s20_k3:pc_dncnn_crop}
	\end{subfigure}
	\begin{subfigure}{0.18\textwidth}
	    \captionsetup{justification=centering}
		\includegraphics[width=\textwidth]{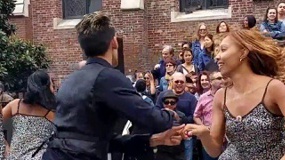}
		\caption{Clean}
		\label{fig:davis_22_1_s20_k3:clean_crop}
	\end{subfigure}
	\begin{subfigure}{0.18\textwidth}
	    \captionsetup{justification=centering}
		\includegraphics[width=\textwidth]{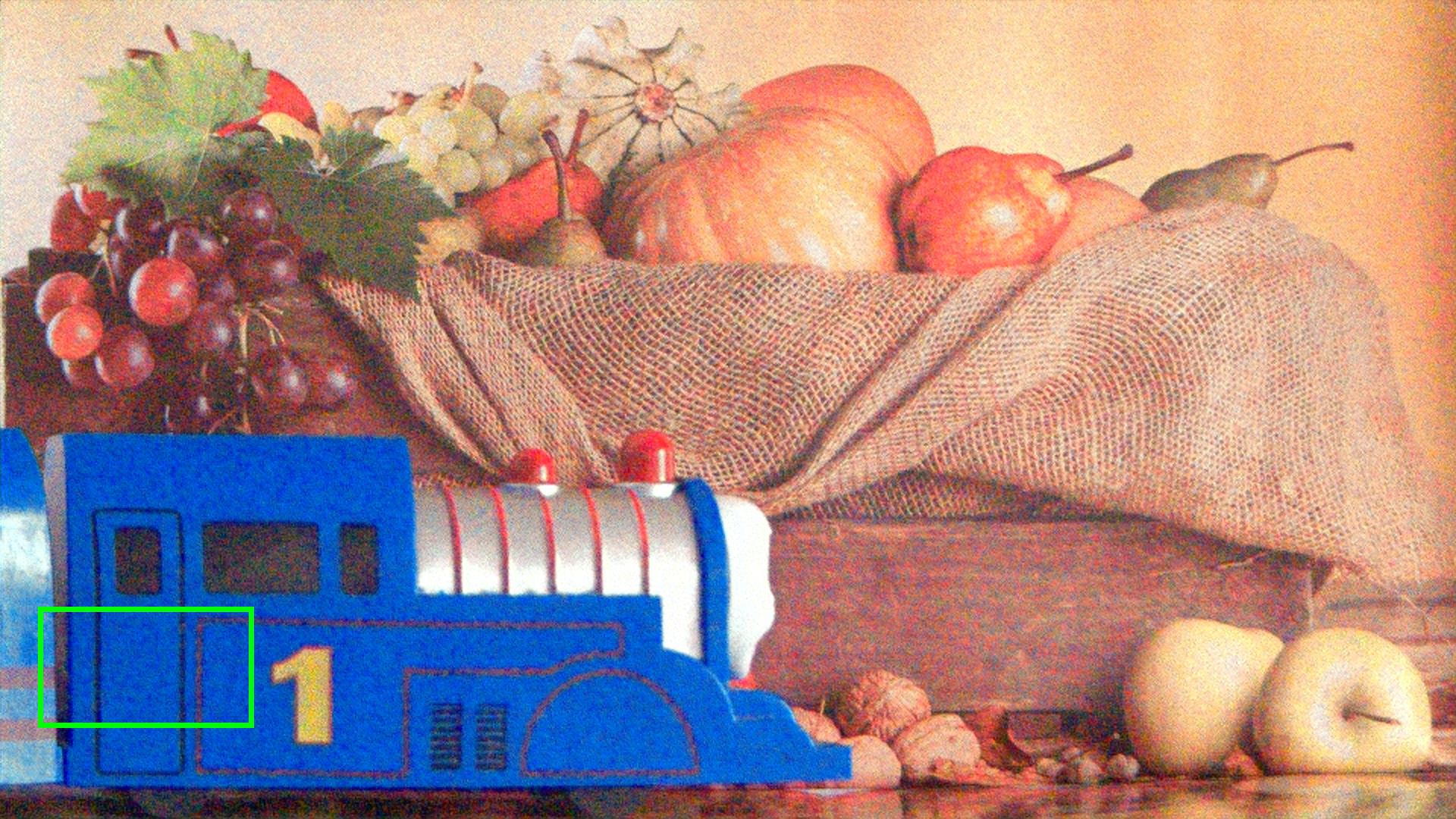}
		\caption{Noisy \\ 26.84 / 0.633}
		\label{fig:crvd_9_3_iso25600:noisy_rect}
	\end{subfigure}
	\begin{subfigure}{0.18\textwidth}
	    \captionsetup{justification=centering}
		\includegraphics[width=\textwidth]{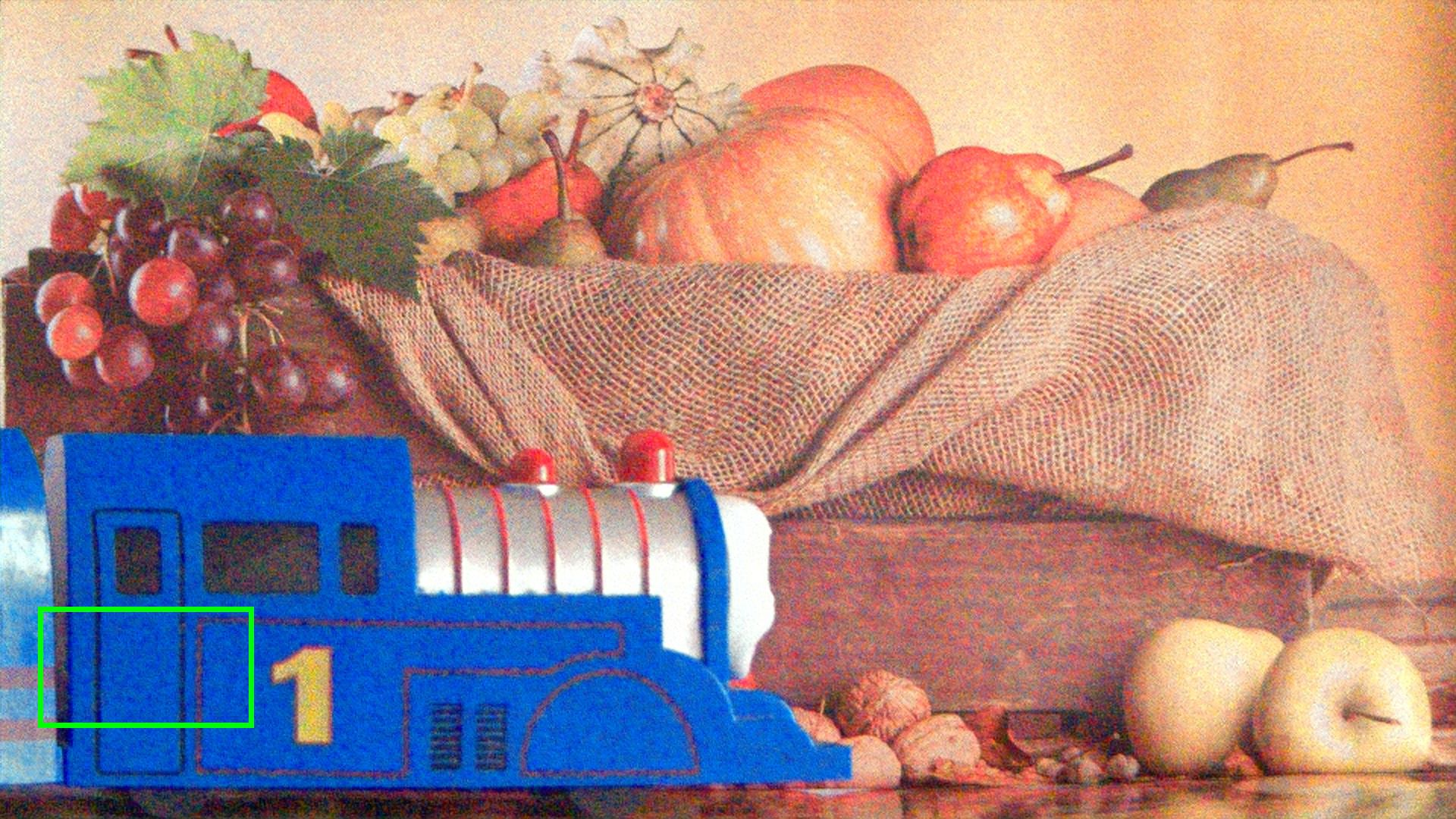}
		\caption{N2N \\ 26.94 / 0.637}
		\label{fig:crvd_9_3_iso25600:n2n_rect}
	\end{subfigure}
	\begin{subfigure}{0.18\textwidth}
	    \captionsetup{justification=centering}
		\includegraphics[width=\textwidth]{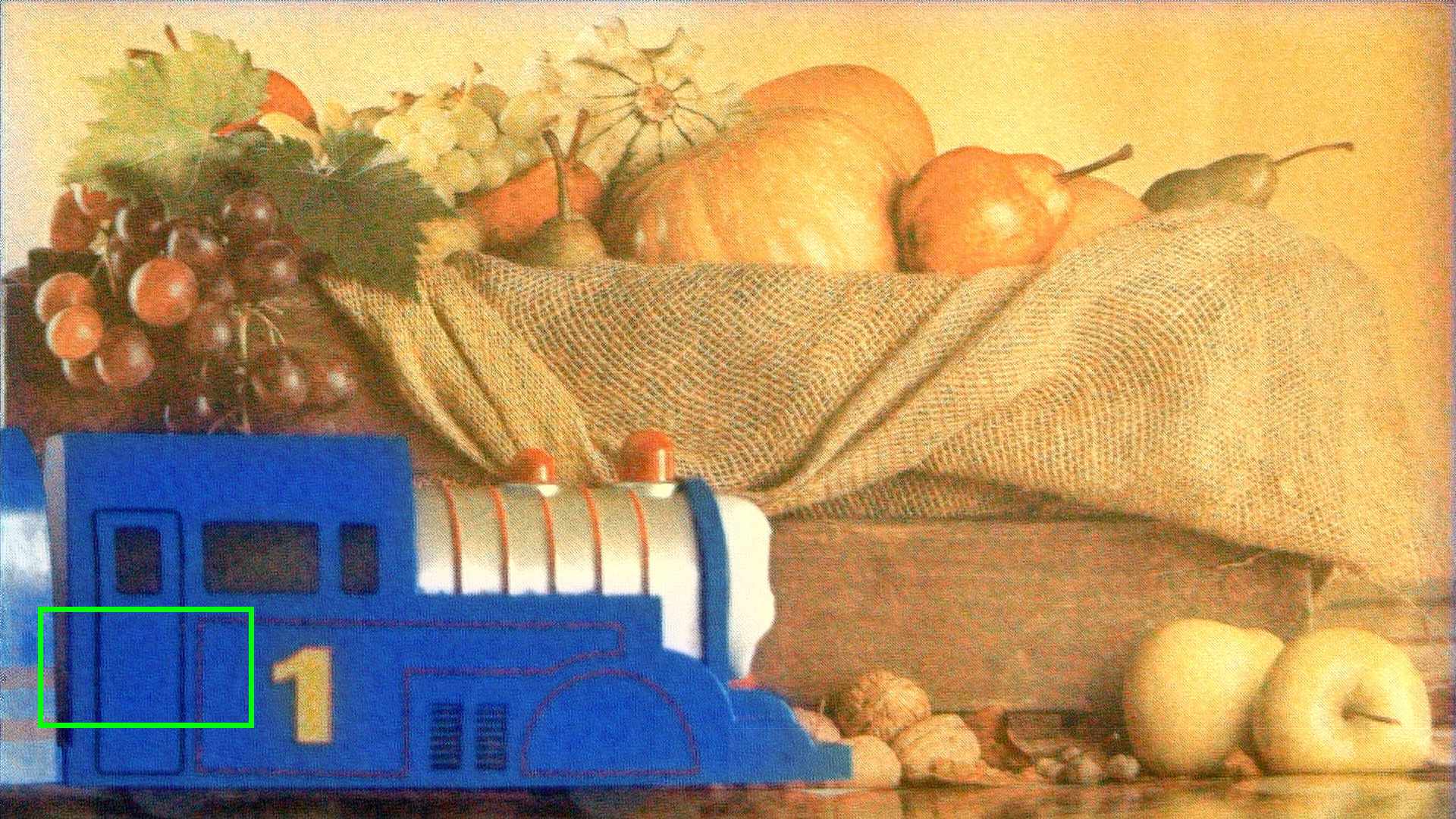}
		\caption{B2U \\ 19.27 / 0.387}
		\label{fig:crvd_9_3_iso25600:b2u_rect}
	\end{subfigure}
	\begin{subfigure}{0.18\textwidth}
	    \captionsetup{justification=centering}
		\includegraphics[width=\textwidth]{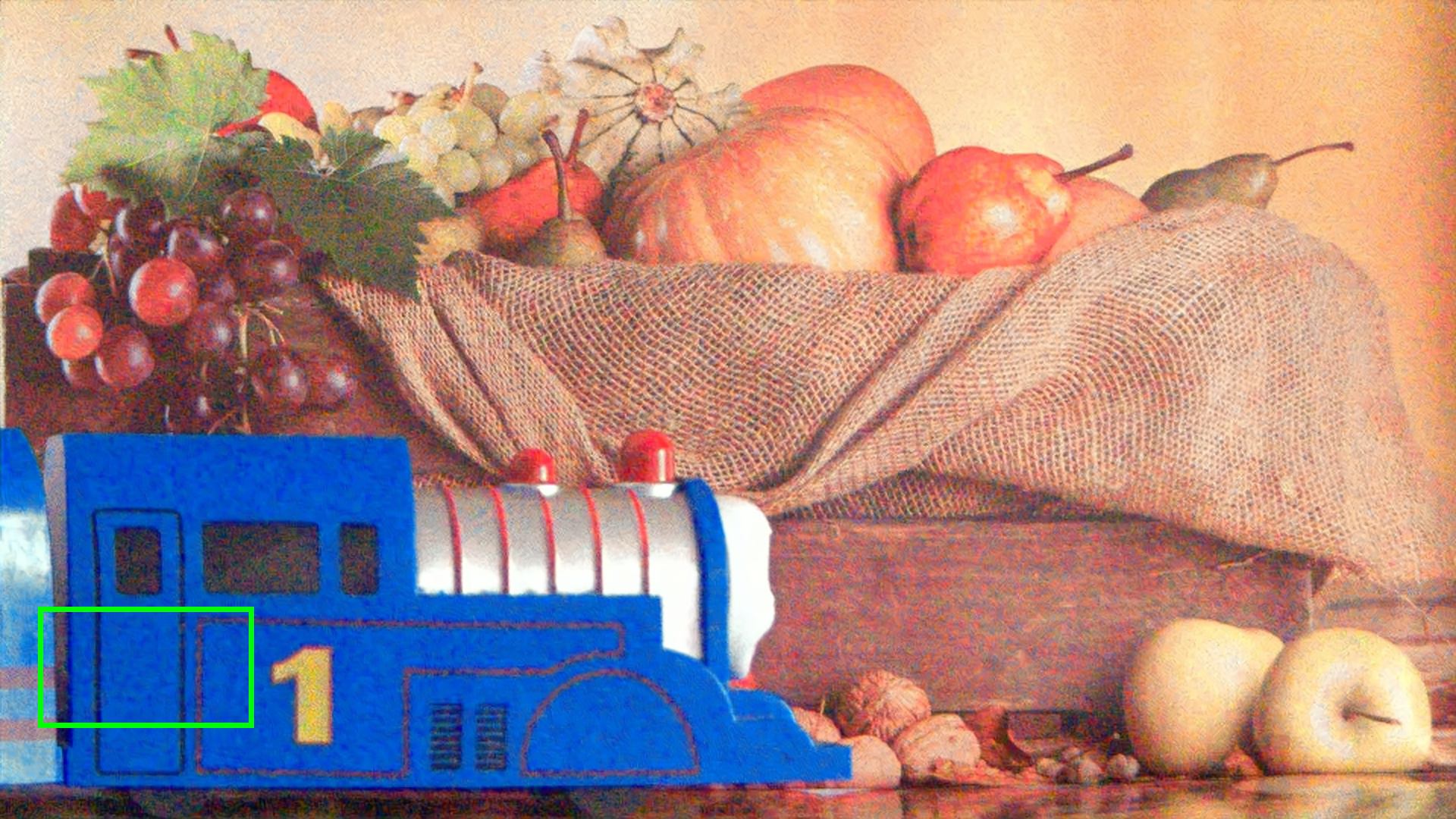}
		\caption{BM3D \\ 28.40 / 0.756}
		\label{fig:crvd_9_3_iso25600:bm3d_rect}
	\end{subfigure}
	\begin{subfigure}{0.18\textwidth}
	    \captionsetup{justification=centering}
		\includegraphics[width=\textwidth]{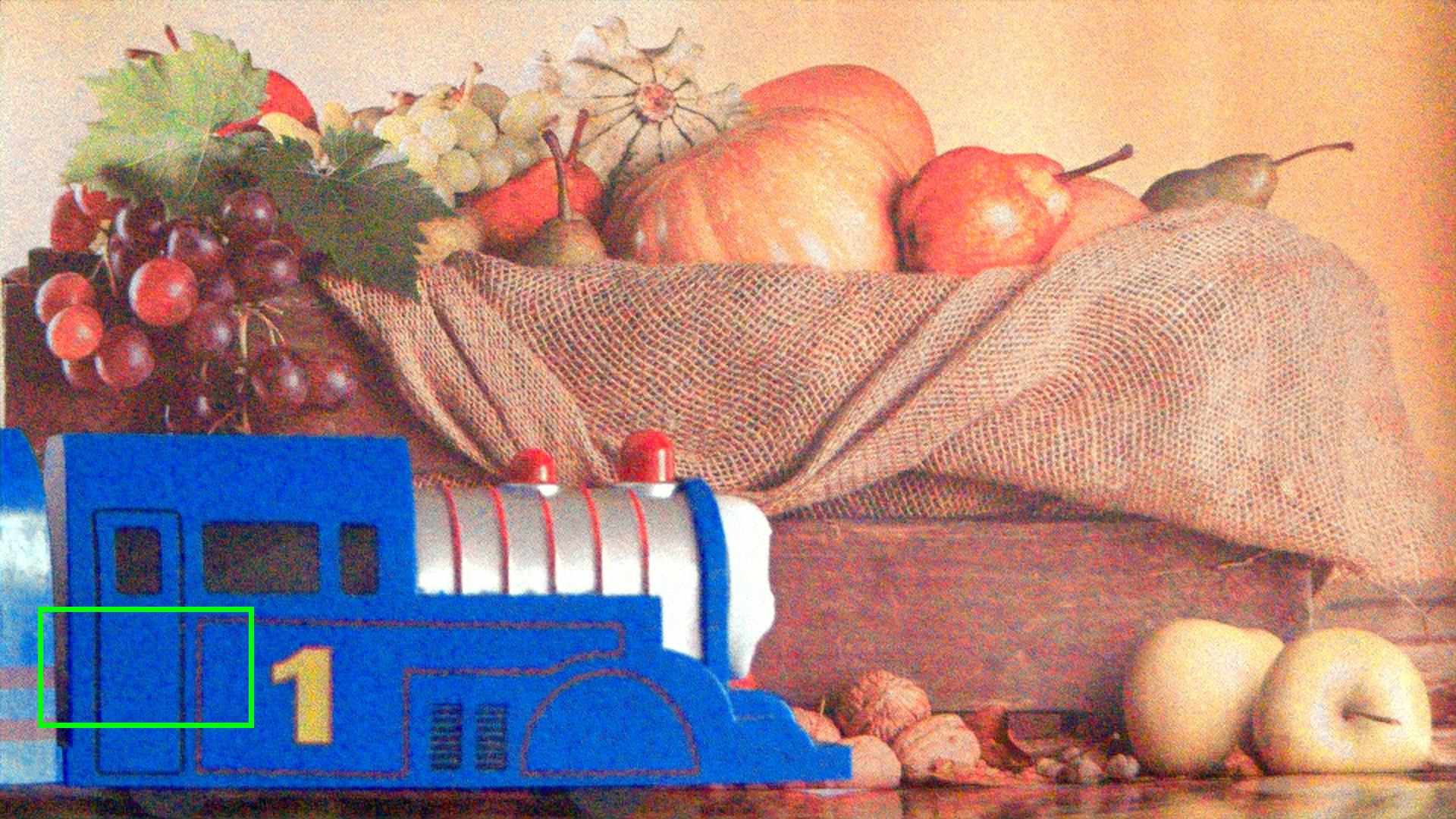}
		\caption{B-DnCNN \\ 26.93 / 0.637}
		\label{fig:crvd_9_3_iso25600:b_dncnn_rect}
	\end{subfigure}
	\begin{subfigure}{0.18\textwidth}
	    \captionsetup{justification=centering}
		\includegraphics[width=\textwidth]{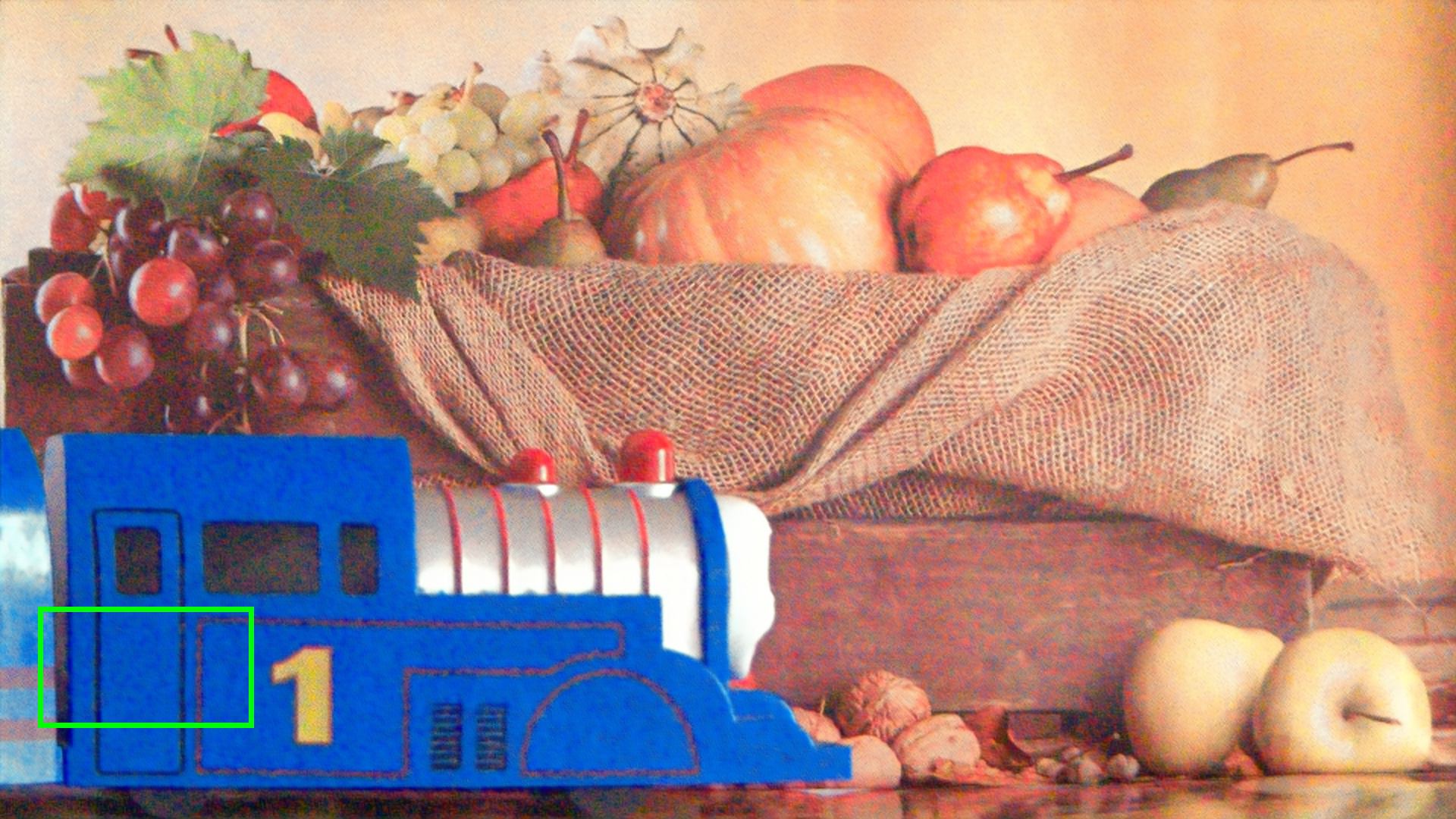}
		\caption{R2R \\ 29.64 / 0.820}
		\label{fig:crvd_9_3_iso25600:r2r_rect}
	\end{subfigure}
	\begin{subfigure}{0.18\textwidth}
	    \captionsetup{justification=centering}
		\includegraphics[width=\textwidth]{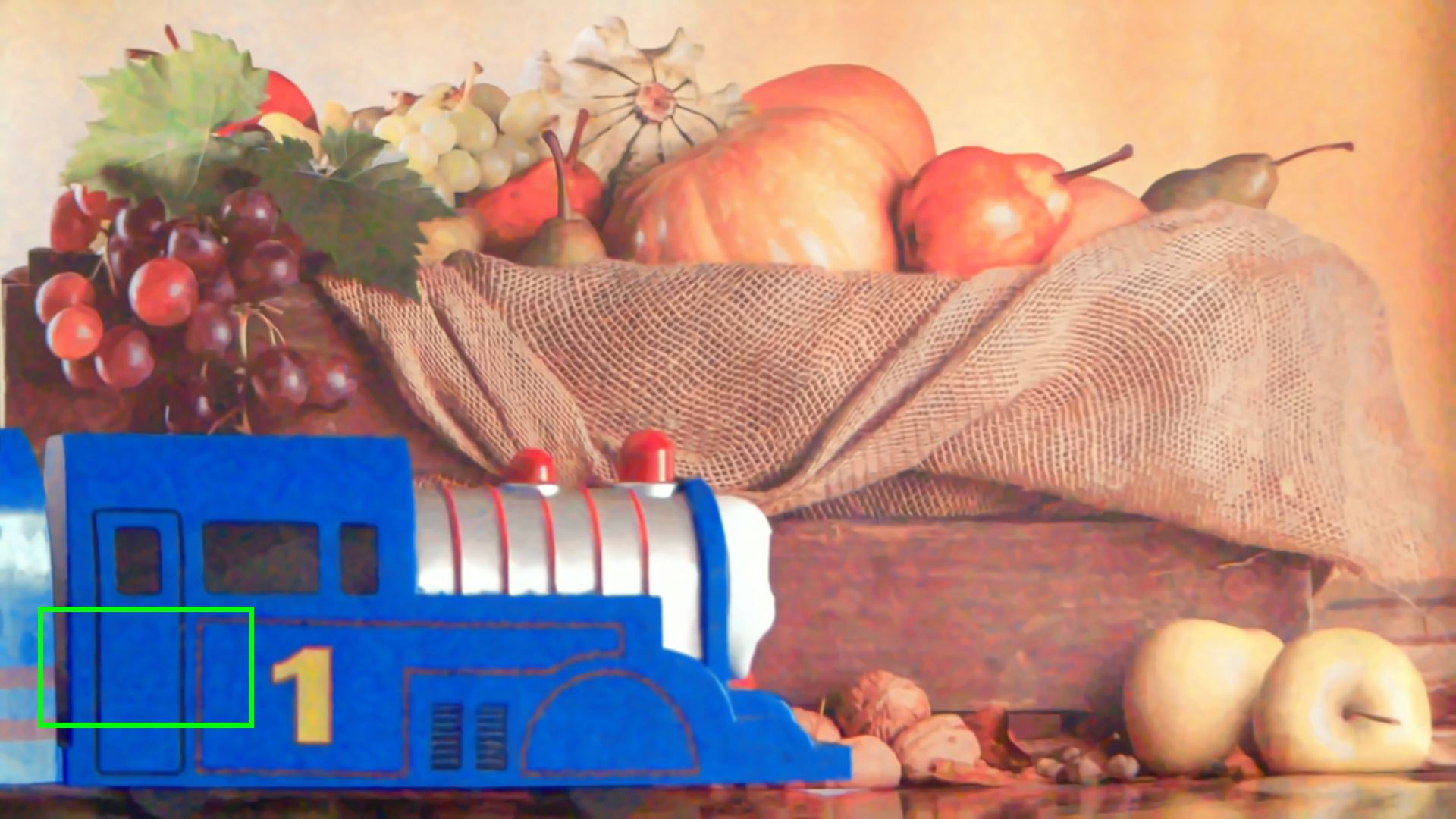}
		\caption{BM3D-O \\ 31.18 / 0.902}
		\label{fig:crvd_9_3_iso25600:bm3d_opt_rect}
	\end{subfigure}
	\begin{subfigure}{0.18\textwidth}
	    \captionsetup{justification=centering}
		\includegraphics[width=\textwidth]{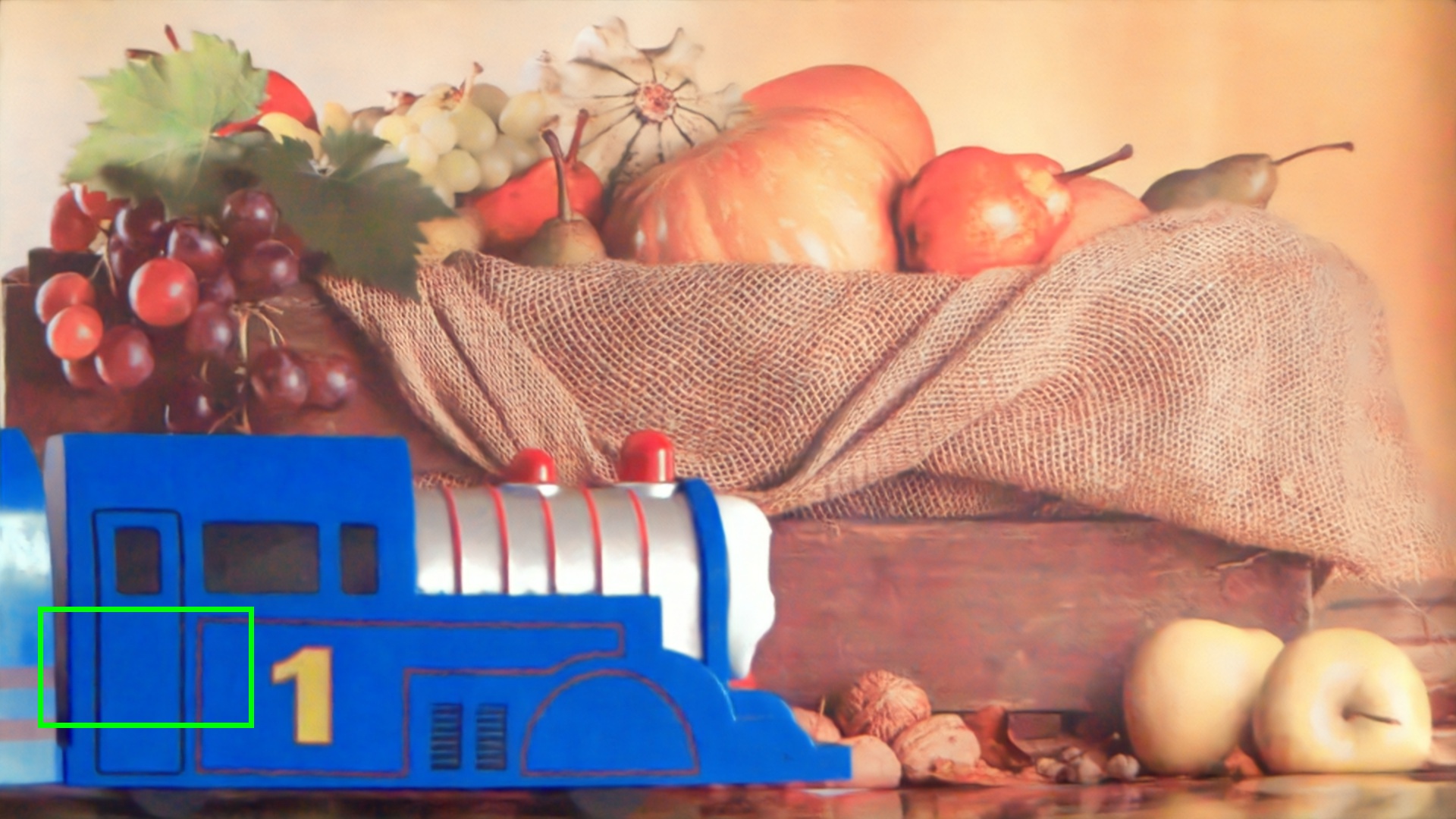}
		\caption{PC-UNet \\ 32.20 / 0.928}
		\label{fig:crvd_9_3_iso25600:pc_unet_rect}
	\end{subfigure}
	\begin{subfigure}{0.18\textwidth}
	    \captionsetup{justification=centering}
		\includegraphics[width=\textwidth]{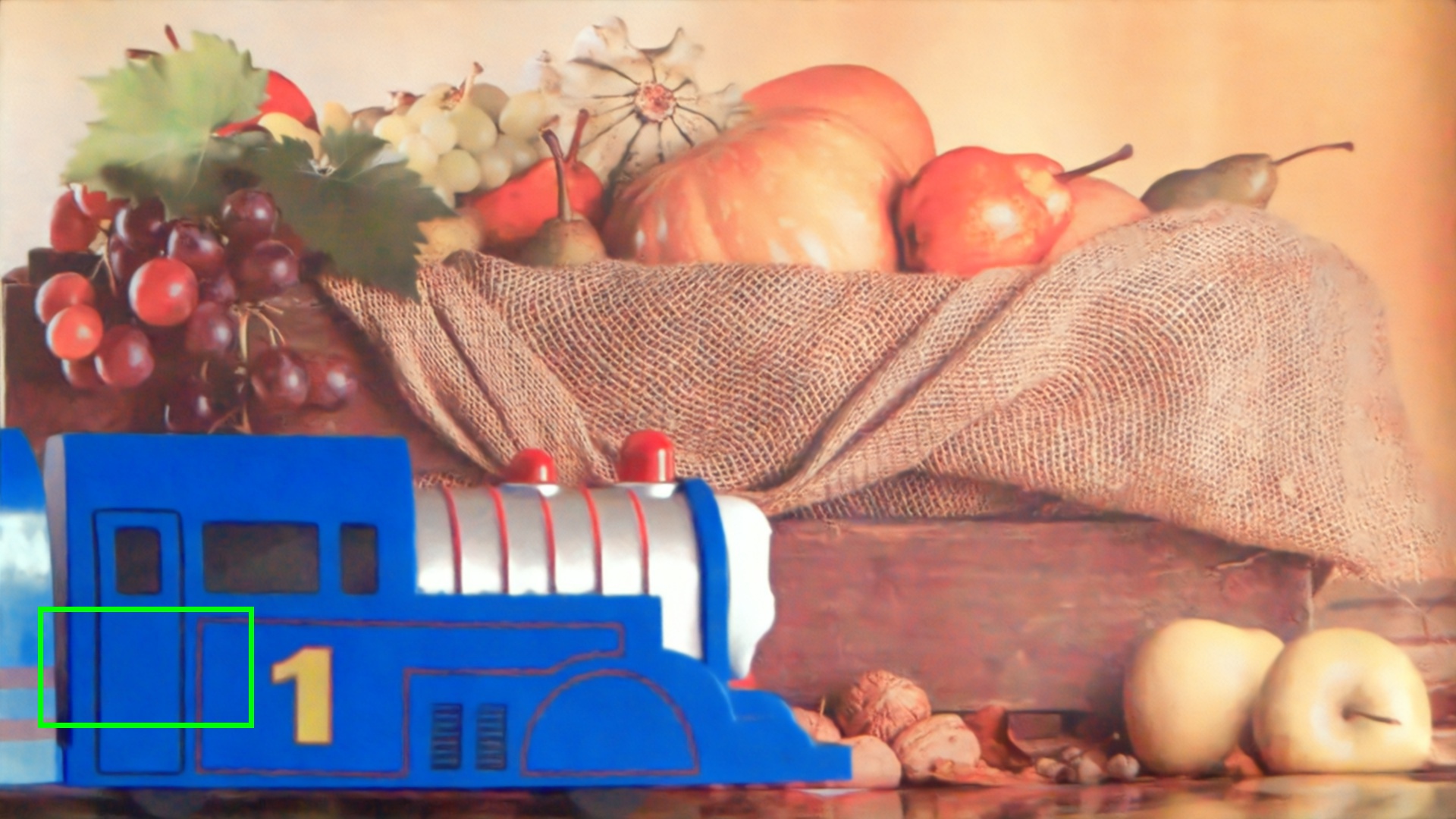}
		\caption{PC-DnCNN \\ 32.04 / 0.927}
		\label{fig:crvd_9_3_iso25600:pc_dncnn_rect}
	\end{subfigure}
	\begin{subfigure}{0.18\textwidth}
	    \captionsetup{justification=centering}
		\includegraphics[width=\textwidth]{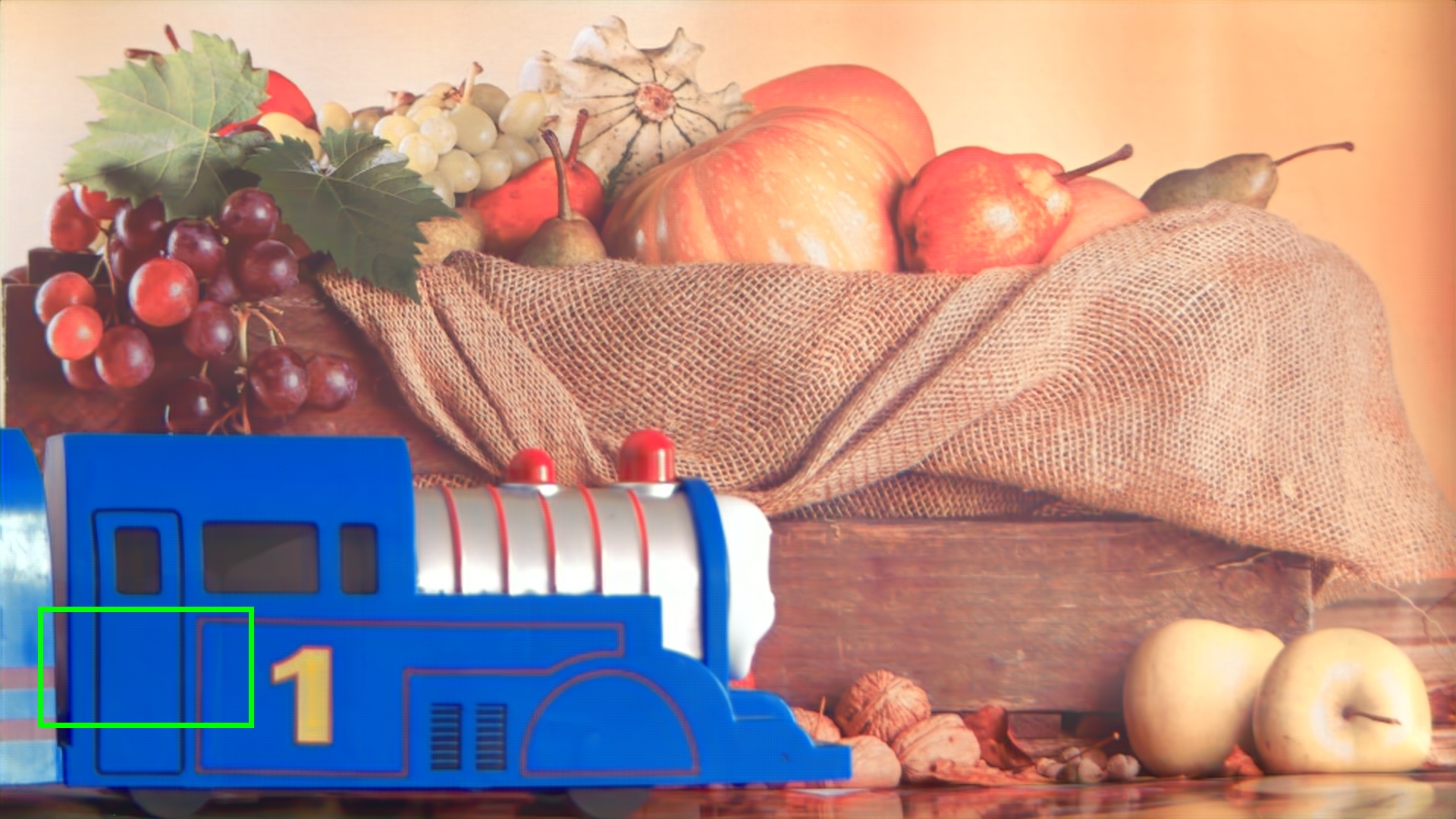}
		\caption{Clean \newline }
		\label{fig:crvd_9_3_iso25600:clean_rect}
	\end{subfigure}
	\begin{subfigure}{0.18\textwidth}
	    \captionsetup{justification=centering}
		\includegraphics[width=\textwidth]{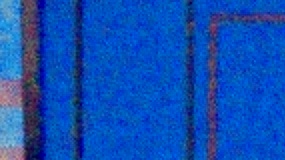}
		\caption{Noisy}
		\label{fig:crvd_9_3_iso25600:noisy_crop}
	\end{subfigure}
	\begin{subfigure}{0.18\textwidth}
	    \captionsetup{justification=centering}
		\includegraphics[width=\textwidth]{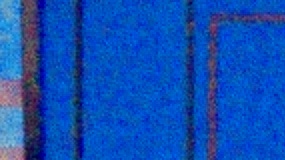}
		\caption{N2N}
		\label{fig:crvd_9_3_iso25600:n2n_crop}
	\end{subfigure}
	\begin{subfigure}{0.18\textwidth}
	    \captionsetup{justification=centering}
		\includegraphics[width=\textwidth]{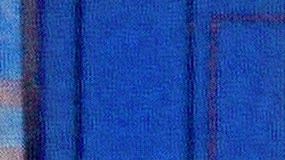}
		\caption{B2U}
		\label{fig:crvd_9_3_iso25600:b2u_crop}
	\end{subfigure}
	\begin{subfigure}{0.18\textwidth}
	    \captionsetup{justification=centering}
		\includegraphics[width=\textwidth]{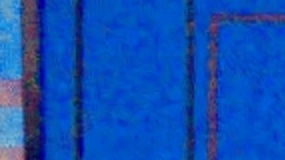}
		\caption{BM3D}
		\label{fig:crvd_9_3_iso25600:bm3d_crop}
	\end{subfigure}
	\begin{subfigure}{0.18\textwidth}
	    \captionsetup{justification=centering}
		\includegraphics[width=\textwidth]{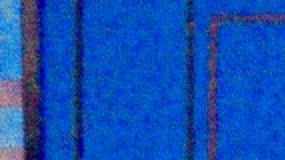}
		\caption{B-DnCNN}
		\label{fig:crvd_9_3_iso25600:b_dncnn_crop}
	\end{subfigure}
	\begin{subfigure}{0.18\textwidth}
	    \captionsetup{justification=centering}
		\includegraphics[width=\textwidth]{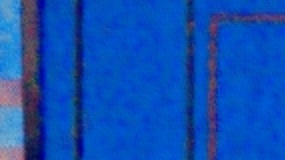}
		\caption{R2R}
		\label{fig:crvd_9_3_iso25600:r2r_crop}
	\end{subfigure}
	\begin{subfigure}{0.18\textwidth}
	    \captionsetup{justification=centering}
		\includegraphics[width=\textwidth]{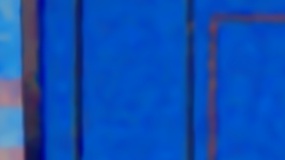}
		\caption{BM3D-O}
		\label{fig:crvd_9_3_iso25600:bm3d_opt_crop}
	\end{subfigure}
	\begin{subfigure}{0.18\textwidth}
	    \captionsetup{justification=centering}
		\includegraphics[width=\textwidth]{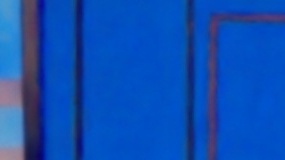}
		\caption{PC-UNet}
		\label{fig:crvd_9_3_iso25600:pc_unet_crop}
	\end{subfigure}
	\begin{subfigure}{0.18\textwidth}
	    \captionsetup{justification=centering}
		\includegraphics[width=\textwidth]{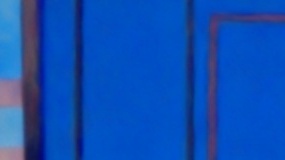}
		\caption{PC-DnCNN}
		\label{fig:crvd_9_3_iso25600:pc_dncnn_crop}
	\end{subfigure}
	\begin{subfigure}{0.18\textwidth}
	    \captionsetup{justification=centering}
		\includegraphics[width=\textwidth]{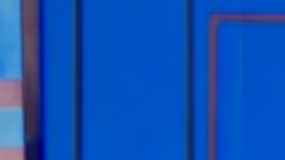}
		\caption{Clean}
		\label{fig:crvd_9_3_iso25600:clean_crop}
	\end{subfigure}
	\caption{Denoising examples with two types of noise. The first four rows show frame 13 of the sequence \emph{salsa} in the DAVIS dataset with correlated Gaussian noise with $\sigma = 20$ and $k = 3$. The last four rows present frame 4 of scene 10 in the CRVD dataset with IS0 25600. As can be seen, in both experiments, oracle BM3D leaves a substantial amount of low-frequency noise unfiltered. In addition, it produces blurred output for the DAVIS frame. Other algorithms, except ours (PC-UNet and PC-DnCNN), fail to remove the noise, while B2U loses stability during the training.}
	\label{fig:davis_22_1_s20_k3_crvd_9_3_iso25600}
\end{figure*}

\subsection{Real-World Noise Removal}
Real-world noise refers to a particular sensor whose model is unknown, and it's distribution may vary with sensor parameters such as ISO, aperture, exposure time, etc.. Finding a dataset for this evaluation is a challenging task. To the best of our knowledge, there are no burst or video datasets with such noise that include ground truth images. For example, it is impossible to use the popular SIDD~\cite{abdelhamed2018high}, DND~\cite{plotz2017benchmarking}, CC~\cite{nam2016holistic}, and PolyU~\cite{xu2018real} datasets, as they only contain single images and not bursts. 

For conducting a real-world image denoising experiment, we use the CRVD~\cite{yue2020supervised} dataset, which consists of 11 groups of noisy pictures taken in a photo laboratory and their ground truth counterparts. Each group captures a different scene, each scene is captured 7 times, there is some movement between each capture.
Each of these groups of 7 images can be considered as an artificial video sequence.
However, the movements of objects and changes in lighting captured in these artificial sequences are incomparably sharper than in a typical video or an image burst. Examples of such artificial video sequences are presented in figure~\ref{fig:crvd_example_s5_s7} in Appendix~\ref{app:additional_results}. 

When applying the patch-craft framework, it is better to avoid sequences with sharp movements and severe lighting changes since the latter makes patch matching difficult. With that in mind, it is interesting that our framework achieves favorable results even when trained on a small and challenging dataset. Thus, among other things, this experiment indicates the robustness of the proposed method.

Since the CRVD dataset is small, we augment it by replicating each sequence 7 times, 
where in each replica a different image is used as a middle frame.
Then, similarly to the correlated Gaussian denoising experiment, we use the middle frames as network inputs, and use the 6 surrounding frames for building the patch-craft targets. 
We test the network using the same 77 CRVD frames by comparing their output with the ground truth images.
Note that the network can not overfit the ground truth images, as they are not available during  training.  

Since the competing methods are not designed for training on CRVD, we adapt the CRVD dataset 
in such a way that each method trains in the conditions close to the ones described in its paper.
R2R uses a set containing 400 images of size $180 \times 180$ pixels, augmenting it with scaling by 4 factors (1, 0.9, 0.8, 0.7). Thus, for training R2R, we use $4 \times 400$ random crops of CRVD images, each crop of size $180 \times 180$. Note that we disable the scaling augmentation since scaling may affect the noise statistics. Also, we choose $\alpha = 2$ which leads to the best denoising results. Unlike R2R, N2N and B2U methods train their networks using $44,328$ images from ImageNet~\cite{deng2009imagenet}. For N2N and B2U we create $44,328$ random crops of the CRVD images, each of size $256 \times 256$. 

Table~\ref{tab:real_im_n_res} summarizes the denoising performance for different ISO values. A visual comparison of the denoised images is shown in Figure~\ref{fig:davis_22_1_s20_k3_crvd_9_3_iso25600} and Figures~\ref{fig:crvd_8_5_iso25600_2_3_iso25600},~\ref{fig:crvd_5_2_iso12800_10_5_iso12800},~and~\ref{fig:crvd_1_5_iso6400_0_4_iso3200} in Appendix~\ref{app:additional_results}. These results lead to a similar conclusion as in the correlated Gaussian denoising experiment: The current SoTA self-supervised methods that we compare with face difficulties in training when the contaminating noise is correlated, and this difficulty strengthens with ISO. For high ISO values, our framework outperforms the O-BM3D in terms of PSNR and SSIM, where the latter tends to leave a noticeable amount of low-frequency noise unfiltered. 

\section{Conclusion}
Recent literature pays relatively little attention to the problem of correlated noise reduction in images, probably due to its toughness. Such methods can be used, for instance, for real-world noise reduction in sRGB color space\footnote{Note that some self-supervised learning methods, including R2R, B2U, and N2N, do succeed and show good denoising performance in \emph{raw-RGB} color space, since the noise in this space has lower spatial and cross-channel correlations~\cite{nam2016holistic,yue2020supervised}. However, this is not the case in sRGB.}, since such noise is usually correlated. This paper proposes a novel self-supervised framework for training a denoiser where the contaminating noise is spatially and cross-channel correlated. The proposed framework relies on the availability of bursts or short video sequences of noisy frames. Our method applies patch matching for building patch-craft images and employs them as training targets. We present a statistical analysis of the target noise that leads to excluding of faulty image pairs from the training set, thereby boosting the obtained denoising performance. 
The proposed framework shows outstanding denoising results compared with the recent SoTA self-supervised training algorithms.\footnote{The code reproducing the results of this paper is available at \href{https://github.com/grishavak/pcst}{https://github.com/grishavak/pcst}.}


{\small
\bibliographystyle{ieee_fullname}
\bibliography{egbib}
}

\appendix

\section{Proof of Lemma~\ref{lem:sgd}}
\label{app:proof_sgd}
\begin{proof}
We begin with rewriting the expression for $\nabla_{\mathbf{\theta}} \tilde{l}$, 
\begin{equation}
    \begin{split}
        \nabla_{\mathbf{\theta}} \tilde{l} & = \nabla_\theta^Tf_{\mathbf{\theta}}\left(\mathbf{y}\right)\left(f_{\mathbf{\theta}}\left(\mathbf{y}\right) - \mathbf{\tilde{x}}\right) = \\
        & = \nabla_\theta^Tf_{\mathbf{\theta}}\left(\mathbf{y}\right)\left(f_{\mathbf{\theta}}\left(\mathbf{y}\right) - \left(\mathbf{x} + \mathbf{w}\right)\right) = \\
        & = \nabla_\theta^Tf_{\mathbf{\theta}}\left(\mathbf{y}\right)\left(f_{\mathbf{\theta}}\left(\mathbf{y}\right) - \mathbf{x}\right) - \nabla_\theta^Tf_{\mathbf{\theta}}\left(\mathbf{y}\right)\mathbf{w} = \\
        & = \nabla_{\mathbf{\theta}} l - \nabla_\theta^Tf_{\mathbf{\theta}}\left(\mathbf{y}\right)\mathbf{w} \;.
    \end{split}
\end{equation}
We proceed to calculating the expectation of $\nabla_{\mathbf{\theta}} \tilde{l}$, 
\begin{equation}
\label{eq:l_tilde_exp}
    \begin{split}
        \E\left[\nabla_{\mathbf{\theta}} \tilde{l}\right] & = \E\left[\E\left[\left.\nabla_{\mathbf{\theta}} \tilde{l}\right\vert\mathbf{x},\mathbf{z}\right]\right] = \\
        & = \E\left[\E\left[\left.\nabla_{\mathbf{\theta}} l - \nabla_\theta^Tf\left(\mathbf{y}\right)\mathbf{w}\right\vert\mathbf{x},\mathbf{z}\right]\right] = \\
        & = \E\left[\E\left[\left.\nabla_{\mathbf{\theta}} l \right\vert\mathbf{x},\mathbf{z}\right] -  \E\left[\left.\nabla_\theta^Tf\left(\mathbf{y}\right)\mathbf{w}\right\vert\mathbf{x},\mathbf{z}\right]\right] = \\
        & = \E\left[\nabla_{\mathbf{\theta}} l -  \nabla_\theta^Tf\left(\mathbf{y}\right)\E\left[\left.\mathbf{w}\right\vert\mathbf{x},\mathbf{z}\right]\right] = \\
        & = \E\left[\nabla_{\mathbf{\theta}} l -  \nabla_\theta^Tf\left(\mathbf{y}\right)\E\left[\mathbf{w}\right]\right] = \\
        & = \E\left[\nabla_{\mathbf{\theta}} l -  \nabla_\theta^Tf\left(\mathbf{y}\right)\cdot 0\right] = \\
        & = \E\left[\nabla_{\mathbf{\theta}} l\right] = \nabla_{\mathbf{\theta}} L\;.
    \end{split}
\end{equation}
The first equality, $\E\left[\nabla_{\mathbf{\theta}} \tilde{l}\right] = \E\left[\E\left[\left.\nabla_{\mathbf{\theta}} \tilde{l}\right\vert\mathbf{x},\mathbf{z}\right]\right]$, is correct by the law of total expectation, and $\E\left[\left.\mathbf{w}\right\vert\mathbf{x},\mathbf{z}\right] = \E\left[\mathbf{w}\right]$ is true due to the independence of $\mathbf{w}$ in $\mathbf{x}$ and $\mathbf{z}$. 
Finally, we get
\begin{equation}
    \begin{split}
        Bias\left[\nabla_{\mathbf{\theta}} \tilde{l}\right] & = \E\left[\nabla_{\mathbf{\theta}} \tilde{l}\right] - \nabla_{\mathbf{\theta}} L = 0 \;,
    \end{split}
\end{equation}
which completes the proof.
\end{proof}

\section{Dependency Reduction -- Mathematical Analysis}
\label{app:dep_red}
In this section we bring proofs of statements used in Section~\ref{sec:dep_red}, organized as follows. In Section~\ref{app:clt} we prove the convergence of $s_{\mathbf{y}, \mathbf{r}}$ to a Gaussian distribution, and in Section~\ref{app:mean_drift} we prove our statements regarding the shift of the mean in the case of dependencies of types~(I)~and~(II).

\subsection{Convergence to Gaussian}
\label{app:clt}
Our goal is to prove that $s_{\mathbf{y}, \mathbf{r}}$ convergences in distribution to a Gaussian. We start with the assumption that the ground truth image is a sum of two statistically independent random variables, $\mathbf{x} = \mathbf{\bar{x}} + \mu_x$, where $\mu_x$ is an image mean (a scalar) and $\mathbf{\bar{x}}$ is a zero-mean vector. Similarly, we assume that $\mathbf{w} = \mathbf{\bar{w}} + \mu_w$. Our additional assumption is that $\mathbf{\bar{x}}$, $\mathbf{z}$, and $\mathbf{\bar{w}}$ are $m$\emph{-dependent} in any dimension and follow zero-mean and \emph{stationary} distributions. 
\begin{definition} [$m$-dependent sequences~\cite{billingsley2008probability}]
\ \newline Let $X_1, X_2, \dots$ be a sequence of random variables. The sequence is $\boldsymbol{m}$\textbf{-dependent} if $\left(X_1, \dots, X_i\right)$ and $\left(X_{i+k}, \dots, X_{i+k+l}\right)$ are independent whenever $k>m$.
\end{definition}
By $m$-dependency in any dimension, we mean that $\left(X_{1,1}, \dots, X_{i,j}\right)$ and $\left(X_{i+k, j+r}, \dots, X_{i+k+l, j+r+p}\right)$ are independent if ${\max\{k, r\}>m}$.
For natural images distribution, stationarity and $m$-dependency are common assumptions. Stationarity means translation invariance, and the $m$-dependency assumption implies that each pixel is dependent only on its local neighborhood of radius $m/2$. Indeed, $m$-dependency is the property that allows effective denoisers to have relatively small respective fields. To proceed with the proof, we need the following definition and theorem.

\begin{definition} [Strongly mixing sequences~\cite{billingsley2008probability}]
\ \newline Let $X_1, X_2, \dots$ be a sequence of random variables, and let $\alpha_k$ be a number such that
\begin{equation*}
    \sup \left|P\left(A \cap B\right) - P\left(A\right)P\left(B\right)\right| \le \alpha_k
\end{equation*}
for $A \in \sigma\left(X_1, \dots, X_i\right)$, $B \in \sigma\left(X_{i+k}, X_{i+k+1}, \dots \right)$, and $i \ge 1$, $k \ge 1$. The sequence is said to be 
\textbf{strongly mixing} if ${\alpha_k \longrightarrow 0}$ as ${k \longrightarrow \infty}$.
\end{definition}

\begin{theorem} [Central limit theorem for strongly mixing sequences, Theorem 1.7 in~\cite{ibragimov1962some}]
\label{th:clt_str_mix}
\ \newline Let $X_1, X_2, \dots$ be a stationary and strongly mixing sequence with $\E\left[X_i\right] = 0$ such that for some $\delta > 0$
\begin{equation*}
    \begin{split}
        \E\left[\left|X_i\right|^{2+\delta}\right] < \infty \quad \text{and} \quad \sum_{i=1}^{\infty}\left(\alpha_i\right)^{\frac{\delta}{2 + \delta}} < \infty \;.
    \end{split}
\end{equation*}
Denote $S_n = X_1 + X_2 + \dots + X_n$, then 
\begin{equation*}
    \lim_{n \rightarrow \infty} \frac{\E\left[S_n^2\right]}{n} = \sigma^2 = \E\left[X_1^2\right] + 2\sum_{i = 1}^{\infty}\E\left[X_1X_i\right] < \infty \;.
\end{equation*}
If $\sigma \ne 0$, then $\frac{S_n}{\sigma\sqrt{n}}$ converges in distribution to $\mathcal{N}\left(0, 1\right)$ as $n$ approaches infinity.
\end{theorem}

It follows from Theorem~\ref{th:clt_str_mix} follows that sum of $m$-dependent stationary sequence converges in distribution to a Gaussian, as stated in the following corollary.
\begin{corollary}
\label{col:clt_m_dep}
Let $X_1, X_2, \dots$ be a stationary and $m$-dependent sequence with $\E\left[X_i\right] = 0$ such that for some ${\delta > 0}$,
${\E\left[\left|X_i\right|^{2+\delta}\right] < \infty}$.
\\Denote ${S_n = X_1 + X_2 + \dots + X_n}$, then 
\begin{equation*}
    \lim_{n \rightarrow \infty} \frac{\E\left[S_n^2\right]}{n} = \sigma^2 = \E\left[X_1^2\right] + 2\sum_{i = 1}^{m}\E\left[X_1X_i\right] < \infty \;.
\end{equation*}
If $\sigma \ne 0$, then $\frac{S_n}{\sigma\sqrt{n}}$ converges in distribution to $\mathcal{N}\left(0, 1\right)$ as $n$ approaches infinity.
\end{corollary}
\begin{proof}
To prove it, it suffices to show that from the $m$-dependency, it follows that for any $\delta > 0$,
\begin{equation}
    \sum_{i=1}^{\infty}\left(\alpha_i\right)^{\frac{\delta}{2 + \delta}} < \infty \;.
\end{equation}
Note that for $m$-dependent sequences, $\alpha_k = 0$ for any ${k > m}$. Thus, the sum becomes finite and thereby bounded,
\begin{equation}
    \begin{split}
        \sum_{k=1}^{\infty}\left(\alpha_k\right)^{\frac{\delta}{2 + \delta}} = \sum_{k=1}^{m}\left(\alpha_k\right)^{\frac{\delta}{2 + \delta}} < \infty \;.
    \end{split}
\end{equation}
\end{proof}

Armed with Corollary~\ref{col:clt_m_dep}, we return to the proof of $s_{\mathbf{y}, \mathbf{r}}$ convergence. Assuming that $\mu_x$ and $\mu_w$ are known, $s_{\mathbf{y}, \mathbf{r}}$ is defined as
\begin{equation}
\label{eq:s_yr}
    s_{\mathbf{y}, \mathbf{r}} = \frac{1}{n^2}\sum_{i, j = 1}^{n} \left(y_{i,j} - \mu_x\right)\left(r_{i,j} - \mu_w\right).
\end{equation}
Note that in practice, $\mu_x$ and $\mu_w$ are unavailable, but they can be estimated empirically,
\begin{equation}
    \mu_{x} \approx \frac{1}{n^2}\sum_{i, j = 1}^{n} y_{i,j} \;, \quad \mu_w \approx \frac{1}{n^2}\sum_{i, j = 1}^{n} r_{i,j}.
\end{equation}
\begin{proposition}
Let $s_{\mathbf{y}, \mathbf{r}}$ be as defined in equation~\ref{eq:s_yr}.
Then, as $n \longrightarrow \infty$, 
\begin{equation*}
    \frac{n}{\sigma}\left(s_{\mathbf{y}, \mathbf{r}} + \sigma_z^2\right) \longrightarrow \mathcal{N}\left(0, 1\right) \;,
\end{equation*}
where $\sigma_z^2 = \E\left[z_{i,j}^2\right]$ and 
\begin{equation*}
    \sigma^2 = \lim_{n \rightarrow \infty} \E\left[s_{\mathbf{y}, \mathbf{w}}^2\right] \;.
\end{equation*}
\end{proposition}

\begin{proof}
We begin by rewriting the expression for $\mathbf{r}$,
\begin{equation}
\label{eq:r_w_m_z}
    \mathbf{r} = \mathbf{\tilde{x}} - \mathbf{y} = \mathbf{x} + \mathbf{w} - \mathbf{x} - \mathbf{z} = \mathbf{w} - \mathbf{z} \;.
\end{equation}
Substituting $\mathbf{r}$ into the expression for $s_{\mathbf{y}, \mathbf{r}}$ we get
\begin{equation}
    \begin{split}
        & s_{\mathbf{y}, \mathbf{r}} + \sigma_z^2 = \frac{1}{n^2}\sum_{i, j = 1}^{n} \left(y_{i,j} - \mu_x\right)\left(r_{i,j} - \mu_w\right) + \sigma_z^2 = \\
        & = \frac{1}{n^2}\sum_{i, j = 1}^{n} \left(\left(\bar{x}_{i,j} + z_{i,j}\right)\left(\bar{w}_{i,j} - z_{i,j}\right) + \sigma_z^2\right) = \\
        & = \frac{1}{n^2}\sum_{i, j = 1}^{n} \xi_{i,j} = \frac{1}{n}\sum_{i = 1}^{n} \left(\frac{1}{n}\sum_{j = 1}^{n}\xi_{i,j}\right) = \frac{1}{n}\sum_{i = 1}^{n} \zeta_{i}\;,
    \end{split}
\end{equation}
where ${\xi_{i,j}  = \left(\bar{x}_{i,j} + z_{i,j}\right)\left(\bar{w}_{i,j} - z_{i,j}\right) + \sigma_z^2}$ and  ${\zeta_{i} = \frac{1}{n}\sum_{j = 1}^{n} \xi_{i,j}}$. Clearly, sequence $\left(\xi_{1,1}, \dots, \xi_{n,n}\right)$ is $m$-dependent in any dimension. 
Therefore $\left(\zeta_{1}, \dots, \zeta_{n}\right)$ is $m$-dependent. In addition, $\E\left[\zeta_{i}\right] = 0$ as $\bar{x}_{i,j}$, $z_{i,j}$, and $\bar{w}_{i,j}$ are zero-mean and independent. Also, $\E\left[\left|\zeta_{i}\right|^{2+\delta}\right] < \infty$ since $\bar{x}_{i,j}$, $z_{i,j}$, and $\bar{w}_{i,j}$ are bounded. Thus, according to Theorem~\ref{th:clt_str_mix}
\begin{equation*}
    \frac{n}{\sigma}\left(s_{\mathbf{y}, \mathbf{r}} + \sigma_z^2\right) \longrightarrow \mathcal{N}\left(0, 1\right) \;.
\end{equation*}
\end{proof}

\subsection{Dependencies of Type~(I)~and~(II)}
\label{app:mean_drift}
We now turn to we prove our statements regarding the type~(I) and type~(II) dependencies. Through the section, we denote by $\sigma_{\alpha,\beta}$ the scalar covariances computed over pairs of images $\alpha, \beta \in \left\{\mathbf{x}, \mathbf{\tilde{x}}, \mathbf{y}, \mathbf{w}, \mathbf{z}\right\}$, where $\sigma_{\alpha}^2$ stands for the scalar variances ($\sigma_{\alpha} = \sqrt{\sigma_{\alpha,\alpha}}$). We start with the proof that type~(I) dependency implies ${\E\left[s_{\mathbf{y}, \mathbf{r}}\right] > -\sigma_z^2}$. Recall that type~(I) dependency is characterized by a positive correlation between $\mathbf{z}$ and $\mathbf{w}$. In addition, assume that there is no dependency of type~(II), i.e., $\mathbf{x}$ and $\mathbf{w}$ are independent.
\begin{proposition}
Suppose that $\mathbf{w}$ is independent of $\mathbf{x}$, but there is a dependency between $\mathbf{w}$ and $\mathbf{z}$ such that $\sigma_{z,w} > 0$. Then  ${\E\left[s_{\mathbf{y}, \mathbf{r}}\right] > -\sigma_z^2}$.
\end{proposition}
\begin{proof}
Substituting equations \ref{eq:r_w_m_z} into $\E\left[s_{\mathbf{y}, \mathbf{r}}\right]$ we get
\begin{equation}
    \begin{split}
        & \E\left[s_{\mathbf{y}, \mathbf{r}}\right] = \sigma_{y,r} = \sigma_{x + z, w - z} = \sigma_{z,w} - \sigma_z^2 > -\sigma_z^2 \;.
    \end{split}
\end{equation}
\end{proof}

We proceed to a discussion of type~(II) dependency, which occurs when $\mathbf{\tilde{x}}$ and $\mathbf{x}$ tend to be dissimilar. Through it, we assume that there is no dependency of type~(I), i.e., $\mathbf{w}$ and $\mathbf{z}$ are independent. First, we show that this dependency is manifested in a negative correlation between $\mathbf{w}$ and $\mathbf{x}$, ${\sigma_{x,w} < 0}$. 
As mentioned in section~\ref{sec:dep_red},
\begin{equation}
    \label{eq:sig_x_xt_eq}
    \sigma_{x,\tilde{x}} = \sigma_x^2 + \sigma_{x,w} \;.
\end{equation}
Recall that $\mathbf{\tilde{x}}$ is built of patches taken from noisy images. Thus, we can write
\begin{equation}
    \mathbf{\tilde{x}} = \mathbf{x}_2 + \mathbf{z}_2 \;,
\end{equation}
where $\mathbf{x}_2$ is a clean image, which may differ from $\mathbf{x}$, and $\mathbf{z}_2$ is an instantiation of input noise, which is independent of $\mathbf{x}$ and $\mathbf{x}_2$. Thus,
\begin{equation}
    \label{eq:sig_x_xt_ineq}
    \sigma_{x,\tilde{x}} = \sigma_{x, x_2 + z_2} = \sigma_{x, x_2} \le \sigma_x\sigma_{x_2} = \sigma_x^2 \;.
\end{equation}
A conclusion of equations~\ref{eq:sig_x_xt_eq}~and~\ref{eq:sig_x_xt_ineq} is that the dissimilarity between $\mathbf{\tilde{x}}$ and $\mathbf{x}$ reduces the value of $\sigma_{x, \tilde{x}}$, which means that $\sigma_{x,w}$ is necessarily negative.
Finally, we prove that for the dependency of type~(II), we get ${\E\left[s_{\mathbf{y}, \mathbf{r}}\right] < -\sigma_z^2}$.
\begin{proposition}
\label{prop:underfit}
Suppose that $\mathbf{w}$ is independent on $\mathbf{z}$, but there is a dependency between $\mathbf{w}$ and $\mathbf{x}$ such that $\sigma_{x,w} < 0$. Then  ${\E\left[s_{\mathbf{y}, \mathbf{r}}\right] < -\sigma_z^2}$.
\end{proposition}
\begin{proof}
Substituting equation \ref{eq:r_w_m_z} into $\E\left[s_{\mathbf{y}, \mathbf{r}}\right]$ we have
\begin{equation}
    \begin{split}
        & \E\left[s_{\mathbf{y}, \mathbf{r}}\right] = \sigma_{y,r} = \sigma_{x + z, w - z} = \sigma_{x, w} - \sigma_z^2 < -\sigma_z^2 \;.
    \end{split}
\end{equation}
\end{proof}

\section{Robustness of the Patch Matching}
\label{sec:robustness}
In section~\ref{sec:dep_red}, we mention that we use large patches for patch matching. This section shows that patches must be large when dealing with correlated noise. The goal of patch matching is to find similar clean patches by checking the $L_2$ distance between their noisy versions. However, does the similarity between the noisy patches imply the similarity between their clean counterparts? Assuming that noise is independent of the image, the answer is yes, provided the patches are large enough. Intuitively, patch size should grow with the noise power, but do the noise correlations matter? In this section, we show that,  in addition to the power of the noise, the patch size is heavily dependent on the noise correlation range. The distance between the noisy patches can be viewed as an estimator of the distance between their clean counterparts. We show that the variance of this estimator can increase dramatically with the noise correlation range. More specifically, we provide a lower bound for the estimator's variance, which depends on the patch size and the noise autocovariance.

In this section, we use the following notations and assumptions. Let $\mathbf{y}^{(1)}, \mathbf{y}^{(2)} \in \mathbb{R}^{n \times n}$ be two noisy patches, and $\mathbf{x}^{(1)}, \mathbf{x}^{(2)} \in \mathbb{R}^{n \times n}$ be their clean versions, such that
\begin{equation}
    \begin{split}
        & \mathbf{y}^{(1)} = \mathbf{x}^{(1)} +  \mathbf{z}^{(1)} \\
        & \mathbf{y}^{(2)} = \mathbf{x}^{(2)} + \mathbf{z}^{(2)} \;.
    \end{split}
\end{equation}
We assume that $\mathbf{z}^{(1)}$ and $\mathbf{z}^{(2)}$ are independent realizations of 2D zero-mean random process $\left\{Z_{i,j}\right\}$ with autocovariance $R_{ZZ}\left(\tau_1, \tau_2\right)$ and denote $\sigma^2_z = R_{ZZ}\left(0, 0\right)$. For this section only, we assume that $\left\{Z_{i,j}\right\}$ is a Gaussian process. We denote by $\delta_x$ the normalized squared $L_2$ distance between $\mathbf{x}^{(1)}$ and $\mathbf{x}^{(2)}$.
Similarly, $\delta_y$ stands for the normalized squared $L_2$ distance between $\mathbf{y}^{(1)}$ and $\mathbf{y}^{(2)}$,
\begin{equation}
    \begin{split}
        & \delta_x = \frac{1}{n^2}\sum_{i,j = 1}^n \left(x^{(2)}_{i,j} - x^{(1)}_{i,j}\right)^2 \\
        & \delta_y = \frac{1}{n^2}\sum_{i,j = 1}^n \left(y^{(2)}_{i,j} - y^{(1)}_{i,j}\right)^2 \;.
    \end{split}
\end{equation}

\begin{theorem}
\label{th:bias_var}
Let $\delta_y$ be an estimator of $\delta_x$, then $\delta_y$ has a constant bias $2\sigma_z^2$, and its variance is bounded from below by
\begin{equation*}
    \begin{split}
        var\left[\delta_y\right] \ge \frac{8}{n^2}\sigma_z^4 \rho \;,
    \end{split}
\end{equation*}
where
\begin{equation*}
    \rho = \frac{1}{n^2}\sum_{{i_1, j_1, i_2, j_2 = 1}}^n \left(\frac{R_{ZZ}\left(i_1 - j_1, i_2 - j_2\right)}{\sigma_z^2}\right)^2 \;.
\end{equation*}
The bound is sharp since equality holds for $\delta_x = 0$.
\end{theorem}
\noindent The proof of the theorem is given in appendix~\ref{app:proof_bias_var}. The theorem shows that the bound is proportional to $\rho$, where$\rho$ can grow fast with the noise correlation range. We illustrate this by the example of $R_{ZZ}$ with bilinear decay,

\begin{equation}
\label{eq:rzz_example}
    \begin{split}
        & R_{ZZ}\left(\tau_1, \tau_2\right) = g\left(\tau_1\right)g\left(\tau_2\right) \\
        & g\left(\tau\right) = \sigma_z \cdot \max \left\{1 - \frac{\left|\tau\right|}{\theta}, 0\right\} \;,
    \end{split}
\end{equation}
where $\frac{1}{\theta}$ is the decay incline.
\begin{proposition}
\label{prop:rho_ge}
Suppose $R_{ZZ}$ as defined in equation~\ref{eq:rzz_example}. Then
\begin{equation*}
    \rho \ge \frac{1}{4}\left(r + \frac{1}{r}\right)^2 \;, \quad r = \min\left\{n, \lfloor\theta\rfloor\right\} \;.
\end{equation*}
The bound is sharp, the equality holds for $\theta = n$ and $\theta = 1$. 
\end{proposition}
\noindent The proof is given in appendix~\ref{app:rho_ge}. Proposition 10 shows that for $R_{ZZ}$ with bilinear decay, $var\left[\delta_y\right]$ exhibits quadratic growth with respect to the correlation range (for $\theta \le n$).

\subsection{Proof of Theorem~\ref{th:bias_var}}
\label{app:proof_bias_var}
\begin{proof}
We begin with introducing notations. Denote three difference images by $\mathbf{d}^{(x)}$, $\mathbf{d}^{(y)}$, and $\mathbf{d}^{(z)}$, where
\begin{equation*}
        \mathbf{d}^{(x)} = \mathbf{x}^{(2)} - \mathbf{x}^{(1)}, \mathbf{d}^{(y)} = \mathbf{y}^{(2)} - \mathbf{y}^{(1)}, \mathbf{d}^{(z)} = \mathbf{z}^{(2)} - \mathbf{z}^{(1)}.
\end{equation*}
In addition, we define $\delta_x$, $\delta_y$, $\delta_z$, and $\delta_{xz}$ as follows
\begin{equation}
    \begin{split}
            & \delta_x = \frac{1}{n^2}\sum_{i,j = 1}^n \left(d^{(x)}_{i,j}\right)^2, \quad \delta_y = \frac{1}{n^2}\sum_{i,j = 1}^n \left(d^{(y)}_{i,j}\right)^2 \;, \\
            & \delta_z = \frac{1}{n^2}\sum_{i,j = 1}^n \left(d^{(z)}_{i,j}\right)^2, \quad \delta_{x,z} = \frac{1}{n^2}\sum_{i, j = 1}^n d^{(x)}_{i,j}d^{(z)}_{i,j} \;, 
    \end{split}
\end{equation}
where $\delta_x$, $\delta_y$, and $\delta_z$ are the normalized $L_2$ norms of $\mathbf{d}^{(x)}$, $\mathbf{d}^{(y)}$, and $\mathbf{d}^{(z)}$, respectively, and $\delta_{xz}$ stands for a mixed expression. 
Then
\begin{equation*}
    \mathbf{d}^{(y)} = \mathbf{y}^{(2)} - \mathbf{y}^{(1)} = \mathbf{x}^{(2)} - \mathbf{x}^{(1)} + \mathbf{z}^{(2)} - \mathbf{z}^{(1)} = \mathbf{d}^{(x)} + \mathbf{d}^{(z)} \;.
\end{equation*}
Rewriting $\delta_y$, we get
\begin{equation}
\label{eq:delta_y}
    \begin{split}
      \delta_y = & \frac{1}{n^2}\sum_{i,j = 1}^n \left(d^{(y)}_{i,j}\right)^2 = \frac{1}{n^2}\sum_{i,j = 1}^n \left(d^{(x)}_{i,j} + d^{(z)}_{i,j}\right)^2 = \\
      = & \frac{1}{n^2}\sum_{i, j = 1}^n \left(d^{(x)}_{i,j}\right)^2 + \frac{2}{n^2}\sum_{i, j = 1}^n d^{(x)}_{i,j}d^{(z)}_{i,j} + \\
      & \frac{1}{n^2}\sum_{i, j = 1}^n \left(d^{(z)}_{i,j}\right)^2 = \delta_x + 2\delta_{x,z} + \delta_z \;.
    \end{split}
\end{equation}
Therefore,
\begin{equation}
\label{eq:bias_dy}
    \begin{split}
        Bias\left[\delta_y\right] & = \E\left[\delta_y\right] - \delta_x = \\
        & = \E\left[\left(\delta_x + 2\delta_{x,z} + \delta_z\right)\right] - \delta_x = \\
        & = 2\E\left[\delta_{x,z}\right] + \E\left[\delta_z\right] \;,
    \end{split}
\end{equation}
where
\begin{equation}
\label{eq:e_dxz}
    \begin{split}
        2\E\left[\delta_{x,z}\right] & = \frac{2}{n^2}\sum_{i, j = 1}^n \E\left[d^{(x)}_{i,j}d^{(z)}_{i,j}\right] = \\
        & = \frac{2}{n^2}\sum_{i, j = 1}^n d^{(x)}_{i,j}\E\left[d^{(z)}_{i,j}\right] = \\
        & = \frac{2}{n^2}\sum_{i, j = 1}^n d^{(x)}_{i,j}\E\left[z^{(2)}_{i,j} - z^{(1)}_{i,j}\right] = \\
        & = \frac{2}{n^2}\sum_{i, j = 1}^n d^{(x)}_{i,j}\left(0 - 0\right) = 0 \;.
    \end{split}
\end{equation}
Since $\mathbf{z}^{(1)}$ and $\mathbf{z}^{(2)}$ are independent,
\begin{equation}
\label{eq:e_dz2}
    \E\left[\left(d^{(z)}_{i,j}\right)^2\right] = \E\left[\left(z^{(2)}_{i,j} - z^{(1)}_{i,j}\right)^2\right] = 2\sigma^2_z \;.
\end{equation}
Then, 
\begin{equation}
    \begin{split}
        \E\left[\delta_z\right] & = \frac{1}{n^2}\sum_{i,j = 1}^n \E\left[\left(d^{(z)}_{i,j}\right)^2\right] = \frac{1}{n^2}\sum_{i,j = 1}^n 2\sigma_z^2 = 2\sigma_z^2 \;.
    \end{split}
\end{equation}
Substituting equations~\ref{eq:e_dxz}~and~\ref{eq:e_dz2} into \ref{eq:bias_dy}, we have
\begin{equation}
     Bias\left[\delta_y\right] = 2\E\left[\delta_{x,z}\right] + \E\left[\delta_z\right] = 0 + 2\sigma^2_z = 2\sigma^2_z \;.
\end{equation}
We proceed to calculate the variance.
\begin{equation}
\label{eq:var_dy}
    \begin{split}
        var\left[\delta_y\right] & = var\left[\left(\delta_y - \delta_x\right)\right] = \\
        & = \E\left[\left(\delta_y - \delta_x\right)^2\right] - \left(\E\left[\left(\delta_y - \delta_x\right)\right]\right)^2 = \\
        & = \E\left[\left(\delta_y - \delta_x\right)^2\right] - \left(Bias\left[\delta_y\right]\right)^2 = \\
        & = \E\left[\left(\delta_y - \delta_x\right)^2\right] - 4\sigma_z^4 \;.
    \end{split}
\end{equation}
Substituting equation~\ref{eq:delta_y} into~\ref{eq:var_dy}, we get
\begin{equation}
    \begin{split}
        var\left[\delta_y\right] & = \E\left[\left(\delta_y - \delta_x\right)^2\right] - 4\sigma_z^4 = \\
        & = \E\left[\left(2\delta_{x,z} + \delta_z\right)^2\right] - 4\sigma_z^4 = \\
        & = 4\E\left[\delta_{x,z}^2\right] + 4\E\left[\delta_{x,z}\delta_z\right] + \E\left[\delta_z^2\right] - 4\sigma_z^4 \;.
    \end{split}
\end{equation}
Since $\E\left[\delta_{x,z}^2\right] \ge 0$, the variance can be bounded from below using the following inequality,
\begin{equation}
\label{eq:var_dy_ineq}
    var\left[\delta_y\right] \ge 4\E\left[\delta_{x,z}\right] + \E\left[\delta_z^2\right] - 4\sigma_z^4 \;,
\end{equation}
where the bound is strict since provided ${\mathbf{d}^{(x)} = \mathbf{0}}$, we have
\begin{equation}
    \begin{split}
        \E\left[\delta_{x,z}^2\right] & = \E\left[\left(\frac{1}{n^2}\sum_{i, j = 1}^n d^{(x)}_{i,j}d^{(z)}_{i,j}\right)^2\right] = \\
        & = \E\left[\left(\frac{1}{n^2}\sum_{i, j = 1}^n 0\cdot d^{(z)}_{i,j}\right)^2\right] = \E\left[0\right] = 0 \;.
    \end{split}
\end{equation}
Next, we show that $\E\left[\delta_{x,z}\delta_z\right] = 0$.
\begin{equation}
\label{eq:e_dxz_dz}
    \begin{split}
        \E\left[\delta_{x,z}\delta_z\right] & = \frac{1}{n^4}\E\left[\left(\sum_{i,j}^n d^{(x)}_{i, j}d^{(z)}_{i,j}\right)\left(\sum_{i,j}^n \left(d^{(z)}_{i,j}\right)^2\right)\right] = \\
        & = \frac{1}{n^4}\E\left[\sum_{i_1,j_1,i_2,j_2}^n d^{(x)}_{i_1,j_1}d^{(z)}_{i_1,j_1}\left(d^{(z)}_{i_2,j_2}\right)^2\right] = \\
        & = \frac{1}{n^4}\sum_{i_1,j_1,i_2,j_2}^n \E\left[d^{(x)}_{i_1,j_1}d^{(z)}_{i_1,j_1}\left(d^{(z)}_{i_2,j_2}\right)^2\right] = \\
        & = \frac{1}{n^4}\sum_{i_1,j_1,i_2,j_2}^n d^{(x)}_{i_1,j_1}\E\left[d^{(z)}_{i_1,j_1}\left(d^{(z)}_{i_2,j_2}\right)^2\right] = \\
        & = \frac{1}{n^4}\sum_{i_1,j_1,i_2,j_2}^n d^{(x)}_{i_1,j_1}\cdot 0 = 0 \;,
    \end{split}
\end{equation}
where $\E\left[d^{(z)}_{i_1,j_1}\left(d^{(z)}_{i_2,j_2}\right)^2\right] = 0$ as it is the third moment of multivariate Gaussian distribution. Substituting equation~\ref{eq:e_dxz_dz} into~\ref{eq:var_dy_ineq}, we get
\begin{equation}
\label{eq:var_inequality}
    var\left[\hat{\delta}_x\right] \ge \E\left[\delta_z^2\right] - 4\sigma_z^4 \;.
\end{equation}
It remains to calculate the value of $\E\left[\delta_z^2\right]$.  
\begin{equation}
\label{eq:exp_delta_2}
    \begin{split}
        \E\left[\delta_z^2\right] & = \E\left[\delta_z\delta_z\right] = \\
        & = \frac{1}{n^4}\E\left[\left(\sum_{i, j = 1}^n \left(d^{(z)}_{i,j}\right)^2\right)\left(\sum_{i, j = 1}^n \left(d^{(z)}_{i,j}\right)^2\right)\right] = \\
        & = \frac{1}{n^4}\E\left[\sum_{i_1, j_1, i_2, j_2 = 1}^n \left(d^{(z)}_{i_1, j_1}\right)^2 \left(d^{(z)}_{i_2,j_2}\right)^2\right] = \\
        & = \frac{1}{n^4}\sum_{i_1, j_1, i_2, j_2 = 1}^n E\left[\left(d^{(z)}_{i_1, j_1}\right)^2 \left(d^{(z)}_{i_2,j_2}\right)^2\right] \;.
    \end{split}
\end{equation}
\begin{equation}
\label{eq:di2_dj2}
    \begin{split}
        & \E\left[\left(d^{(z)}_{i_1, j_1}\right)^2 \left(d^{(z)}_{i_2,j_2}\right)^2\right] = \\
        & \;\; = \E\left[\left(z^{(1)}_{i_1,j_1} - z^{(2)}_{i_i,j_i}\right)^2 \left(z^{(1)}_{i_2,j_2} - z^{(2)}_{i_2,j_2}\right)^2\right] = \\
        & \;\; = \E\left[\left(\left(z^{(1)}_{i_1,j_1}\right)^2 - 2z^{(1)}_{i_1,j_1} z^{(2)}_{i_1,j_1} + \left(z^{(2)}_{i_1,j_1}\right)^2\right) \times \right. \\
        & \qquad\quad\left.\left(\left(z^{(1)}_{i_2,j_2}\right)^2 - 2z^{(1)}_{i_2,j_2} z^{(2)}_{i_2,j_2} + \left(z^{(2)}_{i_2,j_2}\right)^2\right)\right] = \\
        & \;\; = \E\left[\left(z^{(1)}_{i_1,j_1}\right)^2 \left(z^{(1)}_{i_2,j_2}\right)^2\right] - \\
        & \qquad\qquad\qquad 2\E\left[\left(z^{(1)}_{i_1,j_1}\right)^2 z^{(1)}_{i_2,j_2} z^{(2)}_{i_2,j_2}\right] + \\
        & \qquad\qquad\qquad\;\; \E\left[\left(z^{(1)}_{i_1,j_1}\right)^2 \left(z^{(2)}_{i_2,j_2}\right)^2\right] - \\
        & \qquad\qquad\qquad 2\E\left[z^{(1)}_{i_1,j_1} z^{(2)}_{i_1,j_1} \left(z^{(1)}_{i_2,j_2}\right)^2\right] + \\
        & \qquad\qquad\qquad 4\E\left[z^{(1)}_{i_1,j_1} z^{(2)}_{i_1,j_1} z^{(1)}_{i_2,j_2} z^{(2)}_{i_2,j_2}\right] - \\
        & \qquad\qquad\qquad 2\E\left[z^{(1)}_{i_1,j_1} z^{(2)}_{i_1,j_1} \left(z^{(2)}_{i_2,j_2}\right)^2\right] + \\
        & \qquad\qquad\qquad\;\; \E\left[\left(z^{(2)}_{i_1,j_1}\right)^2 \left(z^{(1)}_{i_2,j_2}\right)^2\right] - \\
        & \qquad\qquad\qquad 2\E\left[\left(z^{(2)}_{i_1,j_1}\right)^2 z^{(1)}_{i_2,j_2} z^{(2)}_{i_2,j_2}\right] + \\
        & \qquad\qquad\qquad\;\; \E\left[\left(z^{(2)}_{i_1,j_1}\right)^2 \left(z^{(2)}_{i_2,j_2}\right)^2\right] \;.
    \end{split}
\end{equation}
All summands in equation~\ref{eq:di2_dj2} are the fourth moments of multivariate Gaussian distribution. Using formula for the Gaussian moment, we get
\begin{equation}
\label{eq:4_moment1}
    \begin{split}
        & \E\left[\left(z^{(1)}_{i_1,j_1}\right)^2 \left(z^{(1)}_{i_2,j_2}\right)^2\right] = \\
        & \qquad\qquad = \sigma_z^4 + 2\left(R_{ZZ}(i_1 - j_1, i_2 - j_2)\right)^2
    \end{split}
\end{equation}
\begin{equation}
\label{eq:4_moment2}
    \begin{split}
        & 2\E\left[\left(z^{(1)}_{i_1,j_1}\right)^2 z^{(1)}_{i_2,j_2} z^{(2)}_{i_2,j_2}\right] = \\
        & \qquad\qquad = 2\E\left[\left(z^{(1)}_{i_1,j_1}\right)^2 z^{(1)}_{i_2,j_2}\right] \E\left[z^{(2)}_{i_2,j_2}\right] = 0
    \end{split}
\end{equation}
\begin{equation}
\label{eq:4_moment3}
    \begin{split}
        & \E\left[\left(z^{(1)}_{i_1,j_1}\right)^2 \left(z^{(2)}_{i_2,j_2}\right)^2\right] = \\
        & \qquad\qquad = \E\left[\left(z^{(1)}_{i_1,j_1}\right)^2\right] \E\left[\left(z^{(2)}_{i_2,j_2}\right)^2\right] = \sigma_z^4
    \end{split}
\end{equation}
\begin{equation}
\label{eq:4_moment4}
    \begin{split}
        & 2\E\left[z^{(1)}_{i_1,j_1} z^{(2)}_{i_1,j_1} \left(z^{(1)}_{i_2,j_2}\right)^2\right] = \\
        & \qquad\qquad = 2\E\left[z^{(1)}_{i_1,j_1} \left(z^{(1)}_{i_2,j_2}\right)^2\right] \E\left[z^{(2)}_{i_1,j_1}\right] = 0
    \end{split}
\end{equation}
\begin{equation}
\label{eq:4_moment5}
    \begin{split}
        & 4\E\left[z^{(1)}_{i_1,j_1} z^{(2)}_{i_1,j_1} z^{(1)}_{i_2,j_2} z^{(2)}_{i_2,j_2}\right] = \\
        & \qquad\qquad = 4\E\left[z^{(1)}_{i_1,j_1} z^{(1)}_{i_2,j_2}\right] \E\left[z^{(2)}_{i_1,j_1} z^{(2)}_{i_2,j_2}\right] = \\
        & \qquad\qquad = 4\left(R_{ZZ}(i_1 - j_1, i_2 - j_2)\right)^2
    \end{split}
\end{equation}
\begin{equation}
\label{eq:4_moment6}
    \begin{split}
        & 2\E\left[z^{(1)}_{i_1,j_1} z^{(2)}_{i_1,j_1} \left(z^{(2)}_{i_2,j_2}\right)^2\right] = \\
        & \qquad\qquad = 2\E\left[z^{(1)}_{i_1,j_1}\right] \E\left[z^{(2)}_{i_1,j_1} \left(z^{(2)}_{i_2,j_2}\right)^2\right] = 0
    \end{split}
\end{equation}
\begin{equation}
\label{eq:4_moment7}
    \begin{split}
        & \E\left[\left(z^{(2)}_{i_1,j_1}\right)^2 \left(z^{(1)}_{i_2,j_2}\right)^2\right] = \\
        & \qquad\qquad = \E\left[\left(z^{(2)}_{i_1,j_1}\right)^2\right] \E\left[\left(z^{(1)}_{i_2,j_2}\right)^2\right] = \sigma_z^4
    \end{split}
\end{equation}
\begin{equation}
\label{eq:4_moment8}
    \begin{split}
        & 2\E\left[\left(z^{(2)}_{i_1,j_1}\right)^2 z^{(1)}_{i_2,j_2} z^{(2)}_{i_2,j_2}\right] = \\
        & \qquad\qquad = 2\E\left[\left(z^{(2)}_{i_1,j_1}\right)^2 z^{(2)}_{i_2,j_2}\right] \E\left[z^{(1)}_{i_2,j_2}\right] = 0
    \end{split}
\end{equation}
\begin{equation}
\label{eq:4_moment9}
    \begin{split}
        & \E\left[\left(z^{(2)}_{i_1,j_1}\right)^2 \left(z^{(2)}_{i_2,j_2}\right)^2\right] = \\
        & \qquad\qquad = \sigma_z^4 + 2\left(R_{ZZ}(i_1 - j_1, i_2 - j_2)\right)^2.
    \end{split}
\end{equation}
Summarizing the expressions in equations~\ref{eq:4_moment1}-\ref{eq:4_moment9}, we get
\begin{equation}
\label{eq:e_dz2_dz2}
    \begin{split}
        & \E\left[\left(d^{(z)}_{i_1, j_1}\right)^2 \left(d^{(z)}_{i_2,j_2}\right)^2\right] = \\
        & \qquad\qquad \sigma_z^4 + 2\left(R_{ZZ}(i_1 - j_1, i_2 - j_2)\right)^2 + \\
        & \qquad\qquad \sigma_z^4 + 4\left(R_{ZZ}(i_1 - j_1, i_2 - j_2)\right)^2 + \sigma_z^4 + \\
        & \qquad\qquad \sigma_z^4 + 2\left(R_{ZZ}(i_1 - j_1, i_2 - j_2)\right)^2 = \\
        & \quad = 4\sigma_z^4 + 8\left(R_{ZZ}(i_1 - j_1, i_2 - j_2)\right)^2 \;.
    \end{split}
\end{equation}
Substituting equation~\ref{eq:e_dz2_dz2} into~\ref{eq:exp_delta_2}, we get
\begin{equation}
\label{eq:e_dz2_2}
    \begin{split}
        & \E\left[\delta_z^2\right] = \frac{1}{n^4}\sum_{i_1, j_1, i_2, j_2 = 1}^n E\left[\left(d^{(z)}_{i_1, j_1}\right)^2 \left(d^{(z)}_{i_2,j_2}\right)^2\right] = \\
        & = \frac{1}{n^4}\sum_{i_1, j_1, i_2, j_2 = 1}^n \left(4\sigma_z^4 + 8\left(R_{ZZ}(i_1 - j_1, i_2 - j_2)\right)^2\right) = \\ 
        & = 4\sigma_z^4 + \frac{8}{n^4}\sum_{i_1, j_1, i_2, j_2 = 1}^n \left(R_{ZZ}(i_1 - j_1, i_2 - j_2)\right)^2
    \end{split}
\end{equation}
Substituting equation~\ref{eq:e_dz2_2} into~\ref{eq:var_inequality}, we have
\begin{equation}
    \begin{split}
        & var\left[\delta_y\right] \ge \E\left[\delta_z^2\right] - 4\sigma_z^4 = \\
        & = \frac{8}{n^4}\sum_{i_1, j_1, i_2, j_2 = 1}^n \left(R_{ZZ}\left(i_1 - j_1, i_2 - j_2\right)\right)^2 = \\
        & = \frac{8}{n^2}\sigma_z^4\left(\frac{1}{n^2}\sum_{i_1, j_1, i_2, j_2 = 1}^n \left(\frac{R_{ZZ}\left(i_1 - j_1, i_2 - j_2\right)}{\sigma_z^2}\right)^2\right) \;.
    \end{split}
\end{equation}
Finally, we get
\begin{equation}
    var\left[\delta_y\right] \ge \frac{8}{n^2}\sigma_z^4\rho \;,
\end{equation}
where 
\begin{equation*}
    \rho = \frac{1}{n^2}\sum_{i_1, j_1, i_2, j_2 = 1}^n \left(\frac{R_{ZZ}\left(i_1 - j_1, i_2 - j_2\right)}{\sigma_z^2}\right)^2 \;. 
\end{equation*}
\end{proof}

\subsection{Proof of Proposition~\ref{prop:rho_ge}}
\label{app:rho_ge}
\begin{proof}
Substituting equation~\ref{eq:rzz_example} into $\rho$ in Theorem~\ref{th:bias_var}, we get
\begin{equation}
\label{eq:rzz_rho}
    \begin{split}
        \rho & = \frac{1}{n^2}\sum_{i_1, j_1, i_2, j_2 = 1}^n \left(\frac{R_{ZZ}\left(i_1 - j_1, i_2 - j_2\right)}{\sigma_z^2}\right)^2 = \\
        & = \frac{1}{n^2\sigma_z^4}\sum_{i_1, j_1, i_2, j_2 = 1}^n g^2\left(i_1 - j_1\right)g^2\left(i_2 - j_2\right) = \\
        & = \frac{1}{n^2\sigma_z^4}\sum_{i_1, j_1 = 1}^n g^2\left(i_1 - j_1\right)\sum_{i_2, j_2 = 1}^n g^2\left(i_2 - j_2\right) = \\
        & = \frac{1}{n^2\sigma_z^4}\left(\sum_{i, j = 1}^n g^2\left(i - j\right)\right)^2 \;.
    \end{split}
\end{equation}
It is easy to see that if $a \ge b > 0$, then for any $\tau$,
\begin{equation}
\label{eq:frac_inequality}
    1 - \frac{\left|\tau\right|}{a} \ge 1 - \frac{\left|\tau\right|}{b} \;
\end{equation}
and 
\begin{equation}
\label{eq:max_inequality}
    \max\left\{1 - \frac{\left|\tau\right|}{a}, 0\right\} \ge \max\left\{1 - \frac{\left|\tau\right|}{b}, 0\right\} \;.
\end{equation}

We denote $r = \min\left\{\lfloor\theta\rfloor, n\right\}$, where $\lfloor\cdot\rfloor$ stands for the floor function. Applying Lemma~\ref{lem:scalar_func} on equation~\ref{eq:rzz_rho}, substituting the expression for $g\left(\tau\right)$ in equation~\ref{eq:rzz_example}, and using inequalities~\ref{eq:frac_inequality}~and~\ref{eq:max_inequality}, we get
\begin{equation}
    \begin{split}
        \rho & = \frac{1}{n^2\sigma_z^4}\left(\sum_{i, j = 1}^n g^2\left(i - j\right)\right)^2 = \\
        & = \frac{1}{n^2\sigma_z^4}\left(n\sum_{\tau = -(n - 1)}^{n - 1} \left(1 - \frac{\left|\tau\right|}{n}\right)g^2\left(\tau\right)\right)^2 = \\
        & = \frac{1}{\sigma_z^4}\left(\sum_{\tau = -(n - 1)}^{n - 1} \left(1 - \frac{\left|\tau\right|}{n}\right) \times\right. \\
        & \qquad\qquad\qquad\left.\left(\sigma_z\cdot\max\left\{1 - \frac{\left|\tau\right|}{\theta}, 0\right\}\right)^2\right)^2 \ge \\
        & \ge \left(\sum_{\tau = -(n - 1)}^{n - 1} \left(1 - \frac{\left|\tau\right|}{n}\right) \times\right. \\
        & \qquad\qquad\qquad\left.\left(\max\left\{1 - \frac{\left|\tau\right|}{\lfloor\theta\rfloor}, 0\right\}\right)^2\right)^2 = \\
        & = \left(\sum_{\tau = -(r - 1)}^{r - 1} \left(1 - \frac{\left|\tau\right|}{n}\right)\left(1 - \frac{\left|\tau\right|}{\lfloor\theta\rfloor}\right)^2\right)^2 \ge \\
        & \ge \left(\sum_{\tau = -(r - 1)}^{r - 1} \left(1 - \frac{\left|\tau\right|}{r}\right)\left(1 - \frac{\left|\tau\right|}{r}\right)^2\right)^2 = \\
        & = \left(\sum_{\tau = -(r - 1)}^{r - 1} \left(1 - \frac{\left|\tau\right|}{r}\right)^3\right)^2 = \\
        & = \left(1 + 2\sum_{\tau = 1}^{r - 1} \left(1 - \frac{\tau}{r}\right)^3\right)^2 = \\
        & = \left(1 + \frac{2}{r^3}\sum_{\tau = 1}^{r - 1} \left(r - \tau\right)^3\right)^2 \;.
    \end{split}
\end{equation}
Applying variable change $\tau = r - \tau$ and using the formula for the sum of cubes, we get
\begin{equation}
    \begin{split}
        \rho & \ge \left(1 + \frac{2}{r^3}\sum_{\tau = 1}^{r - 1} \left(r - \tau\right)^3\right)^2 = \\
        & = \left(1 + \frac{2}{r^3}\sum_{\tau = 1}^{r - 1} \tau^3\right)^2 = \\ 
        & = \left(1 + \frac{2}{r^3}\frac{\left(r - 1\right)^2r^2}{4}\right)^2 = \\
        & = \left(1 + \frac{\left(r - 1\right)^2}{2r}\right)^2 = \\
        & = \left(\frac{2r + r^2 - 2r + 1}{2r}\right)^2 = \\
        & = \frac{1}{4}\left(r + \frac{1}{r}\right)^2\;.
    \end{split}
\end{equation}
\end{proof}

\begin{figure}[b]
    \centering
	\begin{subfigure}{0.13\textwidth}
	    \captionsetup{justification=centering}
		\includegraphics[width=\textwidth]{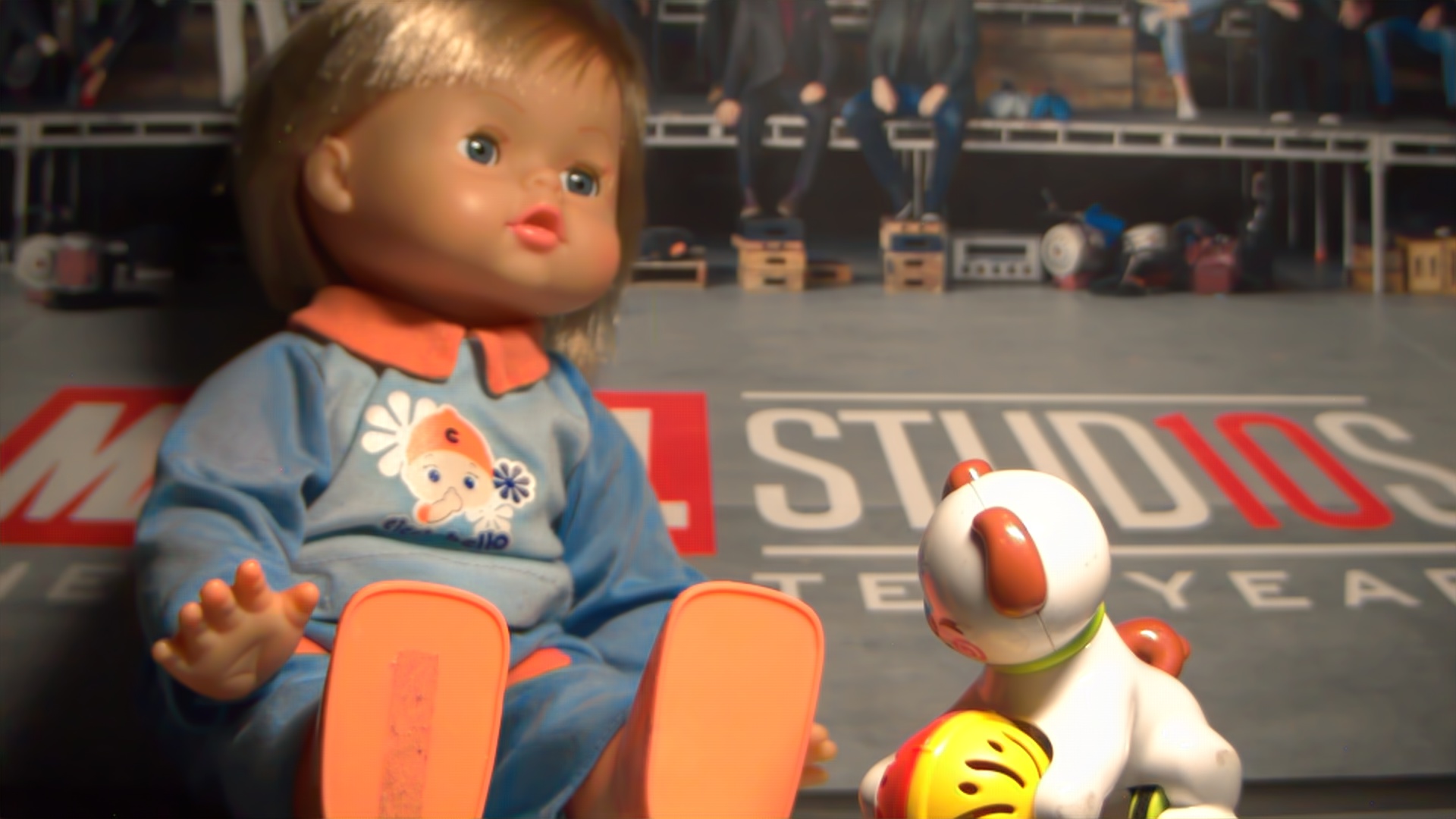}
		\caption{Frame 1}
		\label{fig:crvd_example_s5:f0}
	\end{subfigure}
	\begin{subfigure}{0.13\textwidth}
	    \captionsetup{justification=centering}
		\includegraphics[width=\textwidth]{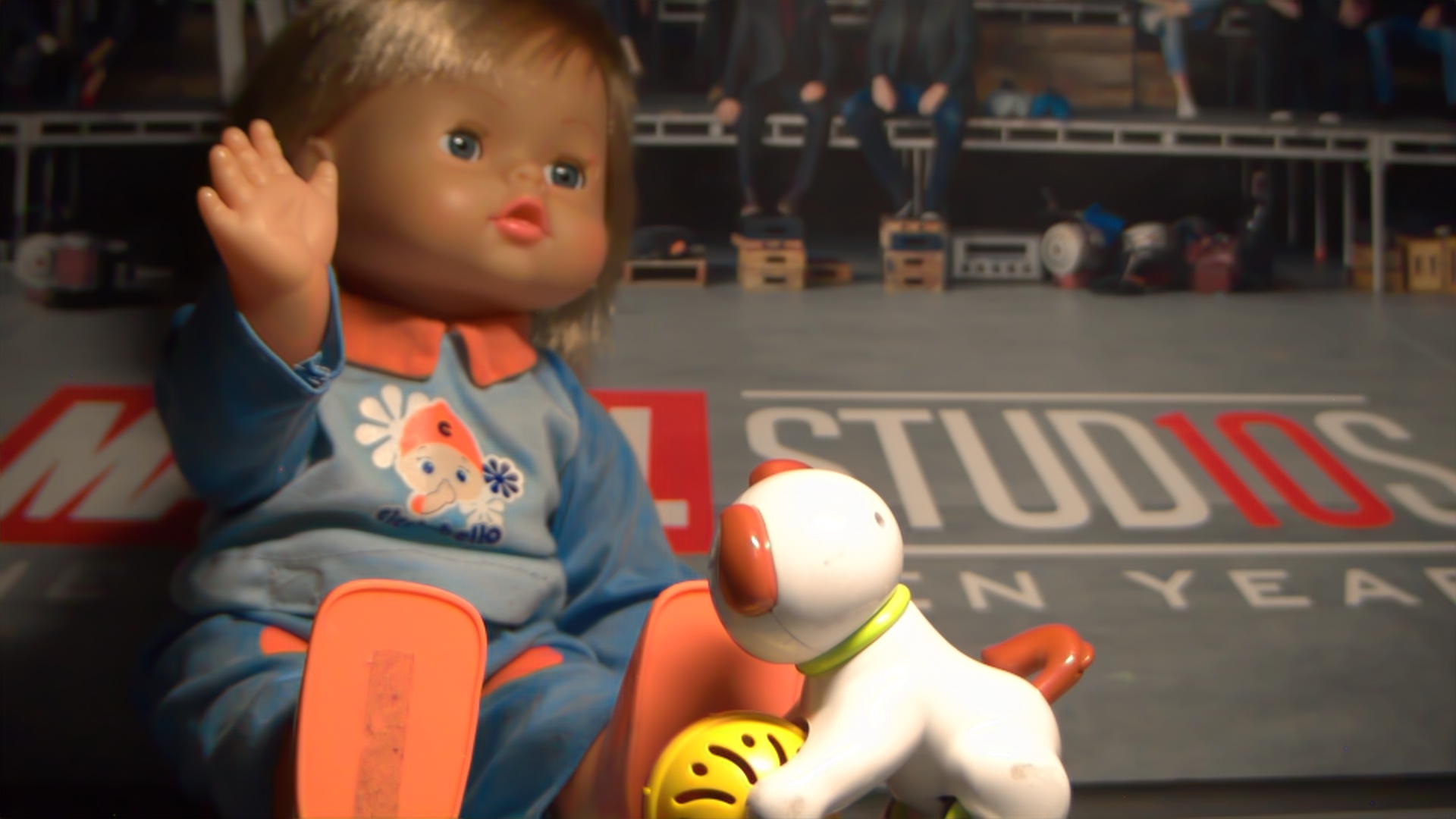}
		\caption{Frame 2}
		\label{fig:crvd_example_s5:f1}
	\end{subfigure}
	\begin{subfigure}{0.13\textwidth}
	    \captionsetup{justification=centering}
		\includegraphics[width=\textwidth]{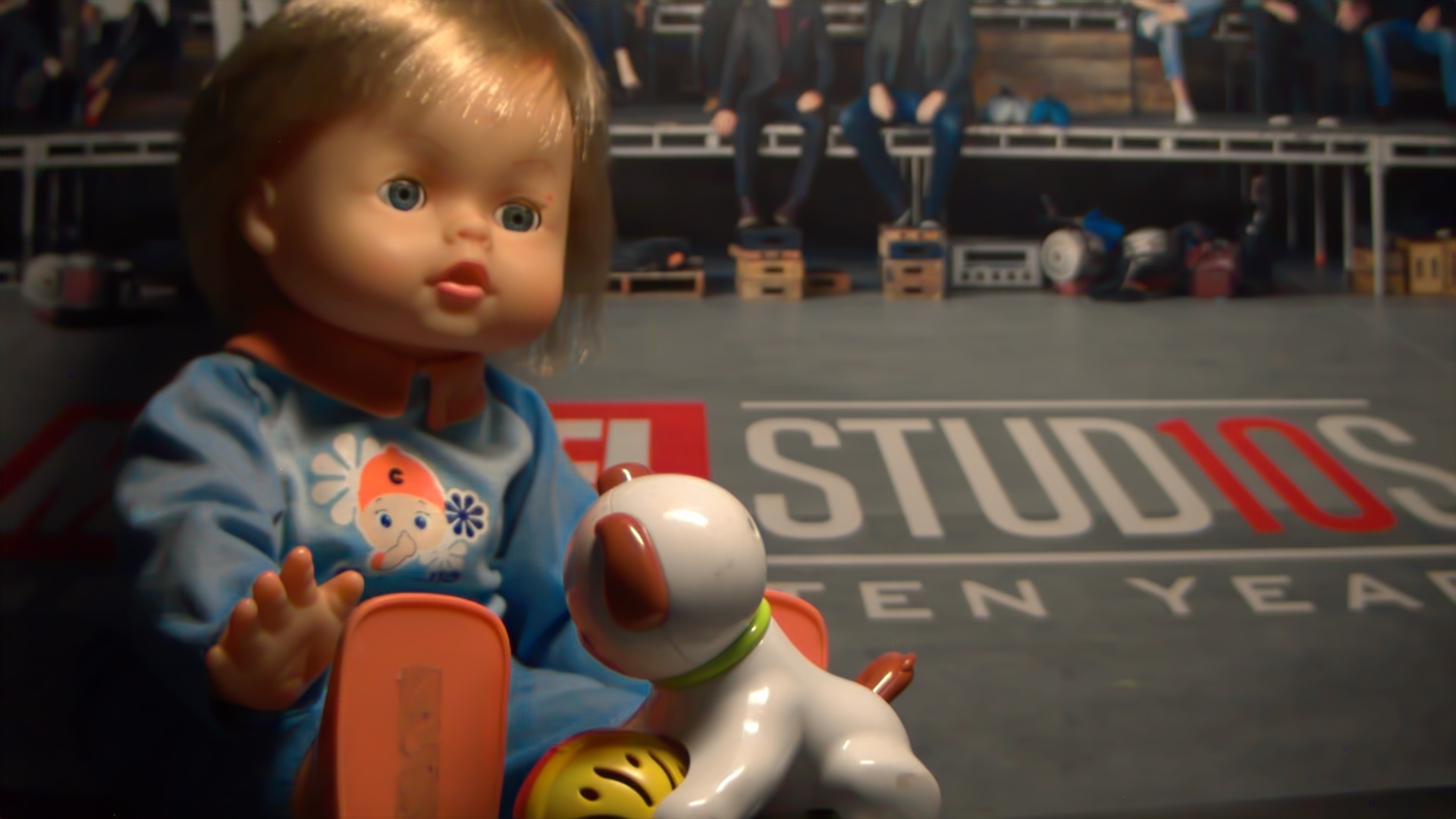}
		\caption{Frame 3}
		\label{fig:crvd_example_s5:f2}
	\end{subfigure}
	\begin{subfigure}{0.13\textwidth}
	    \captionsetup{justification=centering}
		\includegraphics[width=\textwidth]{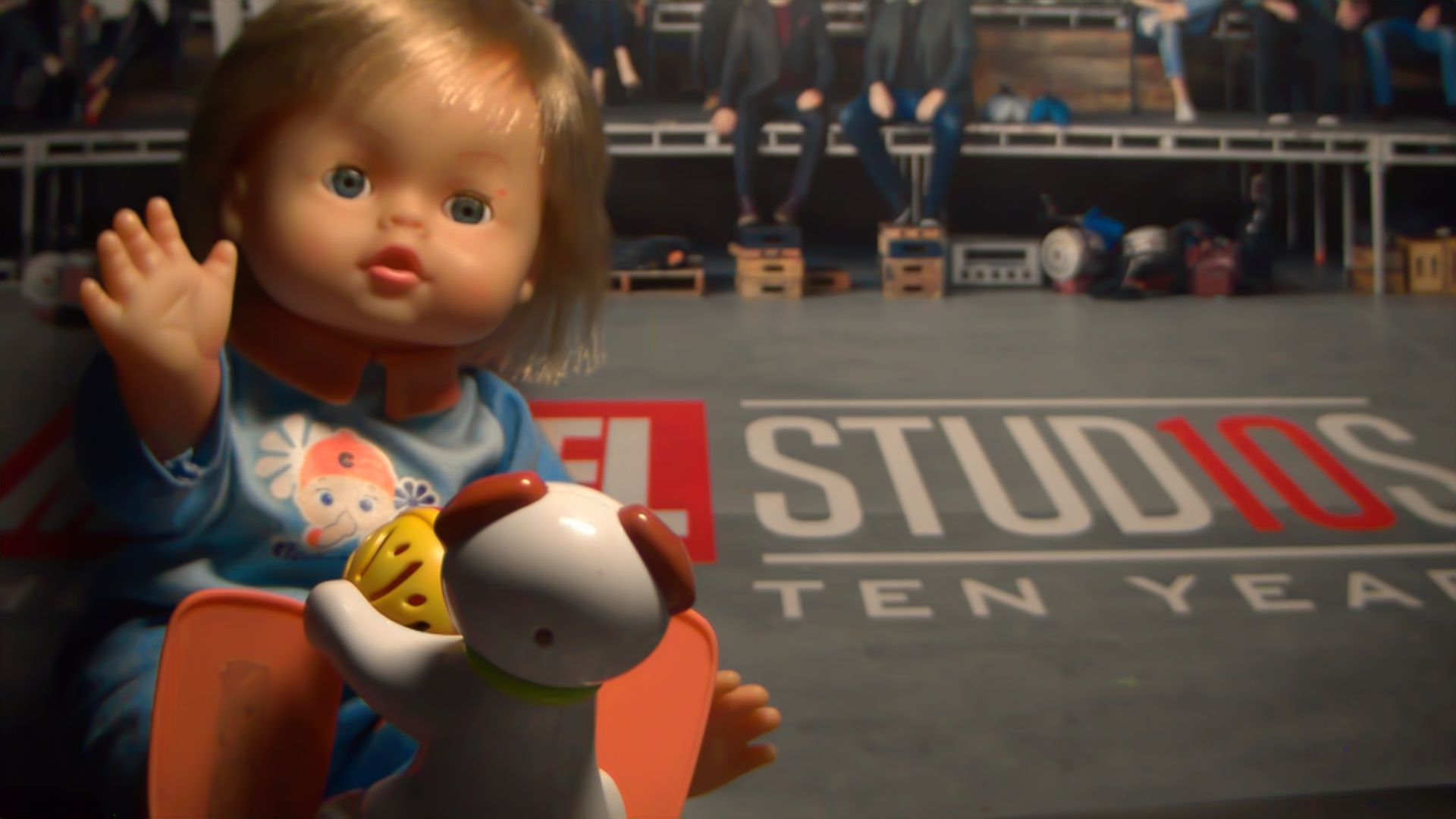}
		\caption{Frame 4}
		\label{fig:crvd_example_s5:f3}
	\end{subfigure}
	\begin{subfigure}{0.13\textwidth}
	    \captionsetup{justification=centering}
		\includegraphics[width=\textwidth]{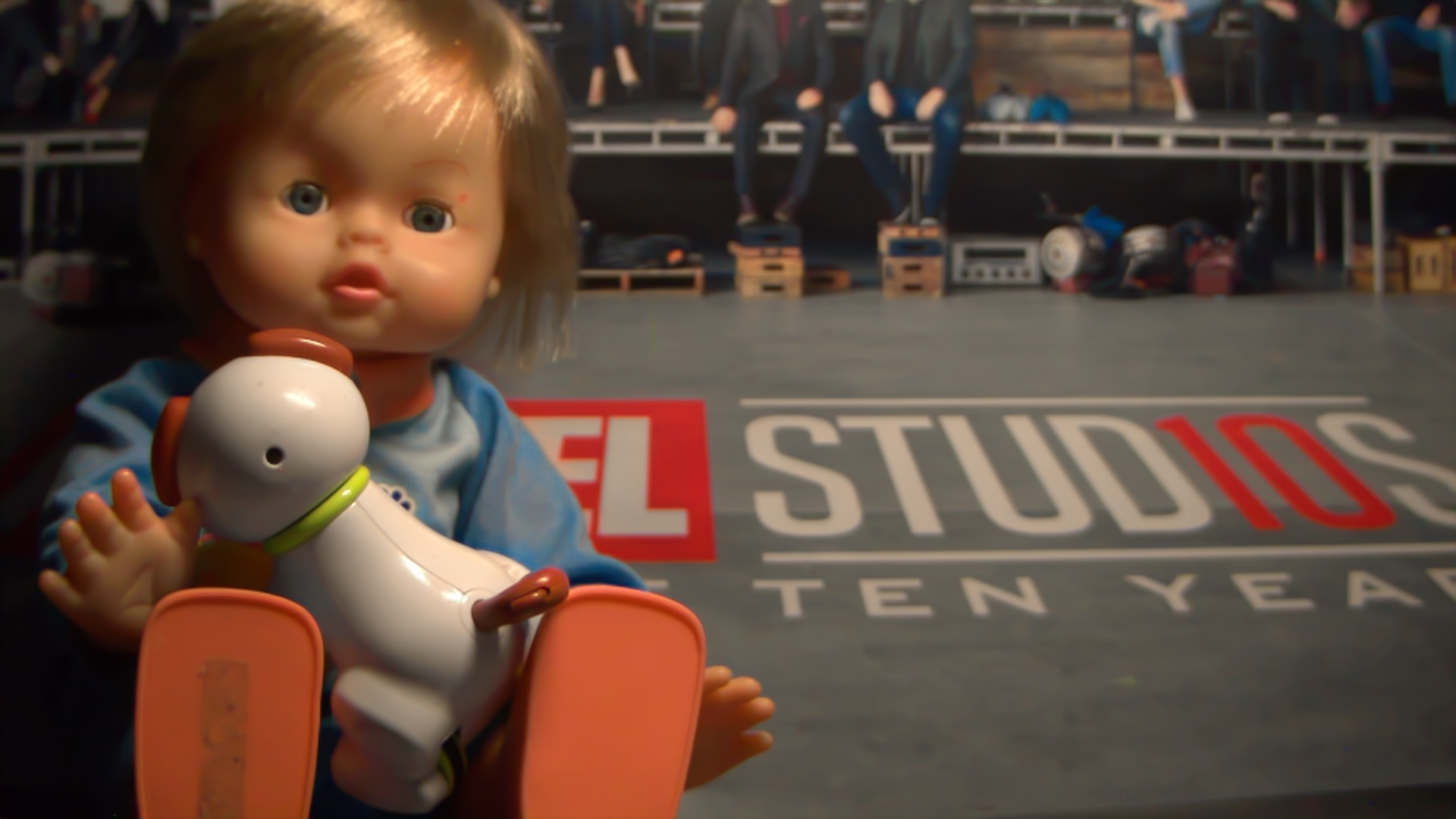}
		\caption{Frame 5}
		\label{fig:crvd_example_s5:f4}
	\end{subfigure}
	\begin{subfigure}{0.13\textwidth}
	    \captionsetup{justification=centering}
		\includegraphics[width=\textwidth]{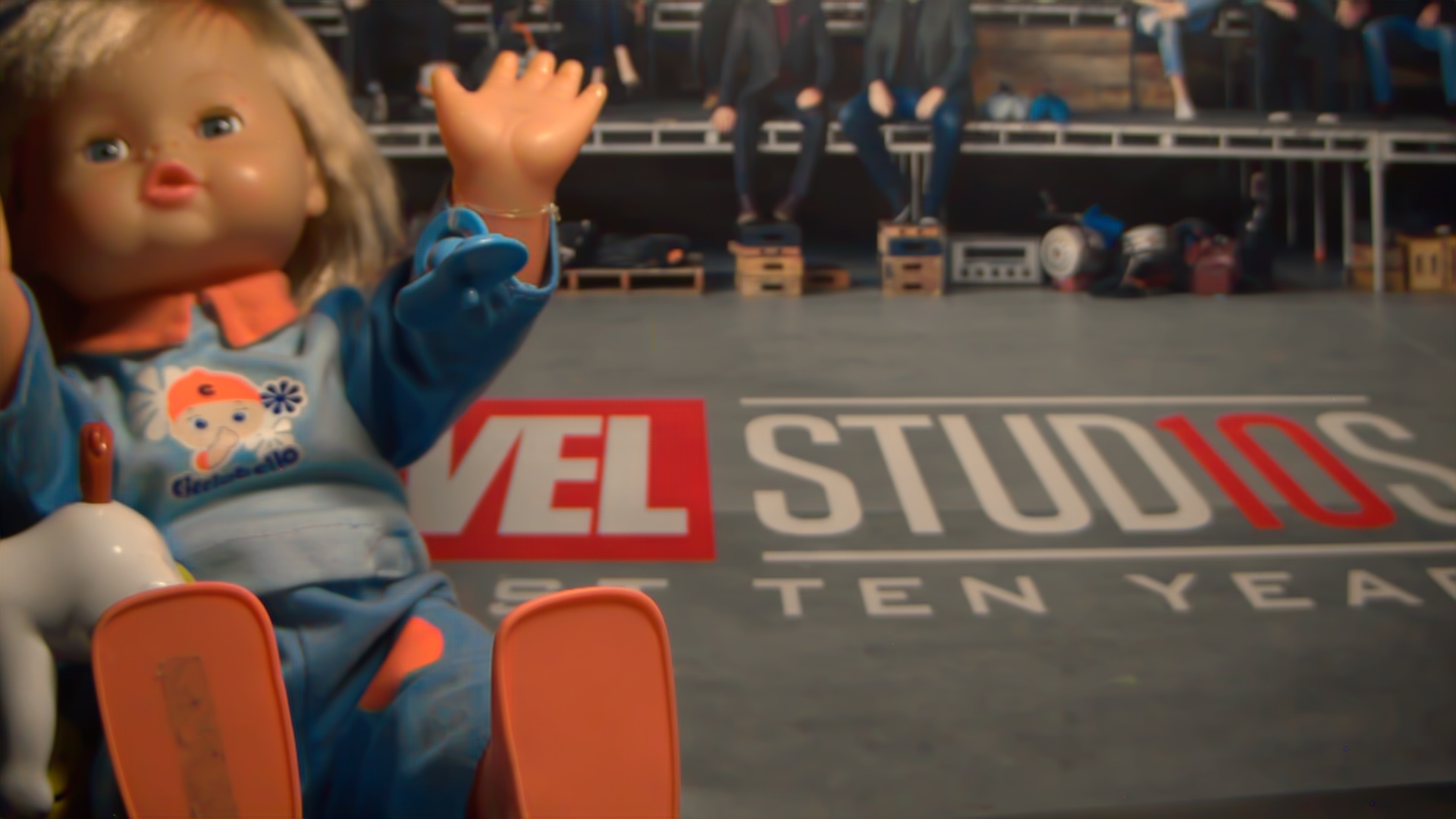}
		\caption{Frame 6}
		\label{fig:crvd_example_s5:f5}
	\end{subfigure}
	\begin{subfigure}{0.13\textwidth}
	    \captionsetup{justification=centering}
		\includegraphics[width=\textwidth]{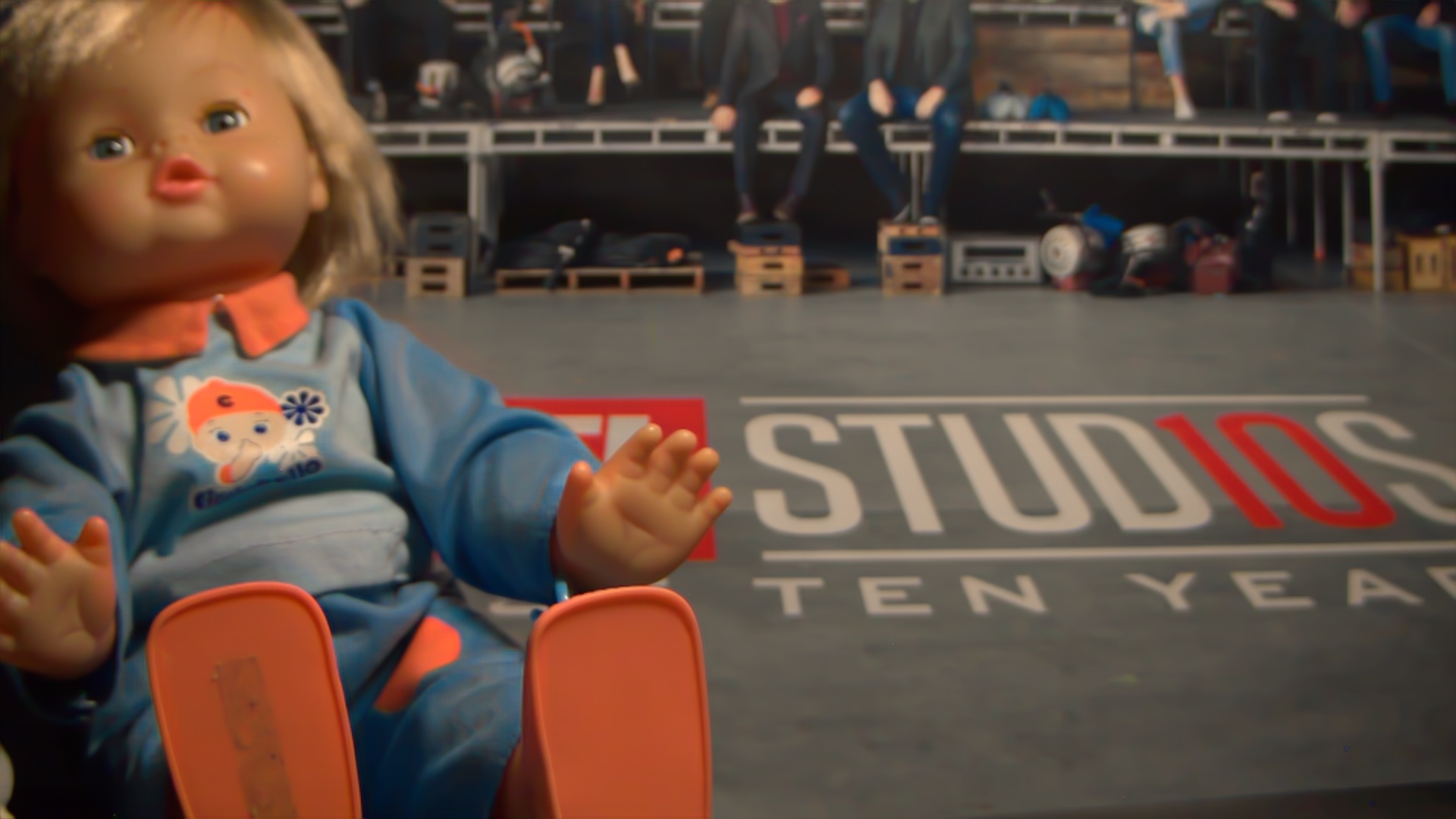}
		\caption{Frame 7}
		\label{fig:crvd_example_s5:f6}
	\end{subfigure}
	\begin{subfigure}{0.13\textwidth}
	    \captionsetup{justification=centering}
		\includegraphics[width=\textwidth]{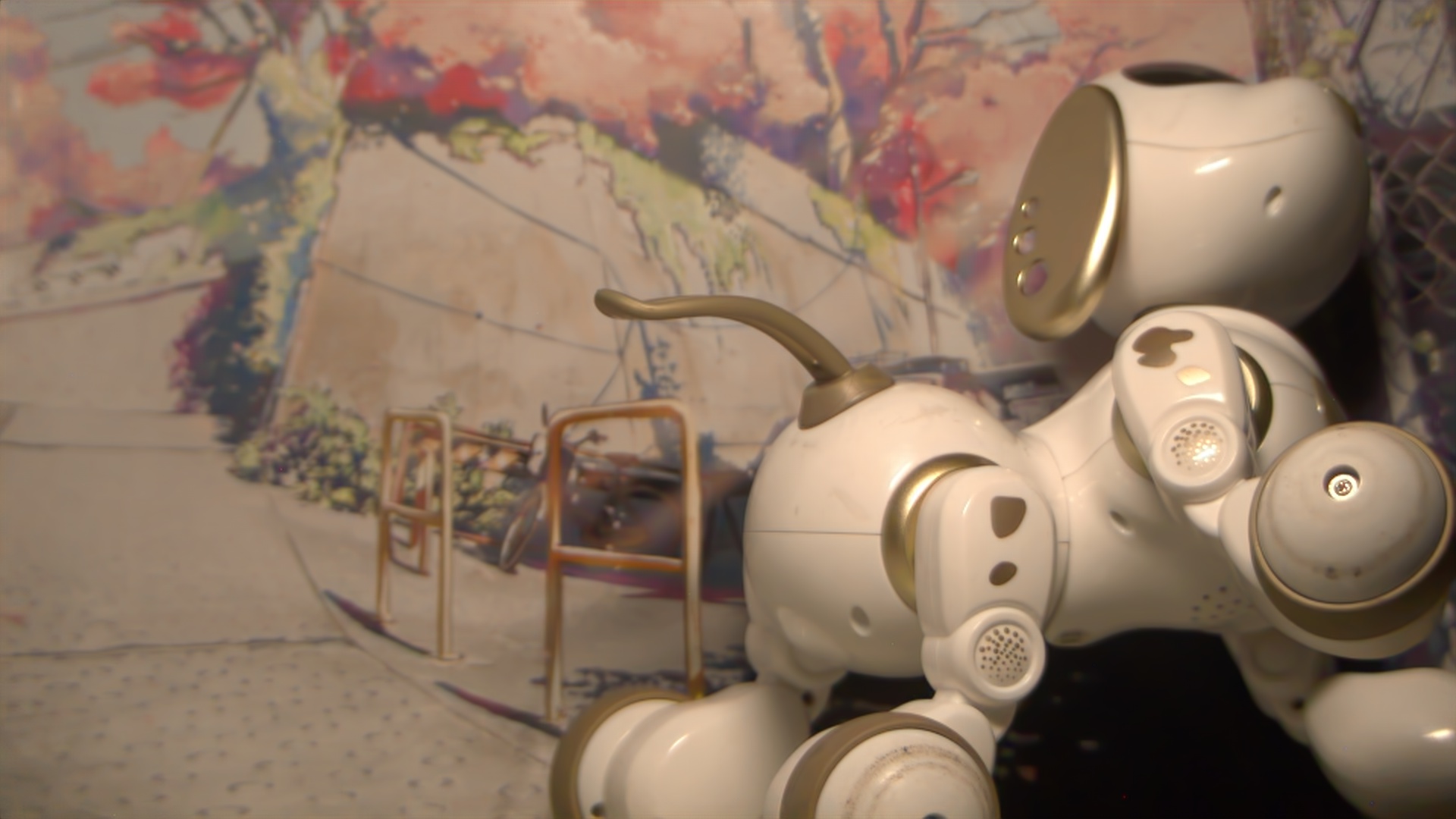}
		\caption{Frame 1}
		\label{fig:crvd_example_s7:f0}
	\end{subfigure}
	\begin{subfigure}{0.13\textwidth}
	    \captionsetup{justification=centering}
		\includegraphics[width=\textwidth]{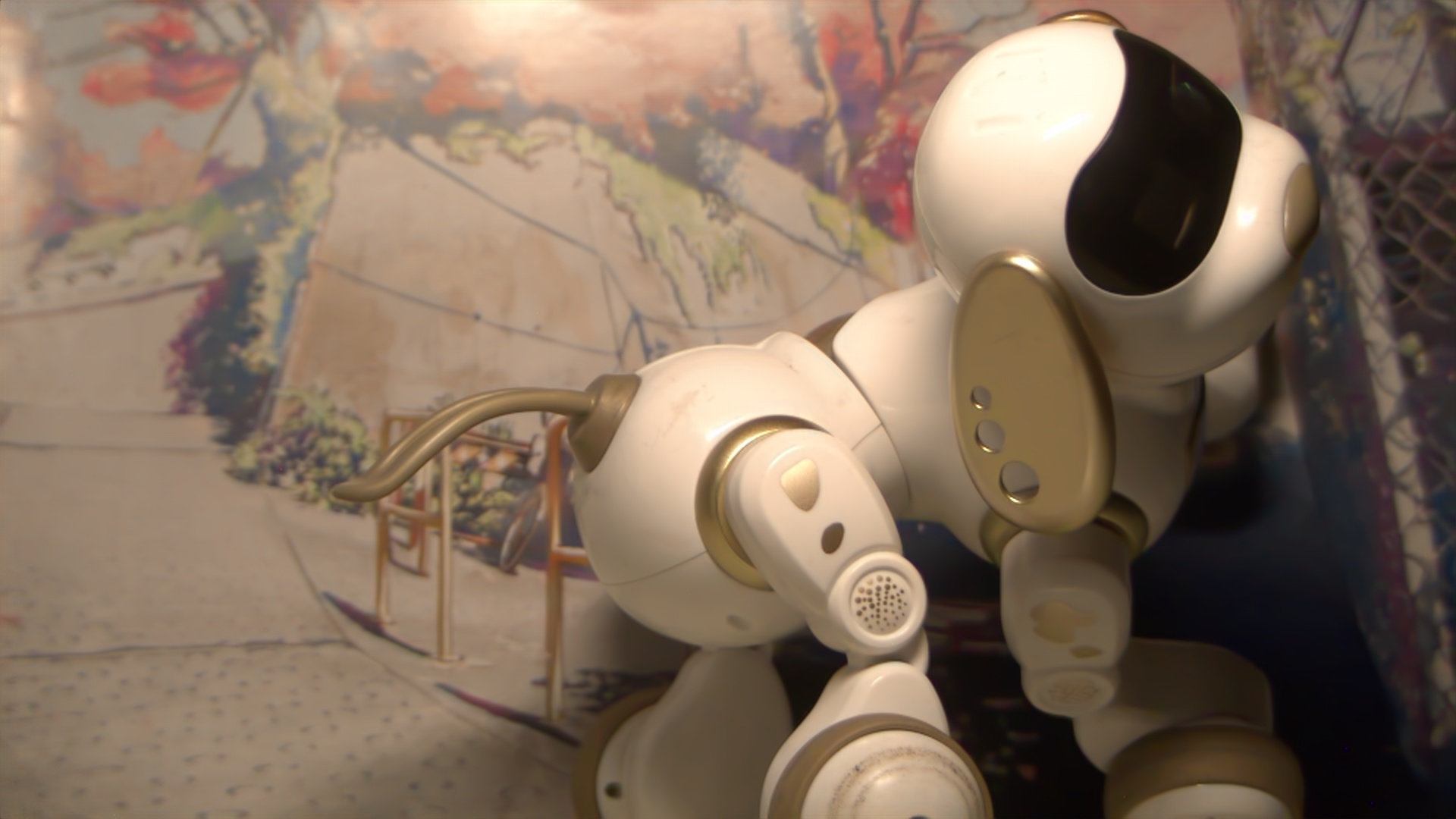}
		\caption{Frame 2}
		\label{fig:crvd_example_s7:f1}
	\end{subfigure}
	\begin{subfigure}{0.13\textwidth}
	    \captionsetup{justification=centering}
		\includegraphics[width=\textwidth]{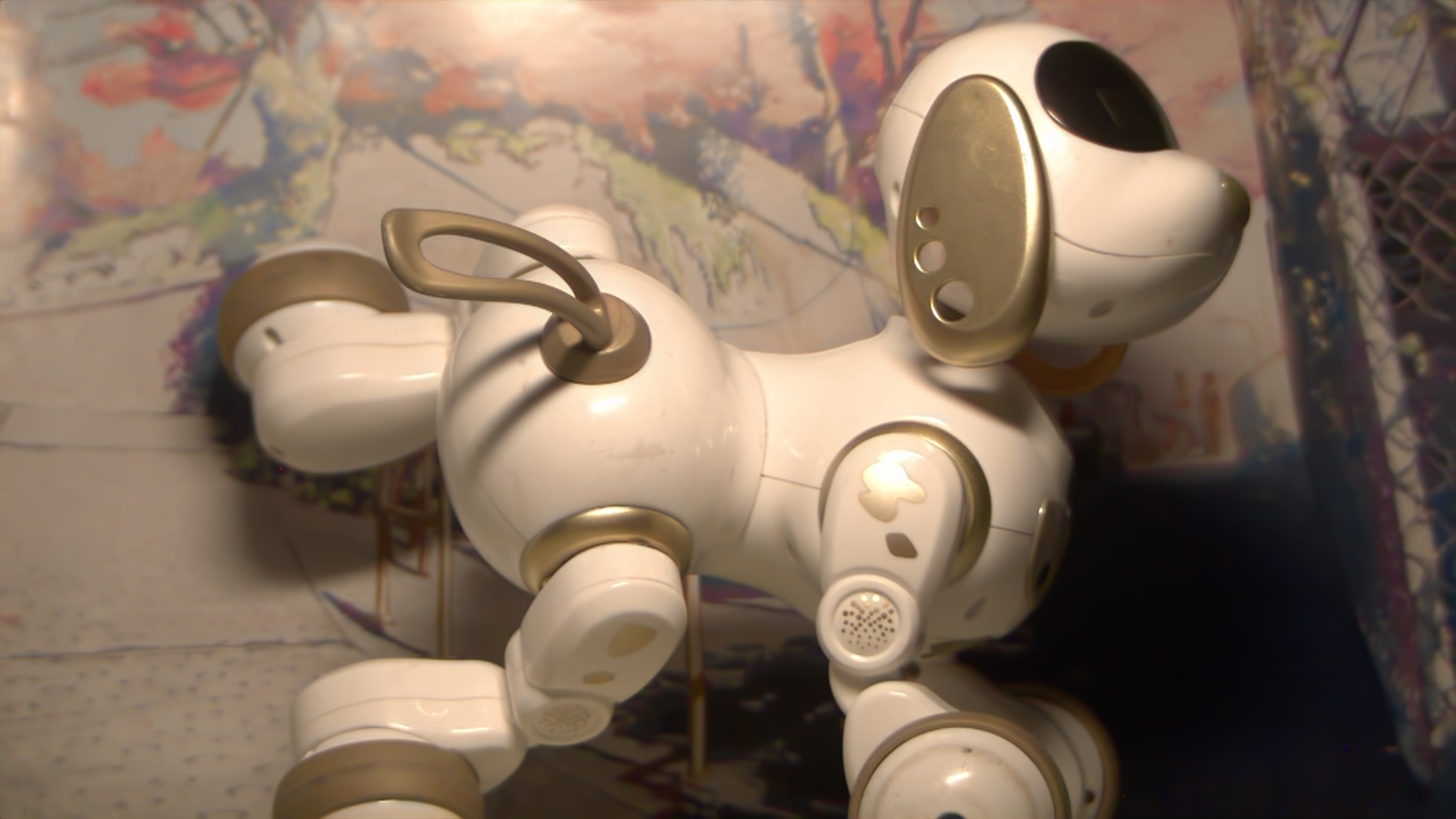}
		\caption{Frame 3}
		\label{fig:crvd_example_s7:f2}
	\end{subfigure}
	\begin{subfigure}{0.13\textwidth}
	    \captionsetup{justification=centering}
		\includegraphics[width=\textwidth]{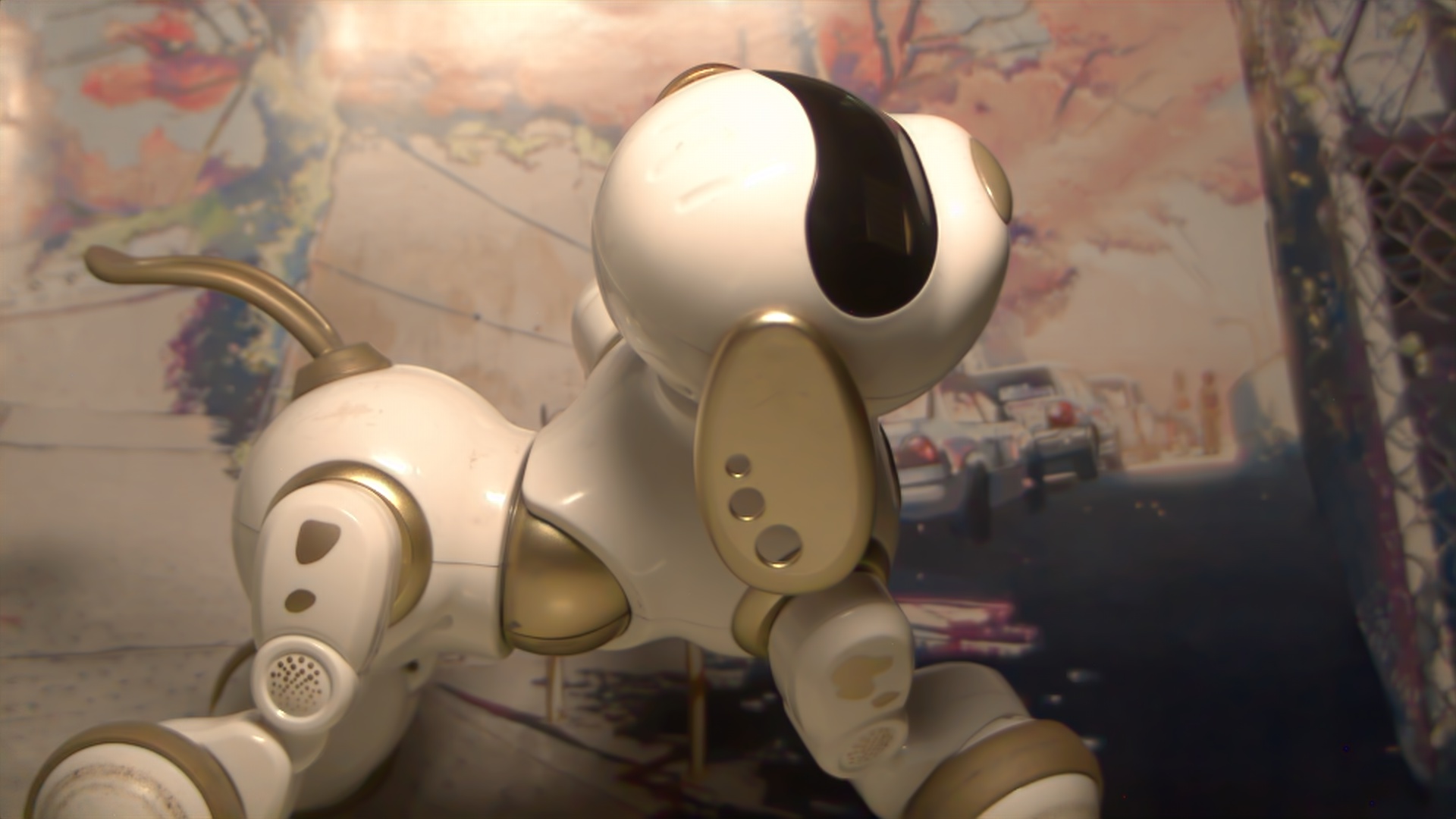}
		\caption{Frame 4}
		\label{fig:crvd_example_s7:f3}
	\end{subfigure}
	\begin{subfigure}{0.13\textwidth}
	    \captionsetup{justification=centering}
		\includegraphics[width=\textwidth]{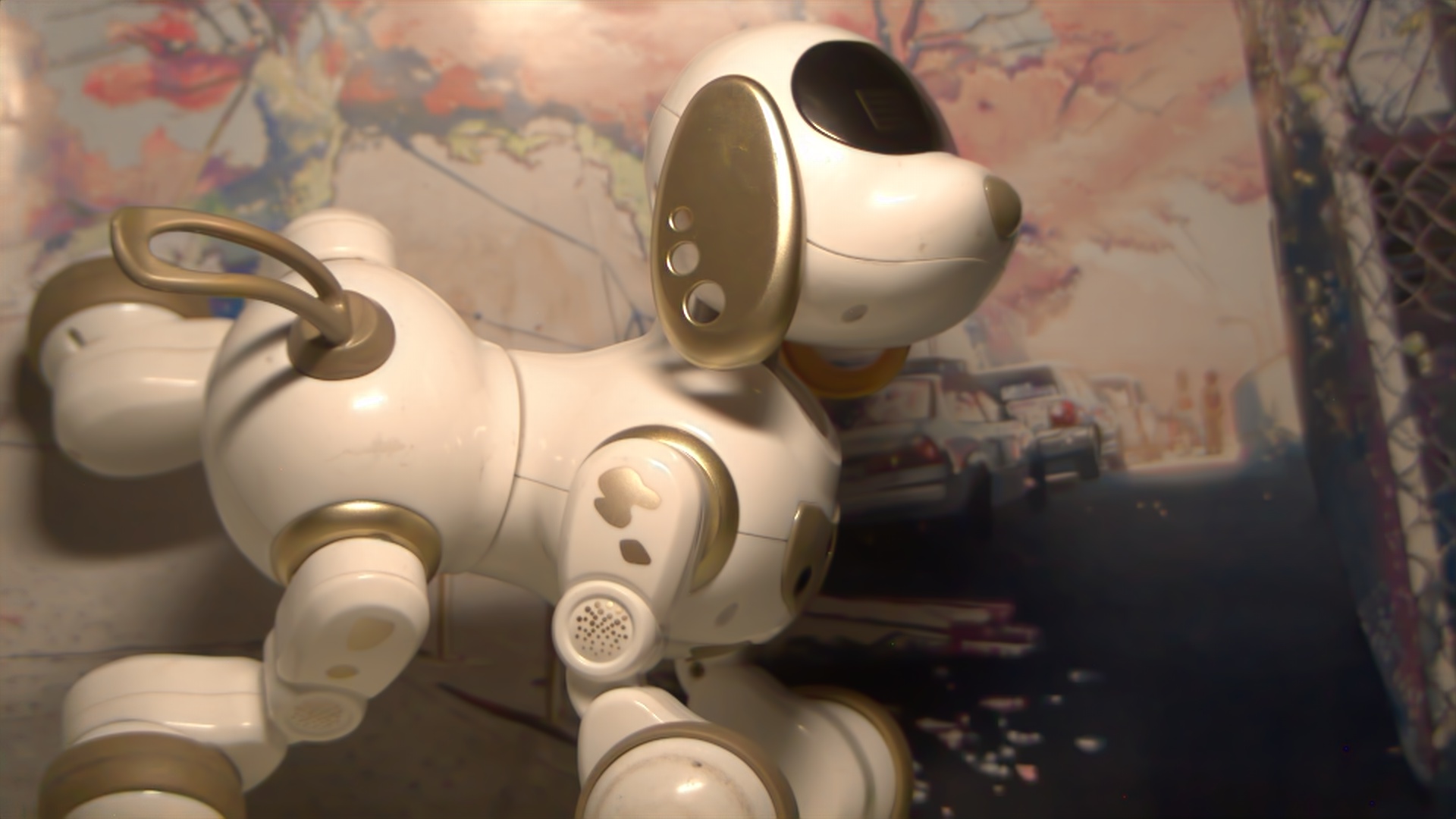}
		\caption{Frame 5}
		\label{fig:crvd_example_s7:f4}
	\end{subfigure}
	\begin{subfigure}{0.13\textwidth}
	    \captionsetup{justification=centering}
		\includegraphics[width=\textwidth]{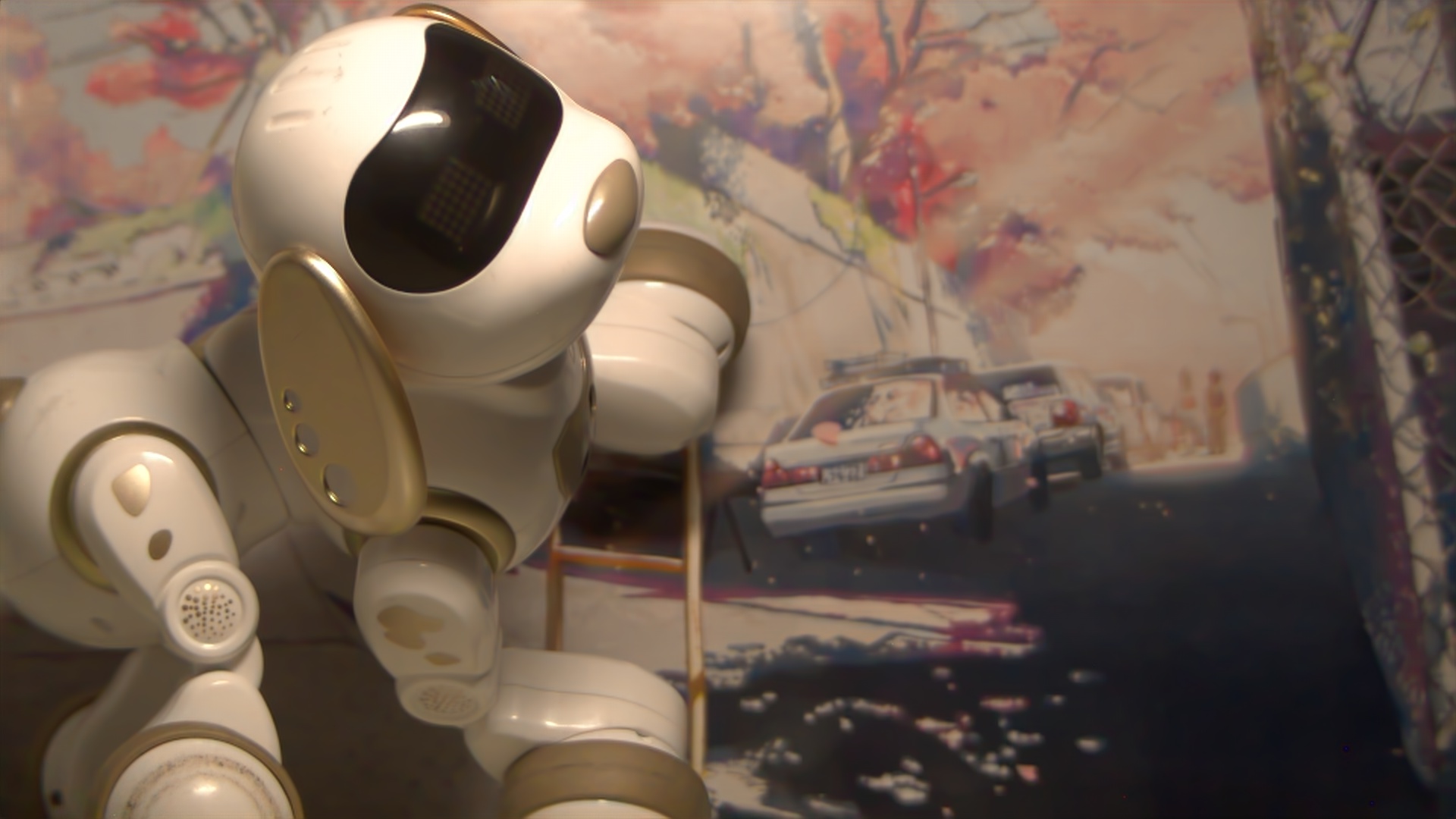}
		\caption{Frame 6}
		\label{fig:crvd_example_s7:f5}
	\end{subfigure}
	\begin{subfigure}{0.13\textwidth}
	    \captionsetup{justification=centering}
		\includegraphics[width=\textwidth]{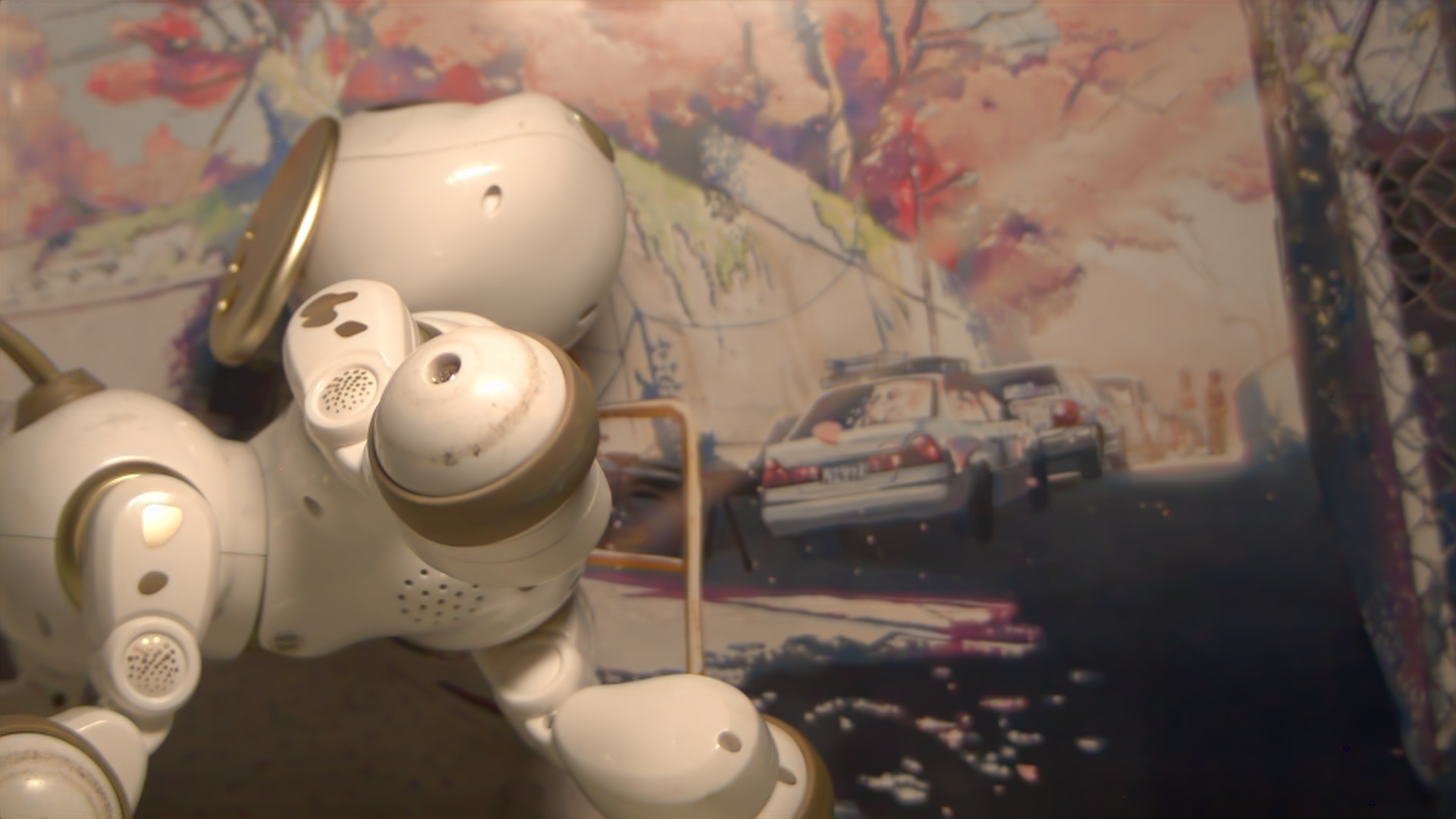}
		\caption{Frame 7}
		\label{fig:crvd_example_s7:f6}
	\end{subfigure}
	\caption{Examples of CRVD sequences. Figures~\ref{fig:crvd_example_s5:f0} to~\ref{fig:crvd_example_s5:f6} present sequence 5, while figures~\ref{fig:crvd_example_s7:f0} to~\ref{fig:crvd_example_s7:f6} shown sequence 7.}
	\label{fig:crvd_example_s5_s7}
\end{figure}

\begin{lemma}
\label{lem:scalar_func}
\begin{equation*}
    \sum_{i, j = 1}^n f\left(i - j\right) = n\sum_{\tau = -\left(n - 1\right)}^{n - 1} \left(1 - \frac{\left|\tau\right|}{n}\right)f\left(\tau\right) \;.
\end{equation*}
\end{lemma}
\begin{proof}
We start by splitting the inner sum into two ranges, $1 \le j \le i$ and $i \le j \le n$
\begin{equation}
\label{eq:lem_dsum_split}
    \begin{split}
      & \sum_{i, j = 1}^n f\left(i - j\right) = \\
      & = \sum_{i = 1}^n\left(\sum_{j = 1}^i f\left(i - j\right) + \sum_{j = i}^n f\left(i - j\right) - f\left(i - i\right)\right) = \\
      & = \sum_{i = 1}^n \sum_{j = 1}^i f\left(i - j\right) + \sum_{i = 1}^n \sum_{j = i}^n f\left(i - j\right) - nf\left(0\right) \;.
    \end{split}
\end{equation}
We calculate the first double sum. Substituting $\tau = i - j + 1$ and changing the order of the double summation, we have
\begin{equation}
\label{eq:1_d_sum}
    \begin{split}
        \sum_{i = 1}^n \sum_{j = 1}^i f\left(i - j\right) & = \sum_{i = 1}^n \sum_{\tau = 1}^{i} f\left(\tau - 1\right) = \\
        & = \sum_{1 \le \tau \le i \le n} f\left(\tau - 1\right) = \\
        & = \sum_{\tau = 1}^n \sum_{i = \tau}^{n} f\left(\tau - 1\right) = \\ 
        & = \sum_{\tau = 1}^n f\left(\tau - 1\right)\sum_{i = \tau}^{n}1 = \\
        & = \sum_{\tau = 1}^n f\left(\tau - 1\right)\left(n - \left(\tau - 1\right)\right) = \\
        & = \sum_{\tau = 0}^{n - 1} \left(n - \tau\right)f\left(\tau\right) = \\
        & = \sum_{\tau = 0}^{n - 1} \left(n - \left|\tau\right|\right)f\left(\tau\right) \;.
    \end{split}
\end{equation}
Substituting variable change $t = n - j + 1$, $k = n - i + 1$ into the second double summation in equation~\ref{eq:lem_dsum_split}, we get
\begin{equation}
\label{eq:d_sum_2}
    \begin{split}
        \sum_{i = 1}^n \sum_{j = i}^n f\left(i - j\right) = & \sum_{k = 1}^n \sum_{t = 1}^k f\left(t - k\right) = \\
        = & \sum_{k = 1}^n \sum_{t = 1}^k g\left(k - t\right) \;,
    \end{split}
\end{equation}
where $g\left(\tau\right) = f\left(-\tau\right)$.
Substituting equation~\ref{eq:1_d_sum} into~\ref{eq:d_sum_2}, we have
\begin{equation}
\label{eq:2_d_sum}
    \begin{split}
            \sum_{i = 1}^n \sum_{j = i}^n f\left(i - j\right) = & \sum_{\tau = 0}^{n - 1} \left(n - \left|\tau\right|\right)g\left(\tau\right) = \\
            = & \sum_{\tau = -\left(n - 1\right)}^{0} \left(n - \left|\tau\right|\right)f\left(\tau\right) \;.
    \end{split}
\end{equation}
Substituting equations~\ref{eq:1_d_sum}~and~\ref{eq:2_d_sum} into equation~\ref{eq:lem_dsum_split} finalizes the proof
\begin{equation}
    \begin{split}
        & \sum_{i, j = 1}^n f\left(i - j\right) = \\
        & \quad = \sum_{i = 1}^n \sum_{j = 1}^i f\left(i - j\right) + \sum_{i = 1}^n \sum_{j = i}^n f\left(i - j\right) - nf\left(0\right) \\
        & \quad = \sum_{\tau = 0}^{n - 1} \left(n - \left|\tau\right|\right)f\left(\tau\right) + \\
        & \qquad\qquad\qquad\sum_{\tau = -\left(n - 1\right)}^{0} \left(n - \left|\tau\right|\right)f\left(\tau\right) - nf\left(0\right) = \\
        & \quad = \sum_{\tau = -\left(n - 1\right)}^{n - 1} \left(n - \left|\tau\right|\right)f\left(\tau\right) \;.
    \end{split}
\end{equation}
\end{proof}

\section{Training Details}
\label{app:training_details}
In all experiments, we train the denoisers for 30 epochs using the Adam~\cite{Kingma2015AdamAM} optimizer while decreasing the learning rate by 0.5 every 5 epochs. The initial learning rate is 0.001 for the correlated Gaussian experiment and 0.00001 for the experinet with real-world noise. During the training, we extract random patches of size $50 \times 50$ from the training data and use a batch size of 32 for correlated Gaussian denoising and 128 for the experiment with real-world noise. The values of $n$ used in the patch search are summarized in tables~\ref{tab:davis_n}~and~\ref{tab:crvd_n}. Examples of CRVD sequences are presented in figure~\ref{fig:crvd_example_s5_s7}.
\begin{table}
    \centering
    \begin{tabular}{|c|c|c|c|c|}
         \hline
         & $\sigma = 5$ & $\sigma = 10$ & $\sigma = 15$ & $\sigma = 20$ \\
         \hline
         $k = 2$ &  19 & 27 & 31 & 33 \\
         \hline
         $k = 3$ &  19 & 37 & 41 & 43 \\
         \hline
         $k = 4$ &  25 & 43 & 43 & 45 \\
         \hline
    \end{tabular}
    \caption{Patch size $n$ for correlated Gaussian experiments.}
    \label{tab:davis_n}
\end{table}
\begin{table}
    \centering
    \begin{tabular}{|c|c|c|c|c|c|}
         \hline
         ISO & 1600 & 3200 & 6400 & 12800 & 25600 \\
         \hline
         $n$ &  15 & 25 & 27 & 35 & 37 \\
         \hline
    \end{tabular}
    \caption{Patch size $n$ for experiments with real-world noise.}
    \label{tab:crvd_n}
\end{table}

\section{Additional Results}
\label{app:additional_results}
This section provides additional visual comparisons of our framework versus leading competitors. Figures~\ref{fig:davis_20_1_s10_k3_11_0_s10_k4},~\ref{fig:davis_2_0_s15_k4_8_1_s20_k4},~and~\ref{fig:davis_4_2_s15_k3_19_0_s5_k4} show denoising examples from the correlated Gaussian experiments, while figures~\ref{fig:crvd_8_5_iso25600_2_3_iso25600},~\ref{fig:crvd_5_2_iso12800_10_5_iso12800},~and~\ref{fig:crvd_1_5_iso6400_0_4_iso3200} present real-world denoising examples. 
\begin{figure*}
    \centering
	\begin{subfigure}{0.18\textwidth}
	    \captionsetup{justification=centering}
		\includegraphics[width=\textwidth]{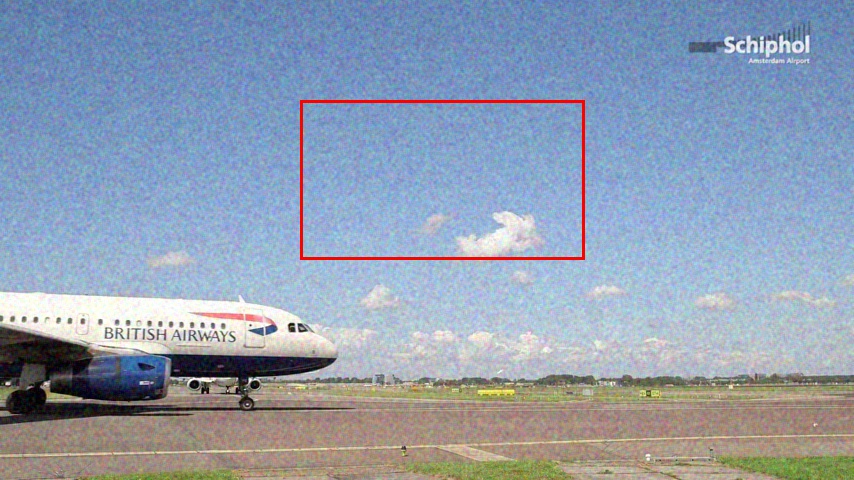}
		\caption{Noisy \\ 28.12 / 0.522}
		\label{fig:davis_20_1_s10_k3:noisy_rect}
	\end{subfigure}
	\begin{subfigure}{0.18\textwidth}
	    \captionsetup{justification=centering}
		\includegraphics[width=\textwidth]{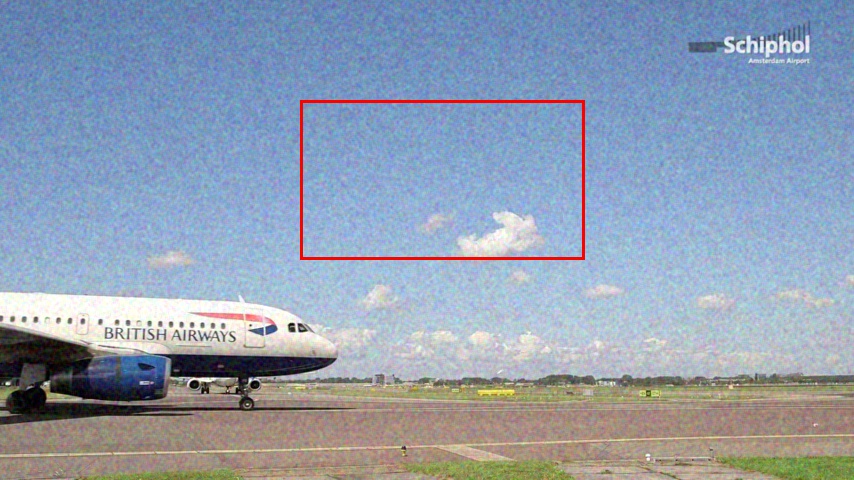}
		\caption{N2N \\ 28.70 / 0.555}
		\label{fig:davis_20_1_s10_k3:n2n_rect}
	\end{subfigure}
	\begin{subfigure}{0.18\textwidth}
	    \captionsetup{justification=centering}
		\includegraphics[width=\textwidth]{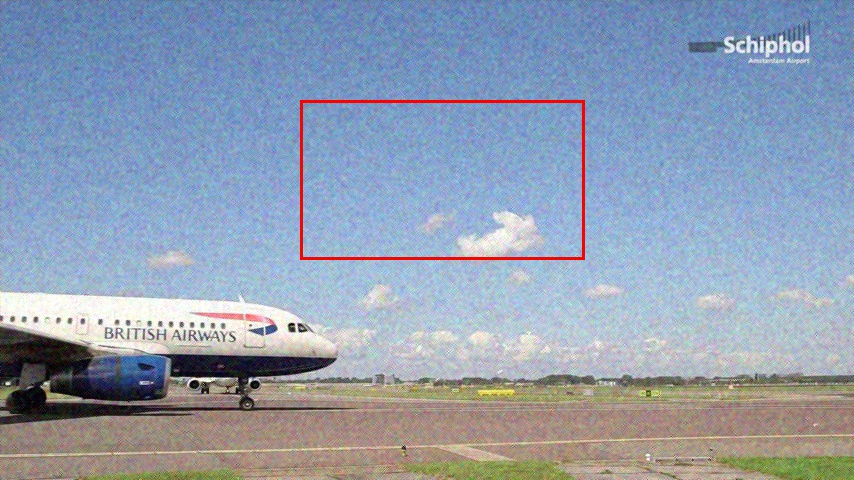}
		\caption{B2U \\ 24.27 / 0.335}
		\label{fig:davis_20_1_s10_k3:b2u_rect}
	\end{subfigure}
	\begin{subfigure}{0.18\textwidth}
	    \captionsetup{justification=centering}
		\includegraphics[width=\textwidth]{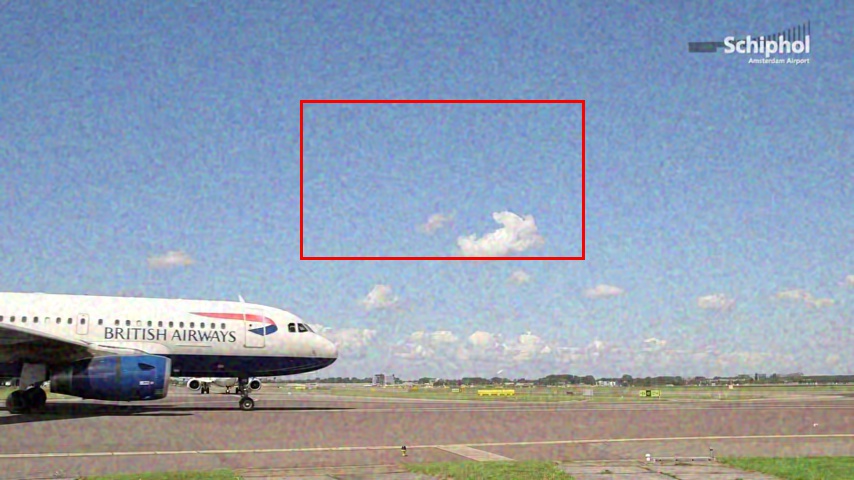}
		\caption{BM3D \\ 31.41 / 0.726}
		\label{fig:davis_20_1_s10_k3:bm3d_rect}
	\end{subfigure}
	\begin{subfigure}{0.18\textwidth}
	    \captionsetup{justification=centering}
		\includegraphics[width=\textwidth]{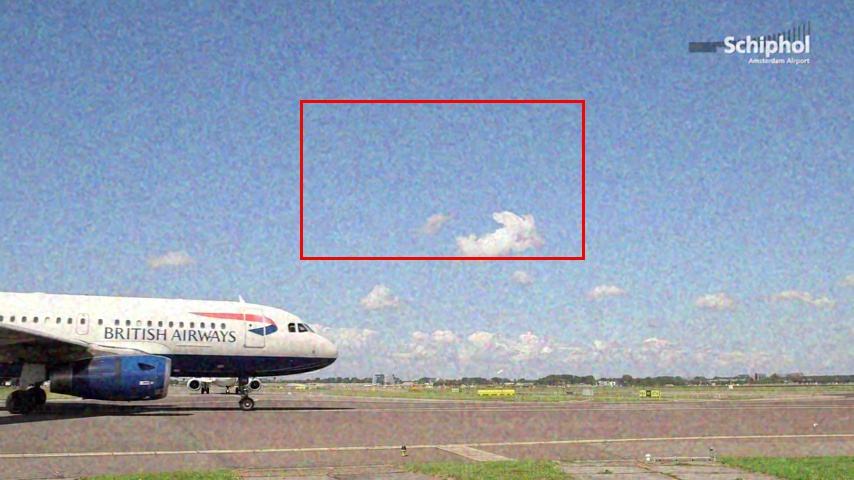}
		\caption{B-DnCNN \\ 30.55 / 0.677}
		\label{fig:davis_20_1_s10_k3:b_dncnn_rect}
	\end{subfigure}
	\begin{subfigure}{0.18\textwidth}
	    \captionsetup{justification=centering}
		\includegraphics[width=\textwidth]{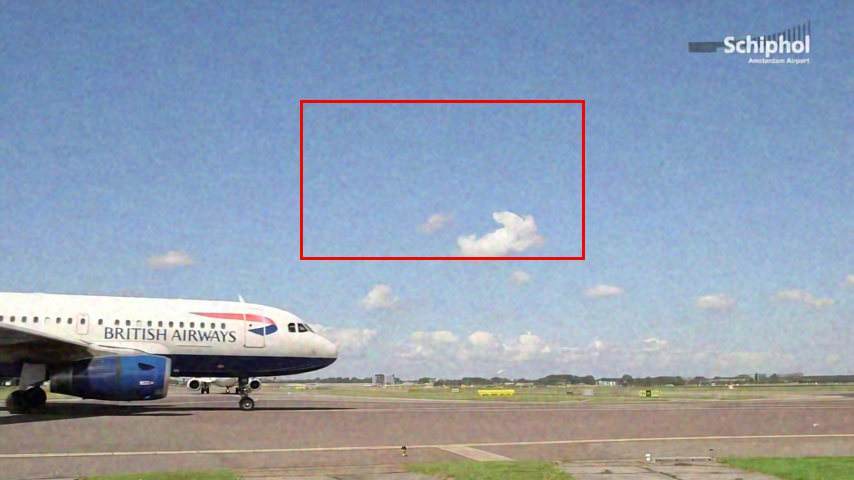}
		\caption{R2R \\ 34.05 / 0.836}
		\label{fig:davis_20_1_s10_k3:r2r_rect}
	\end{subfigure}
	\begin{subfigure}{0.18\textwidth}
	    \captionsetup{justification=centering}
		\includegraphics[width=\textwidth]{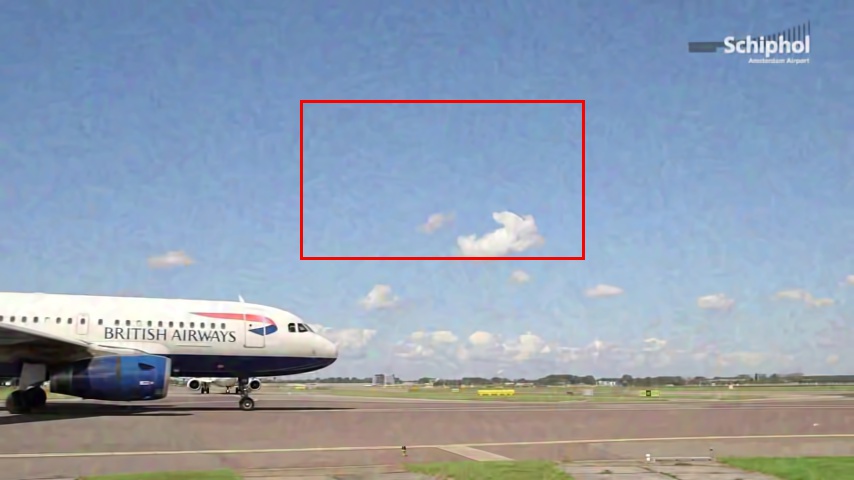}
		\caption{BM3D-O \\ 36.37 / 0.939}
		\label{fig:davis_20_1_s10_k3:bm3d_opt_rect}
	\end{subfigure}
	\begin{subfigure}{0.18\textwidth}
	    \captionsetup{justification=centering}
		\includegraphics[width=\textwidth]{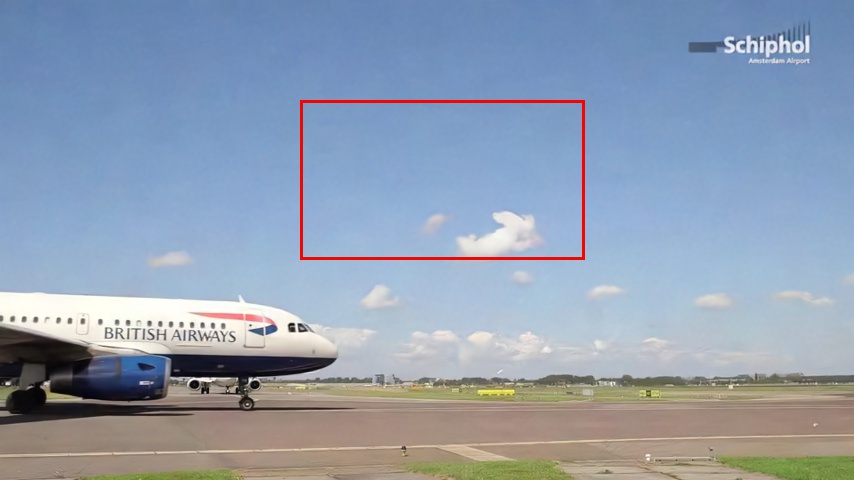}
		\caption{PC-UNet \\ 38.65 / 0.965}
		\label{fig:davis_20_1_s10_k3:pc_unet_rect}
	\end{subfigure}
	\begin{subfigure}{0.18\textwidth}
	    \captionsetup{justification=centering}
		\includegraphics[width=\textwidth]{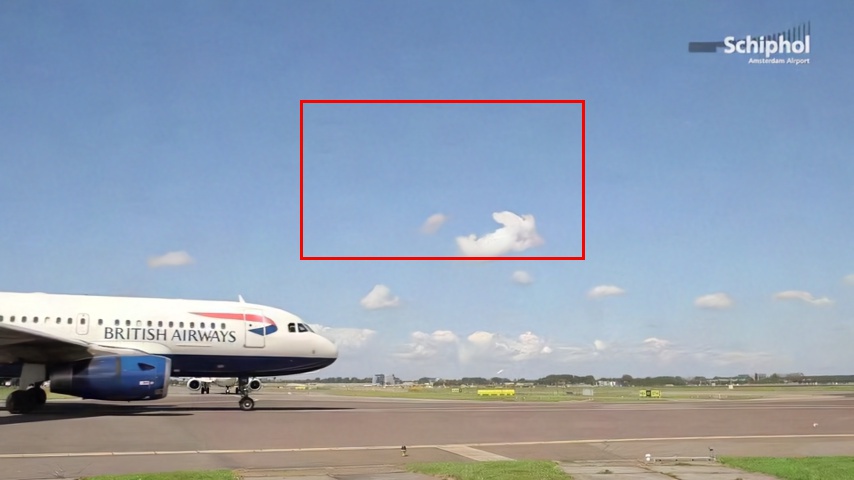}
		\caption{PC-DnCNN \\ 39.02 / 0.968}
		\label{fig:davis_20_1_s10_k3:pc_dncnn_rect}
	\end{subfigure}
	\begin{subfigure}{0.18\textwidth}
	    \captionsetup{justification=centering}
		\includegraphics[width=\textwidth]{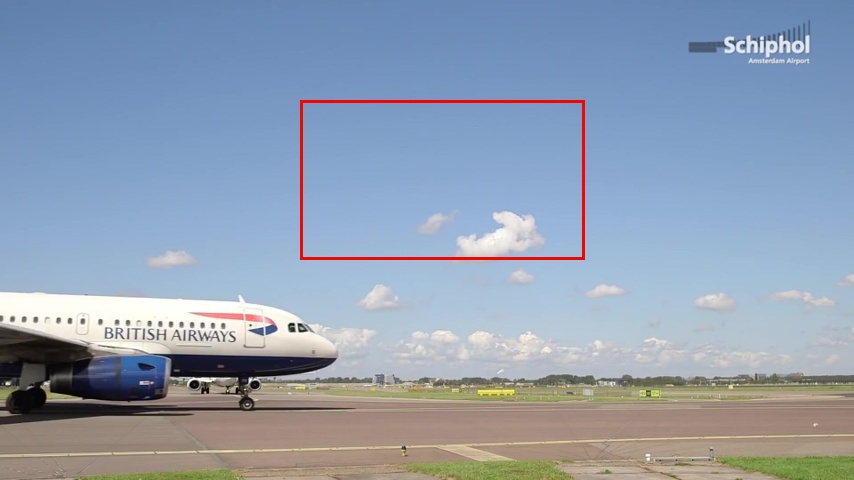}
		\caption{Clean \newline }
		\label{fig:davis_20_1_s10_k3:clean_rect}
	\end{subfigure}
	\begin{subfigure}{0.18\textwidth}
	    \captionsetup{justification=centering}
		\includegraphics[width=\textwidth]{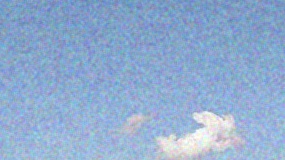}
		\caption{Noisy}
		\label{fig:davis_20_1_s10_k3:noisy_crop}
	\end{subfigure}
	\begin{subfigure}{0.18\textwidth}
	    \captionsetup{justification=centering}
		\includegraphics[width=\textwidth]{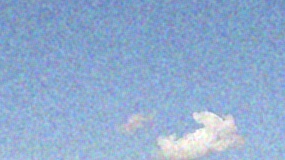}
		\caption{N2N}
		\label{fig:davis_20_1_s10_k3:n2n_crop}
	\end{subfigure}
	\begin{subfigure}{0.18\textwidth}
	    \captionsetup{justification=centering}
		\includegraphics[width=\textwidth]{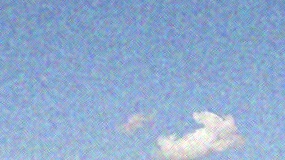}
		\caption{B2U}
		\label{fig:davis_20_1_s10_k3:b2u_crop}
	\end{subfigure}
	\begin{subfigure}{0.18\textwidth}
	    \captionsetup{justification=centering}
		\includegraphics[width=\textwidth]{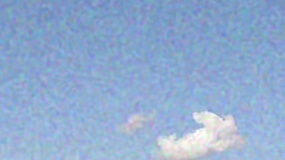}
		\caption{BM3D}
		\label{fig:davis_20_1_s10_k3:bm3d_crop}
	\end{subfigure}
	\begin{subfigure}{0.18\textwidth}
	    \captionsetup{justification=centering}
		\includegraphics[width=\textwidth]{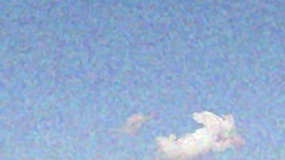}
		\caption{B-DnCNN}
		\label{fig:davis_20_1_s10_k3:b_dncnn_crop}
	\end{subfigure}
	\begin{subfigure}{0.18\textwidth}
	    \captionsetup{justification=centering}
		\includegraphics[width=\textwidth]{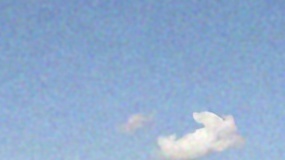}
		\caption{R2R}
		\label{fig:davis_20_1_s10_k3:r2r_crop}
	\end{subfigure}
	\begin{subfigure}{0.18\textwidth}
	    \captionsetup{justification=centering}
		\includegraphics[width=\textwidth]{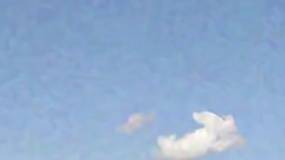}
		\caption{BM3D-O}
		\label{fig:davis_20_1_s10_k3:bm3d_opt_crop}
	\end{subfigure}
	\begin{subfigure}{0.18\textwidth}
	    \captionsetup{justification=centering}
		\includegraphics[width=\textwidth]{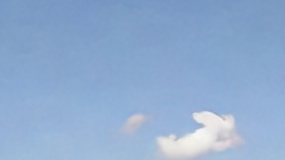}
		\caption{PC-UNet}
		\label{fig:davis_20_1_s10_k3:pc_unet_crop}
	\end{subfigure}
	\begin{subfigure}{0.18\textwidth}
	    \captionsetup{justification=centering}
		\includegraphics[width=\textwidth]{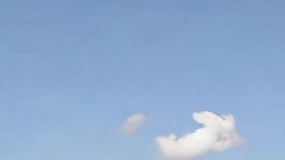}
		\caption{PC-DnCNN}
		\label{fig:davis_20_1_s10_k3:pc_dncnn_crop}
	\end{subfigure}
	\begin{subfigure}{0.18\textwidth}
	    \captionsetup{justification=centering}
		\includegraphics[width=\textwidth]{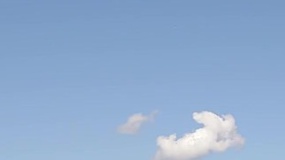}
		\caption{Clean}
		\label{fig:davis_20_1_s10_k3:clean_crop}
	\end{subfigure}
	\begin{subfigure}{0.18\textwidth}
	    \captionsetup{justification=centering}
		\includegraphics[width=\textwidth]{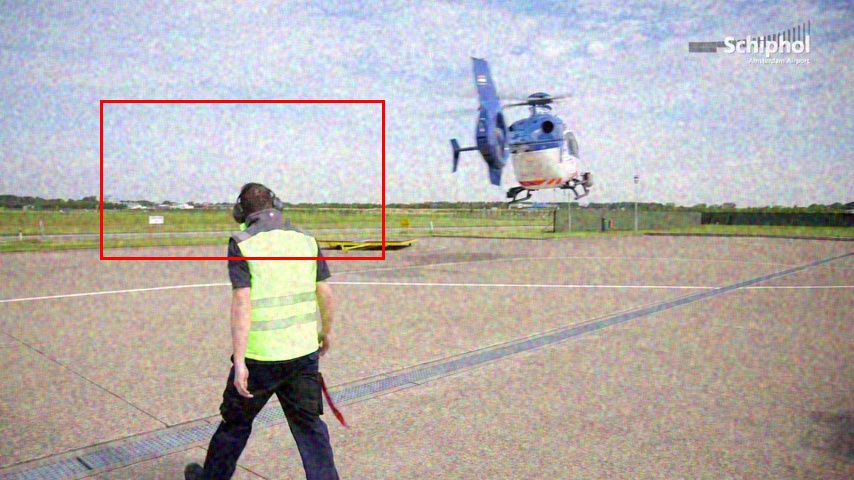}
		\caption{Noisy \\ 28.11 / 0.596}
		\label{fig:davis_11_0_s10_k4:noisy_rect}
	\end{subfigure}
	\begin{subfigure}{0.18\textwidth}
	    \captionsetup{justification=centering}
		\includegraphics[width=\textwidth]{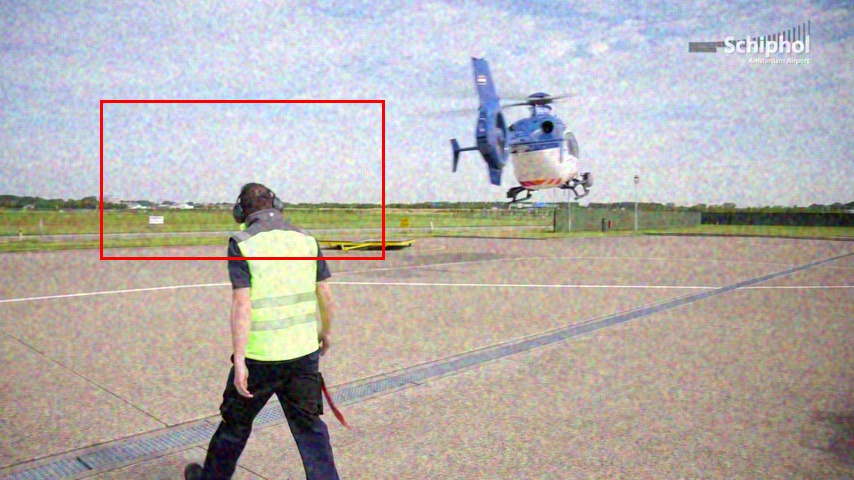}
		\caption{N2N \\ 28.52 / 0.620}
		\label{fig:davis_11_0_s10_k4:n2n_rect}
	\end{subfigure}
	\begin{subfigure}{0.18\textwidth}
	    \captionsetup{justification=centering}
		\includegraphics[width=\textwidth]{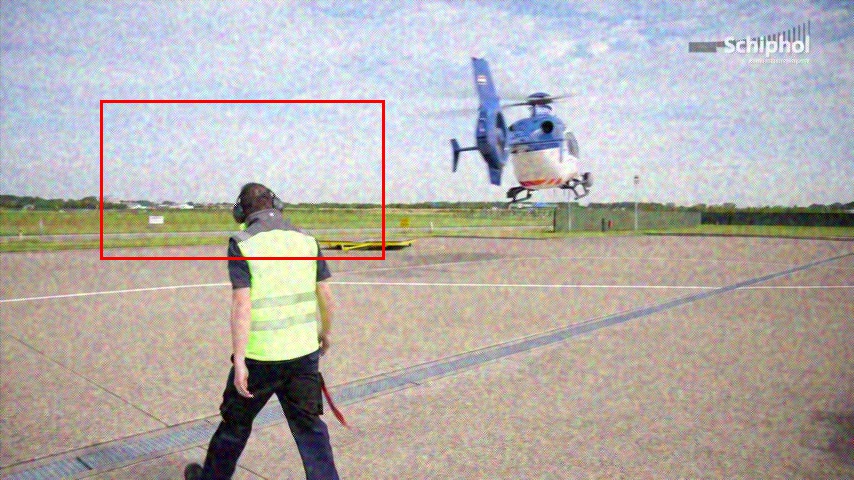}
		\caption{B2U \\ 24.09 / 0.363}
		\label{fig:davis_11_0_s10_k4:b2u_rect}
	\end{subfigure}
	\begin{subfigure}{0.18\textwidth}
	    \captionsetup{justification=centering}
		\includegraphics[width=\textwidth]{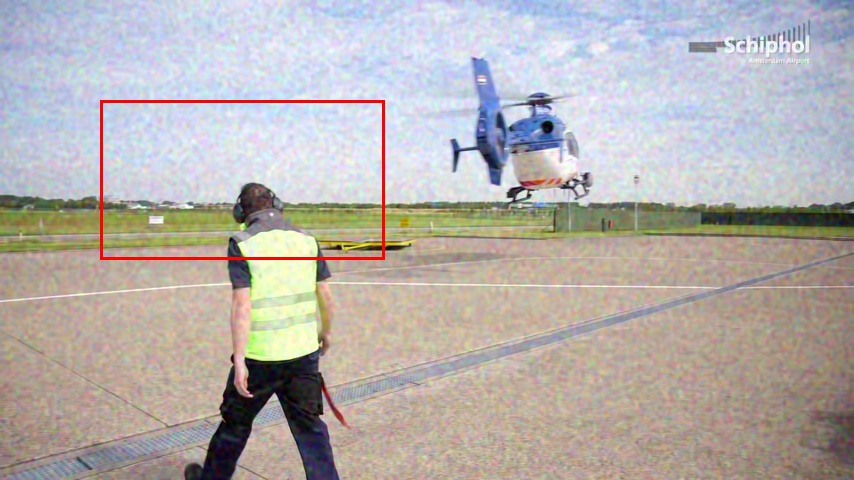}
		\caption{BM3D \\ 30.26 / 0.725}
		\label{fig:davis_11_0_s10_k4:bm3d_rect}
	\end{subfigure}
	\begin{subfigure}{0.18\textwidth}
	    \captionsetup{justification=centering}
		\includegraphics[width=\textwidth]{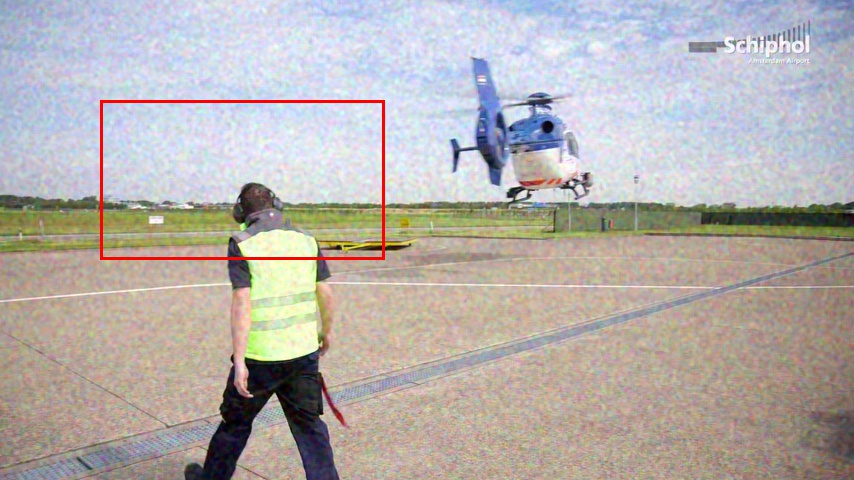}
		\caption{B-DnCNN \\ 29.48 / 0.686}
		\label{fig:davis_11_0_s10_k4:b_dncnn_rect}
	\end{subfigure}
	\begin{subfigure}{0.18\textwidth}
	    \captionsetup{justification=centering}
		\includegraphics[width=\textwidth]{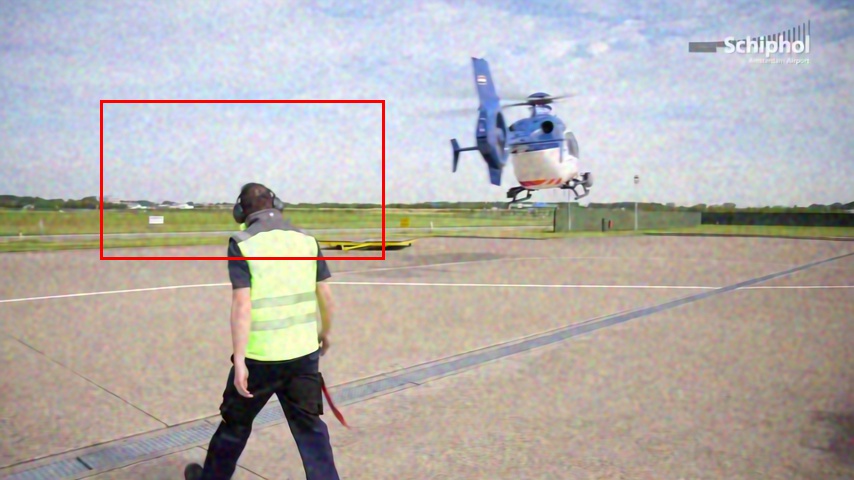}
		\caption{R2R \\ 31.83 / 0.781}
		\label{fig:davis_11_0_s10_k4:r2r_rect}
	\end{subfigure}
	\begin{subfigure}{0.18\textwidth}
	    \captionsetup{justification=centering}
		\includegraphics[width=\textwidth]{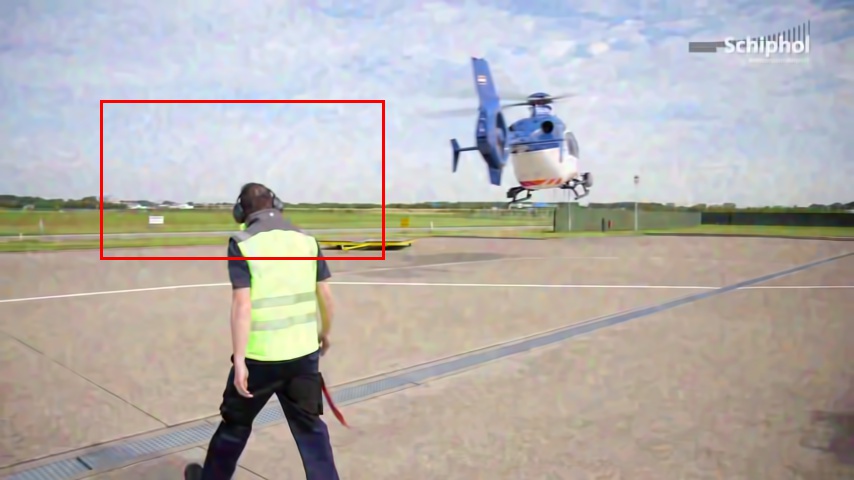}
		\caption{BM3D-O \\ 33.84 / 0.876}
		\label{fig:davis_11_0_s10_k4:bm3d_opt_rect}
	\end{subfigure}
	\begin{subfigure}{0.18\textwidth}
	    \captionsetup{justification=centering}
		\includegraphics[width=\textwidth]{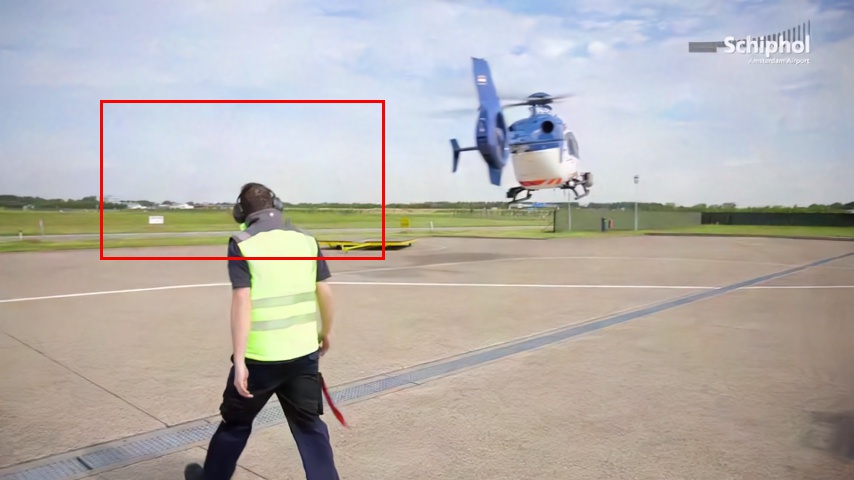}
		\caption{PC-UNet \\ 36.26 / 0.934}
		\label{fig:davis_11_0_s10_k4:pc_unet_rect}
	\end{subfigure}
	\begin{subfigure}{0.18\textwidth}
	    \captionsetup{justification=centering}
		\includegraphics[width=\textwidth]{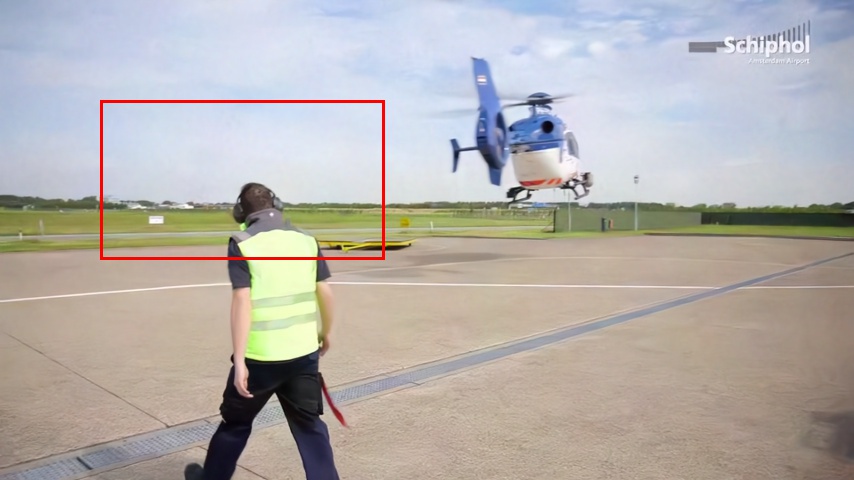}
		\caption{PC-DnCNN \\ 36.42 / 0.936}
		\label{fig:davis_11_0_s10_k4:pc_dncnn_rect}
	\end{subfigure}
	\begin{subfigure}{0.18\textwidth}
	    \captionsetup{justification=centering}
		\includegraphics[width=\textwidth]{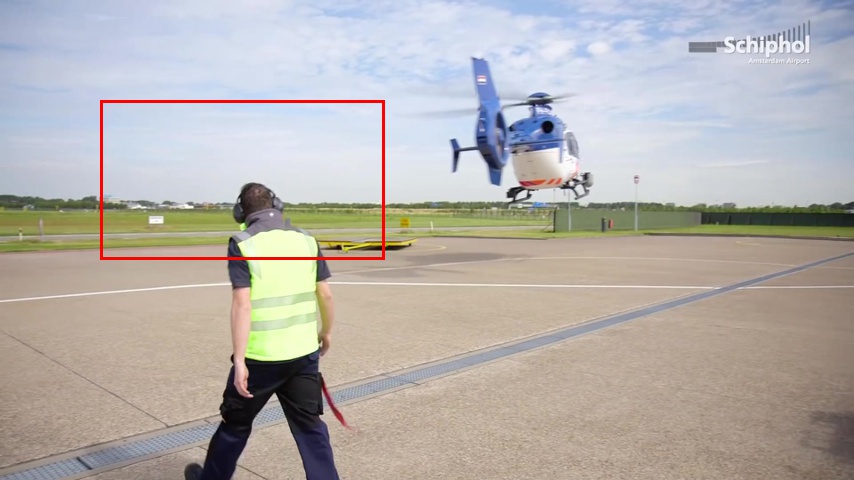}
		\caption{Clean \newline }
		\label{fig:davis_11_0_s10_k4:clean_rect}
	\end{subfigure}
	\begin{subfigure}{0.18\textwidth}
	    \captionsetup{justification=centering}
		\includegraphics[width=\textwidth]{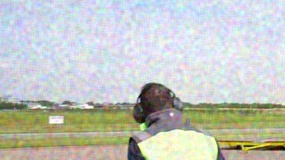}
		\caption{Noisy}
		\label{fig:davis_11_0_s10_k4:noisy_crop}
	\end{subfigure}
	\begin{subfigure}{0.18\textwidth}
	    \captionsetup{justification=centering}
		\includegraphics[width=\textwidth]{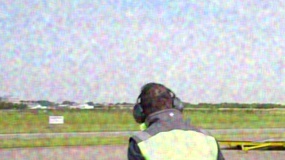}
		\caption{N2N}
		\label{fig:davis_11_0_s10_k4:n2n_crop}
	\end{subfigure}
	\begin{subfigure}{0.18\textwidth}
	    \captionsetup{justification=centering}
		\includegraphics[width=\textwidth]{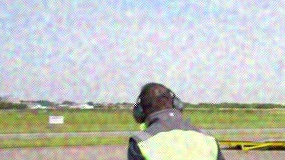}
		\caption{B2U}
		\label{fig:davis_11_0_s10_k4:b2u_crop}
	\end{subfigure}
	\begin{subfigure}{0.18\textwidth}
	    \captionsetup{justification=centering}
		\includegraphics[width=\textwidth]{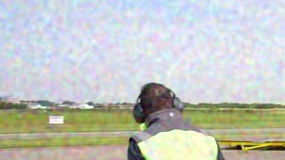}
		\caption{BM3D}
		\label{fig:davis_11_0_s10_k4:bm3d_crop}
	\end{subfigure}
	\begin{subfigure}{0.18\textwidth}
	    \captionsetup{justification=centering}
		\includegraphics[width=\textwidth]{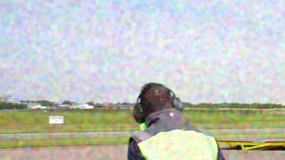}
		\caption{B-DnCNN}
		\label{fig:davis_11_0_s10_k4:b_dncnn_crop}
	\end{subfigure}
	\begin{subfigure}{0.18\textwidth}
	    \captionsetup{justification=centering}
		\includegraphics[width=\textwidth]{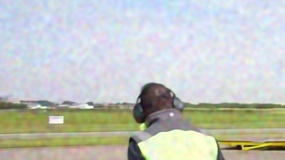}
		\caption{R2R}
		\label{fig:davis_11_0_s10_k4:r2r_crop}
	\end{subfigure}
	\begin{subfigure}{0.18\textwidth}
	    \captionsetup{justification=centering}
		\includegraphics[width=\textwidth]{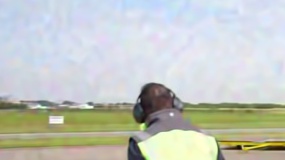}
		\caption{BM3D-O}
		\label{fig:davis_11_0_s10_k4:bm3d_opt_crop}
	\end{subfigure}
	\begin{subfigure}{0.18\textwidth}
	    \captionsetup{justification=centering}
		\includegraphics[width=\textwidth]{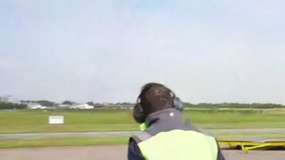}
		\caption{PC-UNet}
		\label{fig:davis_11_0_s10_k4:pc_unet_crop}
	\end{subfigure}
	\begin{subfigure}{0.18\textwidth}
	    \captionsetup{justification=centering}
		\includegraphics[width=\textwidth]{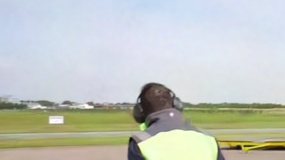}
		\caption{PC-DnCNN}
		\label{fig:davis_11_0_s10_k4:pc_dncnn_crop}
	\end{subfigure}
	\begin{subfigure}{0.18\textwidth}
	    \captionsetup{justification=centering}
		\includegraphics[width=\textwidth]{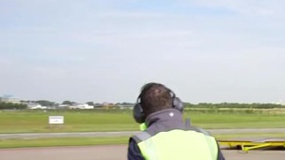}
		\caption{Clean}
		\label{fig:davis_11_0_s10_k4:clean_crop}
	\end{subfigure}
	\caption{Denoising examples with correlated Gaussian noise. The first four rows show frame 13 of the sequence \emph{planes-crossing} with $\sigma = 10$ and $k = 3$. The last four rows present frame 5 of the sequence \emph{helicopter} with $\sigma = 10$ and $k = 4$. As can be seen, oracle BM3D leaves a substantial amount of low-frequency noise unfiltered, while other algorithms, except ours (PC-UNet and PC-DnCNN), do not succeed in removing the noise.}
	\label{fig:davis_20_1_s10_k3_11_0_s10_k4}
\end{figure*}

\begin{figure*}
    \centering
	\begin{subfigure}{0.18\textwidth}
	    \captionsetup{justification=centering}
		\includegraphics[width=\textwidth]{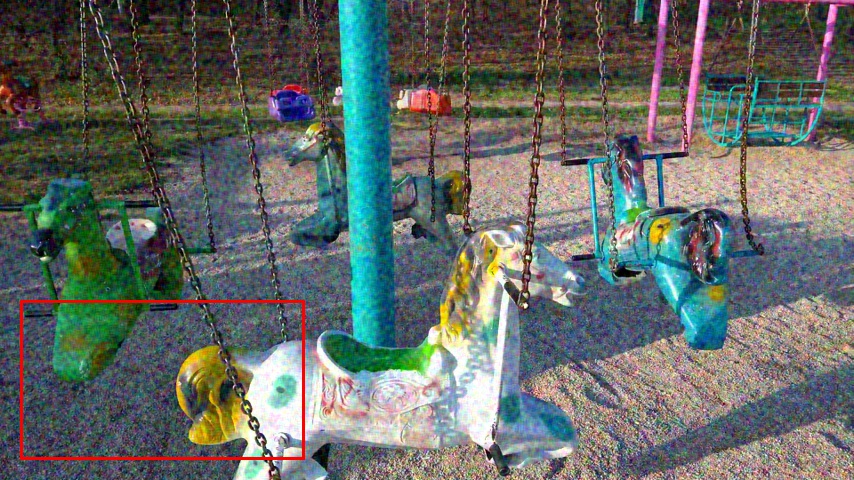}
		\caption{Noisy \\ 24.60 / 0.710}
		\label{fig:davis_2_0_s15_k4:noisy_rect}
	\end{subfigure}
	\begin{subfigure}{0.18\textwidth}
	    \captionsetup{justification=centering}
		\includegraphics[width=\textwidth]{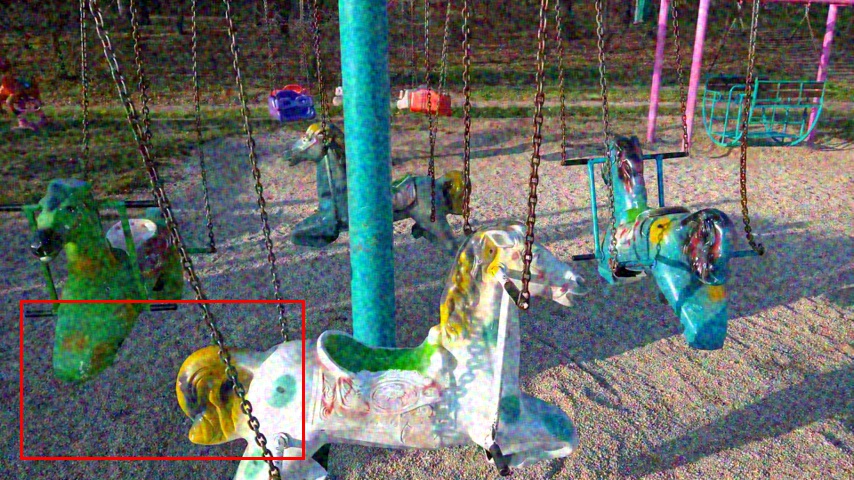}
		\caption{N2N \\ 25.00 / 0.725}
		\label{fig:davis_2_0_s15_k4:n2n_rect}
	\end{subfigure}
	\begin{subfigure}{0.18\textwidth}
	    \captionsetup{justification=centering}
		\includegraphics[width=\textwidth]{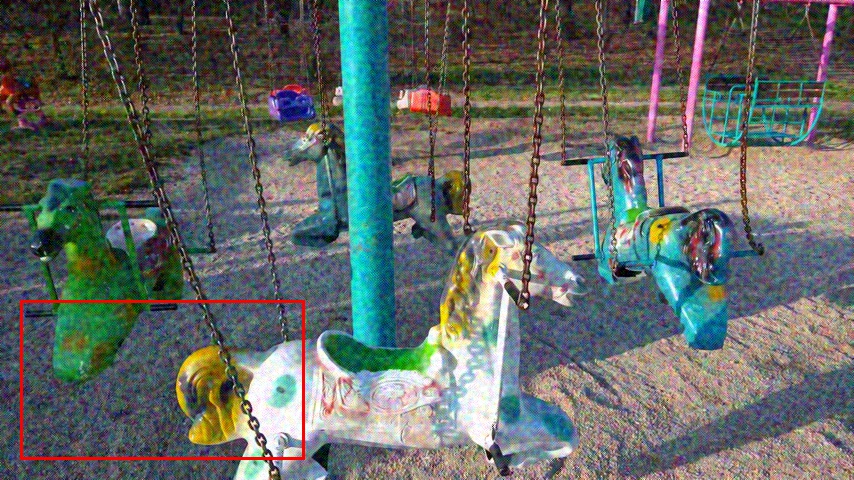}
		\caption{B2U \\ 21.59 / 0.531}
		\label{fig:davis_2_0_s15_k4:b2u_rect}
	\end{subfigure}
	\begin{subfigure}{0.18\textwidth}
	    \captionsetup{justification=centering}
		\includegraphics[width=\textwidth]{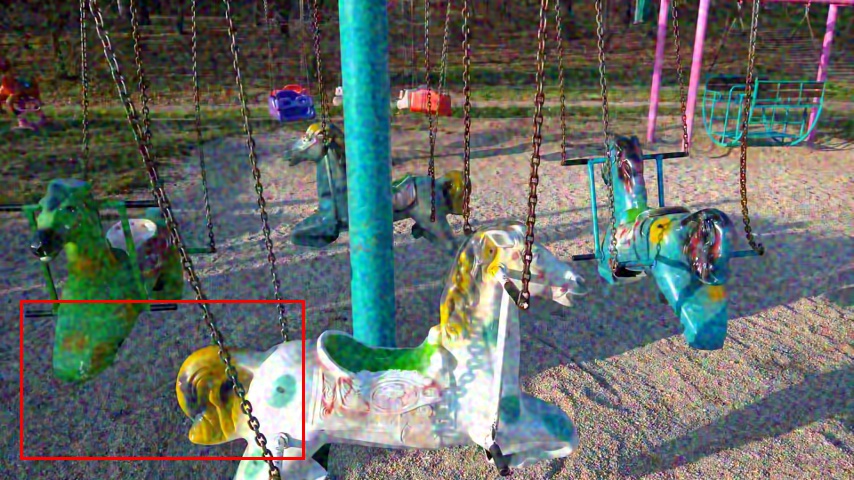}
		\caption{BM3D \\ 26.07 / 0.757}
		\label{fig:davis_2_0_s15_k4:bm3d_rect}
	\end{subfigure}
	\begin{subfigure}{0.18\textwidth}
	    \captionsetup{justification=centering}
		\includegraphics[width=\textwidth]{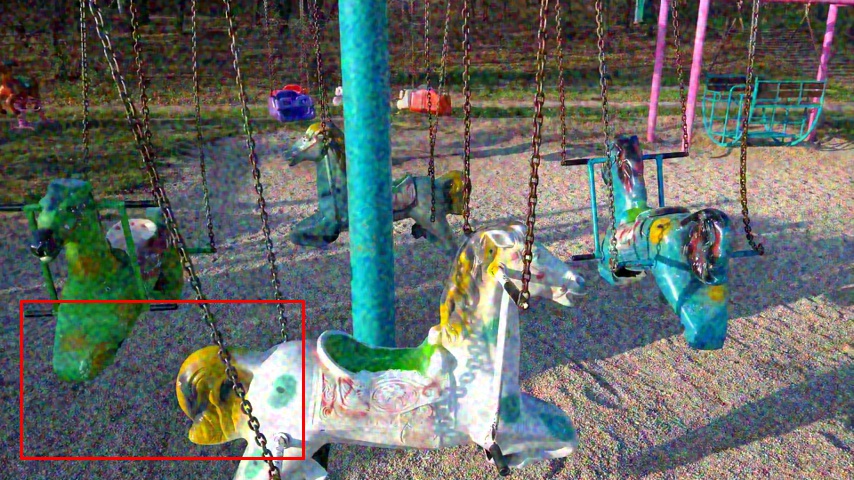}
		\caption{B-DnCNN \\ 25.73 / 0.762}
		\label{fig:davis_2_0_s15_k4:b_dncnn_rect}
	\end{subfigure}
	\begin{subfigure}{0.18\textwidth}
	    \captionsetup{justification=centering}
		\includegraphics[width=\textwidth]{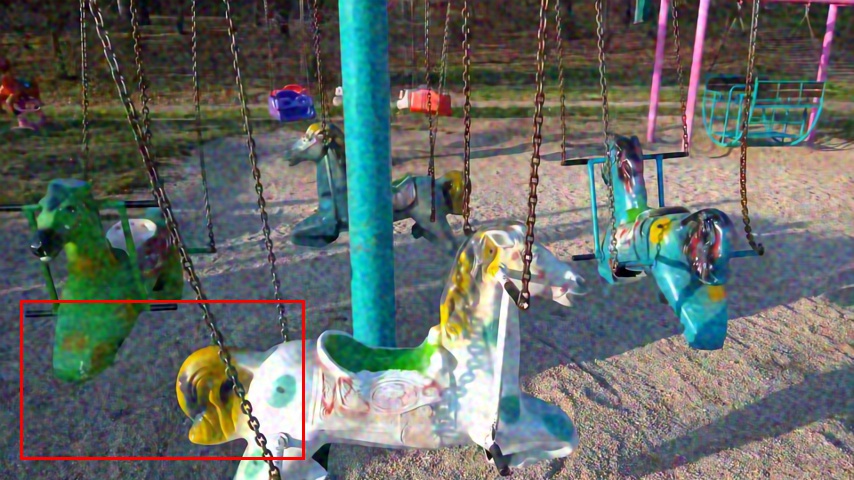}
		\caption{R2R \\ 26.69 / 0.756}
		\label{fig:davis_2_0_s15_k4:r2r_rect}
	\end{subfigure}
	\begin{subfigure}{0.18\textwidth}
	    \captionsetup{justification=centering}
		\includegraphics[width=\textwidth]{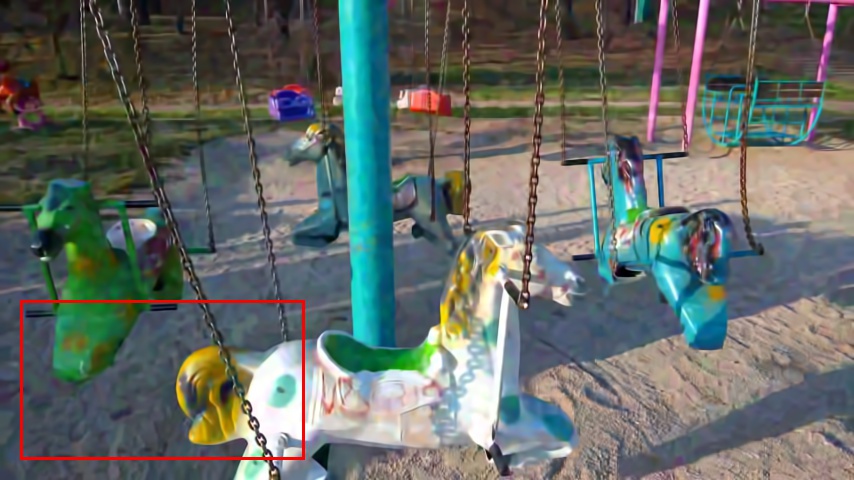}
		\caption{BM3D-O \\ 25.85 / 0.703}
		\label{fig:davis_2_0_s15_k4:bm3d_opt_rect}
	\end{subfigure}
	\begin{subfigure}{0.18\textwidth}
	    \captionsetup{justification=centering}
		\includegraphics[width=\textwidth]{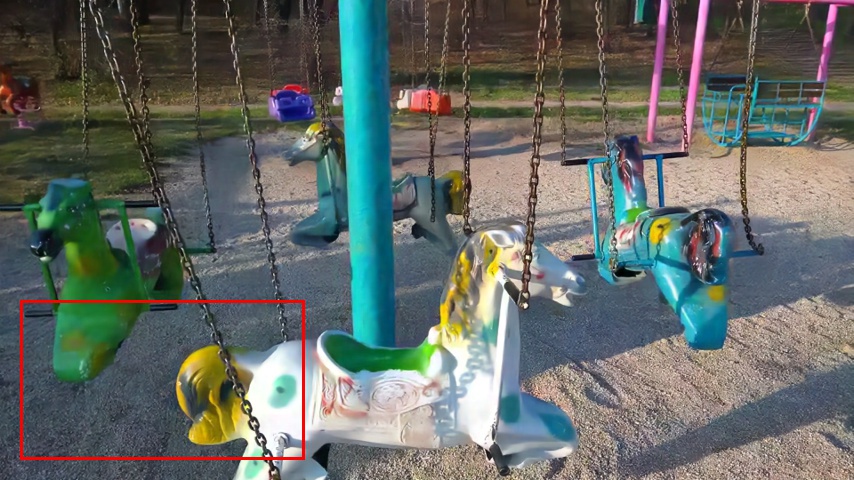}
		\caption{PC-UNet \\ 29.30 / 0.877}
		\label{fig:davis_2_0_s15_k4:pc_unet_rect}
	\end{subfigure}
	\begin{subfigure}{0.18\textwidth}
	    \captionsetup{justification=centering}
		\includegraphics[width=\textwidth]{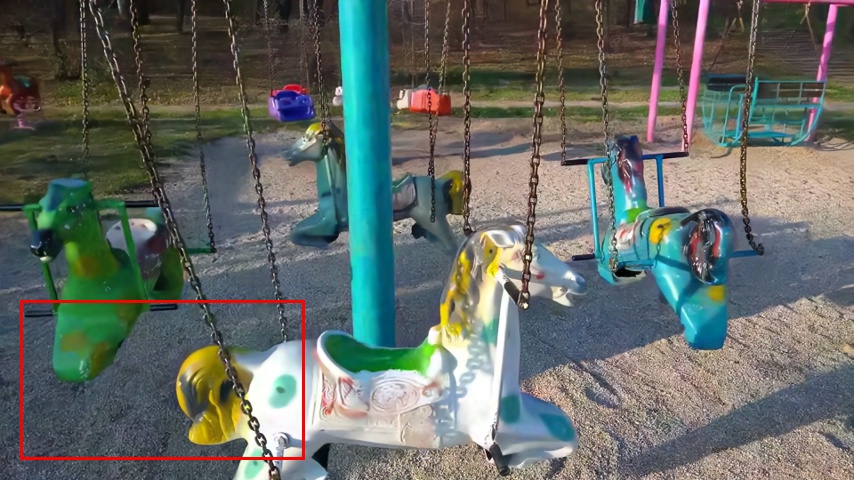}
		\caption{PC-DnCNN \\ 29.29 / 0.881}
		\label{fig:davis_2_0_s15_k4:pc_dncnn_rect}
	\end{subfigure}
	\begin{subfigure}{0.18\textwidth}
	    \captionsetup{justification=centering}
		\includegraphics[width=\textwidth]{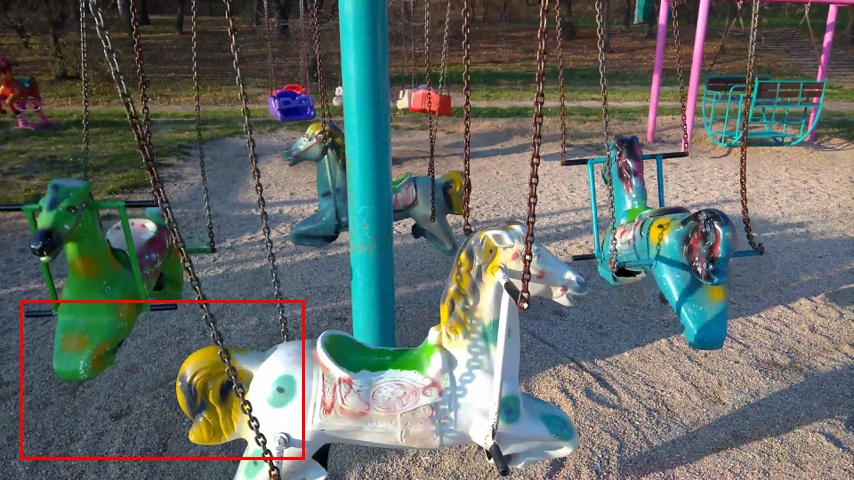}
		\caption{Clean \newline }
		\label{fig:davis_2_0_s15_k4:clean_rect}
	\end{subfigure}
	\begin{subfigure}{0.18\textwidth}
	    \captionsetup{justification=centering}
		\includegraphics[width=\textwidth]{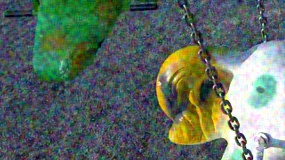}
		\caption{Noisy}
		\label{fig:davis_2_0_s15_k4:noisy_crop}
	\end{subfigure}
	\begin{subfigure}{0.18\textwidth}
	    \captionsetup{justification=centering}
		\includegraphics[width=\textwidth]{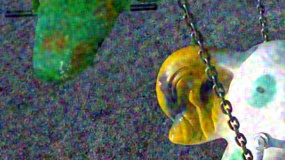}
		\caption{N2N}
		\label{fig:davis_2_0_s15_k4:n2n_crop}
	\end{subfigure}
	\begin{subfigure}{0.18\textwidth}
	    \captionsetup{justification=centering}
		\includegraphics[width=\textwidth]{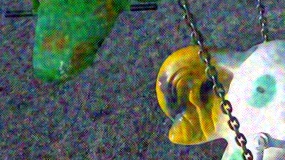}
		\caption{B2U}
		\label{fig:davis_2_0_s15_k4:b2u_crop}
	\end{subfigure}
	\begin{subfigure}{0.18\textwidth}
	    \captionsetup{justification=centering}
		\includegraphics[width=\textwidth]{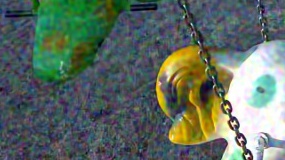}
		\caption{BM3D}
		\label{fig:davis_2_0_s15_k4:bm3d_crop}
	\end{subfigure}
	\begin{subfigure}{0.18\textwidth}
	    \captionsetup{justification=centering}
		\includegraphics[width=\textwidth]{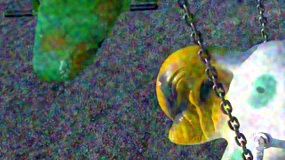}
		\caption{B-DnCNN}
		\label{fig:davis_2_0_s15_k4:b_dncnn_crop}
	\end{subfigure}
	\begin{subfigure}{0.18\textwidth}
	    \captionsetup{justification=centering}
		\includegraphics[width=\textwidth]{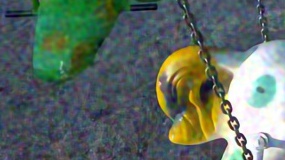}
		\caption{R2R}
		\label{fig:davis_2_0_s15_k4:r2r_crop}
	\end{subfigure}
	\begin{subfigure}{0.18\textwidth}
	    \captionsetup{justification=centering}
		\includegraphics[width=\textwidth]{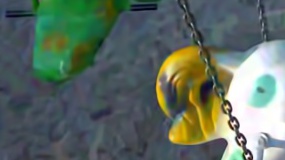}
		\caption{BM3D-O}
		\label{fig:davis_2_0_s15_k4:bm3d_opt_crop}
	\end{subfigure}
	\begin{subfigure}{0.18\textwidth}
	    \captionsetup{justification=centering}
		\includegraphics[width=\textwidth]{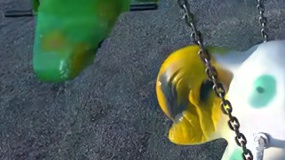}
		\caption{PC-UNet}
		\label{fig:davis_2_0_s15_k4:pc_unet_crop}
	\end{subfigure}
	\begin{subfigure}{0.18\textwidth}
	    \captionsetup{justification=centering}
		\includegraphics[width=\textwidth]{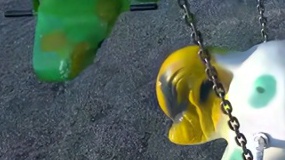}
		\caption{PC-DnCNN}
		\label{fig:davis_2_0_s15_k4:pc_dncnn_crop}
	\end{subfigure}
	\begin{subfigure}{0.18\textwidth}
	    \captionsetup{justification=centering}
		\includegraphics[width=\textwidth]{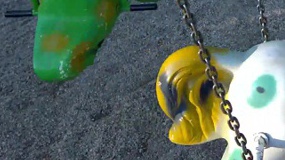}
		\caption{Clean}
		\label{fig:davis_2_0_s15_k4:clean_crop}
	\end{subfigure}
	\begin{subfigure}{0.18\textwidth}
	    \captionsetup{justification=centering}
		\includegraphics[width=\textwidth]{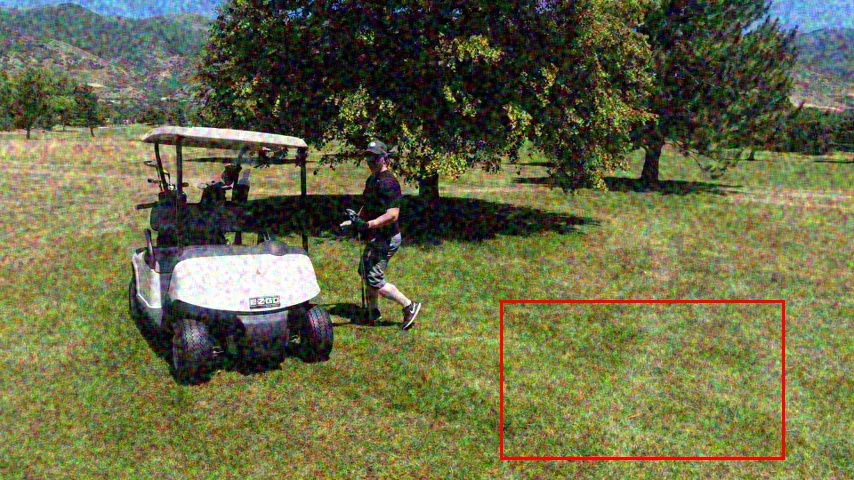}
		\caption{Noisy \\ 22.10 / 0.543}
		\label{fig:davis_8_1_s20_k4:noisy_rect}
	\end{subfigure}
	\begin{subfigure}{0.18\textwidth}
	    \captionsetup{justification=centering}
		\includegraphics[width=\textwidth]{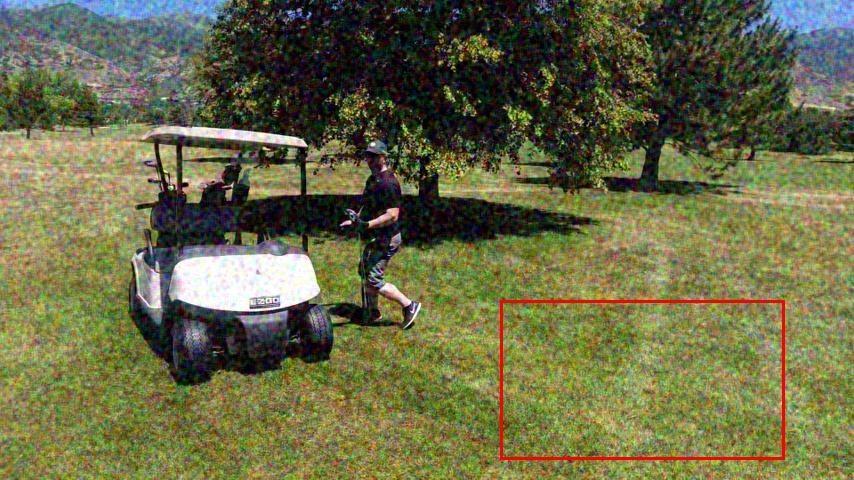}
		\caption{N2N \\ 22.73 / 0.569}
		\label{fig:davis_8_1_s20_k4:n2n_rect}
	\end{subfigure}
	\begin{subfigure}{0.18\textwidth}
	    \captionsetup{justification=centering}
		\includegraphics[width=\textwidth]{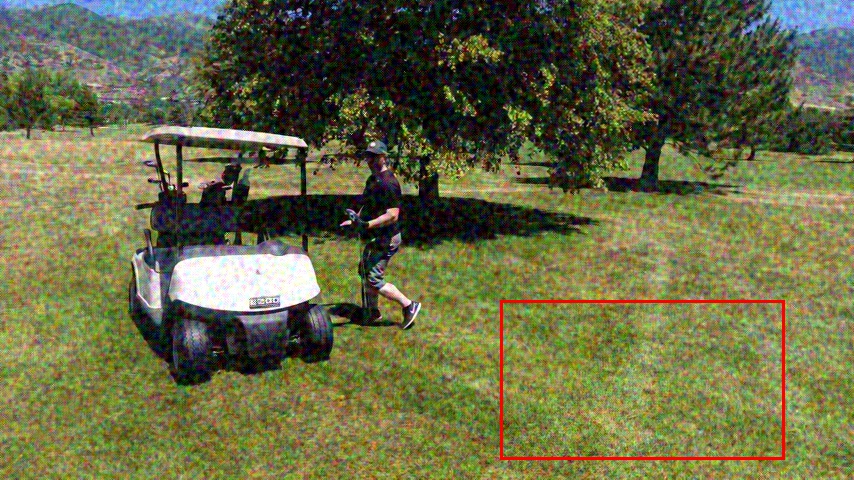}
		\caption{B2U \\ 21.98 / 0.482}
		\label{fig:davis_8_1_s20_k4:b2u_rect}
	\end{subfigure}
	\begin{subfigure}{0.18\textwidth}
	    \captionsetup{justification=centering}
		\includegraphics[width=\textwidth]{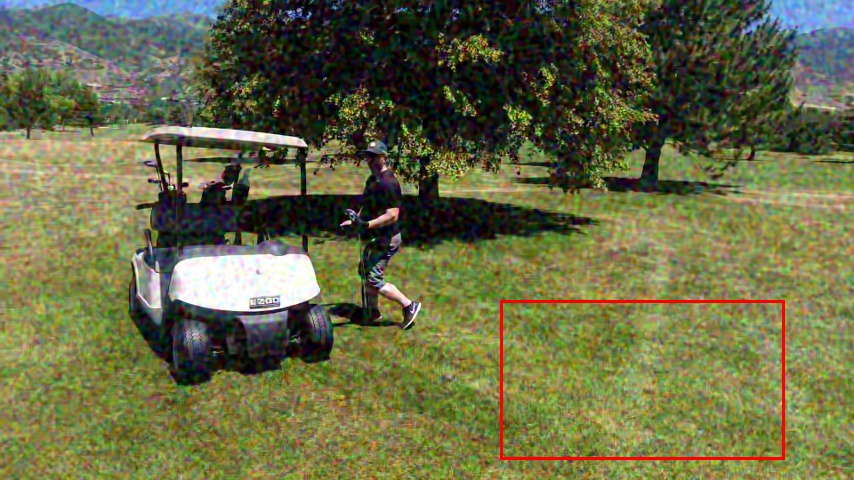}
		\caption{BM3D \\ 23.82 / 0.582}
		\label{fig:davis_8_1_s20_k4:bm3d_rect}
	\end{subfigure}
	\begin{subfigure}{0.18\textwidth}
	    \captionsetup{justification=centering}
		\includegraphics[width=\textwidth]{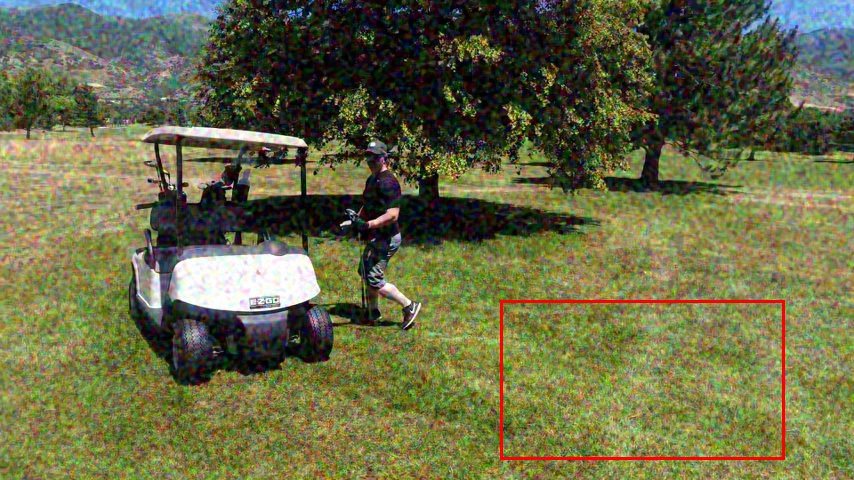}
		\caption{B-DnCNN \\ 23.36 / 0.608}
		\label{fig:davis_8_1_s20_k4:b_dncnn_rect}
	\end{subfigure}
	\begin{subfigure}{0.18\textwidth}
	    \captionsetup{justification=centering}
		\includegraphics[width=\textwidth]{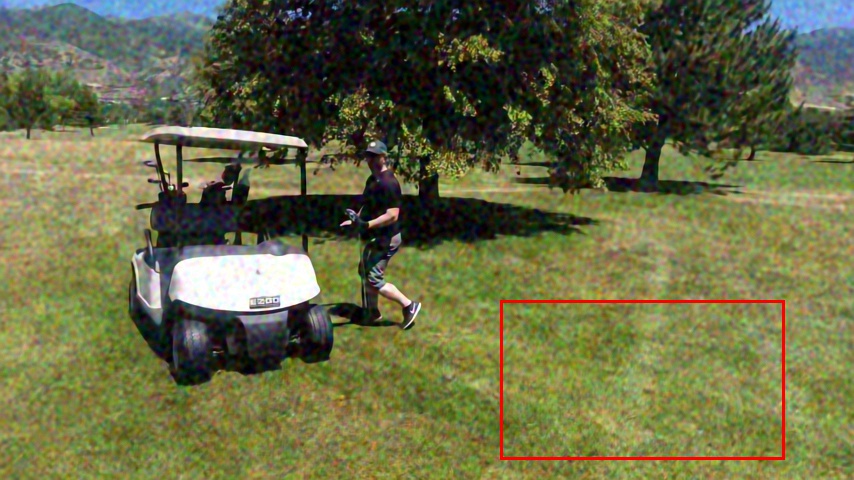}
		\caption{R2R \\ 24.64 / 0.562}
		\label{fig:davis_8_1_s20_k4:r2r_rect}
	\end{subfigure}
	\begin{subfigure}{0.18\textwidth}
	    \captionsetup{justification=centering}
		\includegraphics[width=\textwidth]{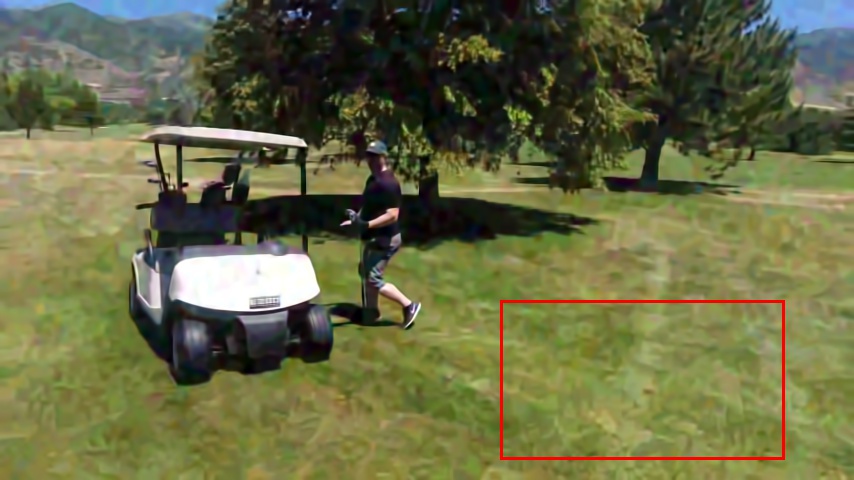}
		\caption{BM3D-O \\ 25.44 / 0.583}
		\label{fig:davis_8_1_s20_k4:bm3d_opt_rect}
	\end{subfigure}
	\begin{subfigure}{0.18\textwidth}
	    \captionsetup{justification=centering}
		\includegraphics[width=\textwidth]{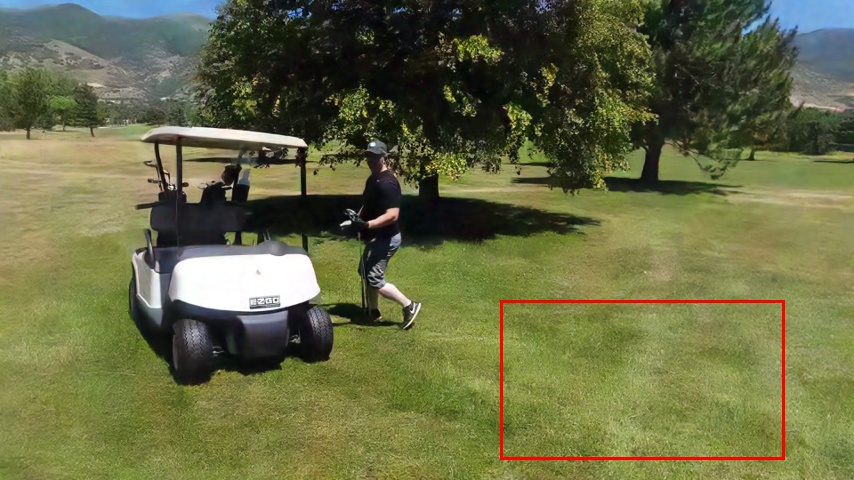}
		\caption{PC-UNet \\ 29.26 / 0.842}
		\label{fig:davis_8_1_s20_k4:pc_unet_rect}
	\end{subfigure}
	\begin{subfigure}{0.18\textwidth}
	    \captionsetup{justification=centering}
		\includegraphics[width=\textwidth]{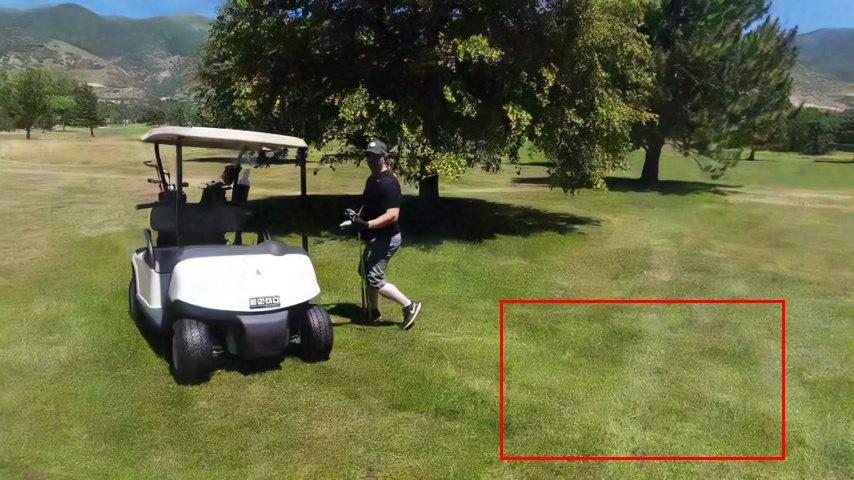}
		\caption{PC-DnCNN \\ 29.37 / 0.844}
		\label{fig:davis_8_1_s20_k4:pc_dncnn_rect}
	\end{subfigure}
	\begin{subfigure}{0.18\textwidth}
	    \captionsetup{justification=centering}
		\includegraphics[width=\textwidth]{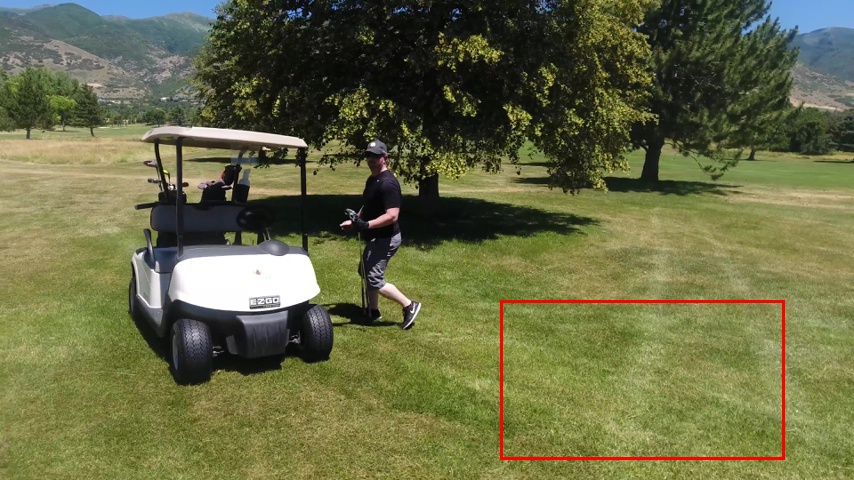}
		\caption{Clean \newline }
		\label{fig:davis_8_1_s20_k4:clean_rect}
	\end{subfigure}
	\begin{subfigure}{0.18\textwidth}
	    \captionsetup{justification=centering}
		\includegraphics[width=\textwidth]{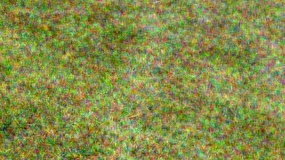}
		\caption{Noisy}
		\label{fig:davis_8_1_s20_k4:noisy_crop}
	\end{subfigure}
	\begin{subfigure}{0.18\textwidth}
	    \captionsetup{justification=centering}
		\includegraphics[width=\textwidth]{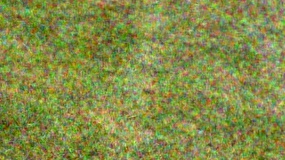}
		\caption{N2N}
		\label{fig:davis_8_1_s20_k4:n2n_crop}
	\end{subfigure}
	\begin{subfigure}{0.18\textwidth}
	    \captionsetup{justification=centering}
		\includegraphics[width=\textwidth]{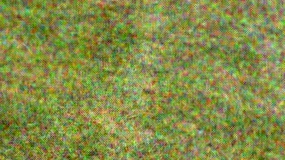}
		\caption{B2U}
		\label{fig:davis_8_1_s20_k4:b2u_crop}
	\end{subfigure}
	\begin{subfigure}{0.18\textwidth}
	    \captionsetup{justification=centering}
		\includegraphics[width=\textwidth]{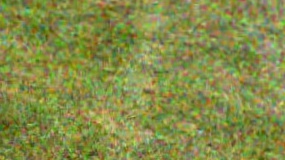}
		\caption{BM3D}
		\label{fig:davis_8_1_s20_k4:bm3d_crop}
	\end{subfigure}
	\begin{subfigure}{0.18\textwidth}
	    \captionsetup{justification=centering}
		\includegraphics[width=\textwidth]{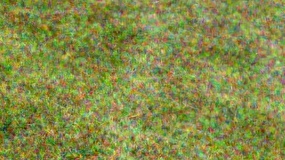}
		\caption{B-DnCNN}
		\label{fig:davis_8_1_s20_k4:b_dncnn_crop}
	\end{subfigure}
	\begin{subfigure}{0.18\textwidth}
	    \captionsetup{justification=centering}
		\includegraphics[width=\textwidth]{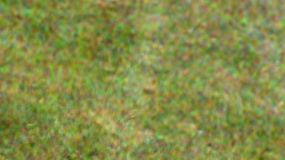}
		\caption{R2R}
		\label{fig:davis_8_1_s20_k4:r2r_crop}
	\end{subfigure}
	\begin{subfigure}{0.18\textwidth}
	    \captionsetup{justification=centering}
		\includegraphics[width=\textwidth]{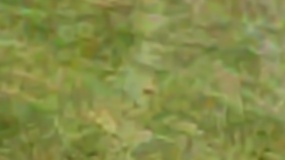}
		\caption{BM3D-O}
		\label{fig:davis_8_1_s20_k4:bm3d_opt_crop}
	\end{subfigure}
	\begin{subfigure}{0.18\textwidth}
	    \captionsetup{justification=centering}
		\includegraphics[width=\textwidth]{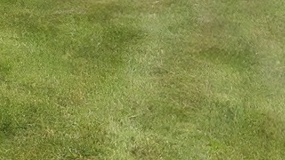}
		\caption{PC-UNet}
		\label{fig:davis_8_1_s20_k4:pc_unet_crop}
	\end{subfigure}
	\begin{subfigure}{0.18\textwidth}
	    \captionsetup{justification=centering}
		\includegraphics[width=\textwidth]{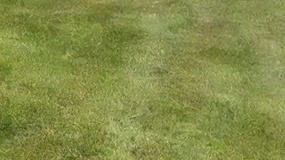}
		\caption{PC-DnCNN}
		\label{fig:davis_8_1_s20_k4:pc_dncnn_crop}
	\end{subfigure}
	\begin{subfigure}{0.18\textwidth}
	    \captionsetup{justification=centering}
		\includegraphics[width=\textwidth]{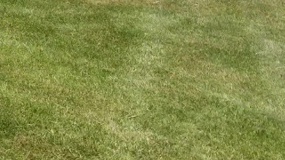}
		\caption{Clean}
		\label{fig:davis_8_1_s20_k4:clean_crop}
	\end{subfigure}
	\caption{Denoising examples with correlated Gaussian noise. The first four rows show frame 5 of the sequence \emph{carousel} with $\sigma = 15$ and $k = 4$. The last four rows present frame 13 of the sequence \emph{golf} with $\sigma = 20$ and $k = 4$. As can be seen, oracle BM3D produces blurred images while leaving a substantial amount of low-frequency noise unfiltered. Other algorithms, except ours (PC-UNet and PC-DnCNN), do not succeed in removing the noise.}
	\label{fig:davis_2_0_s15_k4_8_1_s20_k4}
\end{figure*}

\begin{figure*}
    \centering
	\begin{subfigure}{0.18\textwidth}
	    \captionsetup{justification=centering}
		\includegraphics[width=\textwidth]{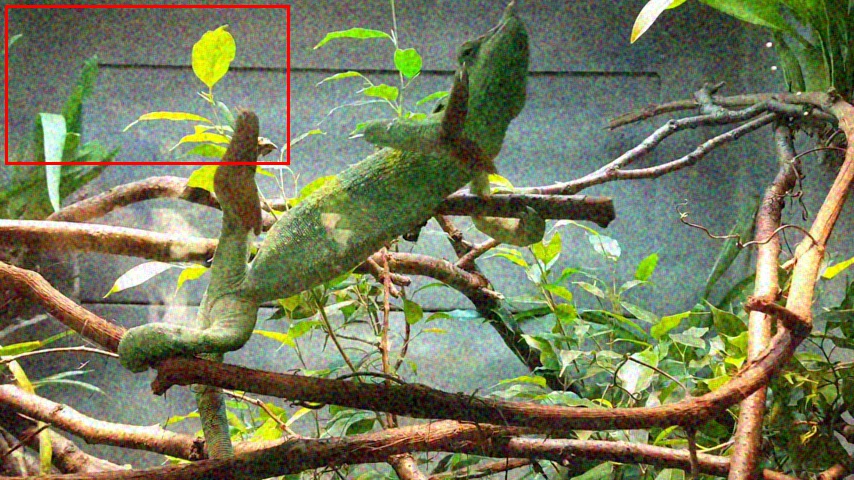}
		\caption{Noisy \\ 24.59 / 0.611}
		\label{fig:davis_4_2_s15_k3:noisy_rect}
	\end{subfigure}
	\begin{subfigure}{0.18\textwidth}
	    \captionsetup{justification=centering}
		\includegraphics[width=\textwidth]{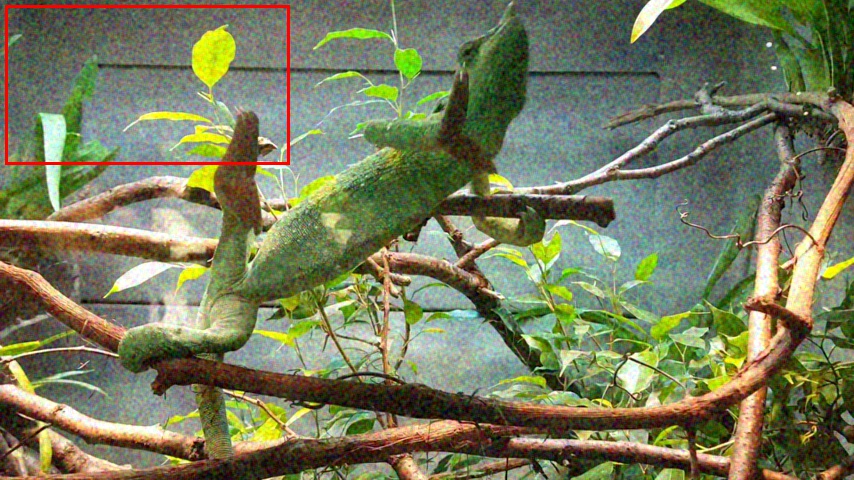}
		\caption{N2N \\ 25.22 / 0.635}
		\label{fig:davis_4_2_s15_k3:n2n_rect}
	\end{subfigure}
	\begin{subfigure}{0.18\textwidth}
	    \captionsetup{justification=centering}
		\includegraphics[width=\textwidth]{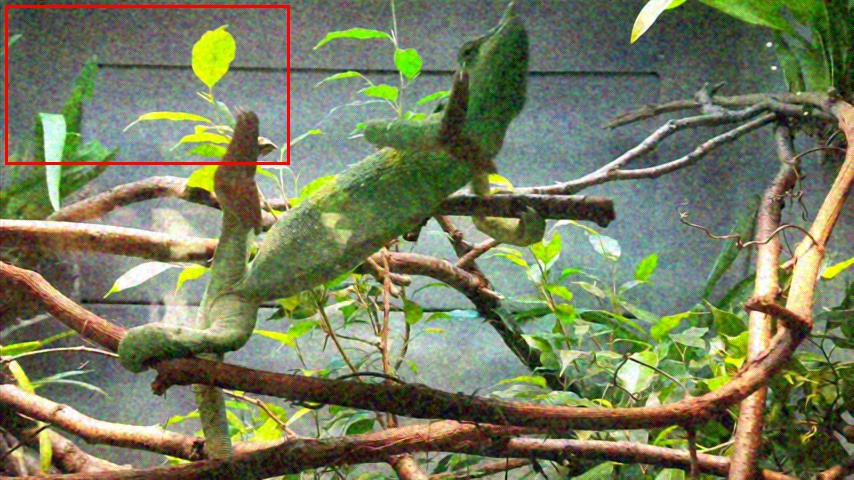}
		\caption{B2U \\ 24.35 / 0.599}
		\label{fig:davis_4_2_s15_k3:b2u_rect}
	\end{subfigure}
	\begin{subfigure}{0.18\textwidth}
	    \captionsetup{justification=centering}
		\includegraphics[width=\textwidth]{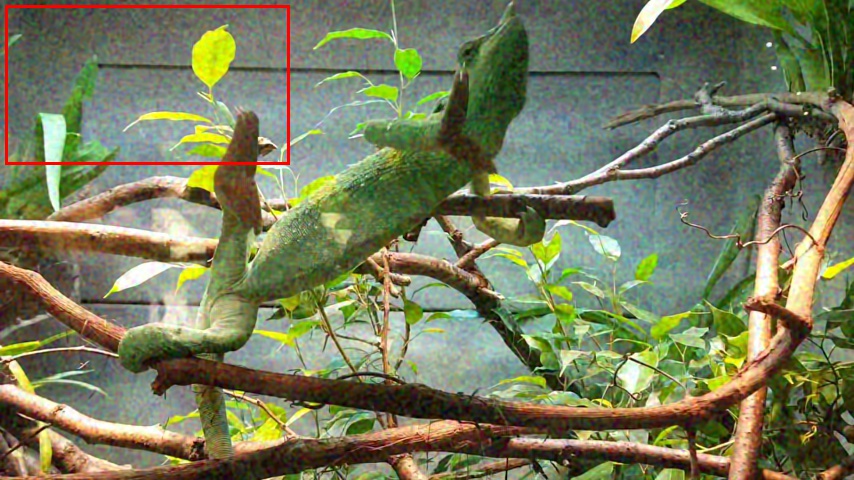}
		\caption{BM3D \\ 27.28 / 0.738}
		\label{fig:davis_4_2_s15_k3:bm3d_rect}
	\end{subfigure}
	\begin{subfigure}{0.18\textwidth}
	    \captionsetup{justification=centering}
		\includegraphics[width=\textwidth]{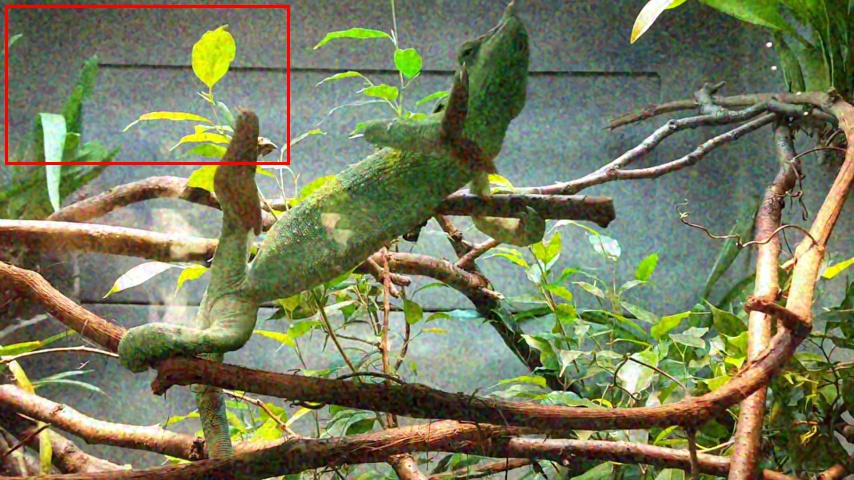}
		\caption{B-DnCNN \\ 26.79 / 0.713}
		\label{fig:davis_4_2_s15_k3:b_dncnn_rect}
	\end{subfigure}
	\begin{subfigure}{0.18\textwidth}
	    \captionsetup{justification=centering}
		\includegraphics[width=\textwidth]{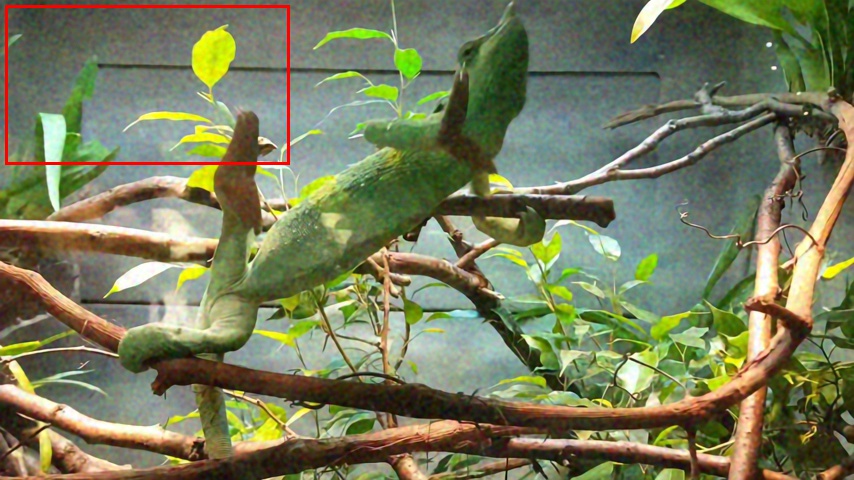}
		\caption{R2R \\ 28.86 / 0.807}
		\label{fig:davis_4_2_s15_k3:r2r_rect}
	\end{subfigure}
	\begin{subfigure}{0.18\textwidth}
	    \captionsetup{justification=centering}
		\includegraphics[width=\textwidth]{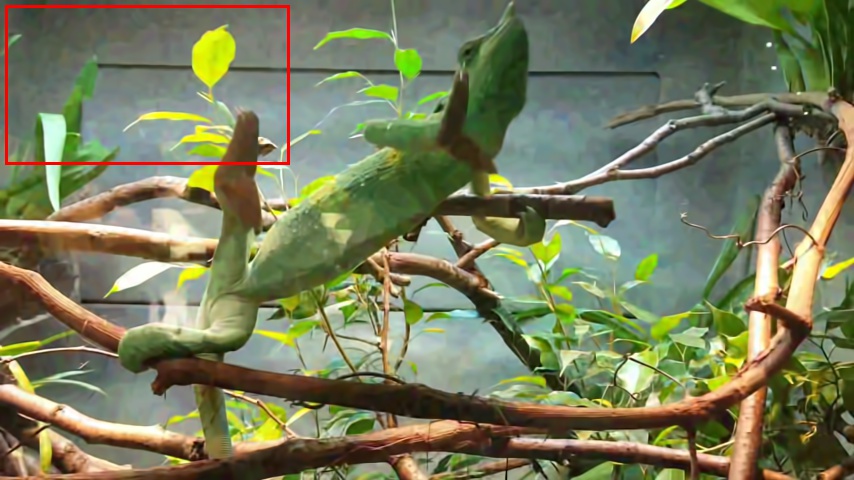}
		\caption{BM3D-O \\ 29.52 / 0.884}
		\label{fig:davis_4_2_s15_k3:bm3d_opt_rect}
	\end{subfigure}
	\begin{subfigure}{0.18\textwidth}
	    \captionsetup{justification=centering}
		\includegraphics[width=\textwidth]{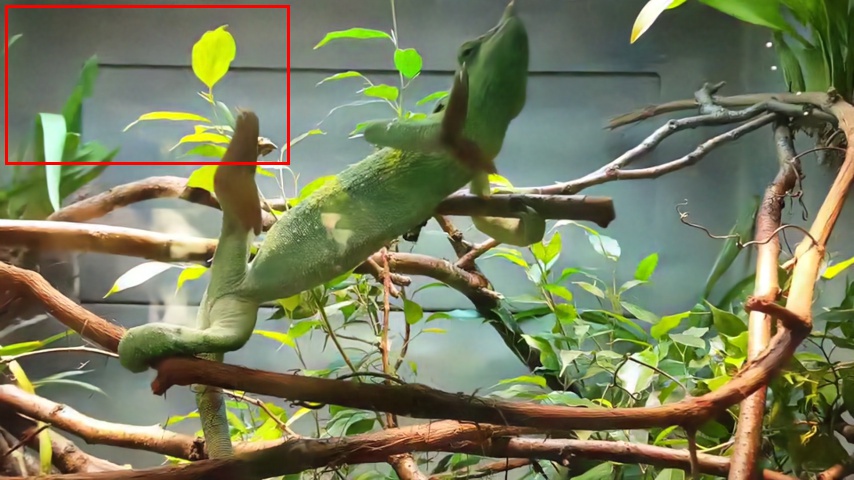}
		\caption{PC-UNet \\ 31.21 / 0.923}
		\label{fig:davis_4_2_s15_k3:pc_unet_rect}
	\end{subfigure}
	\begin{subfigure}{0.18\textwidth}
	    \captionsetup{justification=centering}
		\includegraphics[width=\textwidth]{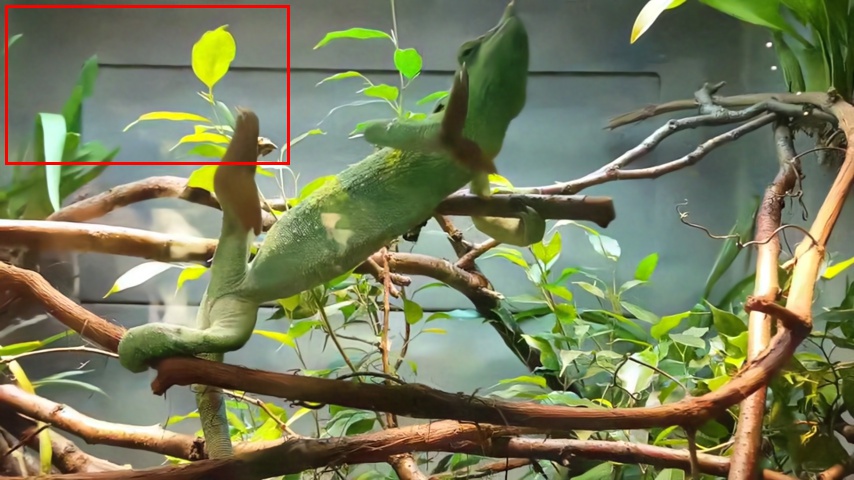}
		\caption{PC-DnCNN \\ 31.45 / 0.926}
		\label{fig:davis_4_2_s15_k3:pc_dncnn_rect}
	\end{subfigure}
	\begin{subfigure}{0.18\textwidth}
	    \captionsetup{justification=centering}
		\includegraphics[width=\textwidth]{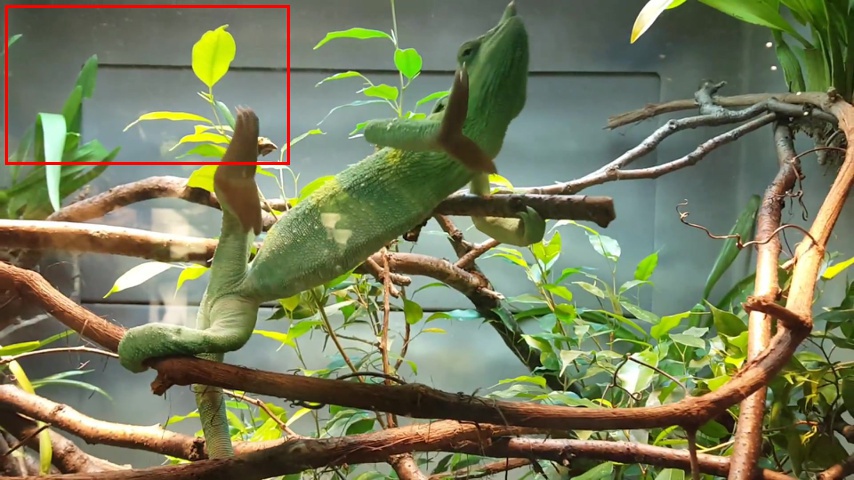}
		\caption{Clean \newline }
		\label{fig:davis_4_2_s15_k3:clean_rect}
	\end{subfigure}
	\begin{subfigure}{0.18\textwidth}
	    \captionsetup{justification=centering}
		\includegraphics[width=\textwidth]{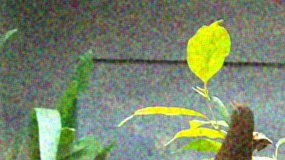}
		\caption{Noisy}
		\label{fig:davis_4_2_s15_k3:noisy_crop}
	\end{subfigure}
	\begin{subfigure}{0.18\textwidth}
	    \captionsetup{justification=centering}
		\includegraphics[width=\textwidth]{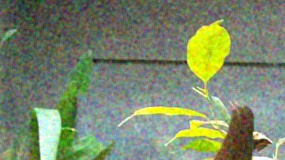}
		\caption{N2N}
		\label{fig:davis_4_2_s15_k3:n2n_crop}
	\end{subfigure}
	\begin{subfigure}{0.18\textwidth}
	    \captionsetup{justification=centering}
		\includegraphics[width=\textwidth]{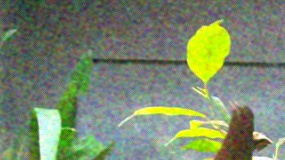}
		\caption{B2U}
		\label{fig:davis_4_2_s15_k3:b2u_crop}
	\end{subfigure}
	\begin{subfigure}{0.18\textwidth}
	    \captionsetup{justification=centering}
		\includegraphics[width=\textwidth]{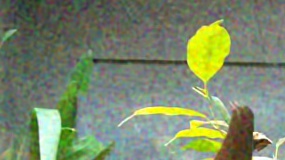}
		\caption{BM3D}
		\label{fig:davis_4_2_s15_k3:bm3d_crop}
	\end{subfigure}
	\begin{subfigure}{0.18\textwidth}
	    \captionsetup{justification=centering}
		\includegraphics[width=\textwidth]{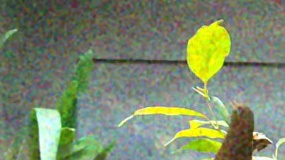}
		\caption{B-DnCNN}
		\label{fig:davis_4_2_s15_k3:b_dncnn_crop}
	\end{subfigure}
	\begin{subfigure}{0.18\textwidth}
	    \captionsetup{justification=centering}
		\includegraphics[width=\textwidth]{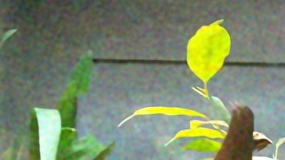}
		\caption{R2R}
		\label{fig:davis_4_2_s15_k3:r2r_crop}
	\end{subfigure}
	\begin{subfigure}{0.18\textwidth}
	    \captionsetup{justification=centering}
		\includegraphics[width=\textwidth]{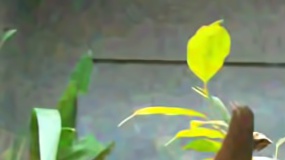}
		\caption{BM3D-O}
		\label{fig:davis_4_2_s15_k3:bm3d_opt_crop}
	\end{subfigure}
	\begin{subfigure}{0.18\textwidth}
	    \captionsetup{justification=centering}
		\includegraphics[width=\textwidth]{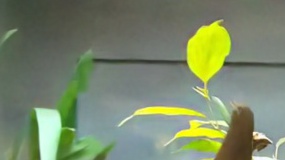}
		\caption{PC-UNet}
		\label{fig:davis_4_2_s15_k3:pc_unet_crop}
	\end{subfigure}
	\begin{subfigure}{0.18\textwidth}
	    \captionsetup{justification=centering}
		\includegraphics[width=\textwidth]{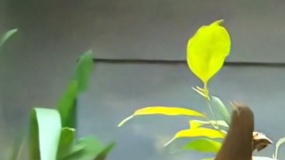}
		\caption{PC-DnCNN}
		\label{fig:davis_4_2_s15_k3:pc_dncnn_crop}
	\end{subfigure}
	\begin{subfigure}{0.18\textwidth}
	    \captionsetup{justification=centering}
		\includegraphics[width=\textwidth]{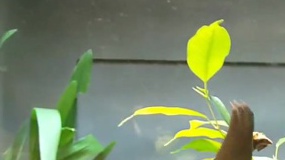}
		\caption{Clean}
		\label{fig:davis_4_2_s15_k3:clean_crop}
	\end{subfigure}
	\begin{subfigure}{0.18\textwidth}
	    \captionsetup{justification=centering}
		\includegraphics[width=\textwidth]{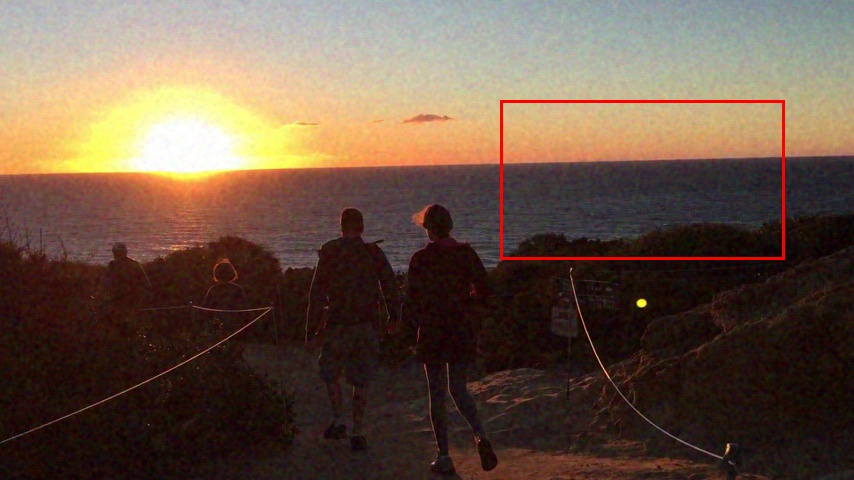}
		\caption{Noisy \\ 34.16 / 0.796}
		\label{fig:davis_19_0_s5_k4:noisy_rect}
	\end{subfigure}
	\begin{subfigure}{0.18\textwidth}
	    \captionsetup{justification=centering}
		\includegraphics[width=\textwidth]{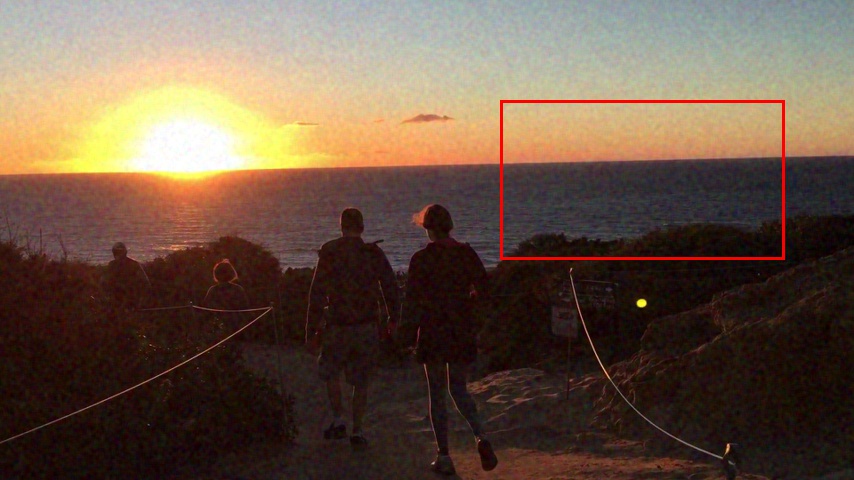}
		\caption{N2N \\ 34.54 / 0.813}
		\label{fig:davis_19_0_s5_k4:n2n_rect}
	\end{subfigure}
	\begin{subfigure}{0.18\textwidth}
	    \captionsetup{justification=centering}
		\includegraphics[width=\textwidth]{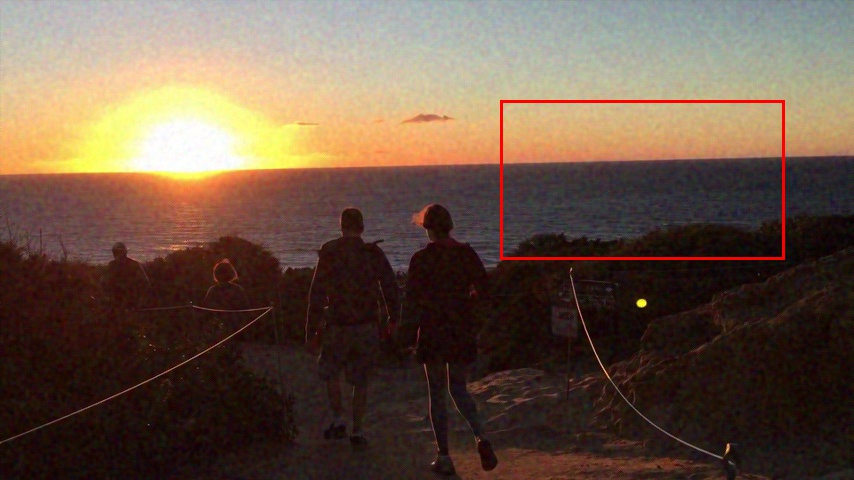}
		\caption{B2U \\ 31.79 / 0.674}
		\label{fig:davis_19_0_s5_k4:b2u_rect}
	\end{subfigure}
	\begin{subfigure}{0.18\textwidth}
	    \captionsetup{justification=centering}
		\includegraphics[width=\textwidth]{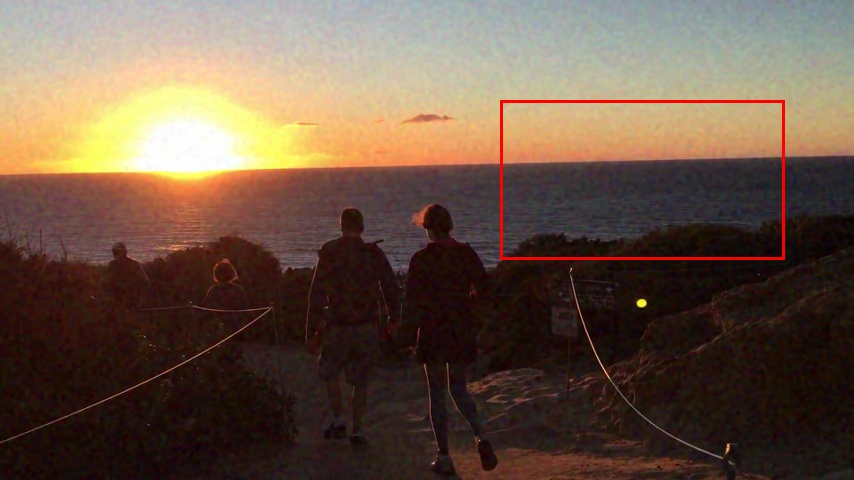}
		\caption{BM3D \\ 36.21 / 0.876}
		\label{fig:davis_19_0_s5_k4:bm3d_rect}
	\end{subfigure}
	\begin{subfigure}{0.18\textwidth}
	    \captionsetup{justification=centering}
		\includegraphics[width=\textwidth]{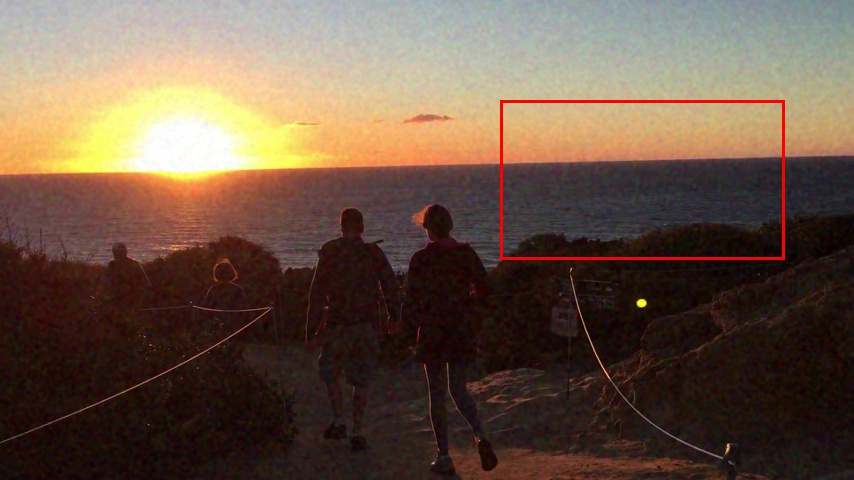}
		\caption{B-DnCNN \\ 35.50 / 0.852}
		\label{fig:davis_19_0_s5_k4:b_dncnn_rect}
	\end{subfigure}
	\begin{subfigure}{0.18\textwidth}
	    \captionsetup{justification=centering}
		\includegraphics[width=\textwidth]{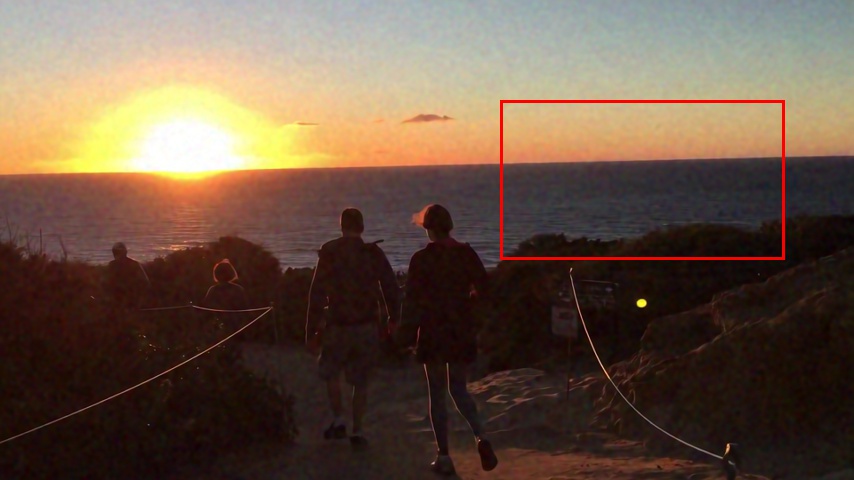}
		\caption{R2R \\ 37.60 / 0.910}
		\label{fig:davis_19_0_s5_k4:r2r_rect}
	\end{subfigure}
	\begin{subfigure}{0.18\textwidth}
	    \captionsetup{justification=centering}
		\includegraphics[width=\textwidth]{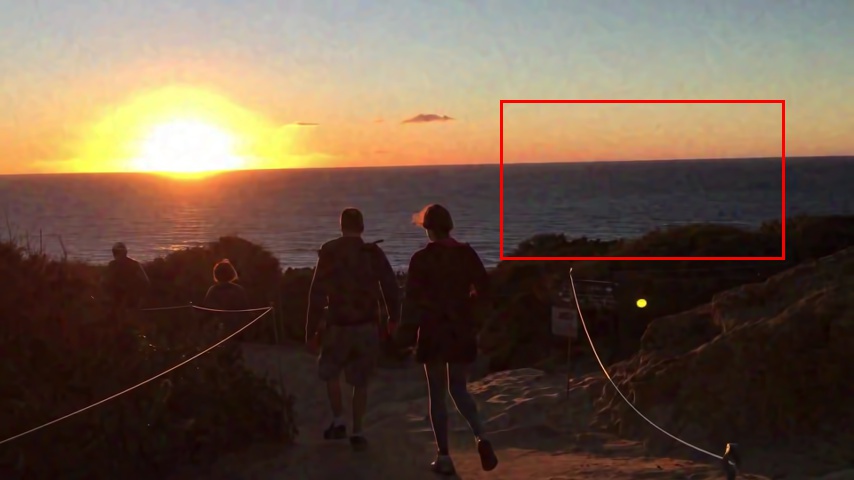}
		\caption{BM3D-O \\ 39.10 / 0.944}
		\label{fig:davis_19_0_s5_k4:bm3d_opt_rect}
	\end{subfigure}
	\begin{subfigure}{0.18\textwidth}
	    \captionsetup{justification=centering}
		\includegraphics[width=\textwidth]{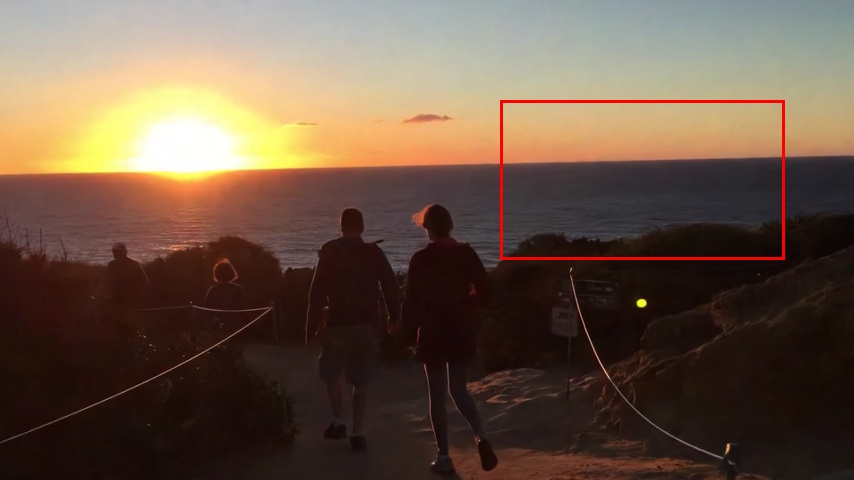}
		\caption{PC-UNet \\ 40.93 / 0.964}
		\label{fig:davis_19_0_s5_k4:pc_unet_rect}
	\end{subfigure}
	\begin{subfigure}{0.18\textwidth}
	    \captionsetup{justification=centering}
		\includegraphics[width=\textwidth]{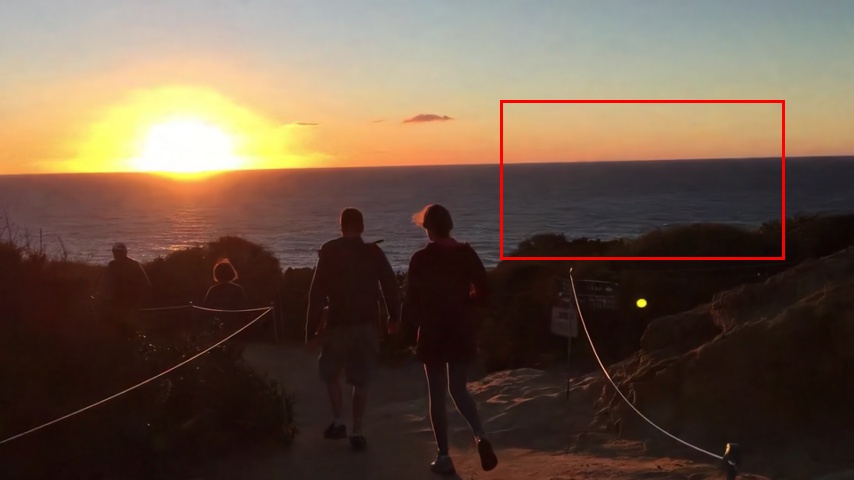}
		\caption{PC-DnCNN \\ 41.10 / 0.965}
		\label{fig:davis_19_0_s5_k4:pc_dncnn_rect}
	\end{subfigure}
	\begin{subfigure}{0.18\textwidth}
	    \captionsetup{justification=centering}
		\includegraphics[width=\textwidth]{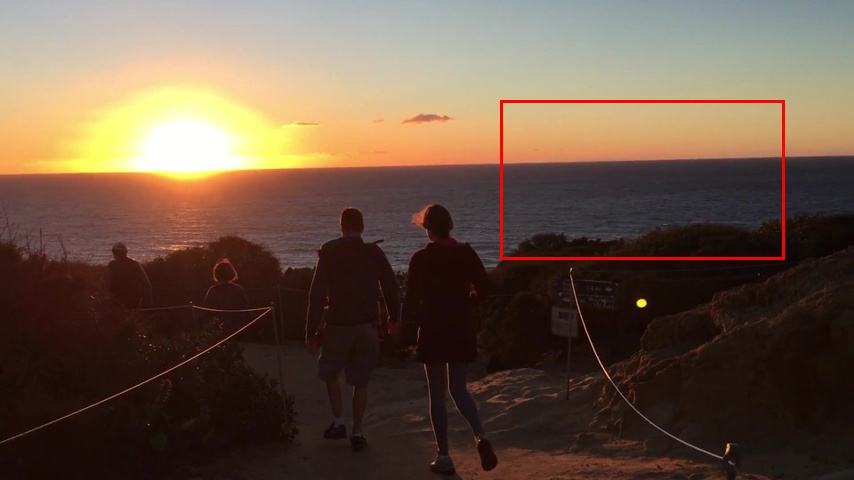}
		\caption{Clean \newline }
		\label{fig:davis_19_0_s5_k4:clean_rect}
	\end{subfigure}
	\begin{subfigure}{0.18\textwidth}
	    \captionsetup{justification=centering}
		\includegraphics[width=\textwidth]{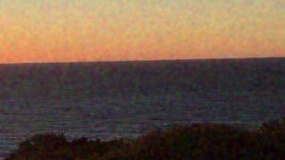}
		\caption{Noisy}
		\label{fig:davis_19_0_s5_k4:noisy_crop}
	\end{subfigure}
	\begin{subfigure}{0.18\textwidth}
	    \captionsetup{justification=centering}
		\includegraphics[width=\textwidth]{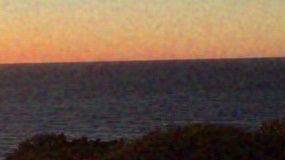}
		\caption{N2N}
		\label{fig:davis_19_0_s5_k4:n2n_crop}
	\end{subfigure}
	\begin{subfigure}{0.18\textwidth}
	    \captionsetup{justification=centering}
		\includegraphics[width=\textwidth]{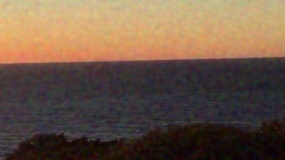}
		\caption{B2U}
		\label{fig:davis_19_0_s5_k4:b2u_crop}
	\end{subfigure}
	\begin{subfigure}{0.18\textwidth}
	    \captionsetup{justification=centering}
		\includegraphics[width=\textwidth]{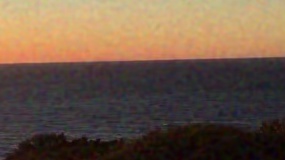}
		\caption{BM3D}
		\label{fig:davis_19_0_s5_k4:bm3d_crop}
	\end{subfigure}
	\begin{subfigure}{0.18\textwidth}
	    \captionsetup{justification=centering}
		\includegraphics[width=\textwidth]{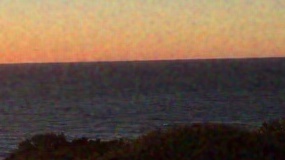}
		\caption{B-DnCNN}
		\label{fig:davis_19_0_s5_k4:b_dncnn_crop}
	\end{subfigure}
	\begin{subfigure}{0.18\textwidth}
	    \captionsetup{justification=centering}
		\includegraphics[width=\textwidth]{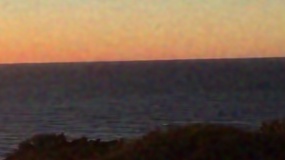}
		\caption{R2R}
		\label{fig:davis_19_0_s5_k4:rt2_crop}
	\end{subfigure}
	\begin{subfigure}{0.18\textwidth}
	    \captionsetup{justification=centering}
		\includegraphics[width=\textwidth]{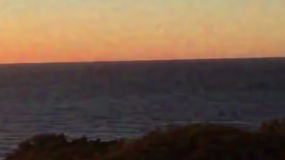}
		\caption{BM3D-O}
		\label{fig:davis_19_0_s5_k4:bm3d_opt_crop}
	\end{subfigure}
	\begin{subfigure}{0.18\textwidth}
	    \captionsetup{justification=centering}
		\includegraphics[width=\textwidth]{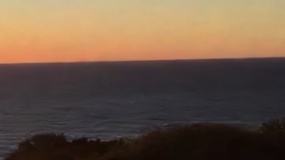}
		\caption{PC-UNet}
		\label{fig:davis_19_0_s5_k4:pc_unet_crop}
	\end{subfigure}
	\begin{subfigure}{0.18\textwidth}
	    \captionsetup{justification=centering}
		\includegraphics[width=\textwidth]{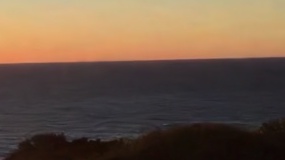}
		\caption{PC-DnCNN}
		\label{fig:davis_19_0_s5_k4:pc_dncnn_crop}
	\end{subfigure}
	\begin{subfigure}{0.18\textwidth}
	    \captionsetup{justification=centering}
		\includegraphics[width=\textwidth]{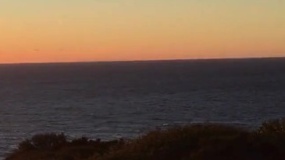}
		\caption{Clean}
		\label{fig:davis_19_0_s5_k4:clean_crop}
	\end{subfigure}
	\caption{Denoising examples with correlated Gaussian noise. The first four rows show frame 24 of the sequence \emph{chameleon} with $\sigma = 15$ and $k = 3$. The last four rows present frame 5 of the sequence \emph{people-sunset} with $\sigma = 5$ and $k = 4$. As can be seen, oracle BM3D leaves a substantial amount of low-frequency noise unfiltered, while other algorithms, except ours (PC-UNet and PC-DnCNN), do not succeed in removing the noise.}
	\label{fig:davis_4_2_s15_k3_19_0_s5_k4}
\end{figure*}

\begin{figure*}
    \centering
	\begin{subfigure}{0.18\textwidth}
	    \captionsetup{justification=centering}
		\includegraphics[width=\textwidth]{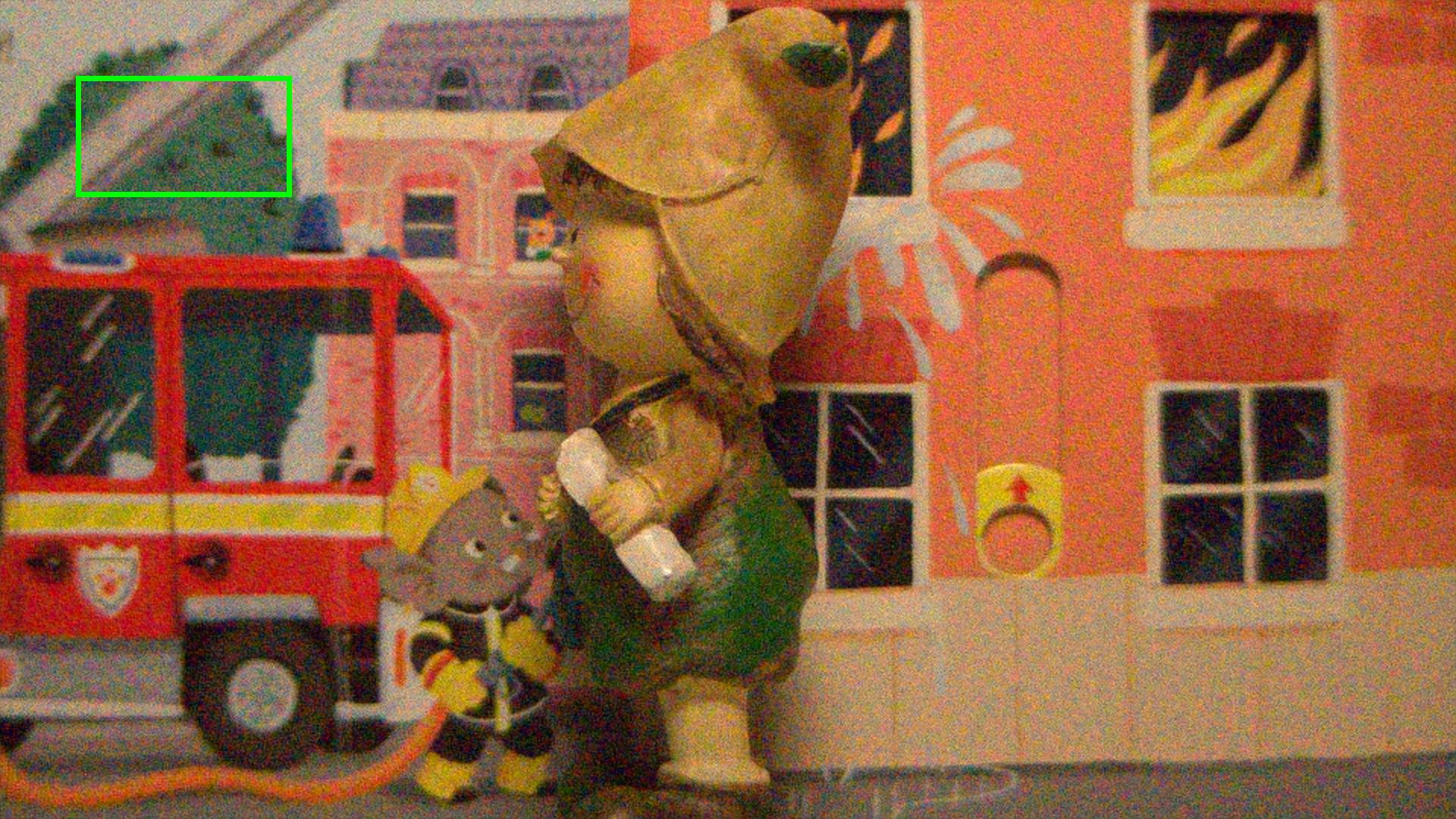}
		\caption{Noisy \\ 25.01 / 0.442}
		\label{fig:crvd_8_5_iso25600:noisy_rect}
	\end{subfigure}
	\begin{subfigure}{0.18\textwidth}
	    \captionsetup{justification=centering}
		\includegraphics[width=\textwidth]{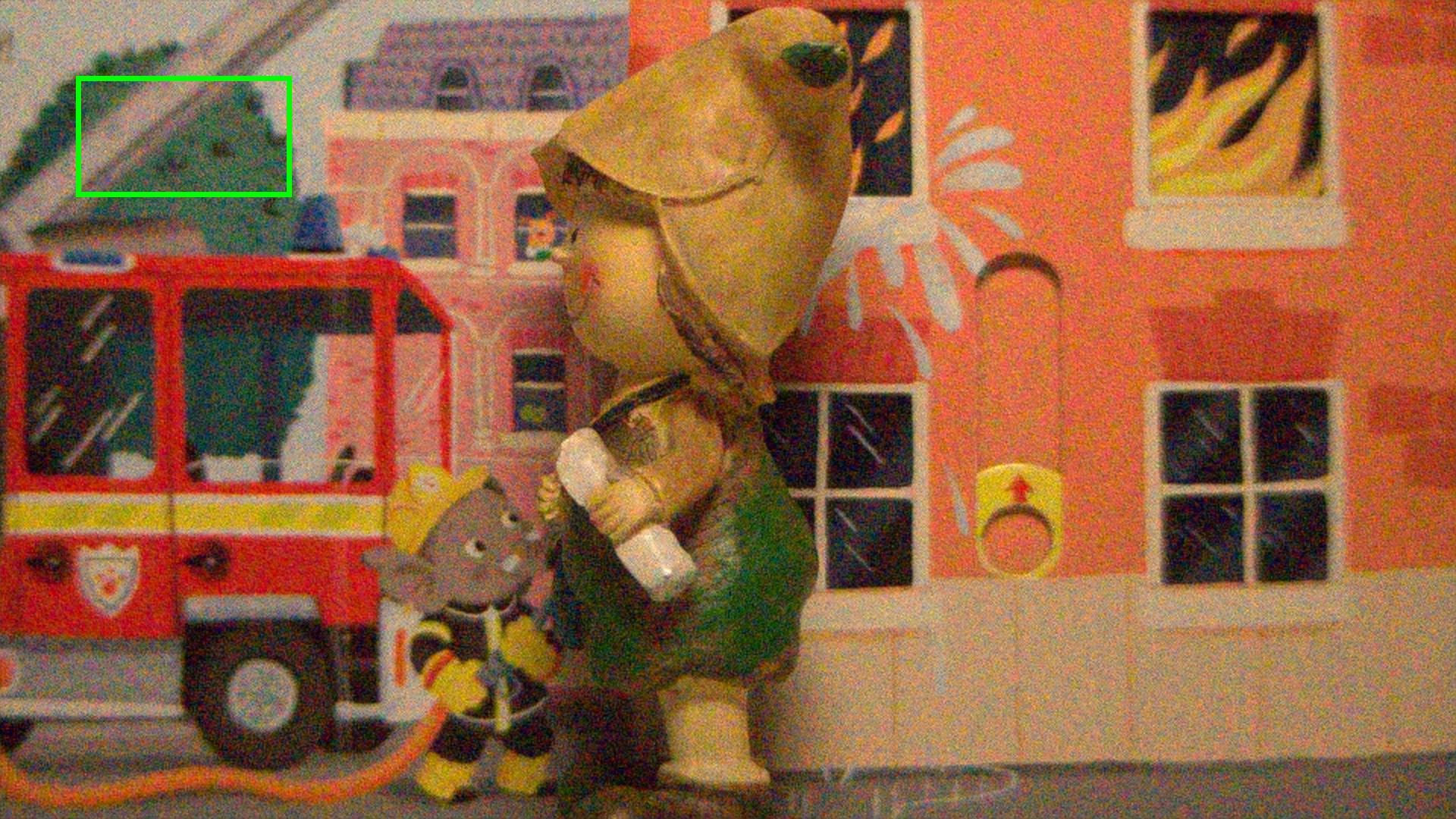}
		\caption{N2N \\ 25.09 / 0.449}
		\label{fig:crvd_8_5_iso25600:n2n2_rect}
	\end{subfigure}
	\begin{subfigure}{0.18\textwidth}
	    \captionsetup{justification=centering}
		\includegraphics[width=\textwidth]{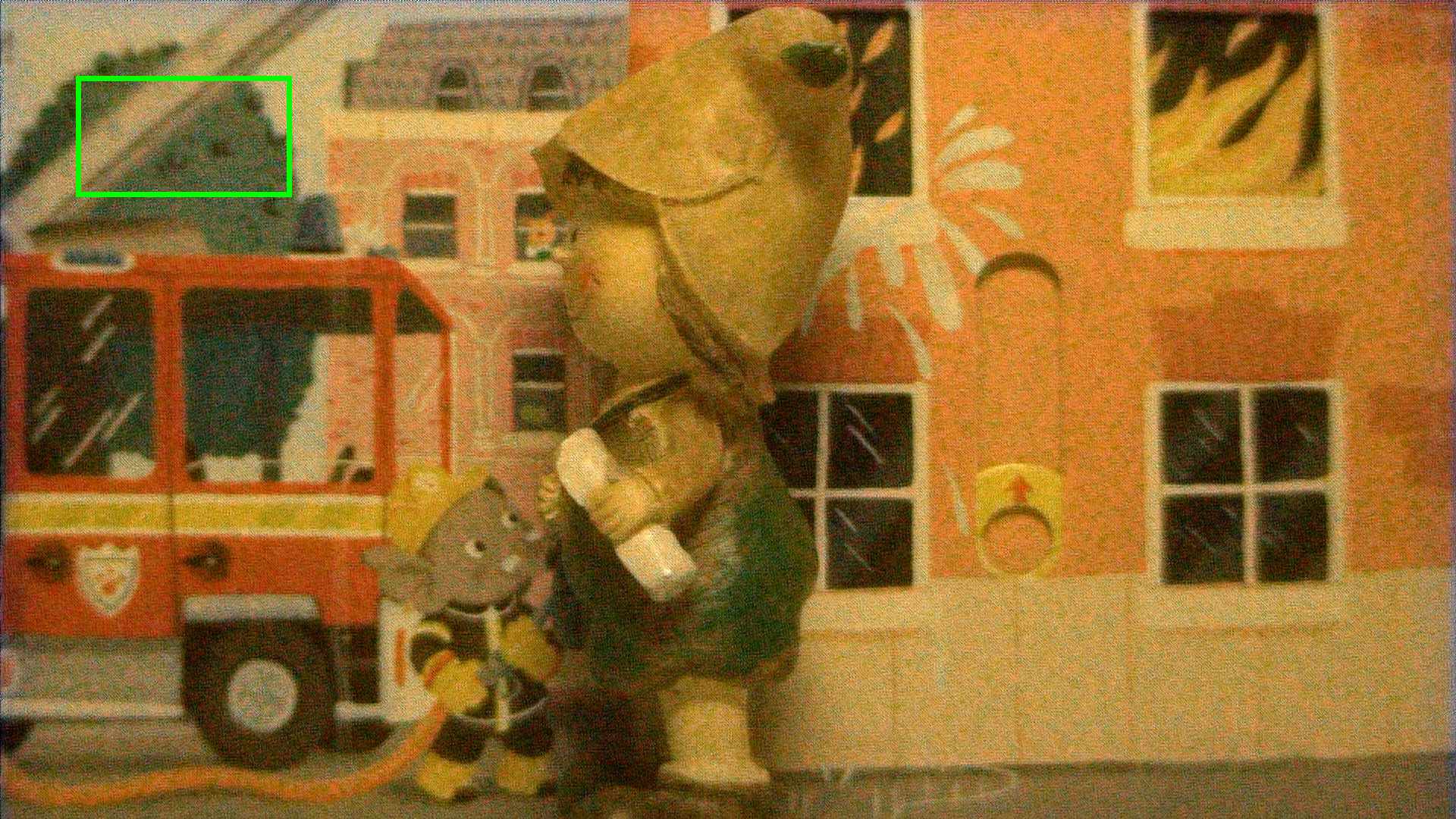}
		\caption{B2U \\ 20.34 / 0.281}
		\label{fig:crvd_8_5_iso25600:b2u_rect}
	\end{subfigure}
	\begin{subfigure}{0.18\textwidth}
	    \captionsetup{justification=centering}
		\includegraphics[width=\textwidth]{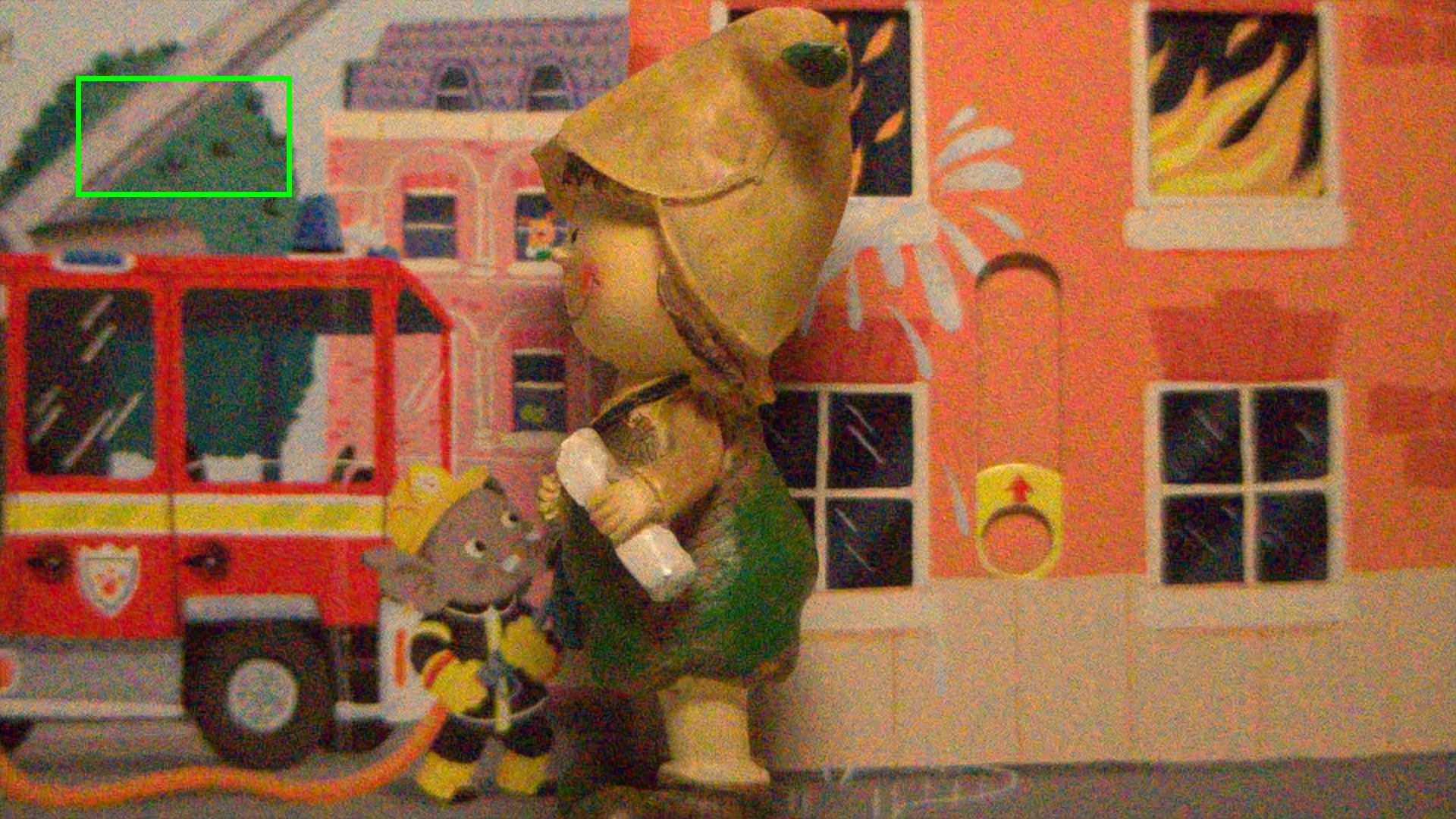}
		\caption{BM3D \\ 26.21 / 0.550}
		\label{fig:crvd_8_5_iso25600:bm3d_rect}
	\end{subfigure}
	\begin{subfigure}{0.18\textwidth}
	    \captionsetup{justification=centering}
		\includegraphics[width=\textwidth]{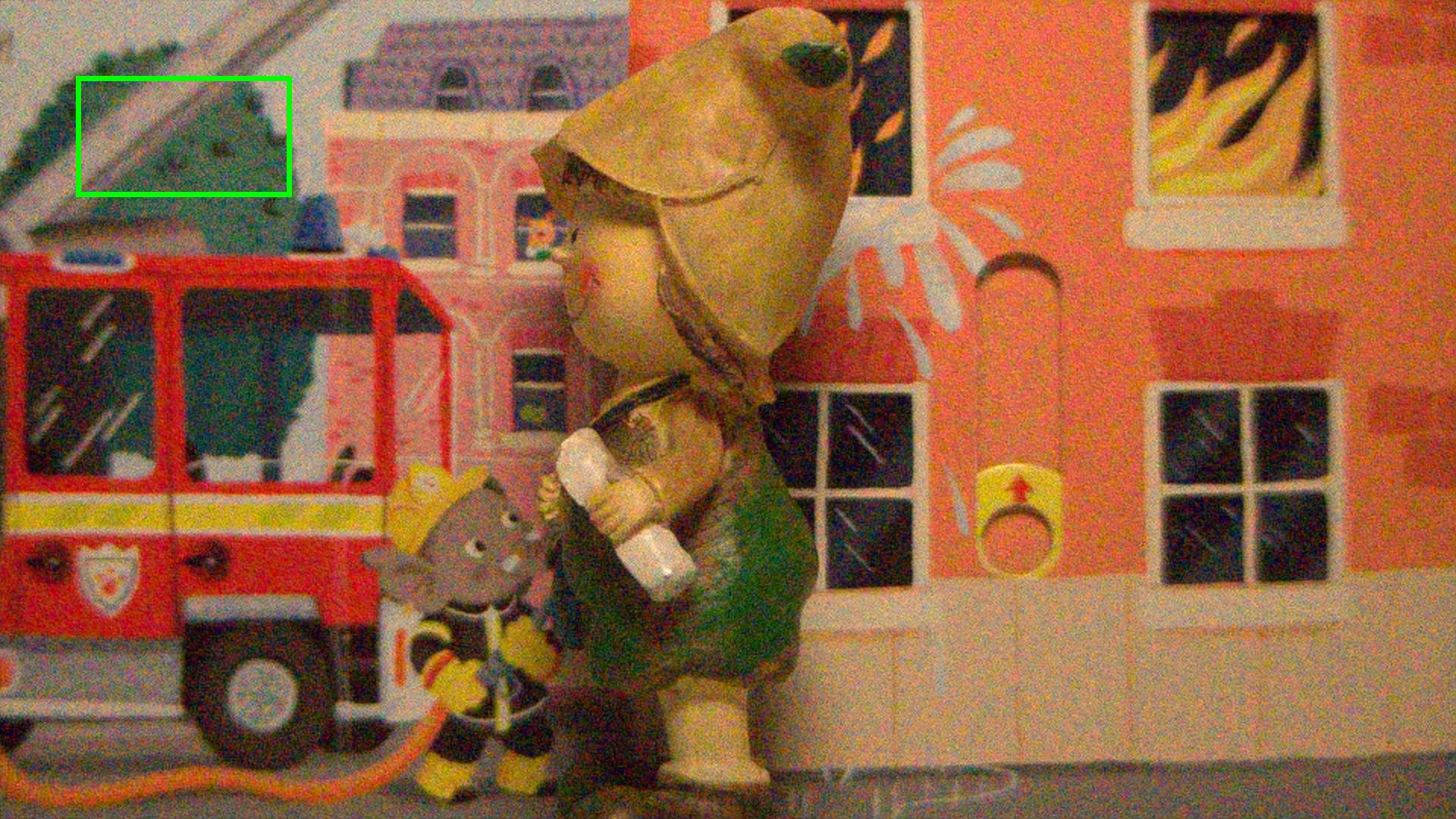}
		\caption{B-DnCNN \\ 25.08 / 0.446}
		\label{fig:crvd_8_5_iso25600:b_dncnn_rect}
	\end{subfigure}
	\begin{subfigure}{0.18\textwidth}
	    \captionsetup{justification=centering}
		\includegraphics[width=\textwidth]{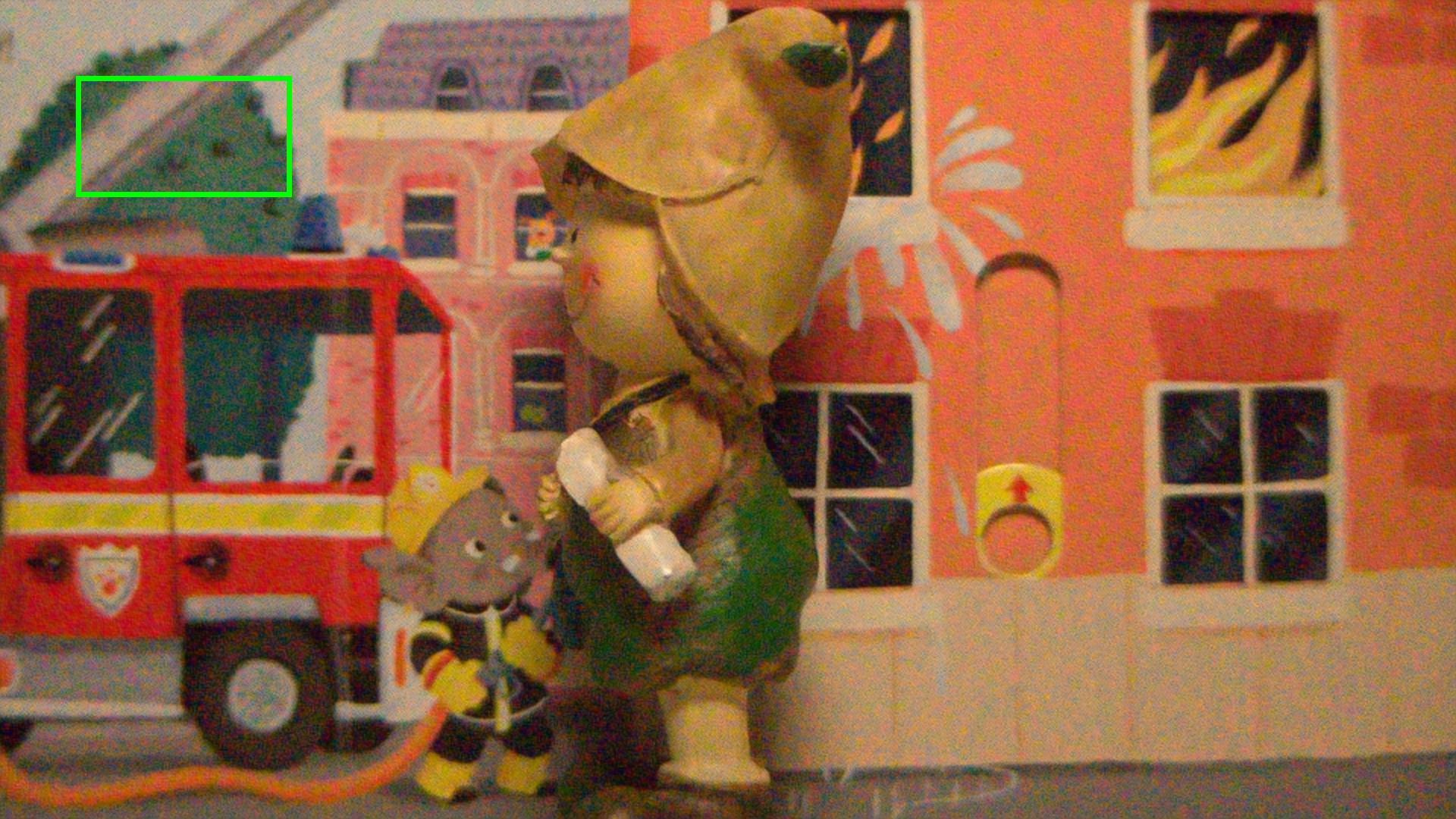}
		\caption{R2R \\ 27.38 / 0.645}
		\label{fig:crvd_8_5_iso25600:r2r_rect}
	\end{subfigure}
	\begin{subfigure}{0.18\textwidth}
	    \captionsetup{justification=centering}
		\includegraphics[width=\textwidth]{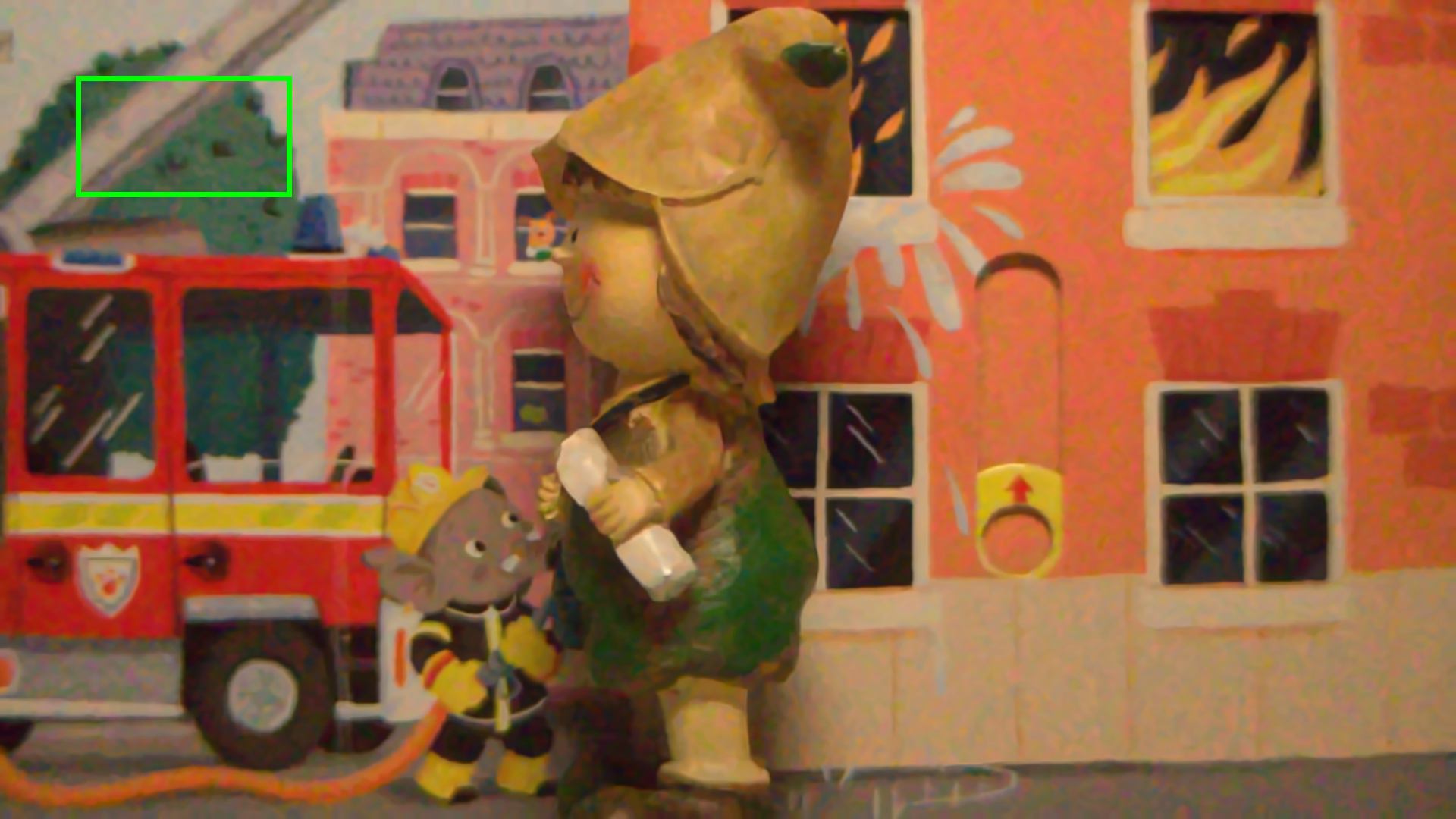}
		\caption{BM3D-O \\ 30.71 / 0.902}
		\label{fig:crvd_8_5_iso25600:bm3d_opt_rect}
	\end{subfigure}
	\begin{subfigure}{0.18\textwidth}
	    \captionsetup{justification=centering}
		\includegraphics[width=\textwidth]{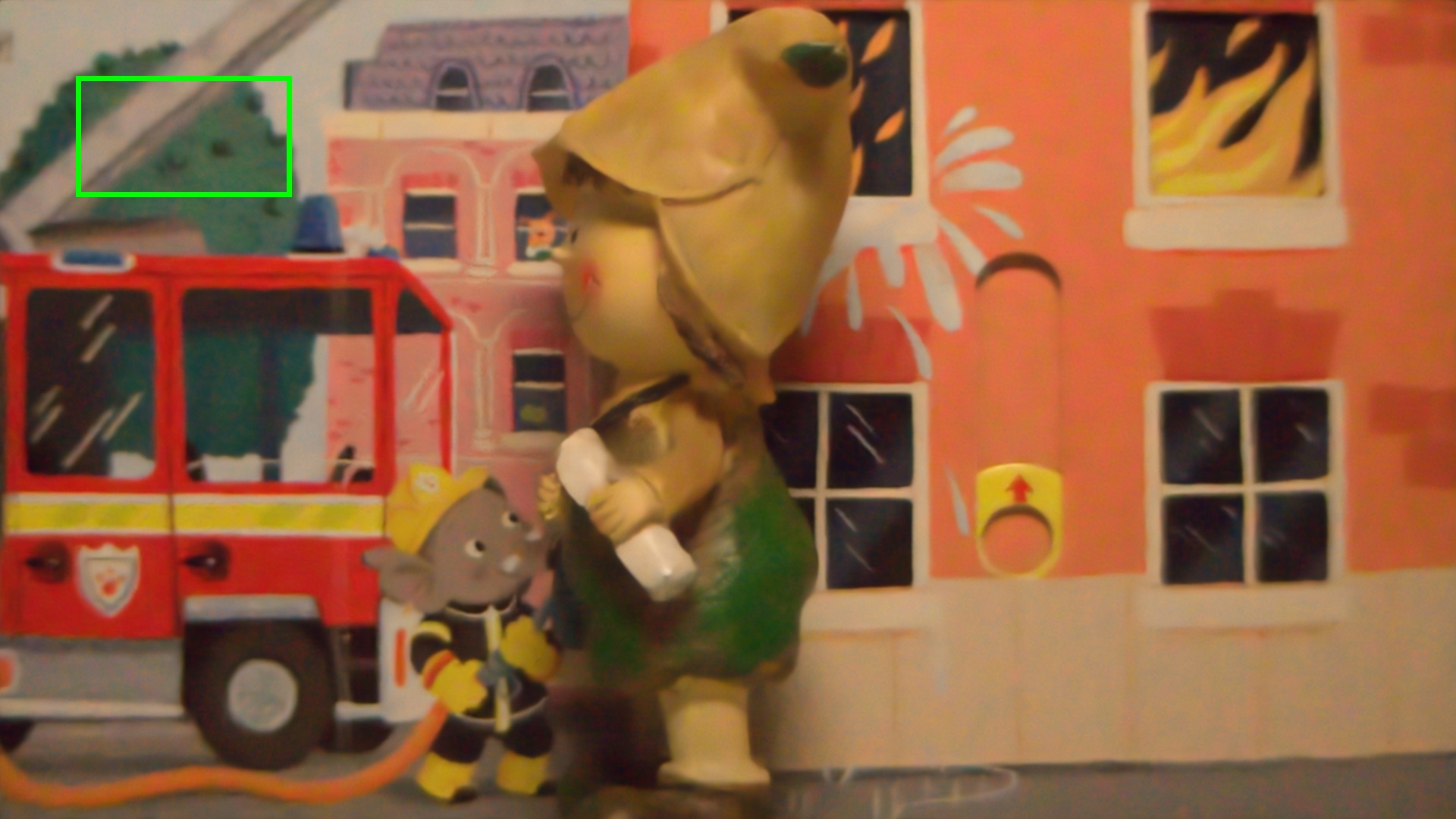}
		\caption{PC-UNet \\ 32.05 / 0.933}
		\label{fig:crvd_8_5_iso25600:pc_unet_rect}
	\end{subfigure}
	\begin{subfigure}{0.18\textwidth}
	    \captionsetup{justification=centering}
		\includegraphics[width=\textwidth]{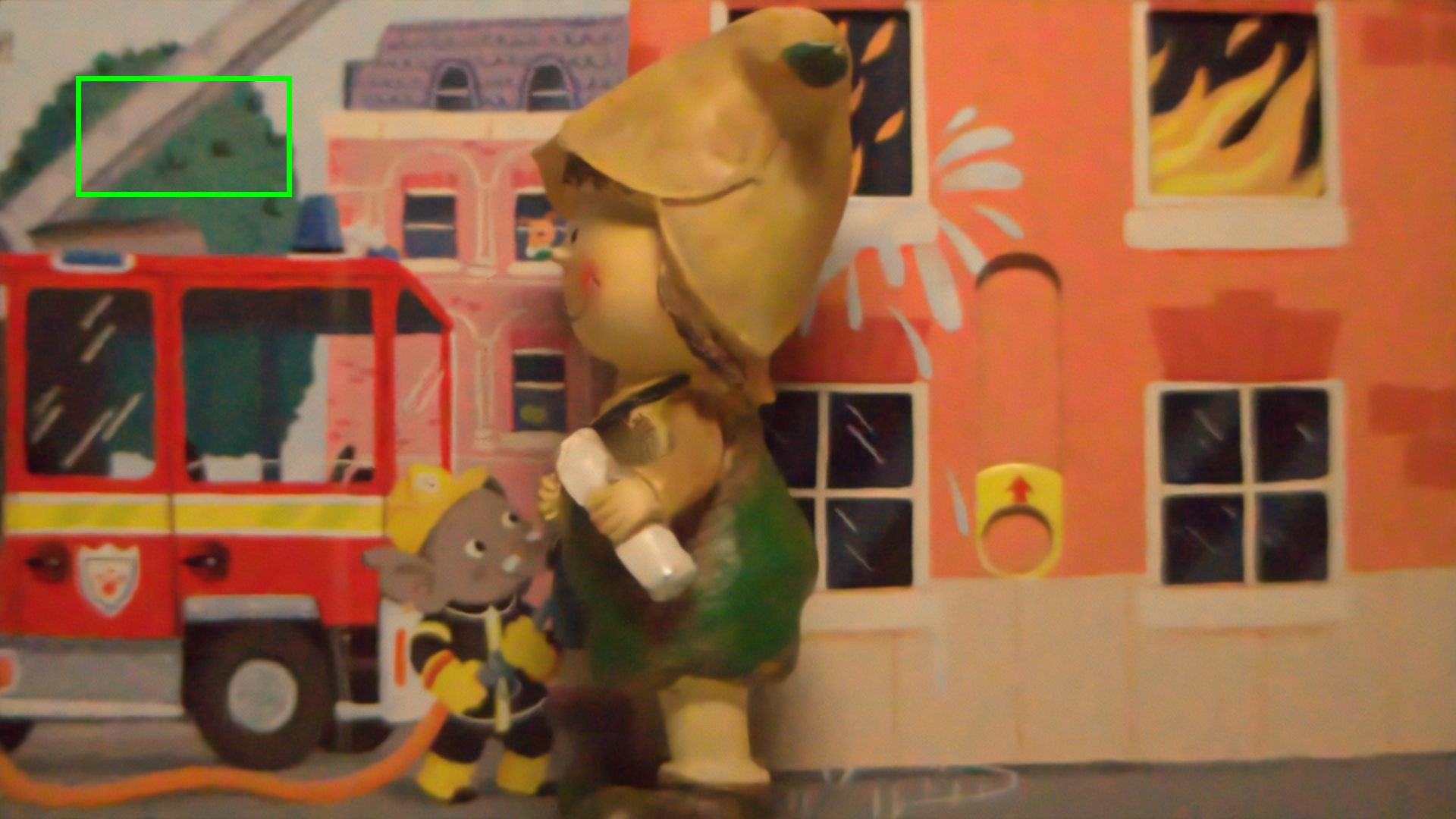}
		\caption{PC-DnCNN \\ 32.00 / 0.932}
		\label{fig:crvd_8_5_iso25600:pc_dncnn_rect}
	\end{subfigure}
	\begin{subfigure}{0.18\textwidth}
	    \captionsetup{justification=centering}
		\includegraphics[width=\textwidth]{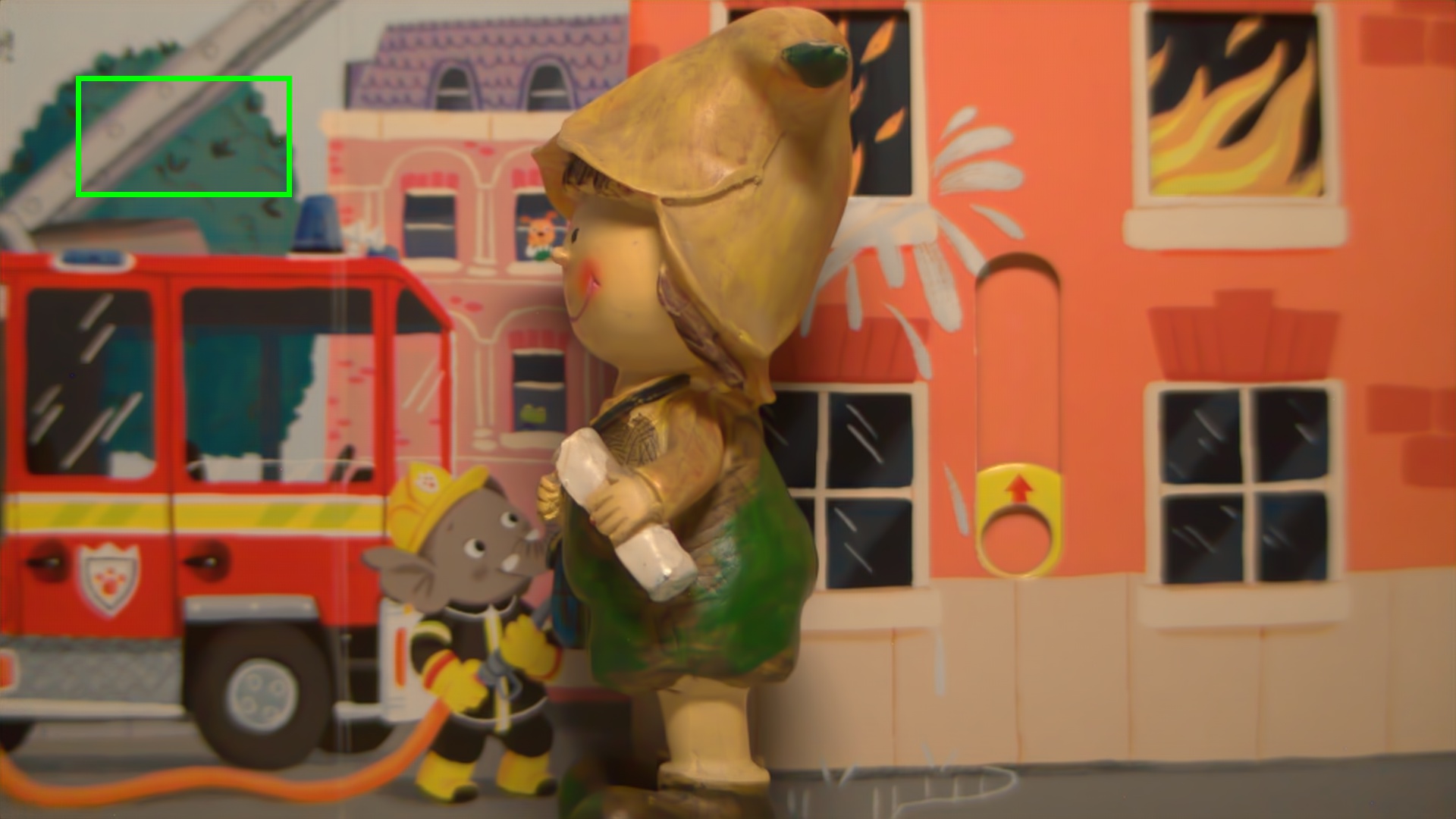}
		\caption{Clean \newline }
		\label{fig:crvd_8_5_iso25600:clean_rect}
	\end{subfigure}
	\begin{subfigure}{0.18\textwidth}
	    \captionsetup{justification=centering}
		\includegraphics[width=\textwidth]{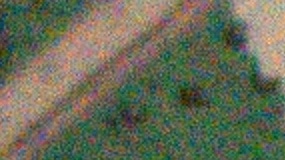}
		\caption{Noisy}
		\label{fig:crvd_8_5_iso25600:noisy_crop}
	\end{subfigure}
	\begin{subfigure}{0.18\textwidth}
	    \captionsetup{justification=centering}
		\includegraphics[width=\textwidth]{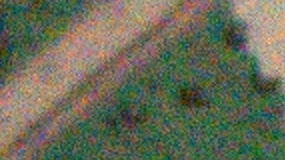}
		\caption{N2N}
		\label{fig:crvd_8_5_iso25600:n2n_crop}
	\end{subfigure}
	\begin{subfigure}{0.18\textwidth}
	    \captionsetup{justification=centering}
		\includegraphics[width=\textwidth]{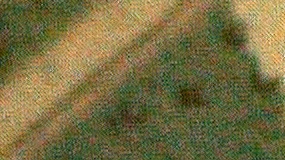}
		\caption{B2U}
		\label{fig:crvd_8_5_iso25600:b2u_crop}
	\end{subfigure}
	\begin{subfigure}{0.18\textwidth}
	    \captionsetup{justification=centering}
		\includegraphics[width=\textwidth]{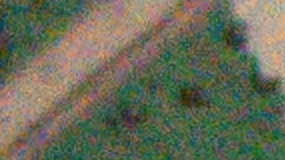}
		\caption{BM3D}
		\label{fig:crvd_8_5_iso25600:bm3d_crop}
	\end{subfigure}
	\begin{subfigure}{0.18\textwidth}
	    \captionsetup{justification=centering}
		\includegraphics[width=\textwidth]{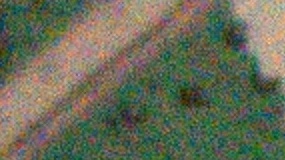}
		\caption{B-DnCNN}
		\label{fig:crvd_8_5_iso25600:b_dncnn_crop}
	\end{subfigure}
	\begin{subfigure}{0.18\textwidth}
	    \captionsetup{justification=centering}
		\includegraphics[width=\textwidth]{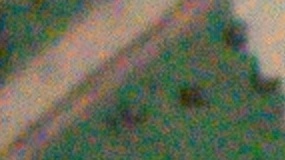}
		\caption{R2R}
		\label{fig:crvd_8_5_iso25600:r2r_crop}
	\end{subfigure}
	\begin{subfigure}{0.18\textwidth}
	    \captionsetup{justification=centering}
		\includegraphics[width=\textwidth]{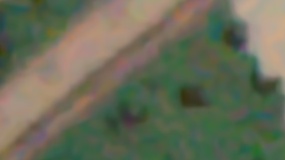}
		\caption{BM3D-O}
		\label{fig:crvd_8_5_iso25600:bm3d_opt_crop}
	\end{subfigure}
	\begin{subfigure}{0.18\textwidth}
	    \captionsetup{justification=centering}
		\includegraphics[width=\textwidth]{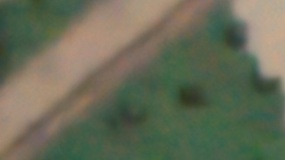}
		\caption{PC-UNet}
		\label{fig:crvd_8_5_iso25600:pc_unet_crop}
	\end{subfigure}
	\begin{subfigure}{0.18\textwidth}
	    \captionsetup{justification=centering}
		\includegraphics[width=\textwidth]{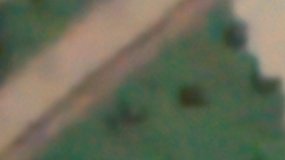}
		\caption{PC-DnCNN}
		\label{fig:crvd_8_5_iso25600:pc_dncnn_crop}
	\end{subfigure}
	\begin{subfigure}{0.18\textwidth}
	    \captionsetup{justification=centering}
		\includegraphics[width=\textwidth]{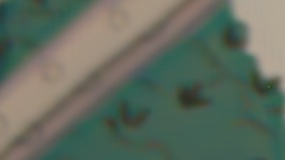}
		\caption{Clean}
		\label{fig:crvd_8_5_iso25600:clean_crop}
	\end{subfigure}
	\begin{subfigure}{0.18\textwidth}
	    \captionsetup{justification=centering}
		\includegraphics[width=\textwidth]{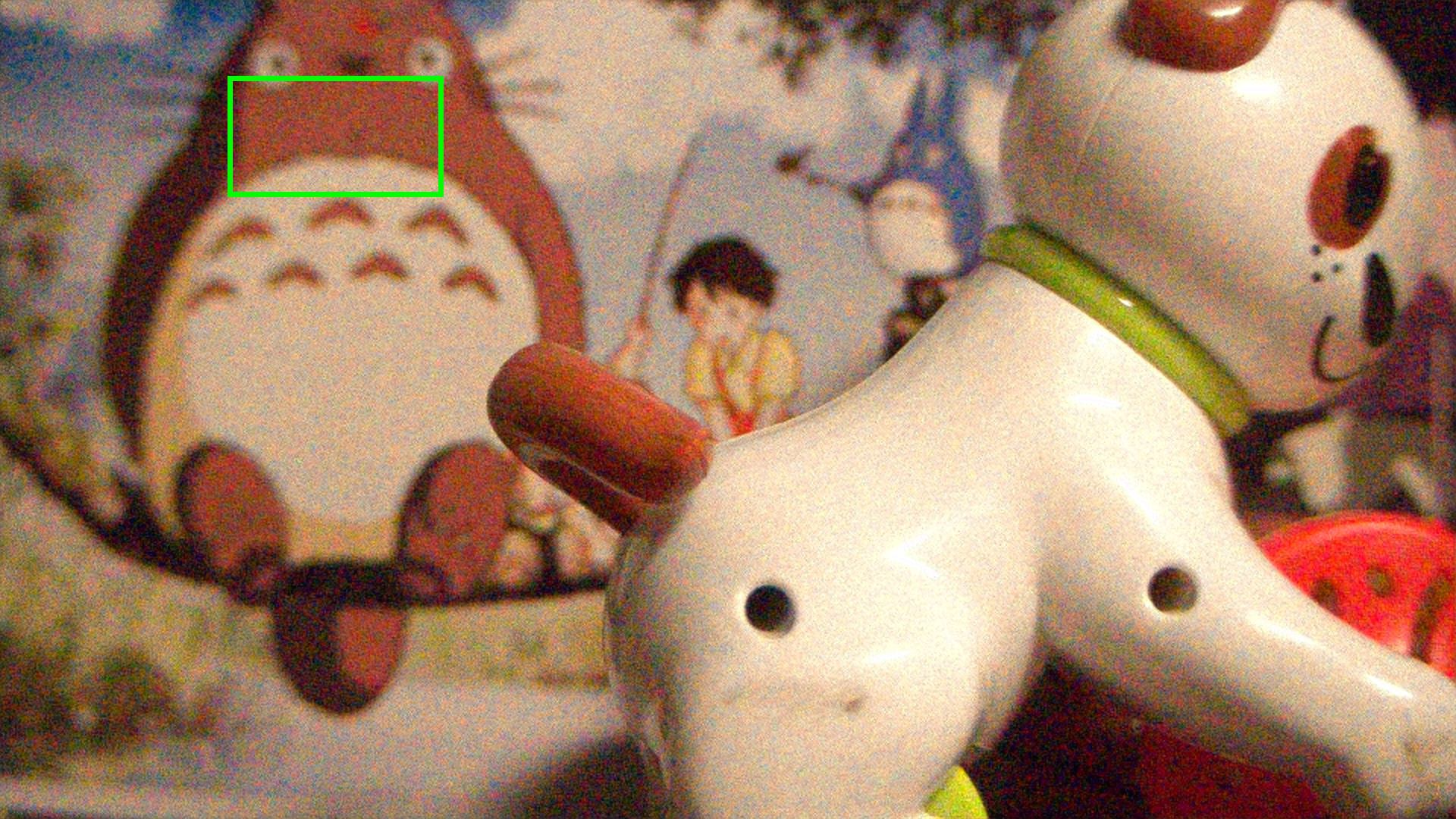}
		\caption{Noisy \\ 25.91 / 0.464}
		\label{fig:crvd_2_3_iso25600:noisy_rect}
	\end{subfigure}
	\begin{subfigure}{0.18\textwidth}
	    \captionsetup{justification=centering}
		\includegraphics[width=\textwidth]{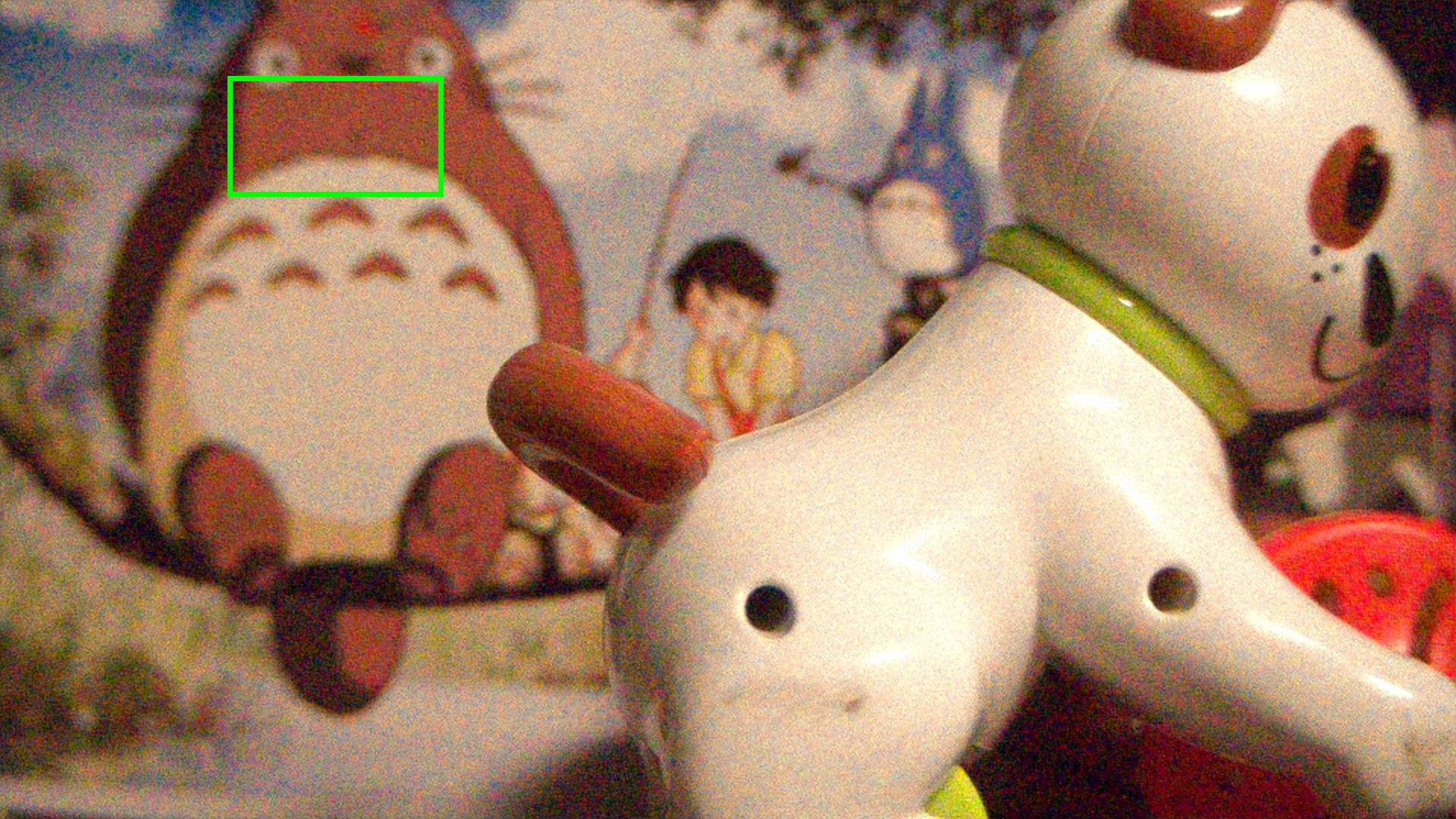}
		\caption{N2N \\ 26.00 / 0.471}
		\label{fig:crvd_2_3_iso25600:n2n_rect}
	\end{subfigure}
	\begin{subfigure}{0.18\textwidth}
	    \captionsetup{justification=centering}
		\includegraphics[width=\textwidth]{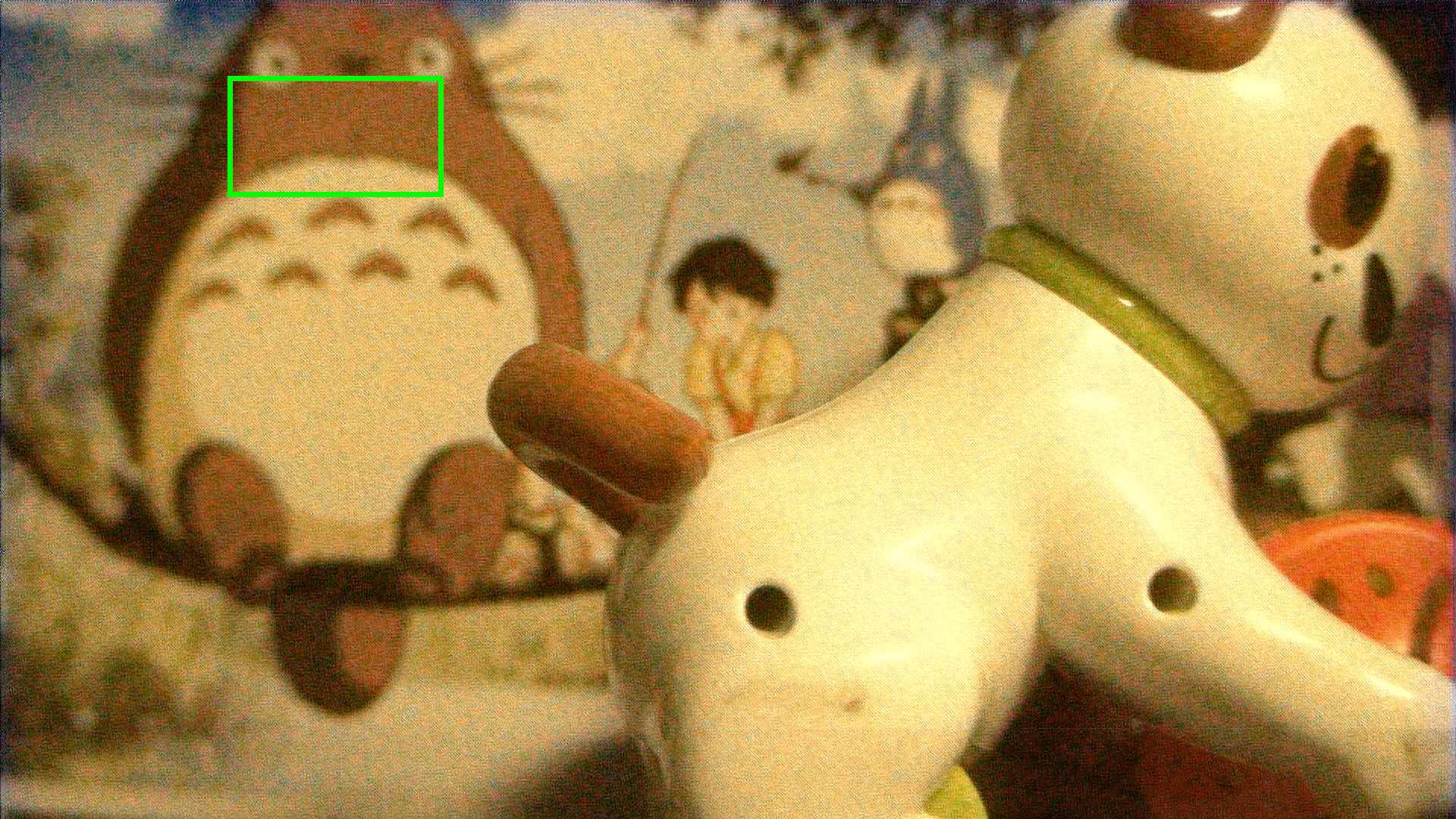}
		\caption{B2U \\ 19.61 / 0.283}
		\label{fig:crvd_2_3_iso25600:b2u_rect}
	\end{subfigure}
	\begin{subfigure}{0.18\textwidth}
	    \captionsetup{justification=centering}
		\includegraphics[width=\textwidth]{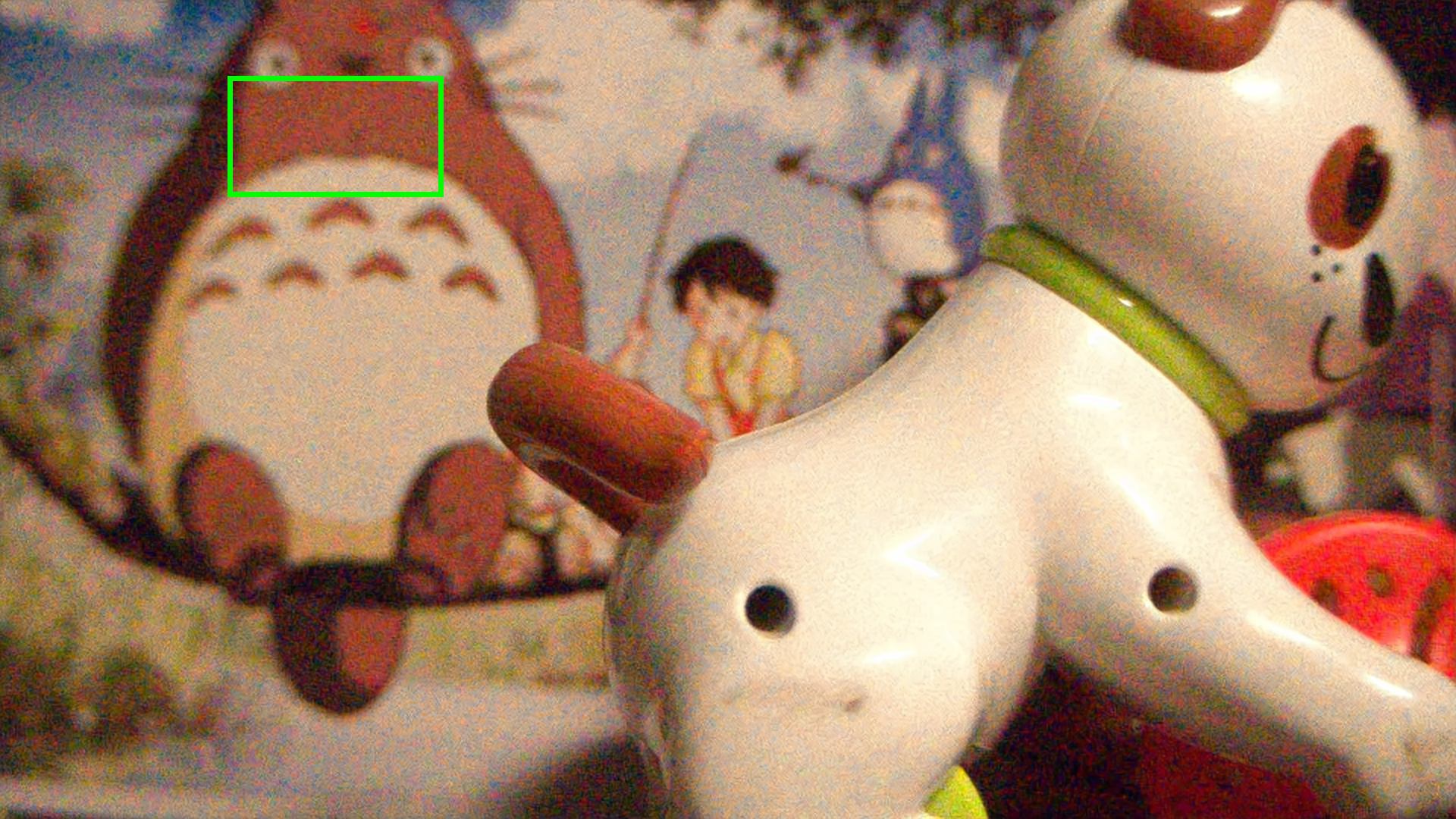}
		\caption{BM3D \\ 27.39 / 0.603}
		\label{fig:crvd_2_3_iso25600:bm3d_rect}
	\end{subfigure}
	\begin{subfigure}{0.18\textwidth}
	    \captionsetup{justification=centering}
		\includegraphics[width=\textwidth]{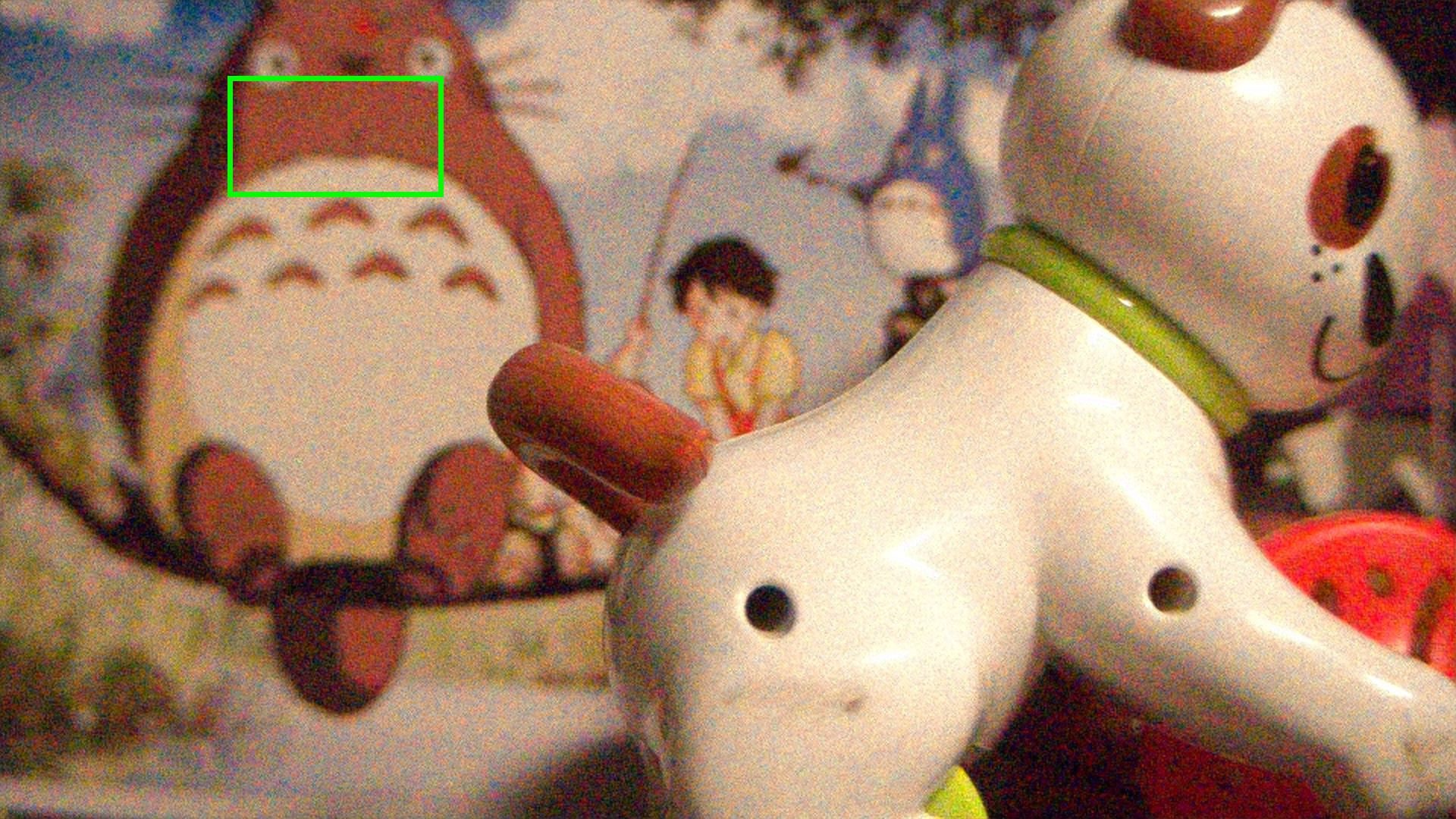}
		\caption{B-DnCNN \\ 25.98 / 0.469}
		\label{fig:crvd_2_3_iso25600:b_dncnn_rect}
	\end{subfigure}
	\begin{subfigure}{0.18\textwidth}
	    \captionsetup{justification=centering}
		\includegraphics[width=\textwidth]{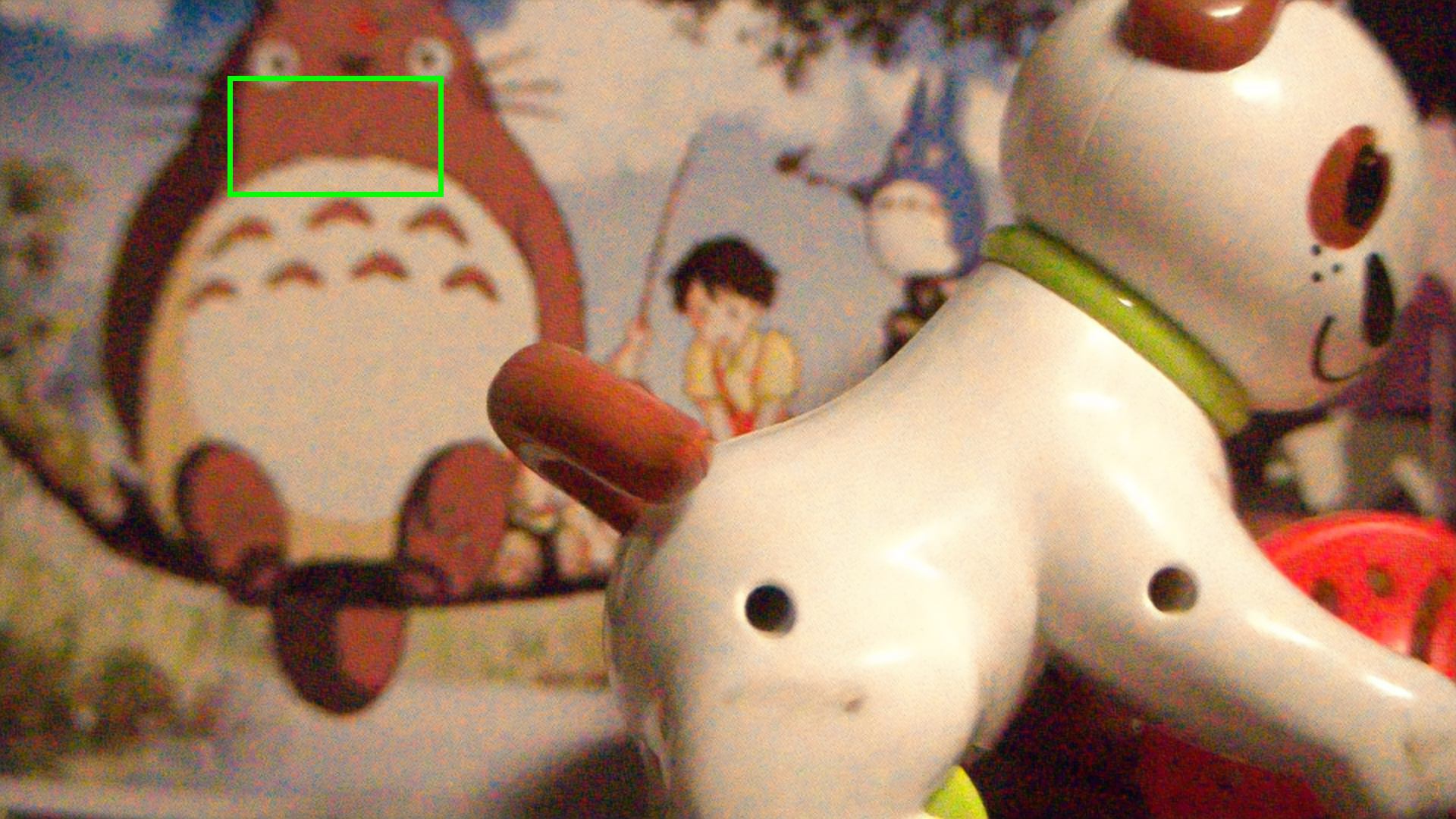}
		\caption{R2R \\ 28.64 / 0.702}
		\label{fig:crvd_2_3_iso25600:r2r_rect}
	\end{subfigure}
	\begin{subfigure}{0.18\textwidth}
	    \captionsetup{justification=centering}
		\includegraphics[width=\textwidth]{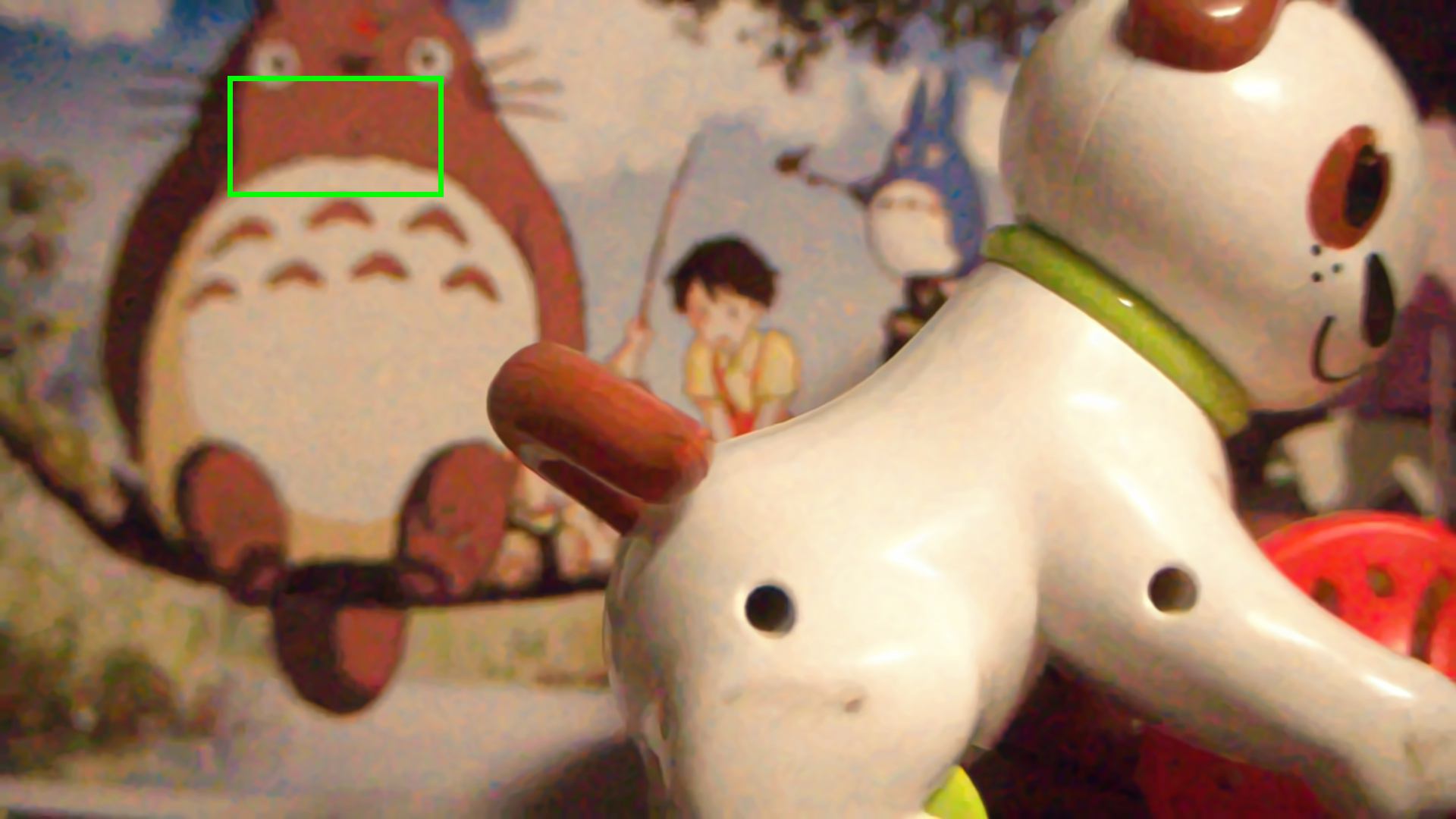}
		\caption{BM3D-O \\ 32.27 / 0.938}
		\label{fig:crvd_2_3_iso25600:bm3d_opt_rect}
	\end{subfigure}
	\begin{subfigure}{0.18\textwidth}
	    \captionsetup{justification=centering}
		\includegraphics[width=\textwidth]{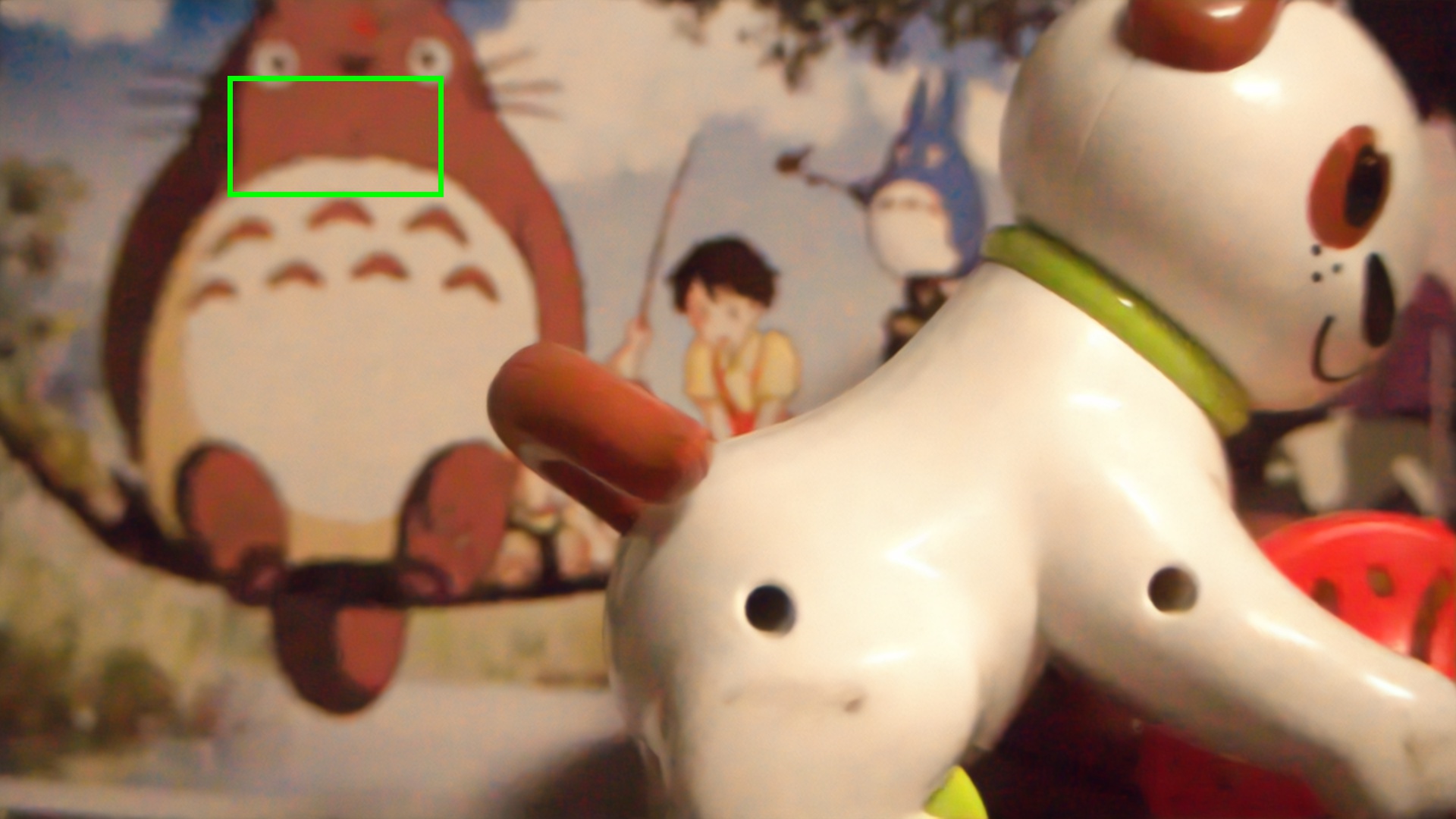}
		\caption{PC-UNet \\ 33.66 / 0.958}
		\label{fig:crvd_2_3_iso25600:pc_unet_rect}
	\end{subfigure}
	\begin{subfigure}{0.18\textwidth}
	    \captionsetup{justification=centering}
		\includegraphics[width=\textwidth]{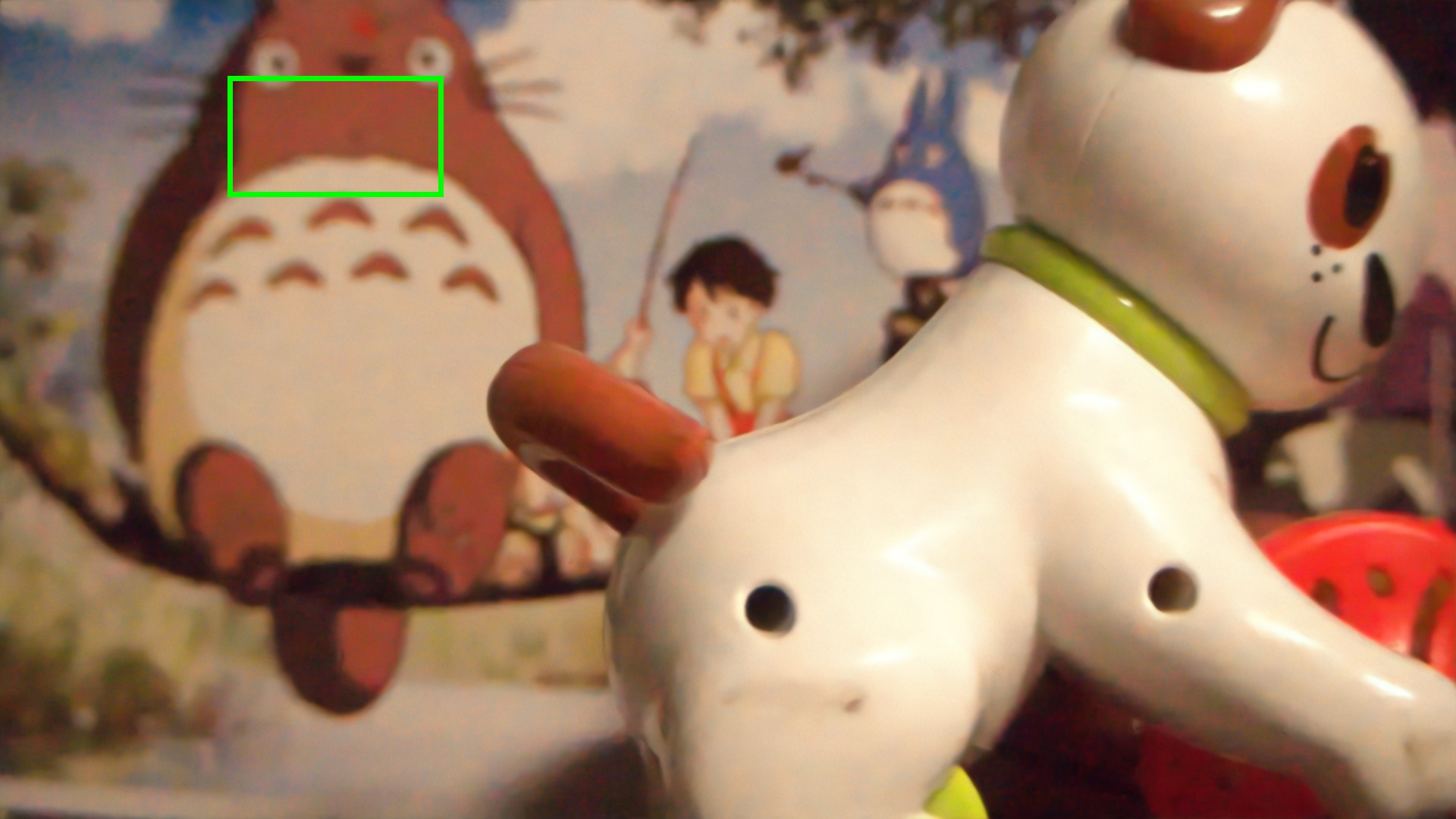}
		\caption{PC-DnCNN \\ 33.50 / 0.957}
		\label{fig:crvd_2_3_iso25600:pc_dncnn_rect}
	\end{subfigure}
	\begin{subfigure}{0.18\textwidth}
	    \captionsetup{justification=centering}
		\includegraphics[width=\textwidth]{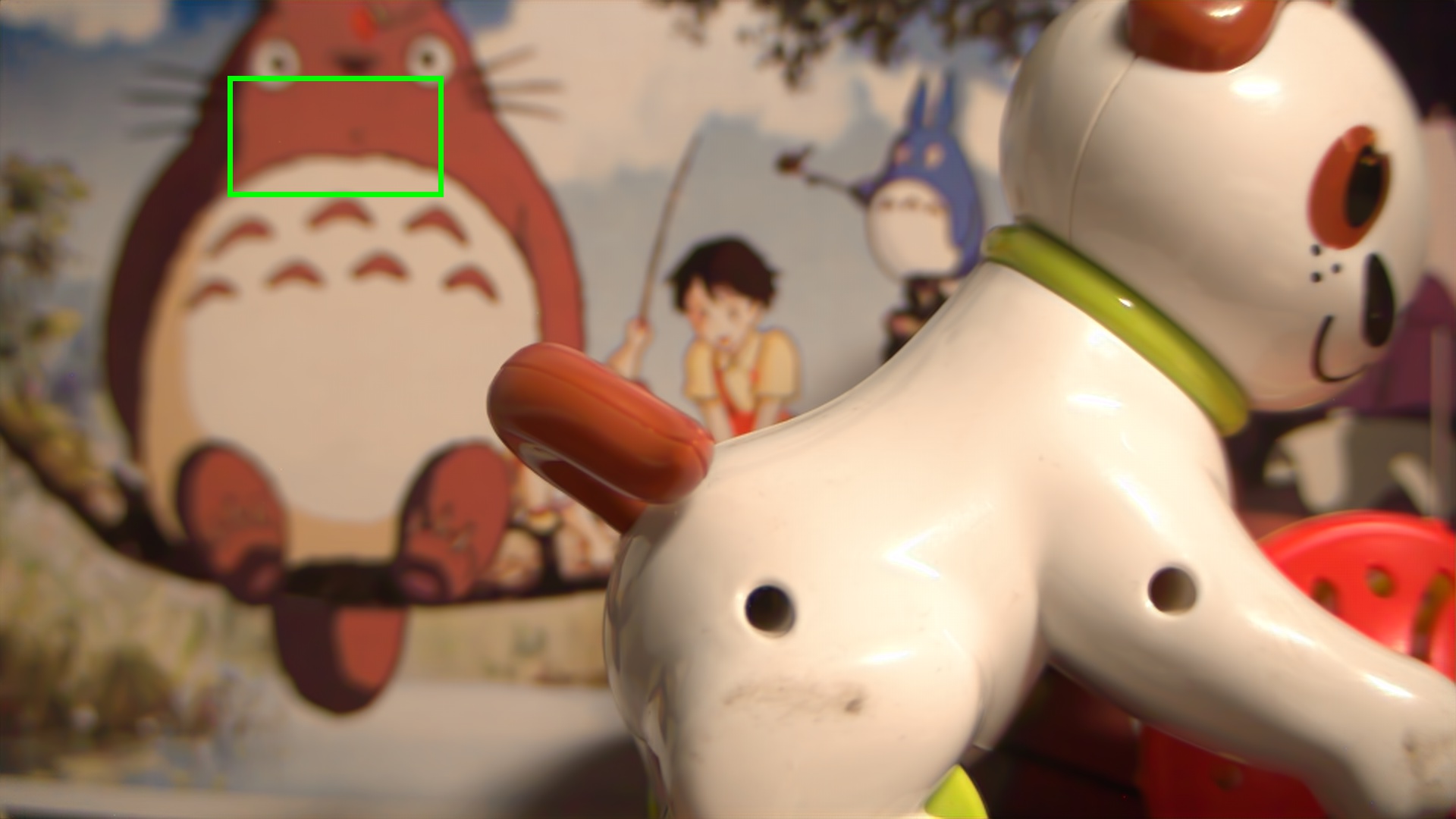}
		\caption{Clean \newline }
		\label{fig:crvd_2_3_iso25600:clean_rect}
	\end{subfigure}
	\begin{subfigure}{0.18\textwidth}
	    \captionsetup{justification=centering}
		\includegraphics[width=\textwidth]{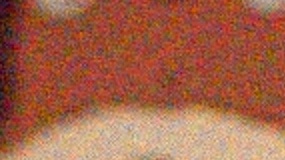}
		\caption{Noisy}
		\label{fig:crvd_2_3_iso25600:noisy_crop}
	\end{subfigure}
	\begin{subfigure}{0.18\textwidth}
	    \captionsetup{justification=centering}
		\includegraphics[width=\textwidth]{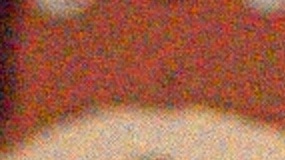}
		\caption{N2N}
		\label{fig:crvd_2_3_iso25600:n2n_crop}
	\end{subfigure}
	\begin{subfigure}{0.18\textwidth}
	    \captionsetup{justification=centering}
		\includegraphics[width=\textwidth]{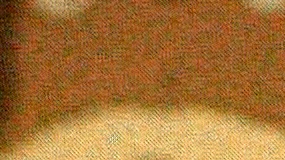}
		\caption{B2U}
		\label{fig:crvd_2_3_iso25600:b2u_crop}
	\end{subfigure}
	\begin{subfigure}{0.18\textwidth}
	    \captionsetup{justification=centering}
		\includegraphics[width=\textwidth]{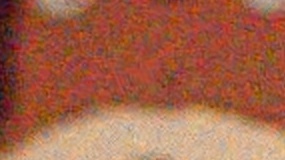}
		\caption{BM3D}
		\label{fig:crvd_2_3_iso25600:bm3d_crop}
	\end{subfigure}
	\begin{subfigure}{0.18\textwidth}
	    \captionsetup{justification=centering}
		\includegraphics[width=\textwidth]{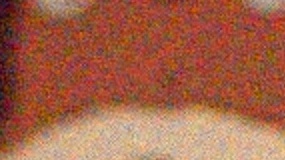}
		\caption{B-DnCNN}
		\label{fig:crvd_2_3_iso25600:b_dncnn_crop}
	\end{subfigure}
	\begin{subfigure}{0.18\textwidth}
	    \captionsetup{justification=centering}
		\includegraphics[width=\textwidth]{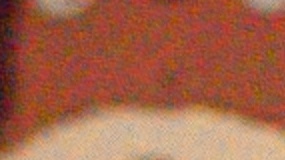}
		\caption{R2R}
		\label{fig:crvd_2_3_iso25600:r2r_crop}
	\end{subfigure}
	\begin{subfigure}{0.18\textwidth}
	    \captionsetup{justification=centering}
		\includegraphics[width=\textwidth]{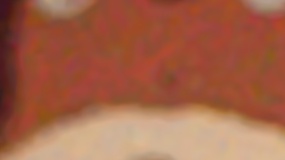}
		\caption{BM3D-O}
		\label{fig:crvd_2_3_iso25600:bm3d_opt_crop}
	\end{subfigure}
	\begin{subfigure}{0.18\textwidth}
	    \captionsetup{justification=centering}
		\includegraphics[width=\textwidth]{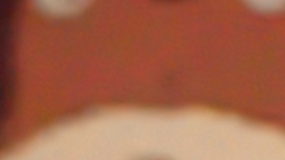}
		\caption{PC-UNet}
		\label{fig:crvd_2_3_iso25600:pc_unet_crop}
	\end{subfigure}
	\begin{subfigure}{0.18\textwidth}
	    \captionsetup{justification=centering}
		\includegraphics[width=\textwidth]{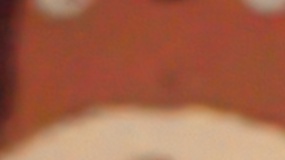}
		\caption{PC-DnCNN}
		\label{fig:crvd_2_3_iso25600:pc_dncnn_crop}
	\end{subfigure}
	\begin{subfigure}{0.18\textwidth}
	    \captionsetup{justification=centering}
		\includegraphics[width=\textwidth]{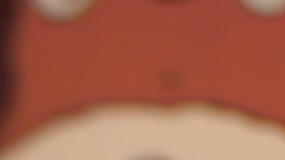}
		\caption{Clean}
		\label{fig:crvd_2_3_iso25600:clean_crop}
	\end{subfigure}
	\caption{Denoising examples with real-world noise. The first four rows show frame 6 of scene 9. The last four rows present frame 4 of scene 2. Both images are captured with ISO 25600. As can be seen, oracle BM3D leaves a substantial amount of low-frequency noise unfiltered, while other algorithms, except ours (PC-UNet and PC-DnCNN), do not succeed in removing the noise.}
	\label{fig:crvd_8_5_iso25600_2_3_iso25600}
\end{figure*}

\begin{figure*}
    \centering
	\begin{subfigure}{0.18\textwidth}
	    \captionsetup{justification=centering}
		\includegraphics[width=\textwidth]{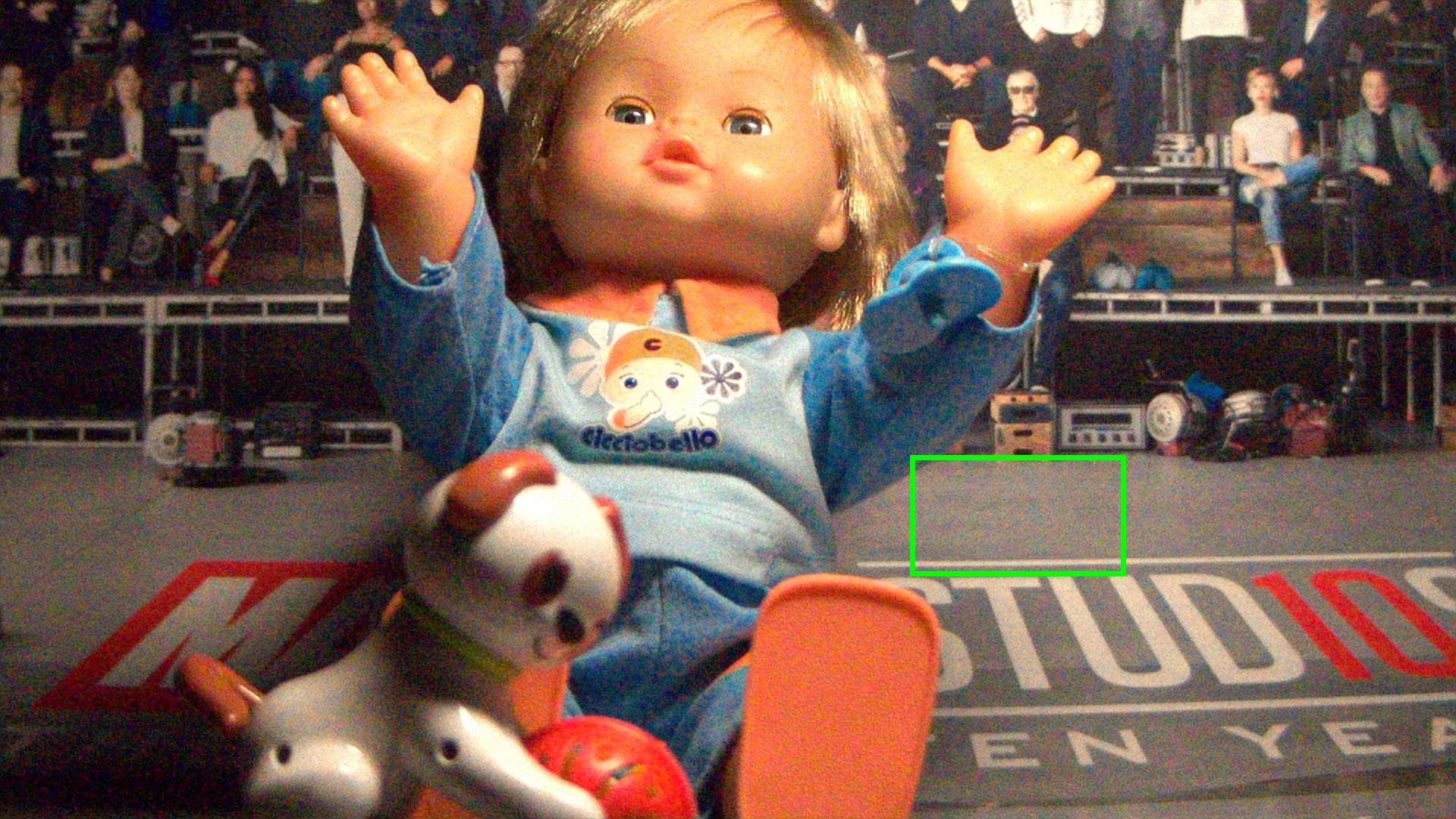}
		\caption{Noisy \\ 28.29 / 0.651}
		\label{fig:crvd_5_2_iso12800:noisy_rect}
	\end{subfigure}
	\begin{subfigure}{0.18\textwidth}
	    \captionsetup{justification=centering}
		\includegraphics[width=\textwidth]{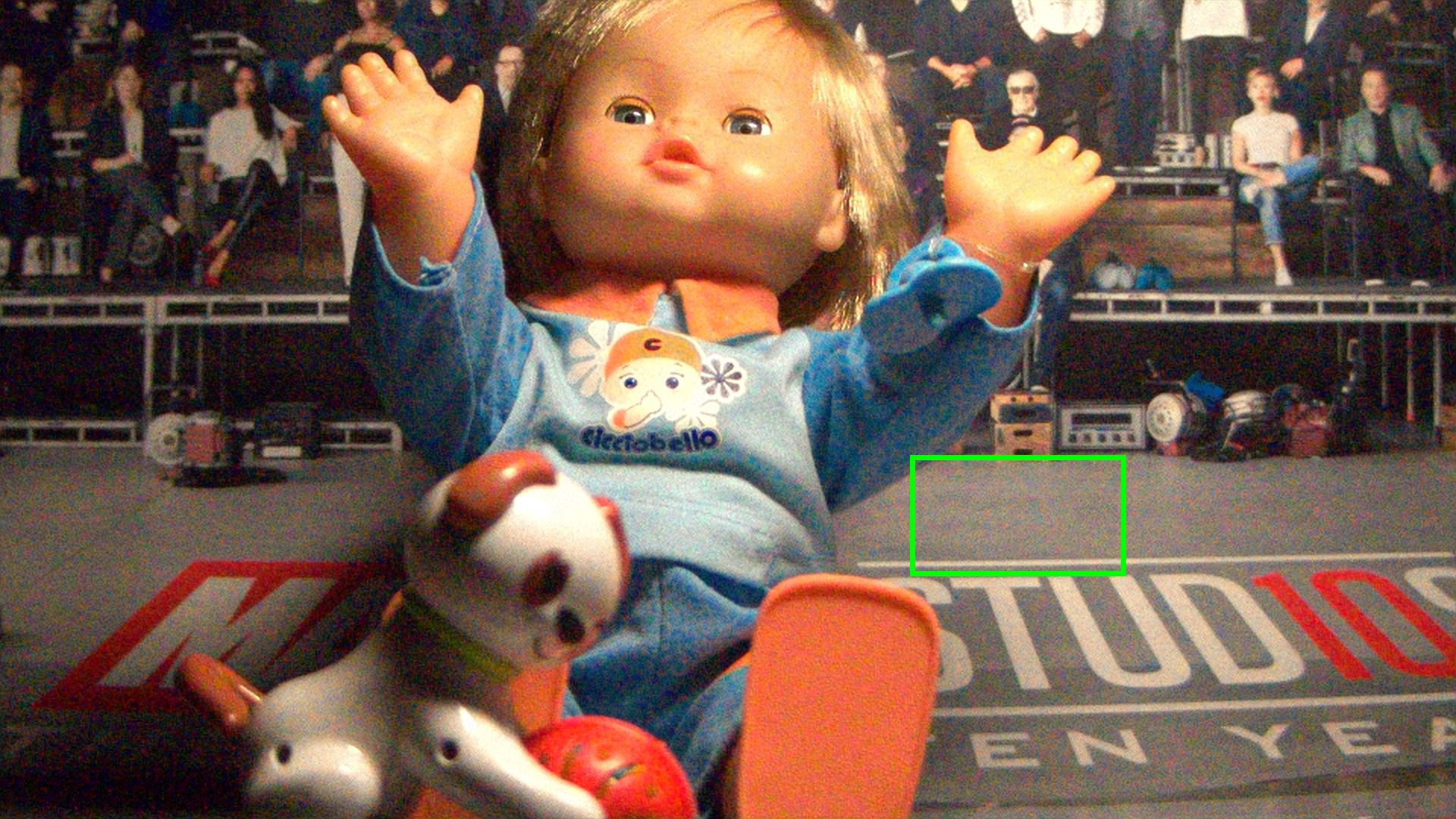}
		\caption{N2N \\ 28.37 / 0.657}
		\label{fig:crvd_5_2_iso12800:n2n_rect}
	\end{subfigure}
	\begin{subfigure}{0.18\textwidth}
	    \captionsetup{justification=centering}
		\includegraphics[width=\textwidth]{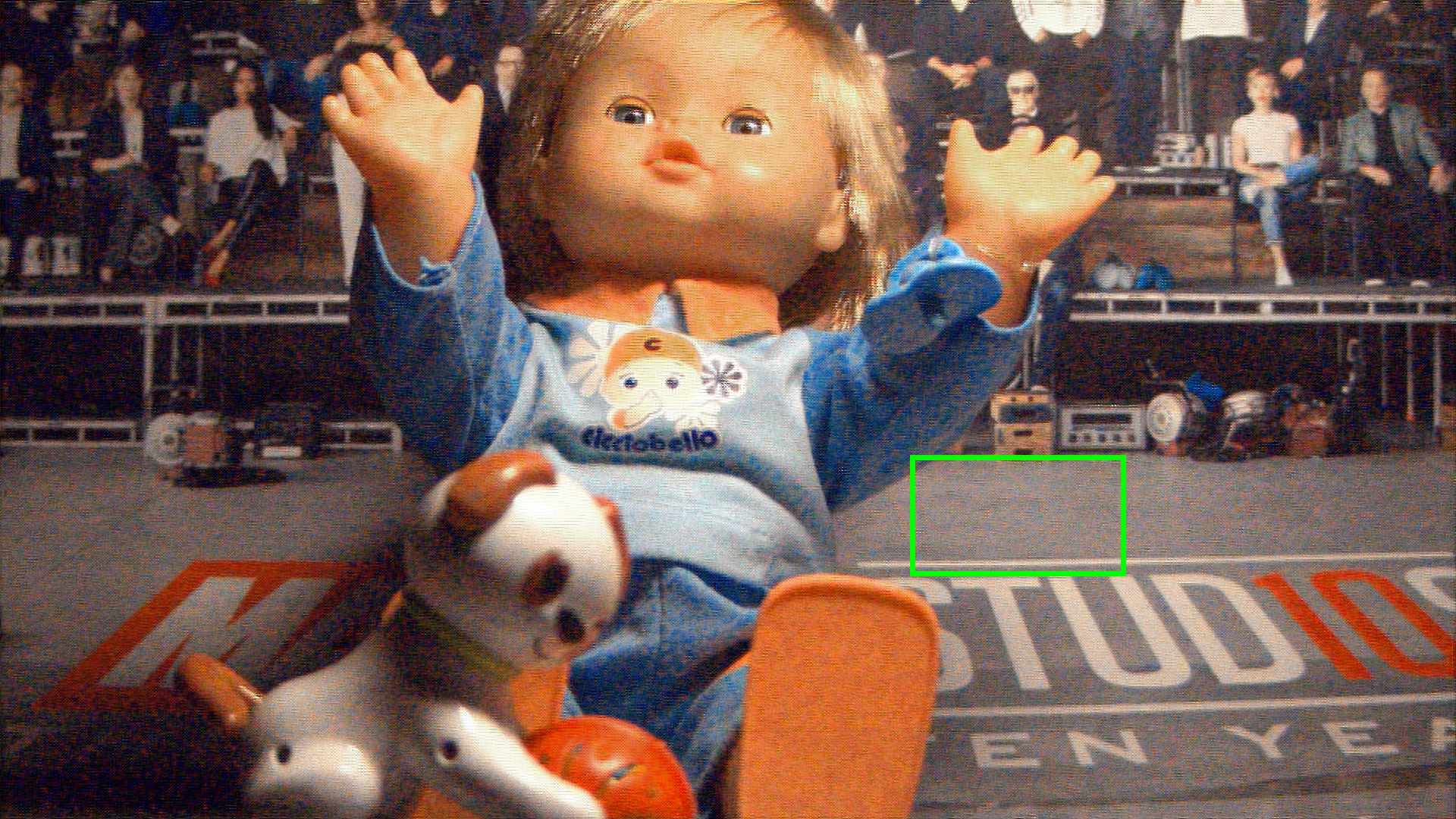}
		\caption{B2U \\ 23.96 / 0.446}
		\label{fig:crvd_5_2_iso12800:b2u_rect}
	\end{subfigure}
	\begin{subfigure}{0.18\textwidth}
	    \captionsetup{justification=centering}
		\includegraphics[width=\textwidth]{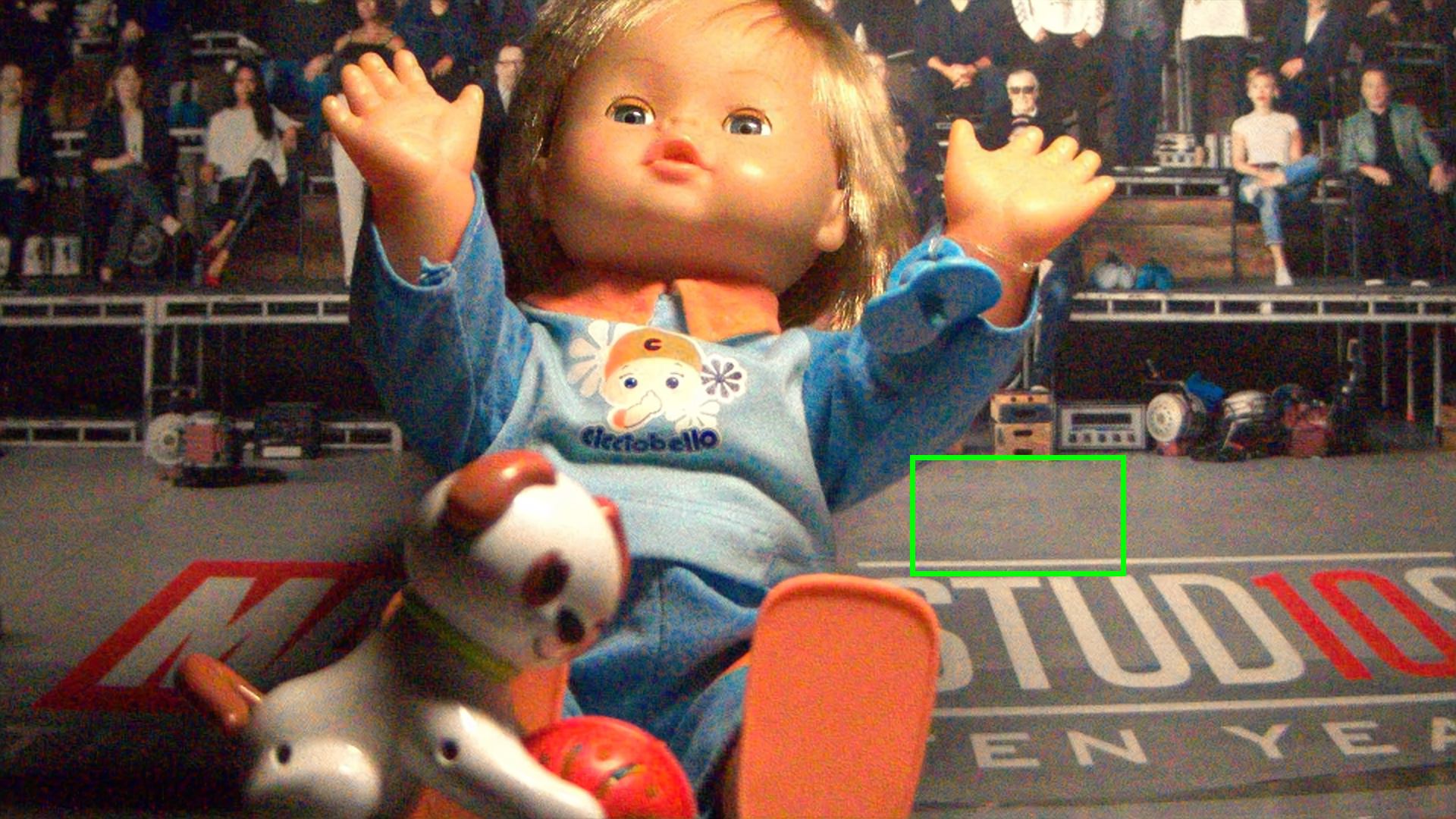}
		\caption{BM3D \\ 29.30 / 0.725}
		\label{fig:crvd_5_2_iso12800:bm3d_rect}
	\end{subfigure}
	\begin{subfigure}{0.18\textwidth}
	    \captionsetup{justification=centering}
		\includegraphics[width=\textwidth]{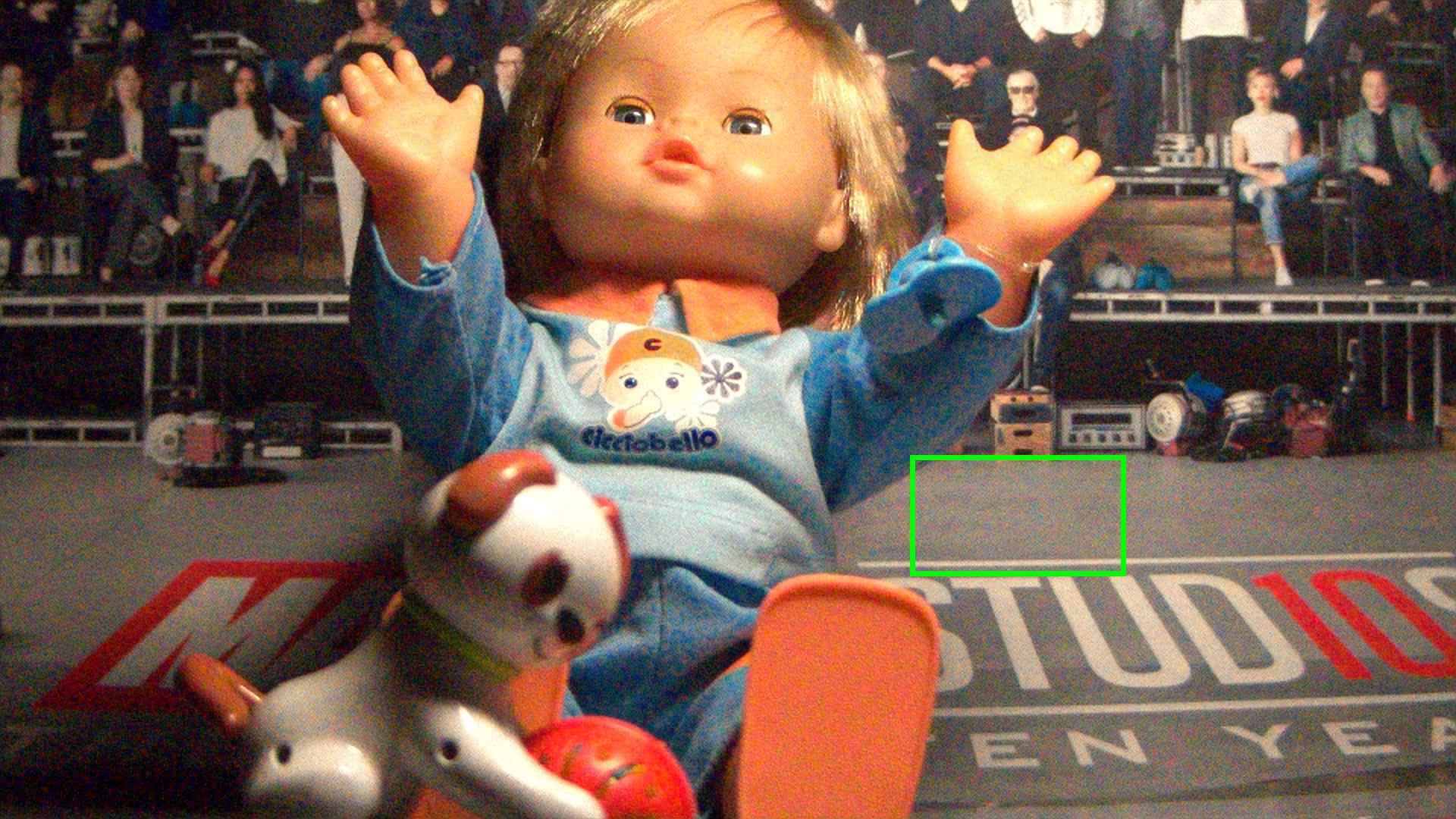}
		\caption{B-DnCNN \\ 28.36 / 0.655}
		\label{fig:crvd_5_2_iso12800:b_dncnn_rect}
	\end{subfigure}
	\begin{subfigure}{0.18\textwidth}
	    \captionsetup{justification=centering}
		\includegraphics[width=\textwidth]{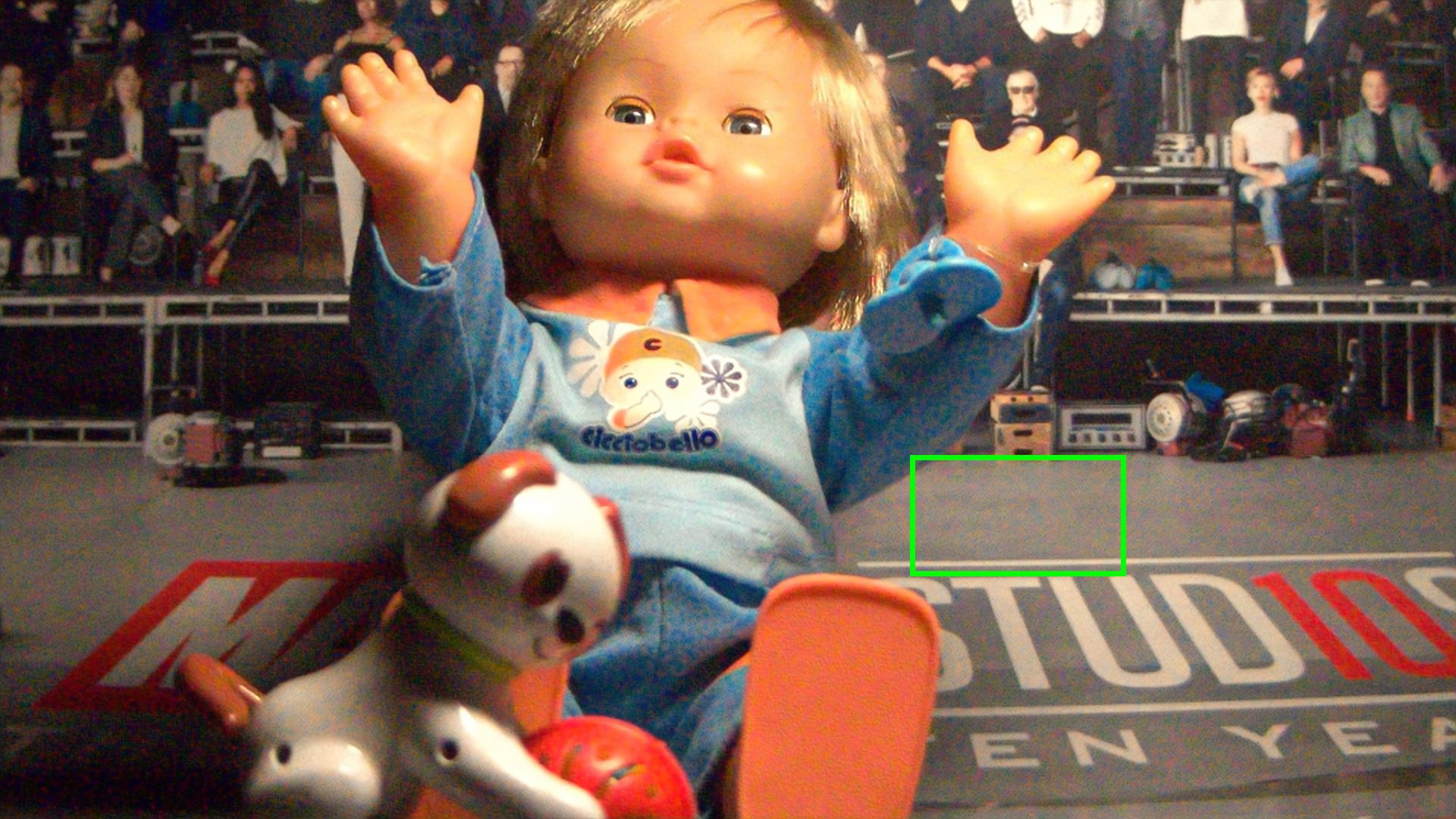}
		\caption{R2R \\ 30.43 / 0.788}
		\label{fig:crvd_5_2_iso12800:r2r_rect}
	\end{subfigure}
	\begin{subfigure}{0.18\textwidth}
	    \captionsetup{justification=centering}
		\includegraphics[width=\textwidth]{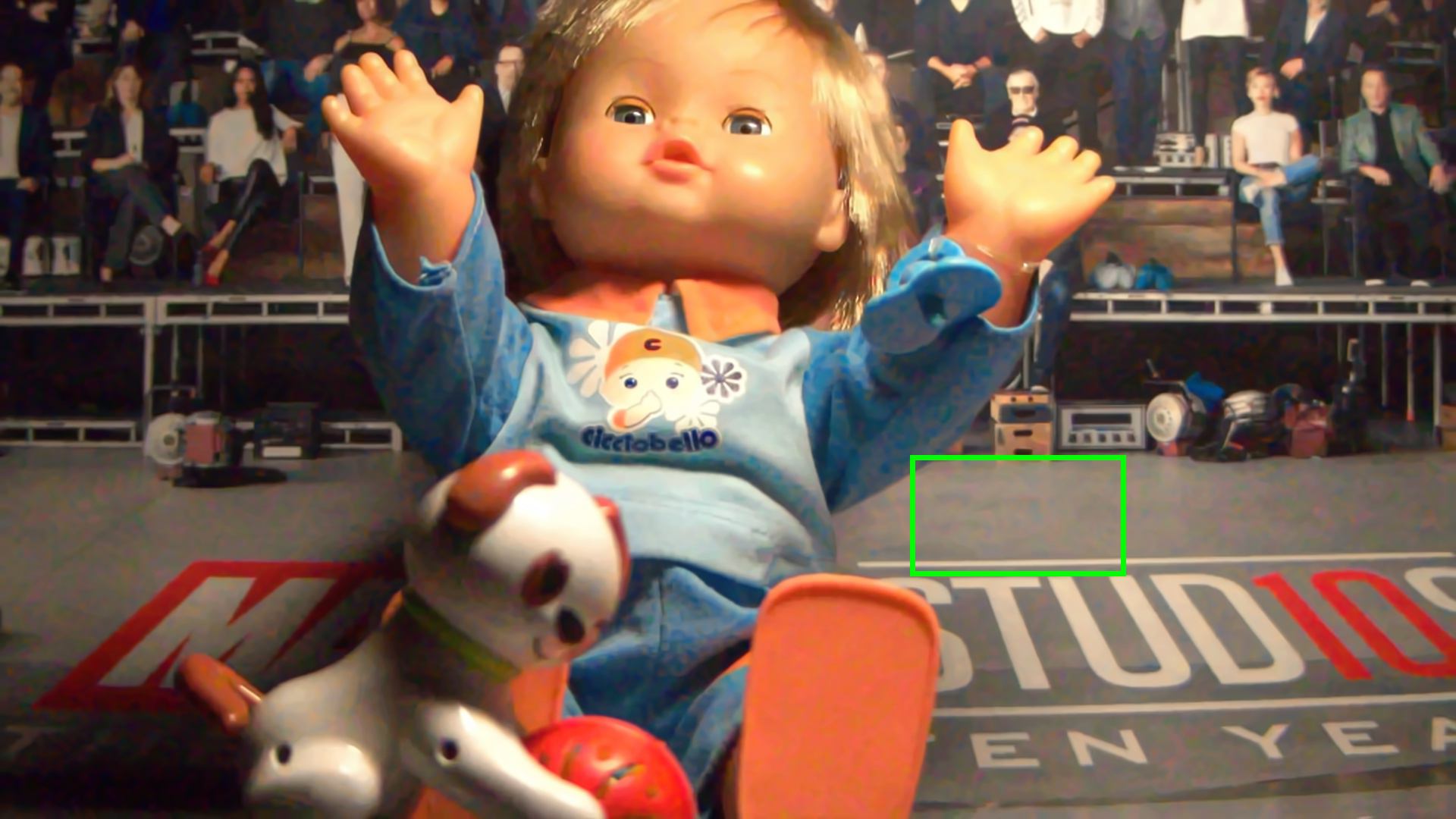}
		\caption{BM3D-O \\ 33.40 / 0.933}
		\label{fig:crvd_5_2_iso12800:bm3d_opt_rect}
	\end{subfigure}
	\begin{subfigure}{0.18\textwidth}
	    \captionsetup{justification=centering}
		\includegraphics[width=\textwidth]{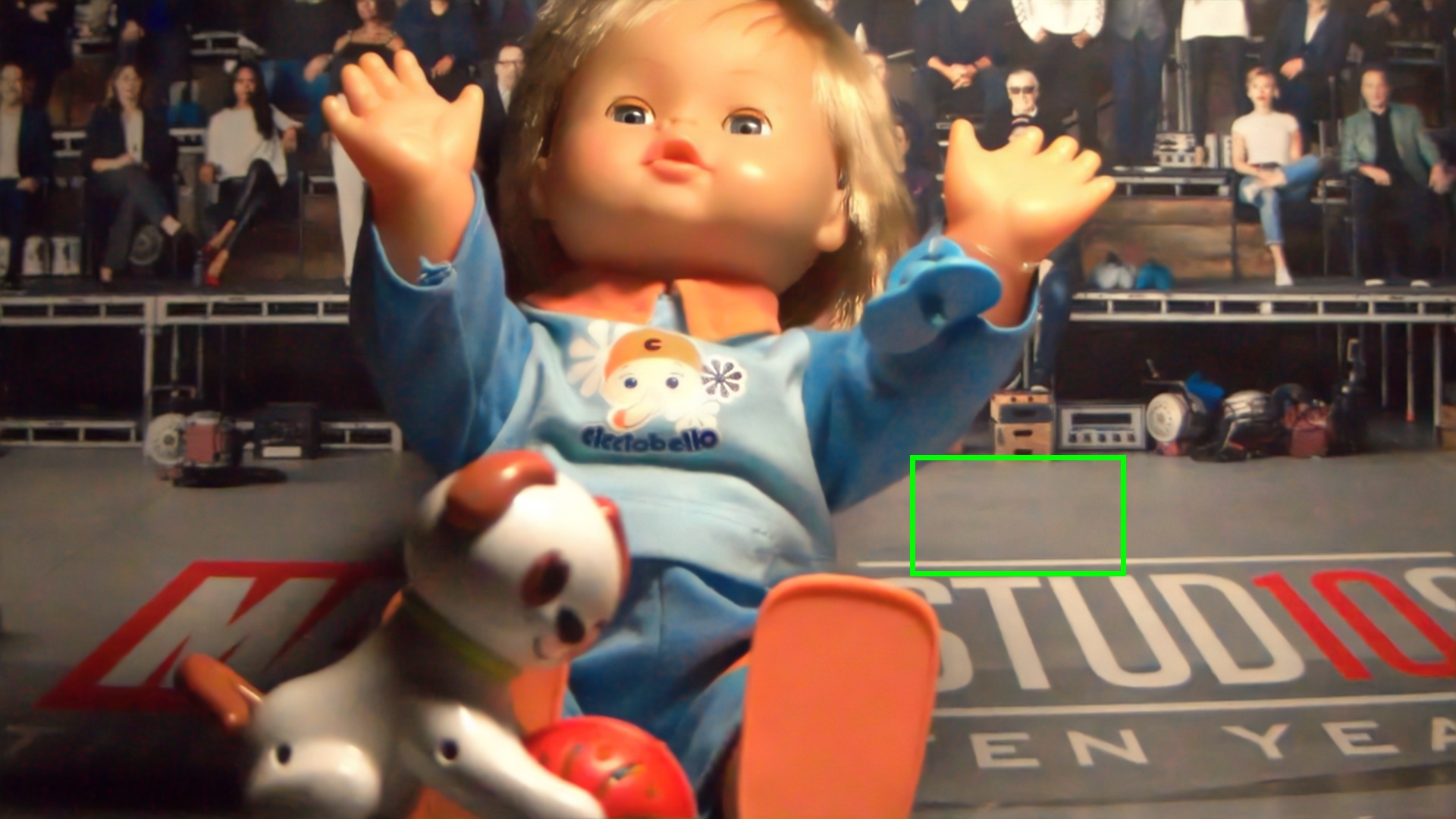}
		\caption{PC-UNet \\ 33.99 / 0.946}
		\label{fig:crvd_5_2_iso12800:pc_unet_rect}
	\end{subfigure}
	\begin{subfigure}{0.18\textwidth}
	    \captionsetup{justification=centering}
		\includegraphics[width=\textwidth]{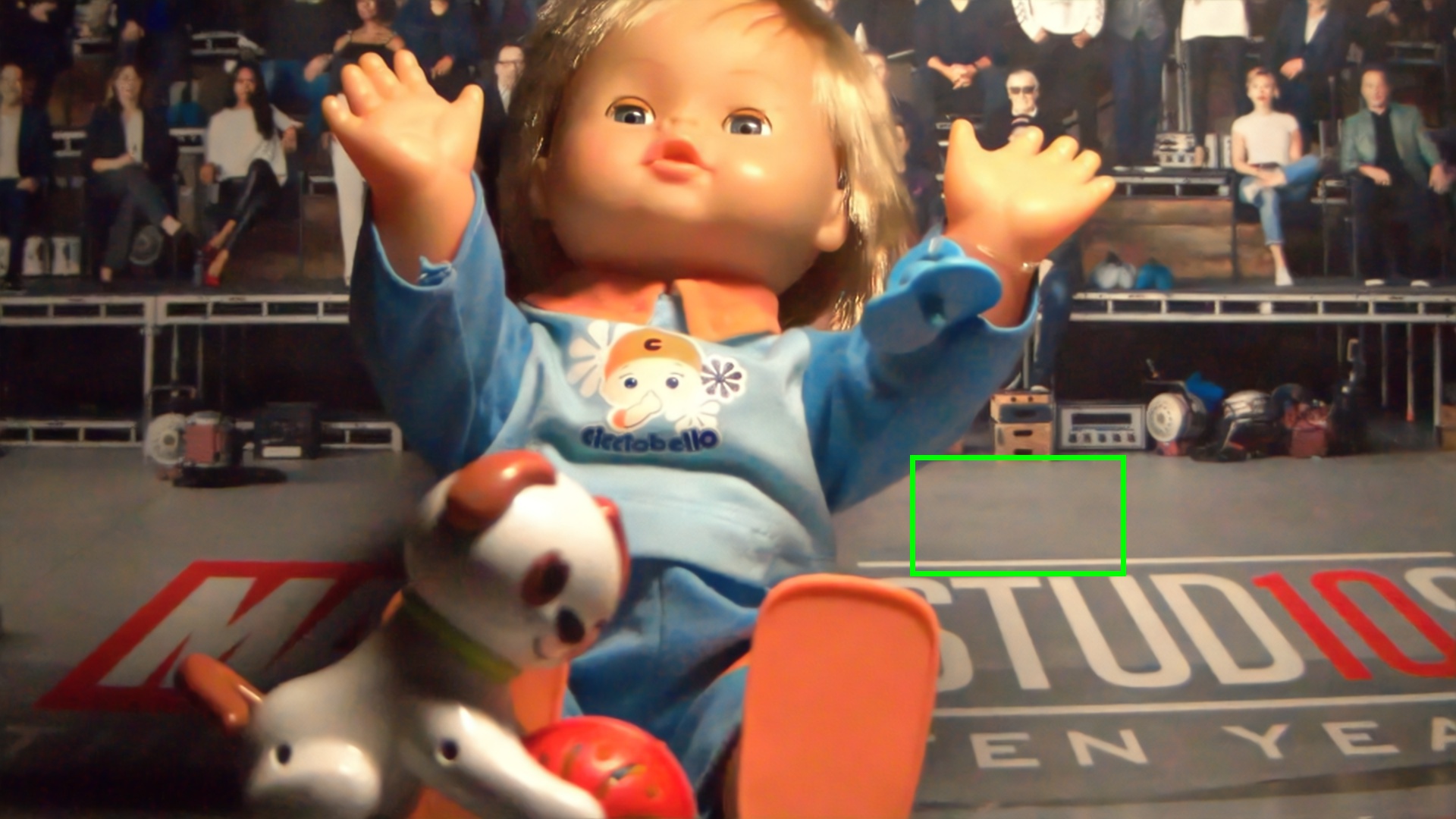}
		\caption{PC-DnCNN \\ 34.37 / 0.948}
		\label{fig:crvd_5_2_iso12800:pc_dncnn_rect}
	\end{subfigure}
	\begin{subfigure}{0.18\textwidth}
	    \captionsetup{justification=centering}
		\includegraphics[width=\textwidth]{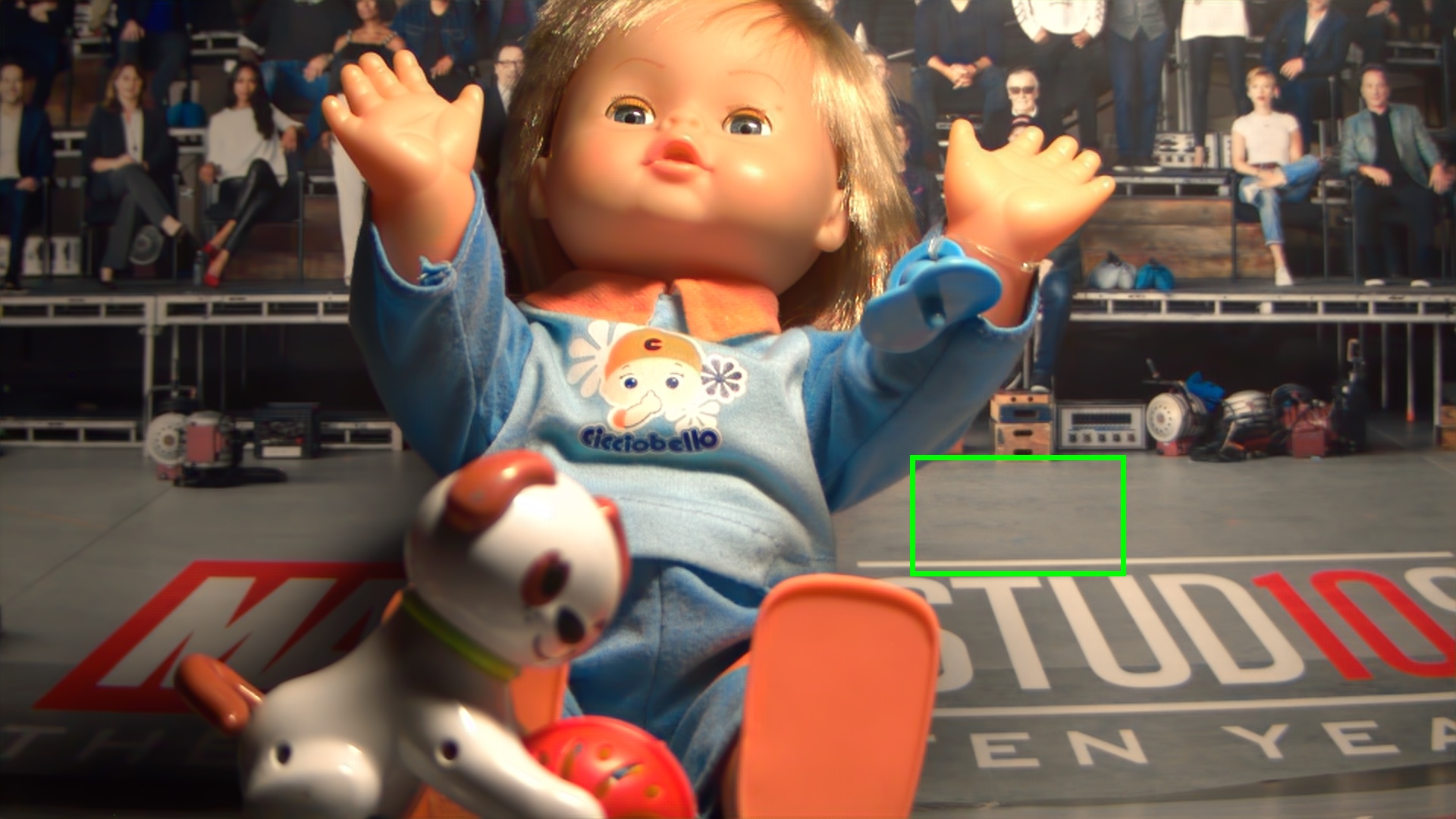}
		\caption{Clean \newline }
		\label{fig:crvd_5_2_iso12800:clean_rect}
	\end{subfigure}
	\begin{subfigure}{0.18\textwidth}
	    \captionsetup{justification=centering}
		\includegraphics[width=\textwidth]{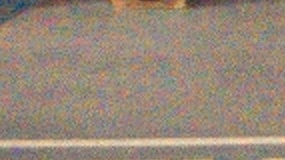}
		\caption{Noisy}
		\label{fig:crvd_5_2_iso12800:noisy_crop}
	\end{subfigure}
	\begin{subfigure}{0.18\textwidth}
	    \captionsetup{justification=centering}
		\includegraphics[width=\textwidth]{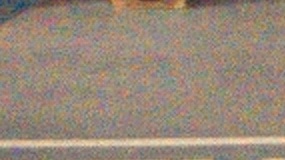}
		\caption{N2N}
		\label{fig:crvd_5_2_iso12800:n2n_crop}
	\end{subfigure}
	\begin{subfigure}{0.18\textwidth}
	    \captionsetup{justification=centering}
		\includegraphics[width=\textwidth]{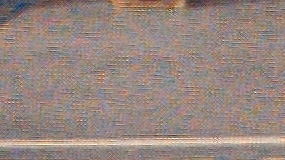}
		\caption{B2U}
		\label{fig:crvd_5_2_iso12800:b2u_crop}
	\end{subfigure}
	\begin{subfigure}{0.18\textwidth}
	    \captionsetup{justification=centering}
		\includegraphics[width=\textwidth]{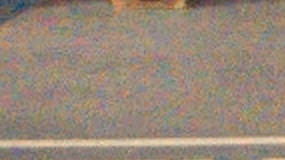}
		\caption{BM3D}
		\label{fig:crvd_5_2_iso12800:bm3d_crop}
	\end{subfigure}
	\begin{subfigure}{0.18\textwidth}
	    \captionsetup{justification=centering}
		\includegraphics[width=\textwidth]{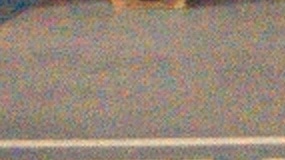}
		\caption{B-DnCNN}
		\label{fig:crvd_5_2_iso12800:b_dncnn_crop}
	\end{subfigure}
	\begin{subfigure}{0.18\textwidth}
	    \captionsetup{justification=centering}
		\includegraphics[width=\textwidth]{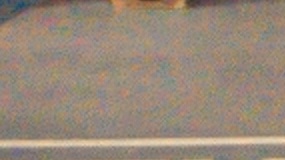}
		\caption{R2R}
		\label{fig:crvd_5_2_iso12800:r2r_crop}
	\end{subfigure}
	\begin{subfigure}{0.18\textwidth}
	    \captionsetup{justification=centering}
		\includegraphics[width=\textwidth]{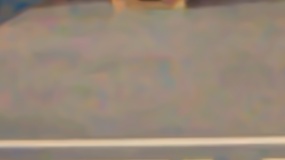}
		\caption{BM3D-O}
		\label{fig:crvd_5_2_iso12800:bm3d_opt_crop}
	\end{subfigure}
	\begin{subfigure}{0.18\textwidth}
	    \captionsetup{justification=centering}
		\includegraphics[width=\textwidth]{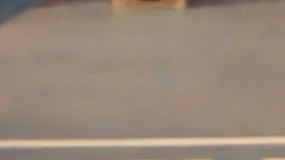}
		\caption{PC-UNet}
		\label{fig:crvd_5_2_iso12800:pc_unet_crop}
	\end{subfigure}
	\begin{subfigure}{0.18\textwidth}
	    \captionsetup{justification=centering}
		\includegraphics[width=\textwidth]{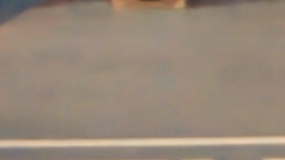}
		\caption{PC-DnCNN}
		\label{fig:crvd_5_2_iso12800:pc_dncnn_crop}
	\end{subfigure}
	\begin{subfigure}{0.18\textwidth}
	    \captionsetup{justification=centering}
		\includegraphics[width=\textwidth]{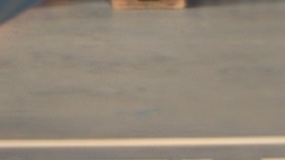}
		\caption{Clean}
		\label{fig:crvd_5_2_iso12800:clean_crop}
	\end{subfigure}
	\begin{subfigure}{0.18\textwidth}
	    \captionsetup{justification=centering}
		\includegraphics[width=\textwidth]{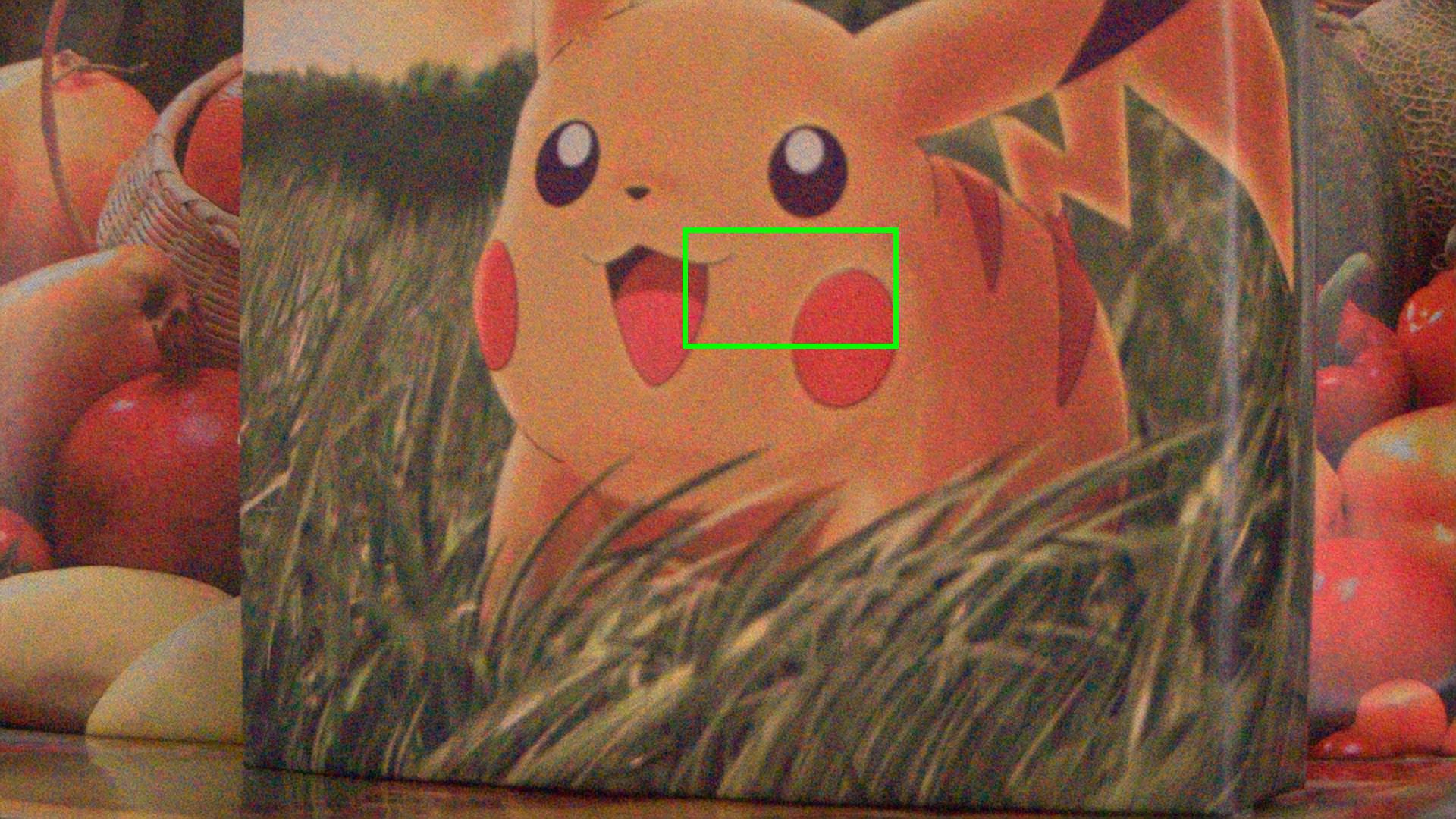}
		\caption{Noisy \\ 27.99 / 0.589}
		\label{fig:crvd_10_5_iso12800:noisy_rect}
	\end{subfigure}
	\begin{subfigure}{0.18\textwidth}
	    \captionsetup{justification=centering}
		\includegraphics[width=\textwidth]{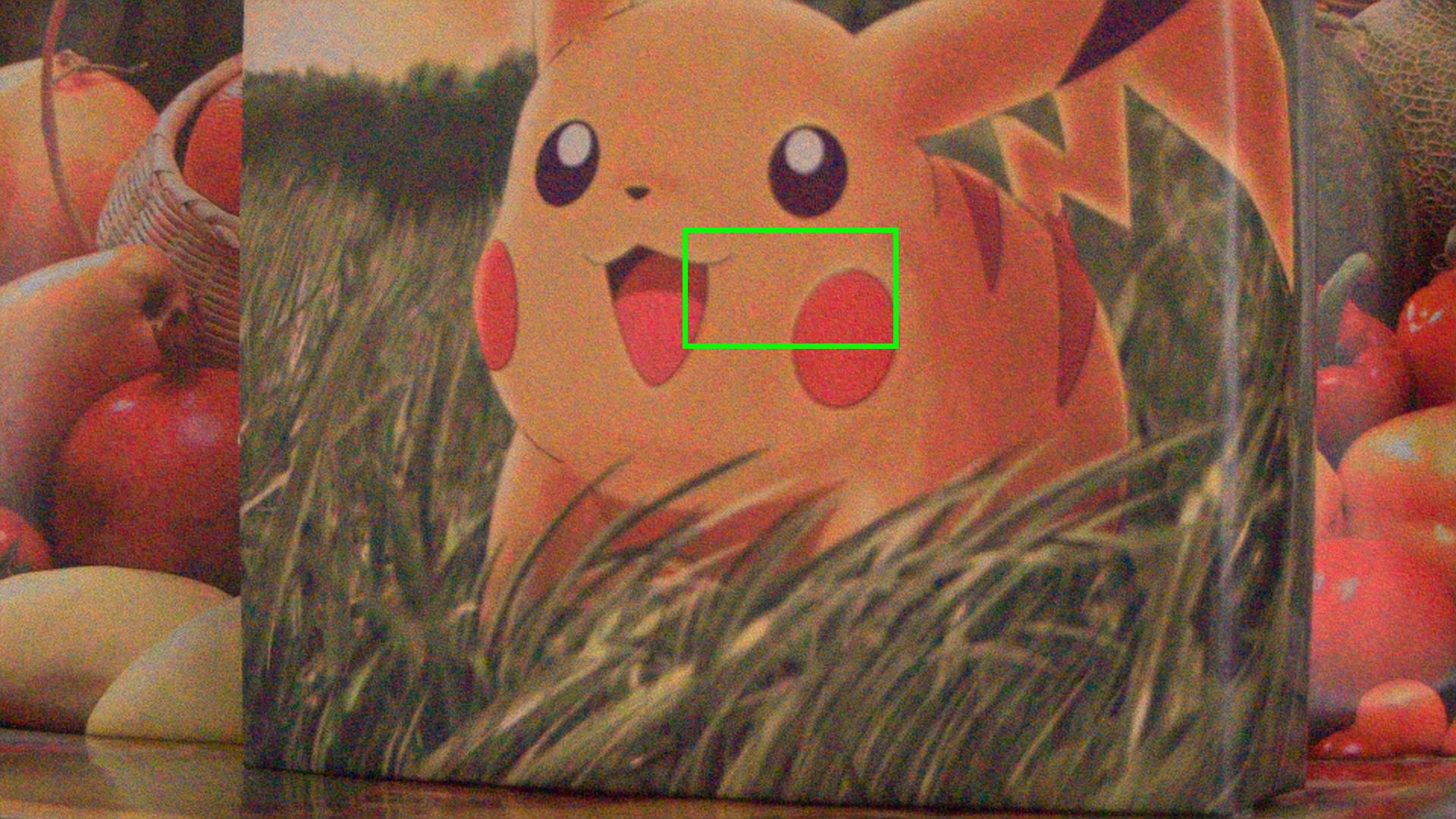}
		\caption{N2N \\ 28.05 / 0.594}
		\label{fig:crvd_10_5_iso12800:n2n_rect}
	\end{subfigure}
	\begin{subfigure}{0.18\textwidth}
	    \captionsetup{justification=centering}
		\includegraphics[width=\textwidth]{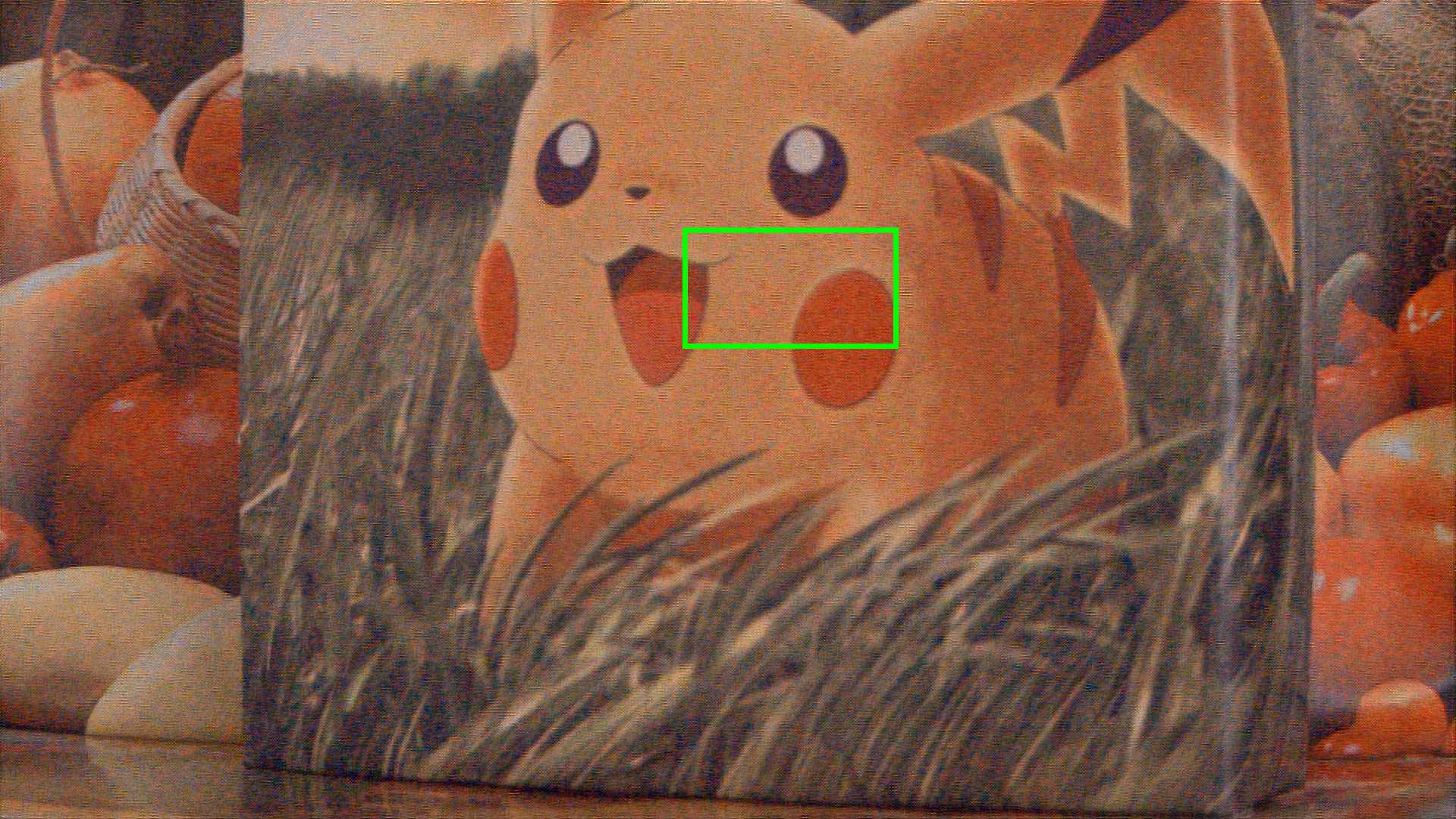}
		\caption{B2U \\ 24.46 / 0.403}
		\label{fig:crvd_10_5_iso12800:b2u_rect}
	\end{subfigure}
	\begin{subfigure}{0.18\textwidth}
	    \captionsetup{justification=centering}
		\includegraphics[width=\textwidth]{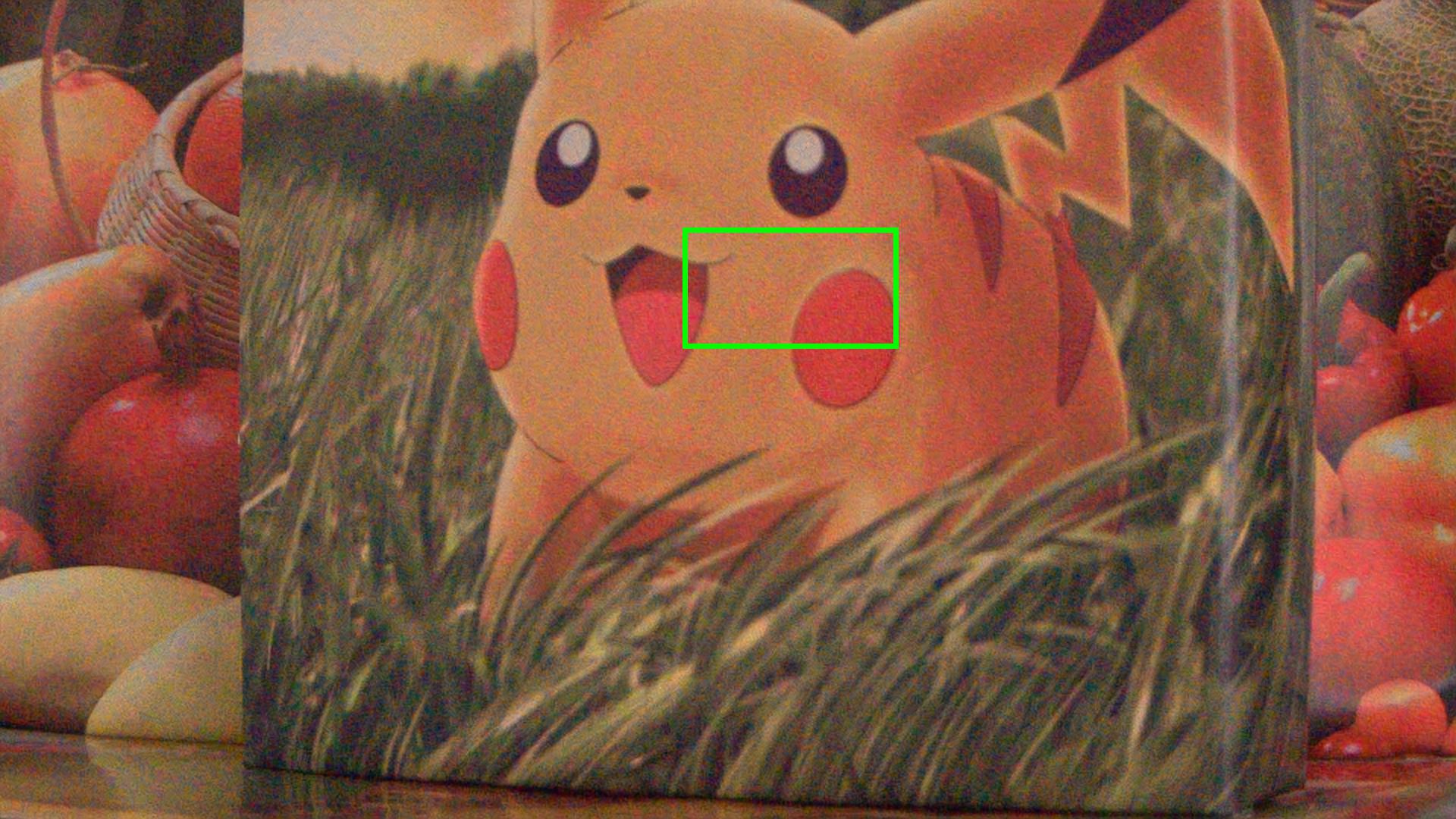}
		\caption{BM3D \\ 29.03 / 0.678}
		\label{fig:crvd_10_5_iso12800:bm3d_rect}
	\end{subfigure}
	\begin{subfigure}{0.18\textwidth}
	    \captionsetup{justification=centering}
		\includegraphics[width=\textwidth]{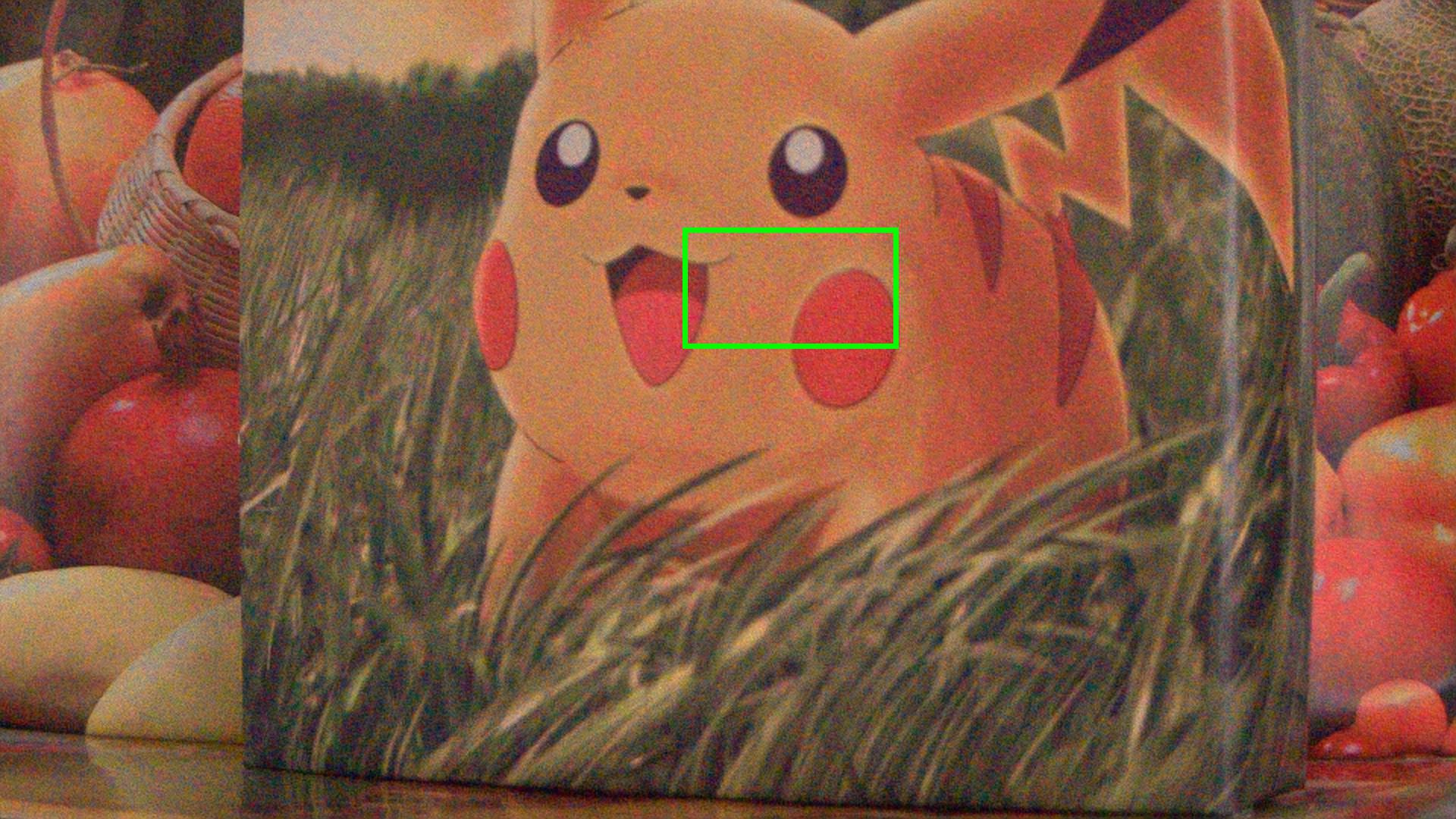}
		\caption{B-DnCNN \\ 28.07 / 0.593}
		\label{fig:crvd_10_5_iso12800:b_dncnn_rect}
	\end{subfigure}
	\begin{subfigure}{0.18\textwidth}
	    \captionsetup{justification=centering}
		\includegraphics[width=\textwidth]{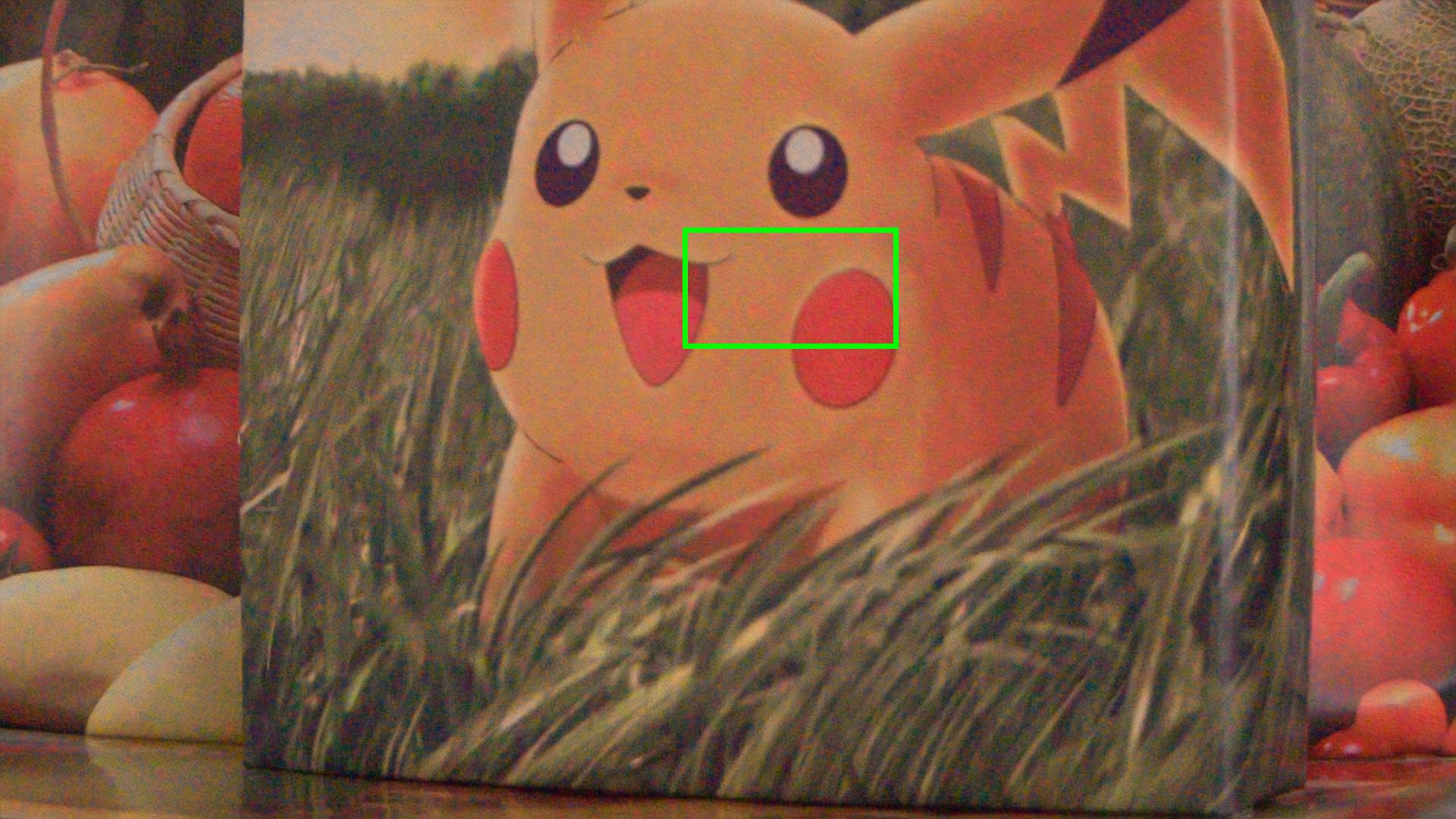}
		\caption{R2R \\ 30.27 / 0.762}
		\label{fig:crvd_10_5_iso12800:r2r_rect}
	\end{subfigure}
	\begin{subfigure}{0.18\textwidth}
	    \captionsetup{justification=centering}
		\includegraphics[width=\textwidth]{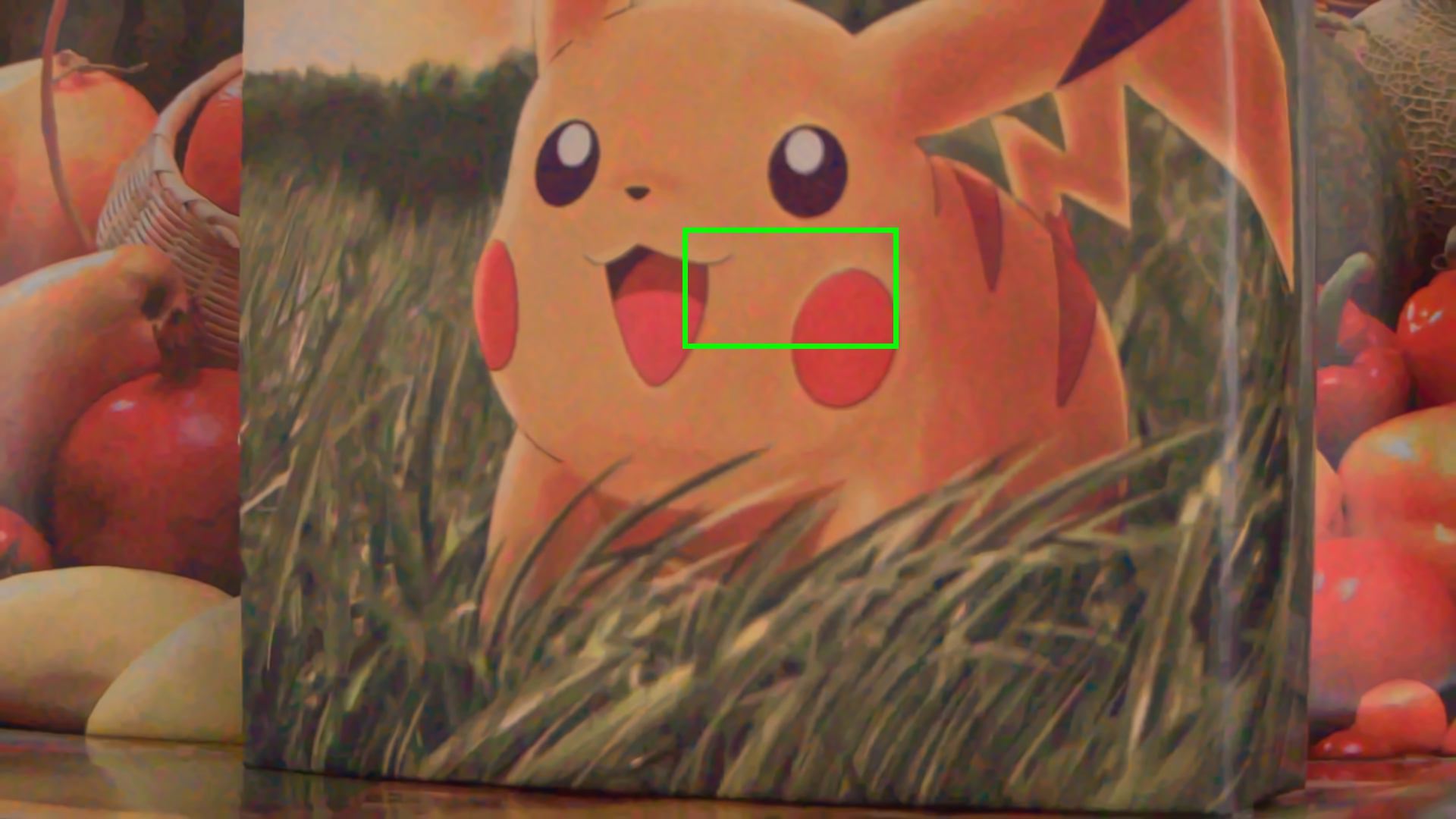}
		\caption{BM3D-O \\ 33.52 / 0.937}
		\label{fig:crvd_10_5_iso12800:bm3d_opt_rect}
	\end{subfigure}
	\begin{subfigure}{0.18\textwidth}
	    \captionsetup{justification=centering}
		\includegraphics[width=\textwidth]{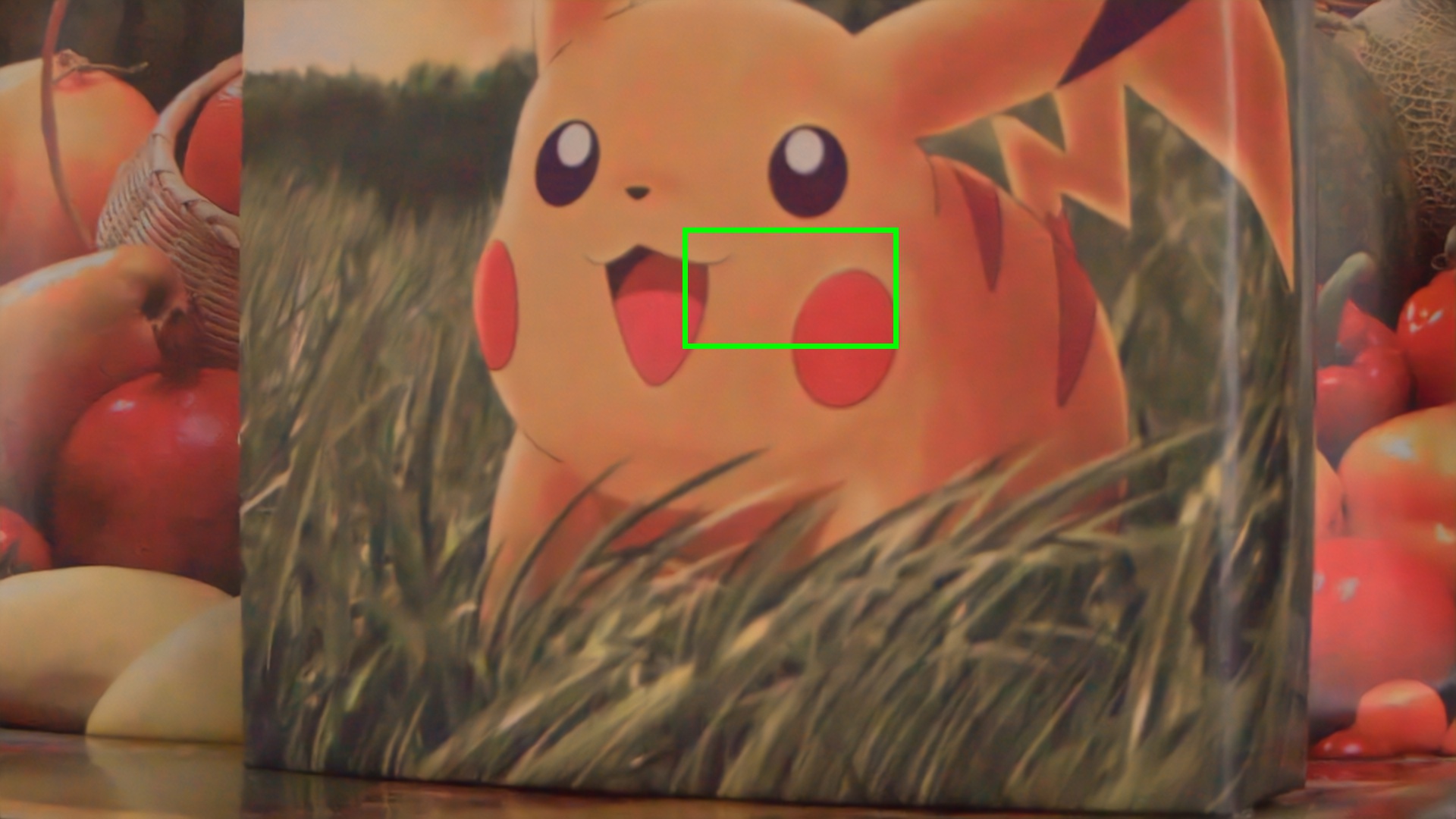}
		\caption{PC-UNet \\ 34.97 / 0.959}
		\label{fig:crvd_10_5_iso12800:pc_unet_rect}
	\end{subfigure}
	\begin{subfigure}{0.18\textwidth}
	    \captionsetup{justification=centering}
		\includegraphics[width=\textwidth]{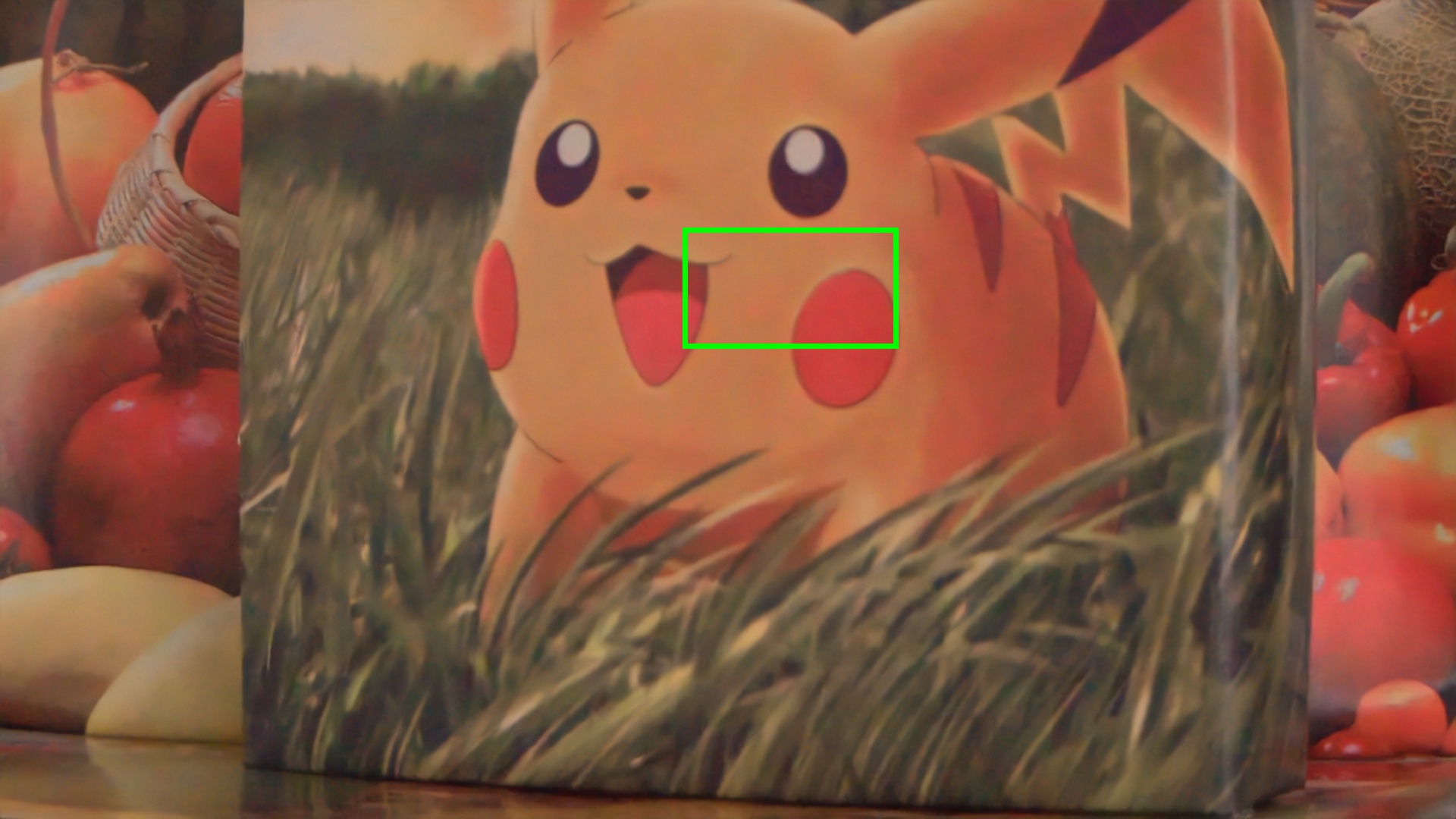}
		\caption{PC-DnCNN \\ 34.90 / 0.956}
		\label{fig:crvd_10_5_iso12800:pc_dncnn_rect}
	\end{subfigure}
	\begin{subfigure}{0.18\textwidth}
	    \captionsetup{justification=centering}
		\includegraphics[width=\textwidth]{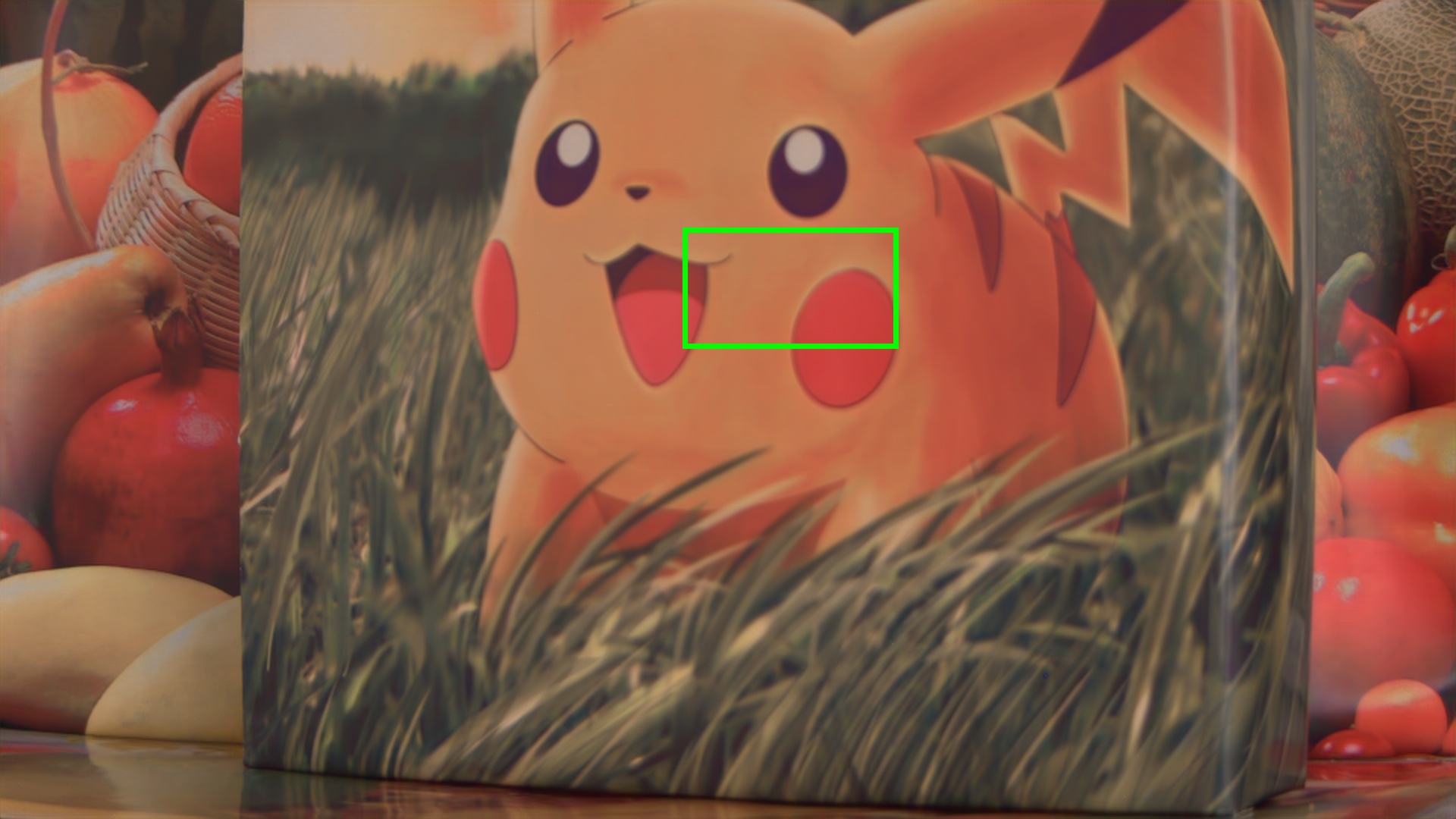}
		\caption{Clean \newline }
		\label{fig:crvd_10_5_iso12800:clean_rect}
	\end{subfigure}
	\begin{subfigure}{0.18\textwidth}
	    \captionsetup{justification=centering}
		\includegraphics[width=\textwidth]{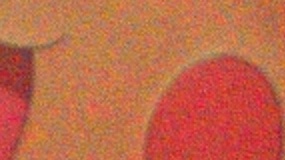}
		\caption{Noisy}
		\label{fig:crvd_10_5_iso12800:noisy_crop}
	\end{subfigure}
	\begin{subfigure}{0.18\textwidth}
	    \captionsetup{justification=centering}
		\includegraphics[width=\textwidth]{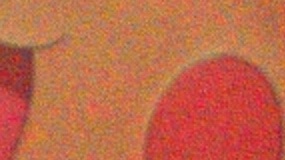}
		\caption{N2N}
		\label{fig:crvd_10_5_iso12800:n2n_crop}
	\end{subfigure}
	\begin{subfigure}{0.18\textwidth}
	    \captionsetup{justification=centering}
		\includegraphics[width=\textwidth]{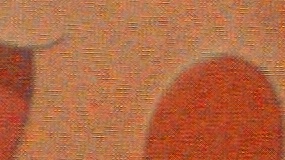}
		\caption{B2U}
		\label{fig:crvd_10_5_iso12800:b2u_crop}
	\end{subfigure}
	\begin{subfigure}{0.18\textwidth}
	    \captionsetup{justification=centering}
		\includegraphics[width=\textwidth]{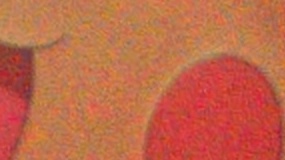}
		\caption{BM3D}
		\label{fig:crvd_10_5_iso12800:bm3d_crop}
	\end{subfigure}
	\begin{subfigure}{0.18\textwidth}
	    \captionsetup{justification=centering}
		\includegraphics[width=\textwidth]{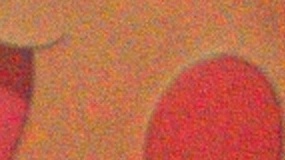}
		\caption{B-DnCNN}
		\label{fig:crvd_10_5_iso12800:b_dncnn_crop}
	\end{subfigure}
	\begin{subfigure}{0.18\textwidth}
	    \captionsetup{justification=centering}
		\includegraphics[width=\textwidth]{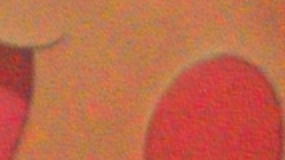}
		\caption{R2R}
		\label{fig:crvd_10_5_iso12800:r2r_crop}
	\end{subfigure}
	\begin{subfigure}{0.18\textwidth}
	    \captionsetup{justification=centering}
		\includegraphics[width=\textwidth]{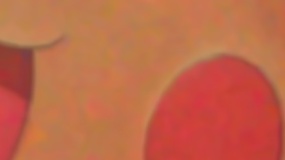}
		\caption{BM3D-O}
		\label{fig:crvd_10_5_iso12800:bm3d_opt_crop}
	\end{subfigure}
	\begin{subfigure}{0.18\textwidth}
	    \captionsetup{justification=centering}
		\includegraphics[width=\textwidth]{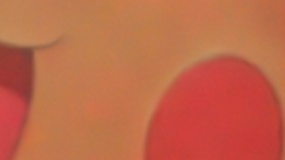}
		\caption{PC-UNet}
		\label{fig:crvd_10_5_iso12800:pc_unet_crop}
	\end{subfigure}
	\begin{subfigure}{0.18\textwidth}
	    \captionsetup{justification=centering}
		\includegraphics[width=\textwidth]{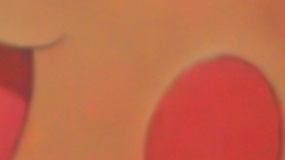}
		\caption{PC-DnCNN}
		\label{fig:crvd_10_5_iso12800:pc_dncnn_crop}
	\end{subfigure}
	\begin{subfigure}{0.18\textwidth}
	    \captionsetup{justification=centering}
		\includegraphics[width=\textwidth]{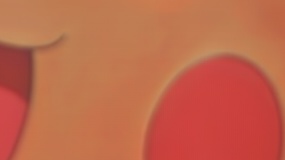}
		\caption{Clean}
		\label{fig:crvd_10_5_iso12800:clean_crop}
	\end{subfigure}
	\caption{Denoising examples with real-world noise. The first four rows show frame 3 of scene 6. The last four rows present frame 6 of scene 11. Both images are captured with ISO 12800. As can be seen, oracle BM3D leaves a substantial amount of low-frequency noise unfiltered, while other algorithms, except ours (PC-UNet and PC-DnCNN), do not succeed in removing the noise.}
	\label{fig:crvd_5_2_iso12800_10_5_iso12800}
\end{figure*}

\begin{figure*}
    \centering
	\begin{subfigure}{0.18\textwidth}
	    \captionsetup{justification=centering}
		\includegraphics[width=\textwidth]{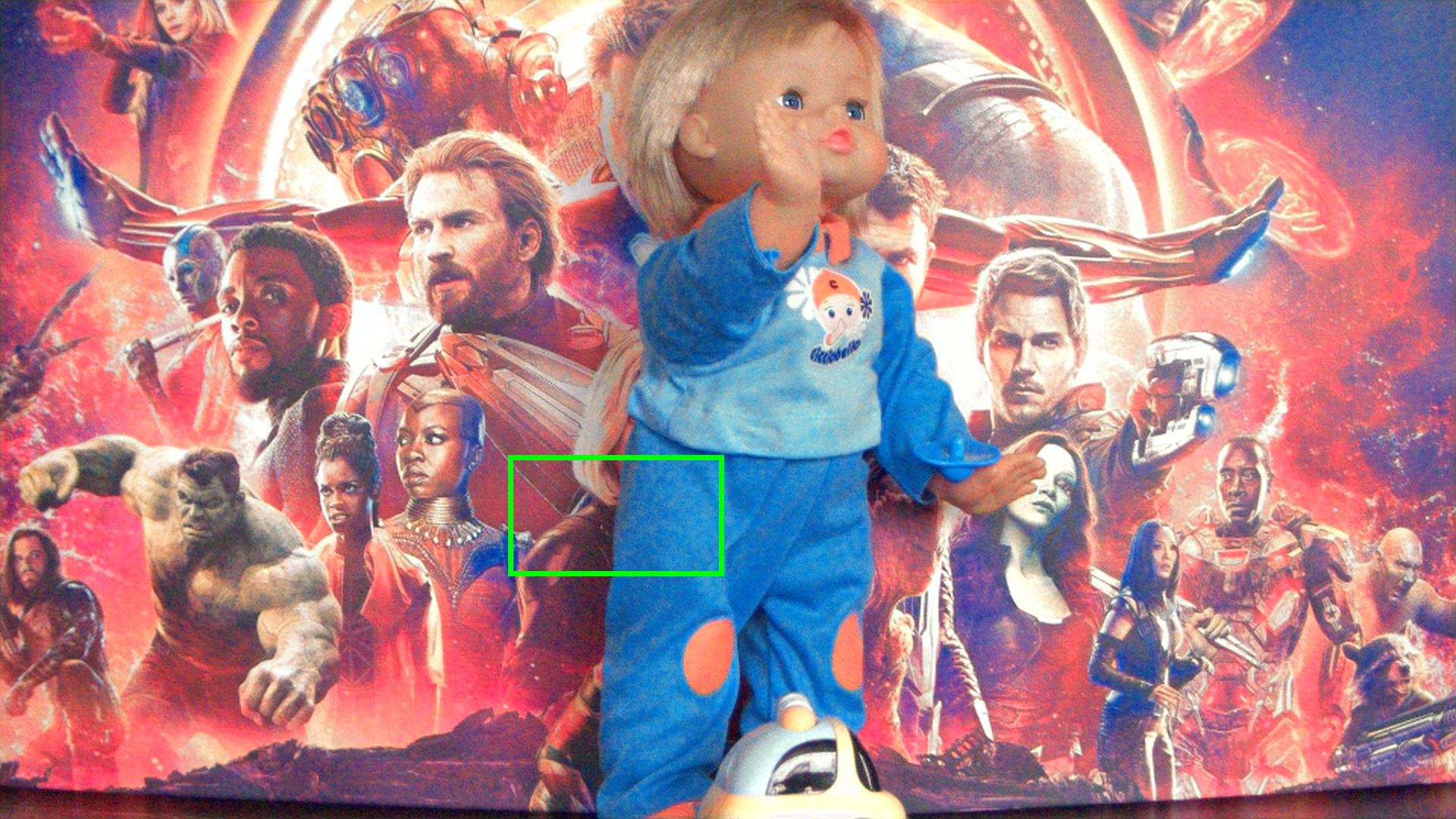}
		\caption{Noisy \\ 32.88 / 0.864}
		\label{fig:crvd_1_5_iso6400:noisy_rect}
	\end{subfigure}
	\begin{subfigure}{0.18\textwidth}
	    \captionsetup{justification=centering}
		\includegraphics[width=\textwidth]{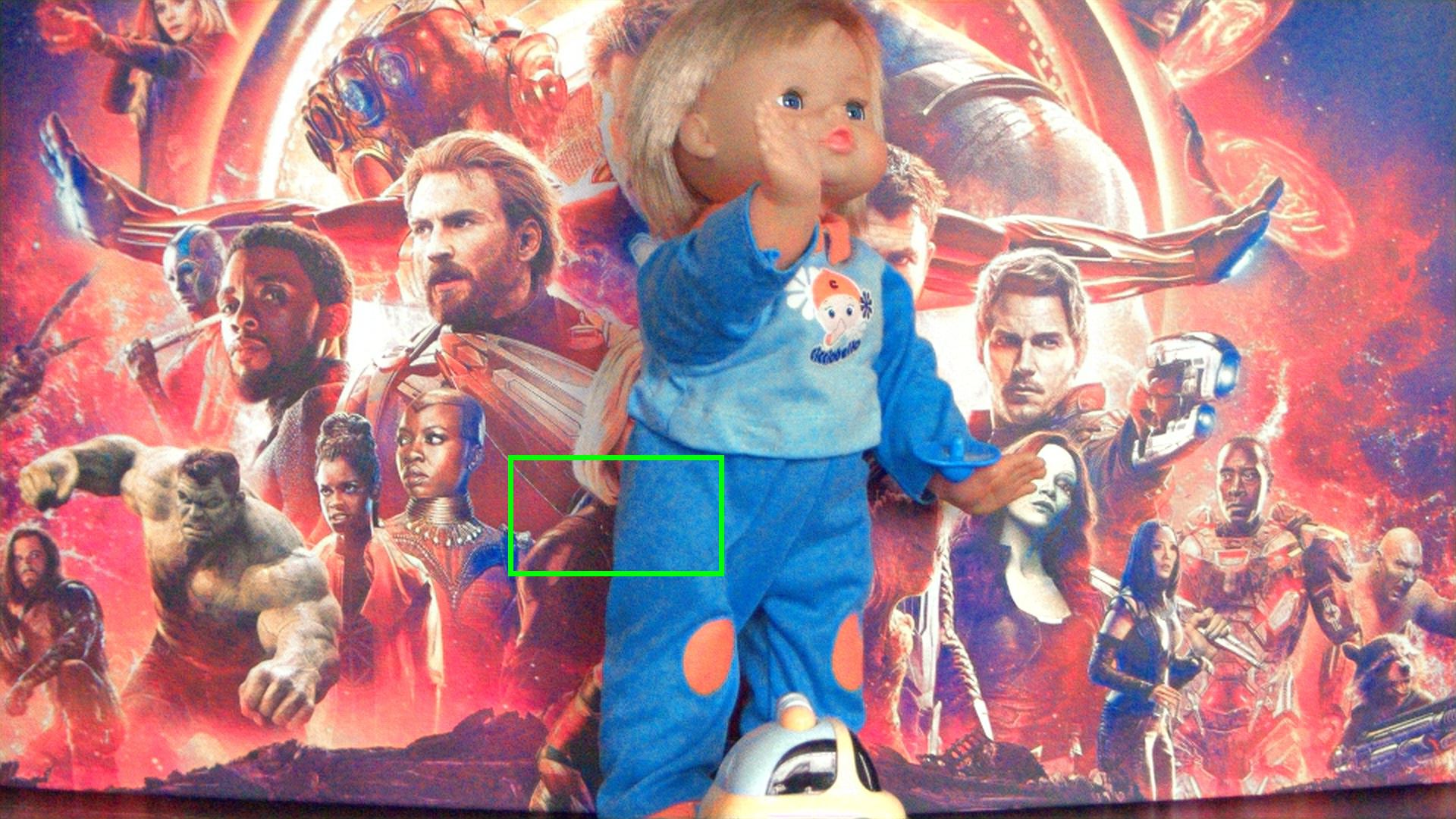}
		\caption{N2N \\ 32.96 / 0.866}
		\label{fig:crvd_1_5_iso6400:n2n_rect}
	\end{subfigure}
	\begin{subfigure}{0.18\textwidth}
	    \captionsetup{justification=centering}
		\includegraphics[width=\textwidth]{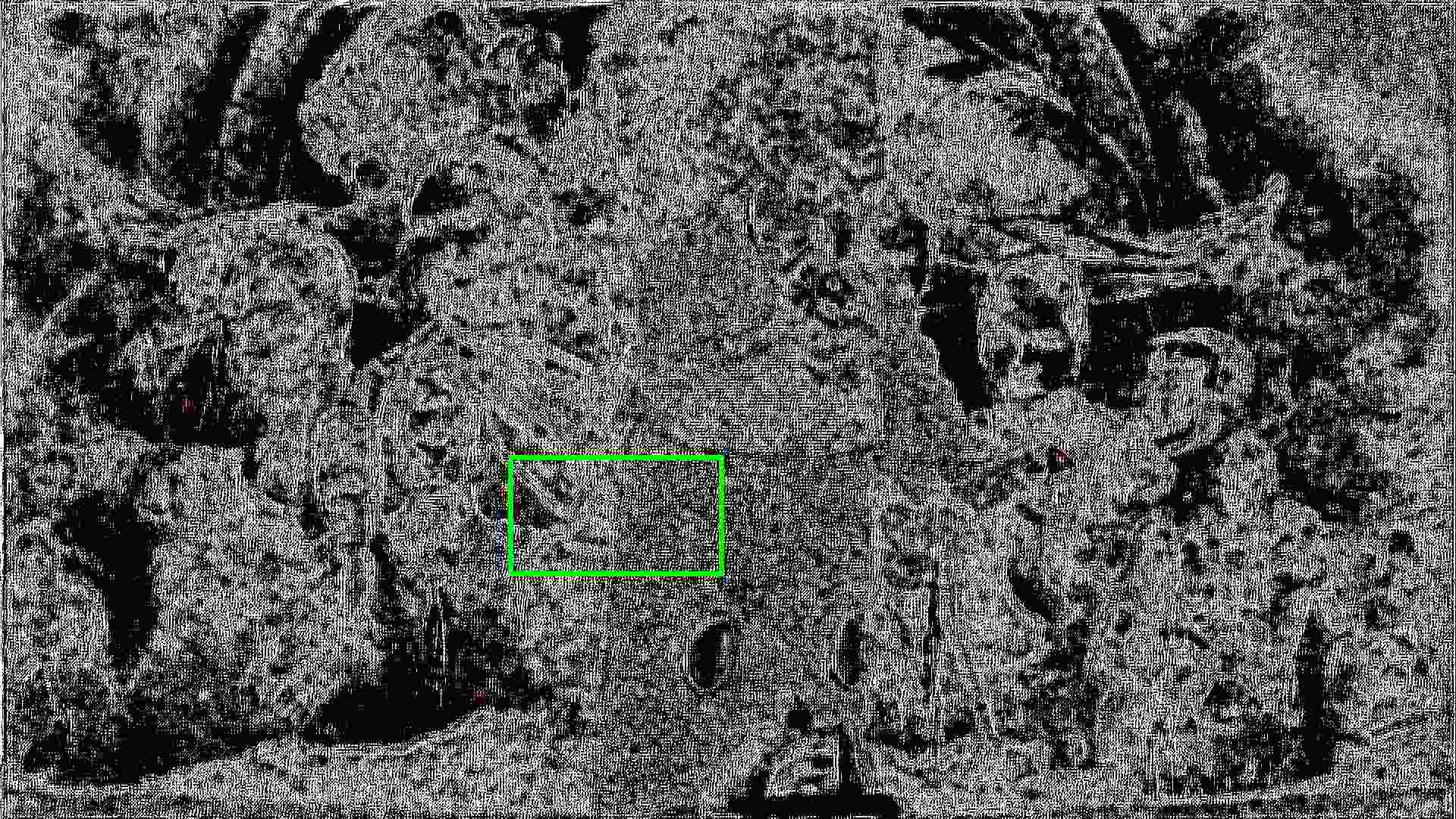}
		\caption{B2U \\ 4.36 / -0.001}
		\label{fig:crvd_1_5_iso6400:b2u_rect}
	\end{subfigure}
	\begin{subfigure}{0.18\textwidth}
	    \captionsetup{justification=centering}
		\includegraphics[width=\textwidth]{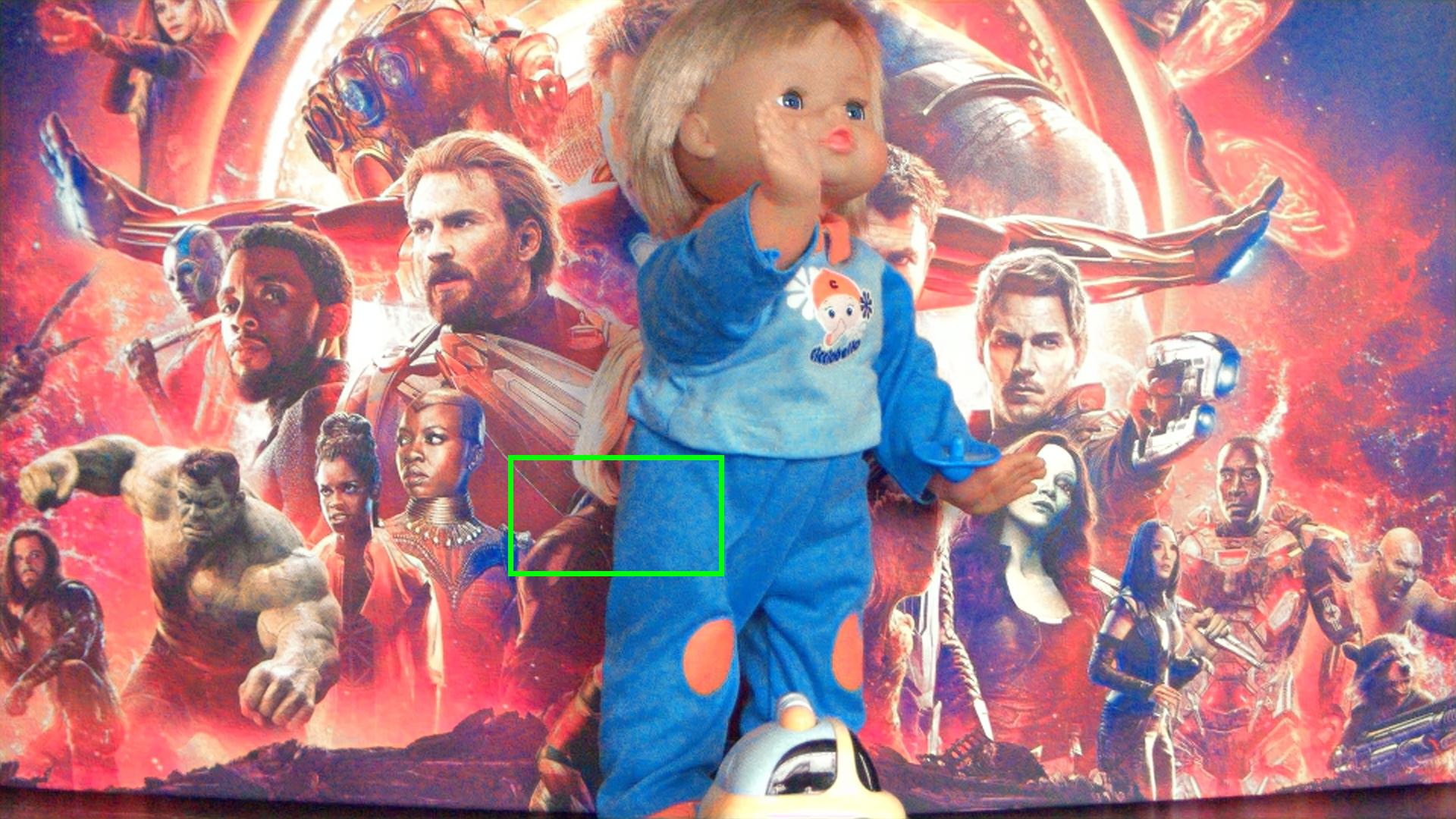}
		\caption{BM3D \\ 33.68 / 0.894}
		\label{fig:crvd_1_5_iso6400:bm3d_rect}
	\end{subfigure}
	\begin{subfigure}{0.18\textwidth}
	    \captionsetup{justification=centering}
		\includegraphics[width=\textwidth]{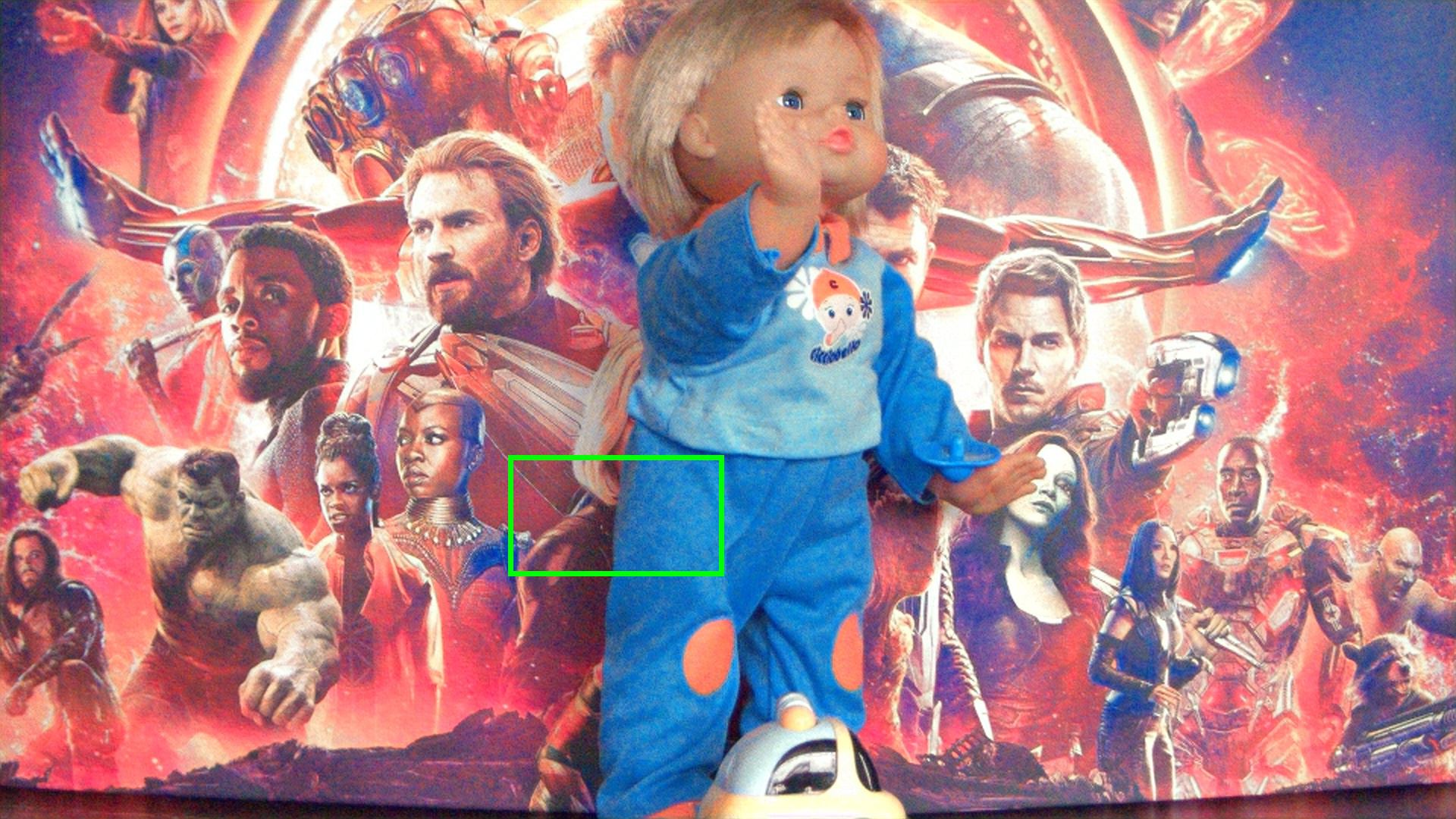}
		\caption{B-DnCNN \\ 32.95 / 0.866}
		\label{fig:crvd_1_5_iso6400:b_dncnn_rect}
	\end{subfigure}
	\begin{subfigure}{0.18\textwidth}
	    \captionsetup{justification=centering}
		\includegraphics[width=\textwidth]{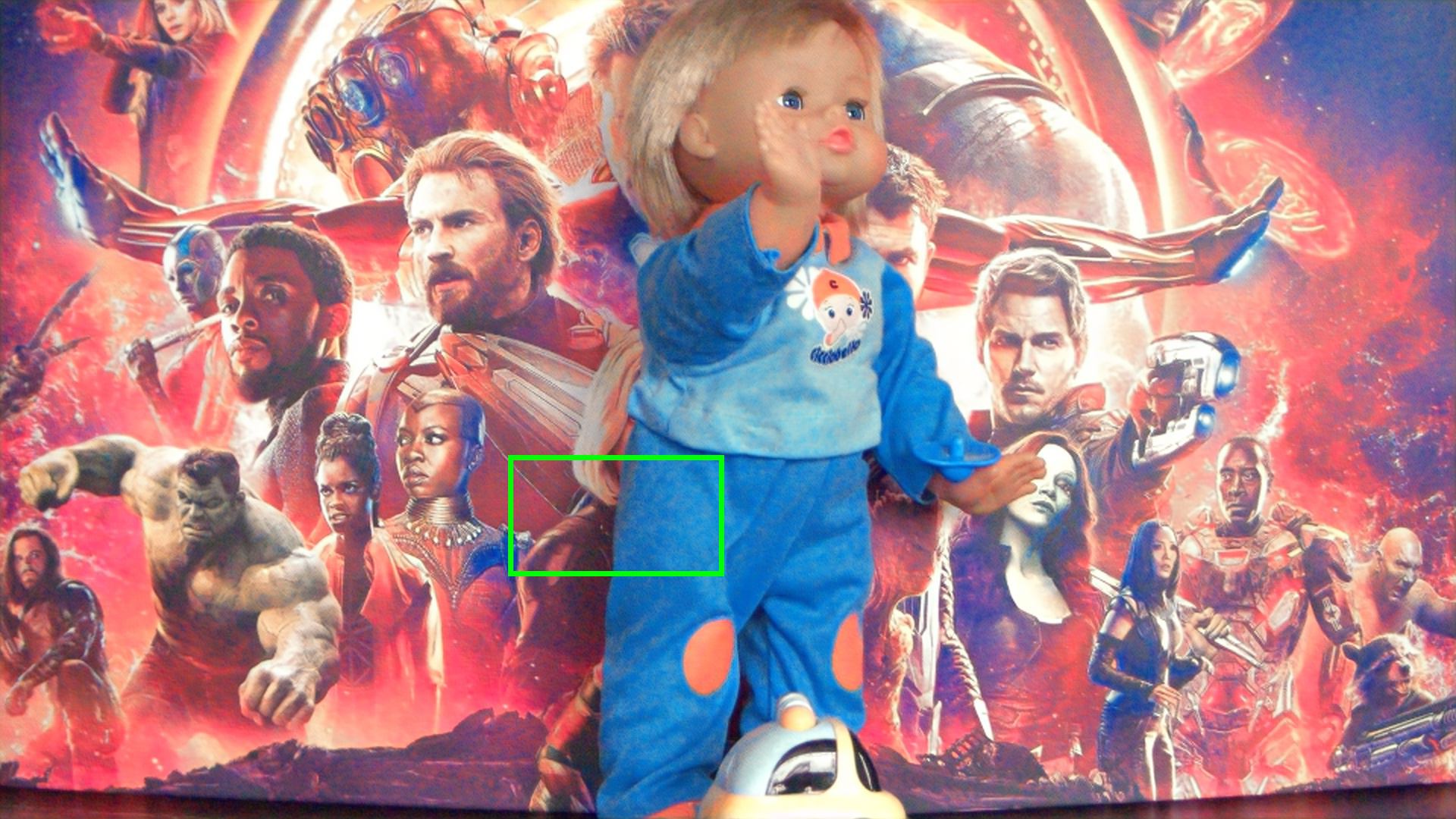}
		\caption{R2R \\ 34.65 / 0.921}
		\label{fig:crvd_1_5_iso6400:r2r_rect}
	\end{subfigure}
	\begin{subfigure}{0.18\textwidth}
	    \captionsetup{justification=centering}
		\includegraphics[width=\textwidth]{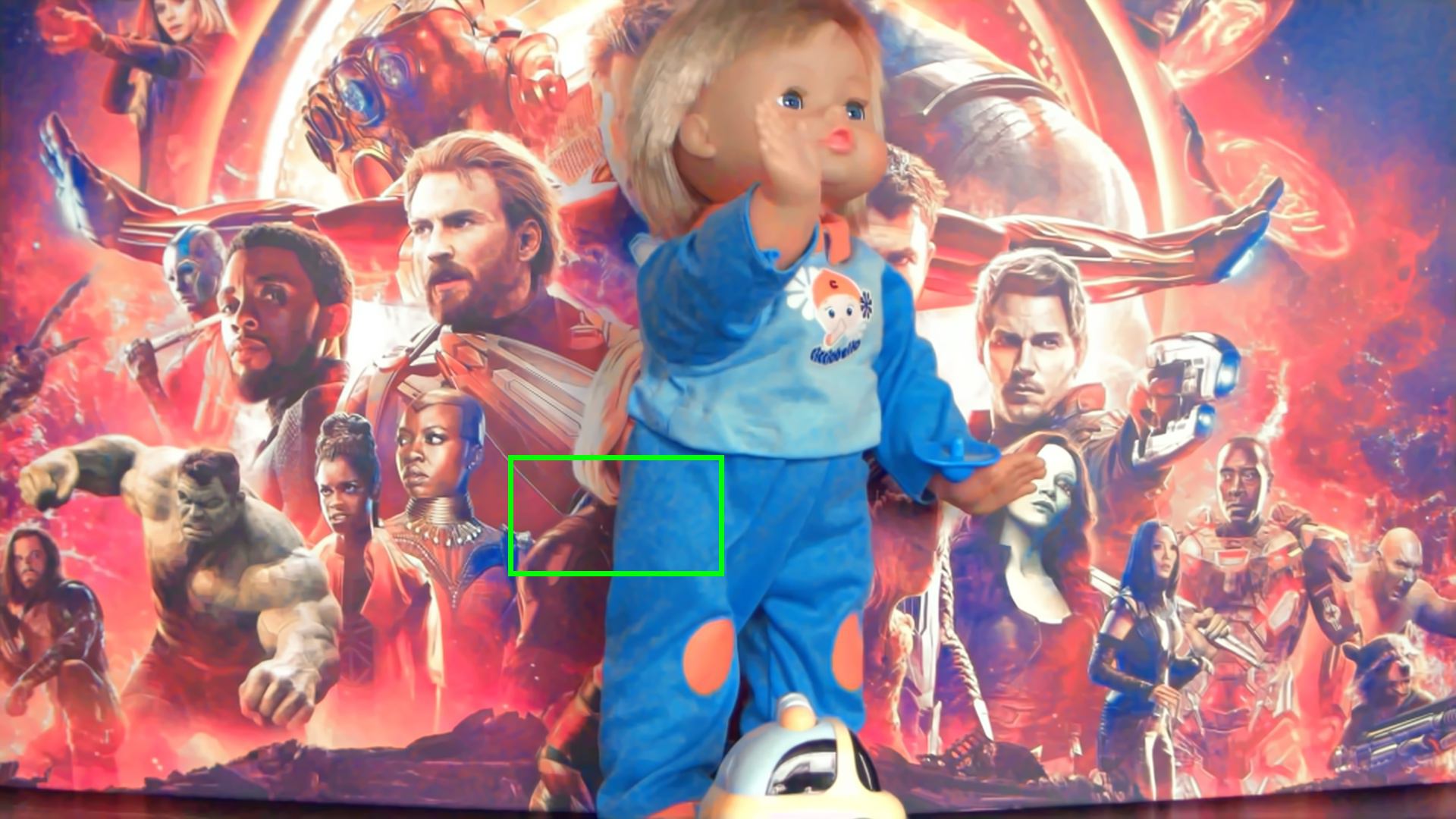}
		\caption{BM3D-O \\ 34.59 / 0.925}
		\label{fig:crvd_1_5_iso6400:bm3d_opt_rect}
	\end{subfigure}
	\begin{subfigure}{0.18\textwidth}
	    \captionsetup{justification=centering}
		\includegraphics[width=\textwidth]{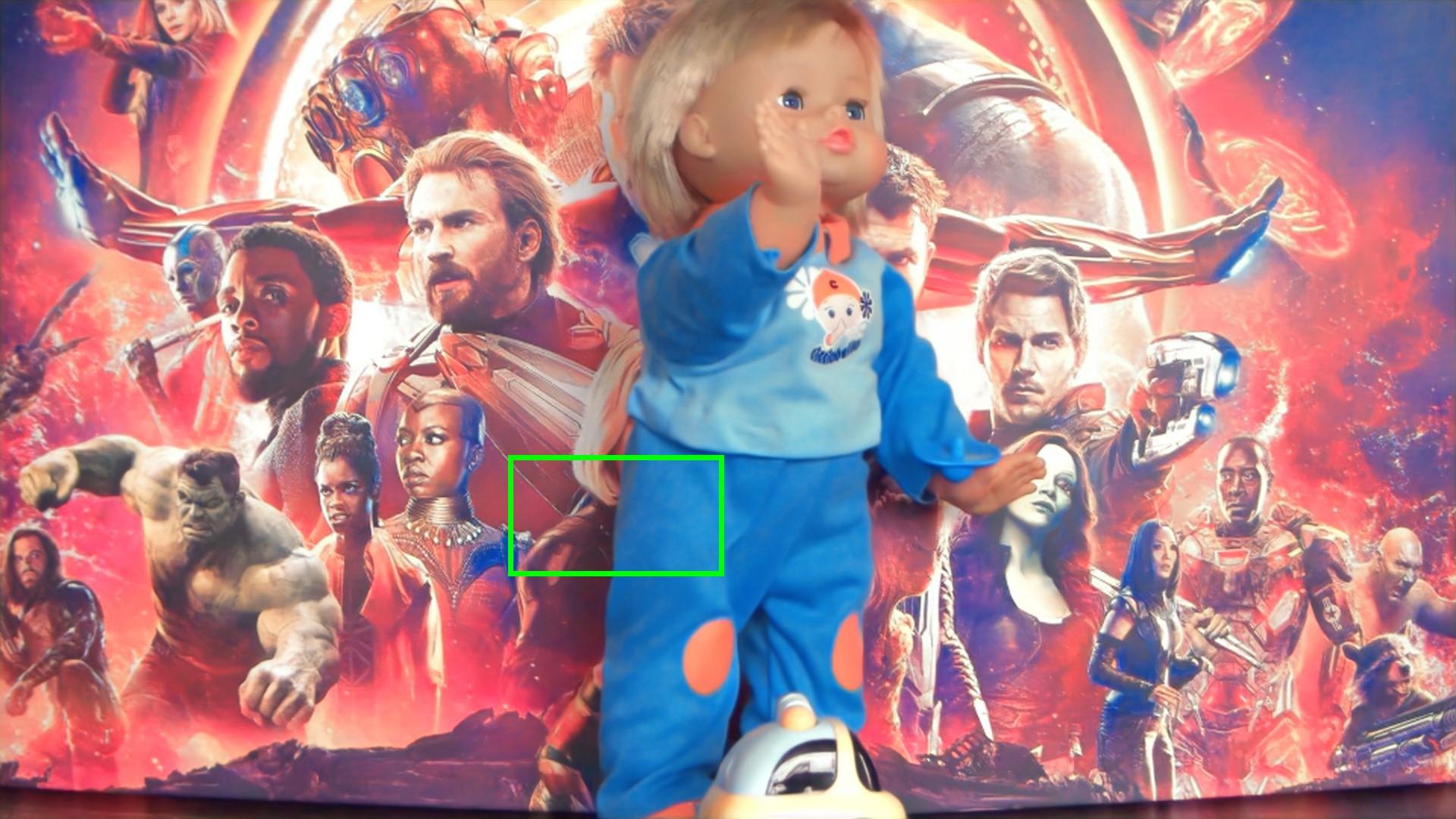}
		\caption{PC-UNet \\ 35.82 / 0.948}
		\label{fig:crvd_1_5_iso6400:pc_unet_rect}
	\end{subfigure}
	\begin{subfigure}{0.18\textwidth}
	    \captionsetup{justification=centering}
		\includegraphics[width=\textwidth]{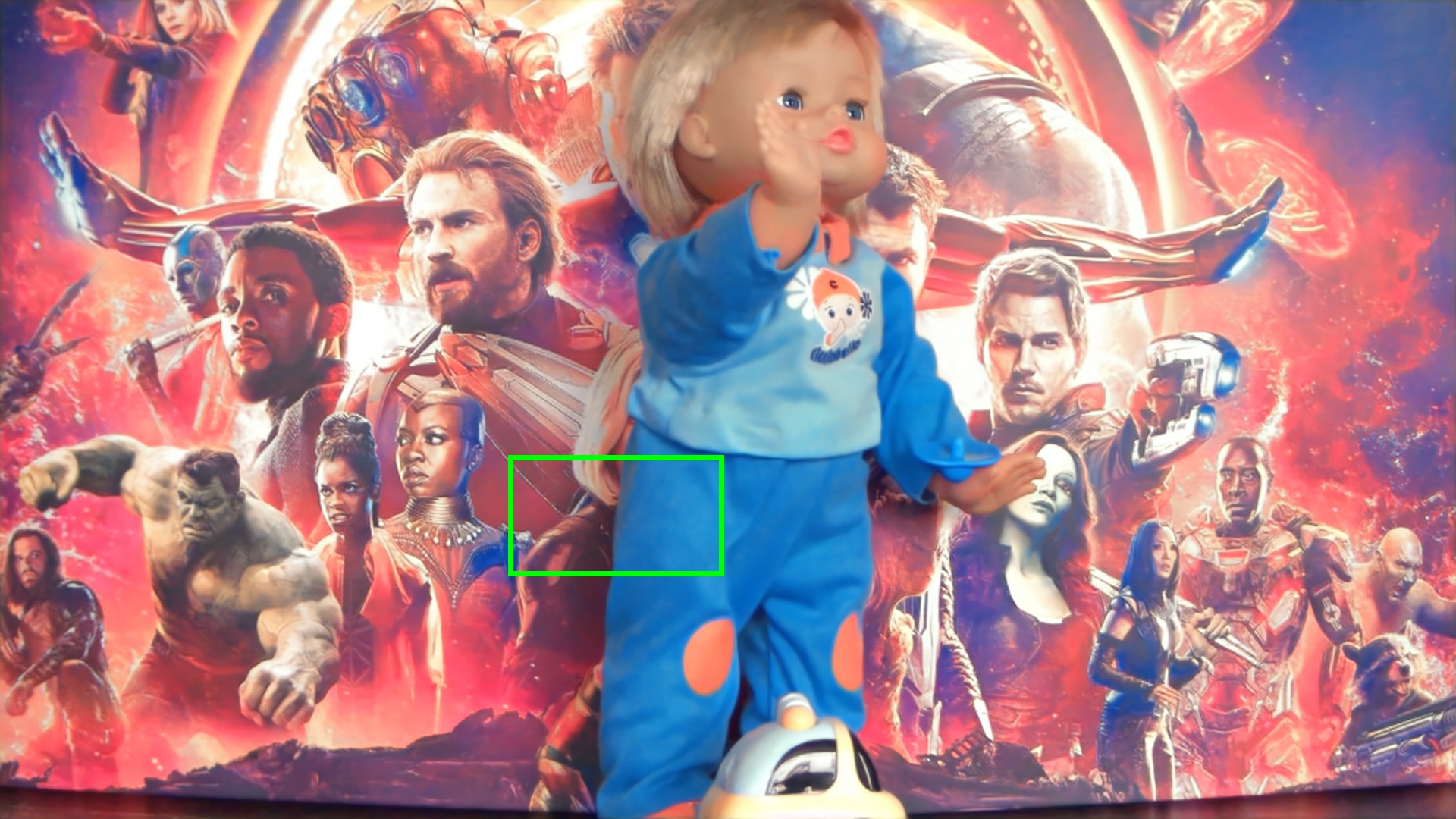}
		\caption{PC-DnCNN \\ 35.80 / 0.947}
		\label{fig:crvd_1_5_iso6400:pc_dncnn_rect}
	\end{subfigure}
	\begin{subfigure}{0.18\textwidth}
	    \captionsetup{justification=centering}
		\includegraphics[width=\textwidth]{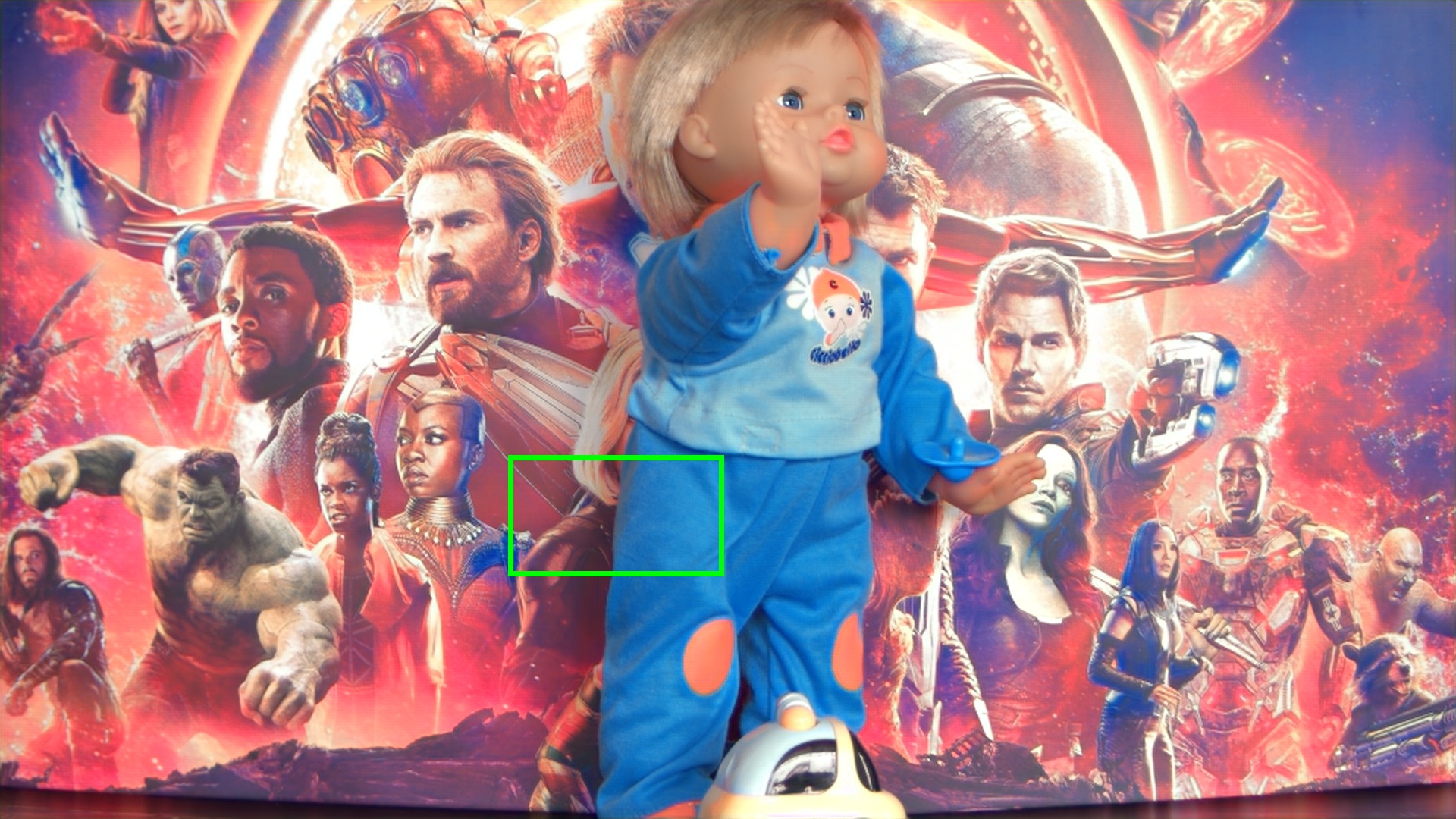}
		\caption{Clean \newline }
		\label{fig:crvd_1_5_iso6400:clean_rect}
	\end{subfigure}
	\begin{subfigure}{0.18\textwidth}
	    \captionsetup{justification=centering}
		\includegraphics[width=\textwidth]{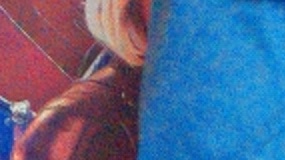}
		\caption{Noisy}
		\label{fig:crvd_1_5_iso6400:noisy_crop}
	\end{subfigure}
	\begin{subfigure}{0.18\textwidth}
	    \captionsetup{justification=centering}
		\includegraphics[width=\textwidth]{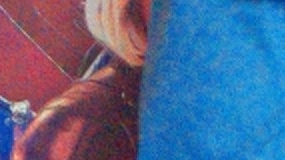}
		\caption{N2N}
		\label{fig:crvd_1_5_iso6400:n2n_crop}
	\end{subfigure}
	\begin{subfigure}{0.18\textwidth}
	    \captionsetup{justification=centering}
		\includegraphics[width=\textwidth]{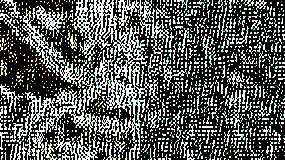}
		\caption{B2U}
		\label{fig:crvd_1_5_iso6400:b2u_crop}
	\end{subfigure}
	\begin{subfigure}{0.18\textwidth}
	    \captionsetup{justification=centering}
		\includegraphics[width=\textwidth]{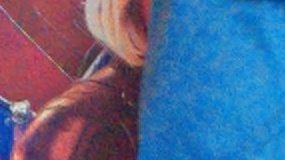}
		\caption{BM3D}
		\label{fig:crvd_1_5_iso6400:bm3d_crop}
	\end{subfigure}
	\begin{subfigure}{0.18\textwidth}
	    \captionsetup{justification=centering}
		\includegraphics[width=\textwidth]{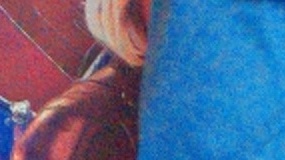}
		\caption{B-DnCNN}
		\label{fig:crvd_1_5_iso6400:b_dncnn_crop}
	\end{subfigure}
	\begin{subfigure}{0.18\textwidth}
	    \captionsetup{justification=centering}
		\includegraphics[width=\textwidth]{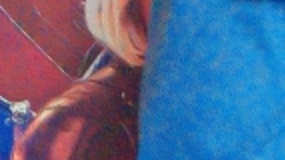}
		\caption{R2R}
		\label{fig:crvd_1_5_iso6400:r2r_crop}
	\end{subfigure}
	\begin{subfigure}{0.18\textwidth}
	    \captionsetup{justification=centering}
		\includegraphics[width=\textwidth]{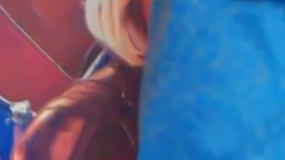}
		\caption{BM3D-O}
		\label{fig:crvd_1_5_iso6400:bm3d_opt_crop}
	\end{subfigure}
	\begin{subfigure}{0.18\textwidth}
	    \captionsetup{justification=centering}
		\includegraphics[width=\textwidth]{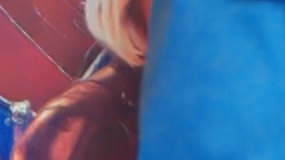}
		\caption{PC-UNet}
		\label{fig:crvd_1_5_iso6400:pc_unet_crop}
	\end{subfigure}
	\begin{subfigure}{0.18\textwidth}
	    \captionsetup{justification=centering}
		\includegraphics[width=\textwidth]{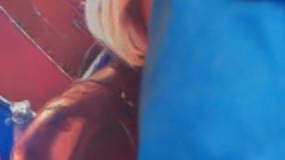}
		\caption{PC-DnCNN}
		\label{fig:crvd_1_5_iso6400:pc_dncnn_crop}
	\end{subfigure}
	\begin{subfigure}{0.18\textwidth}
	    \captionsetup{justification=centering}
		\includegraphics[width=\textwidth]{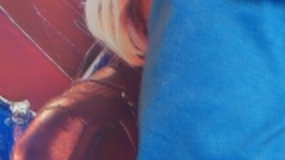}
		\caption{Clean}
		\label{fig:crvd_1_5_iso6400:clean_crop}
	\end{subfigure}
	\begin{subfigure}{0.18\textwidth}
	    \captionsetup{justification=centering}
		\includegraphics[width=\textwidth]{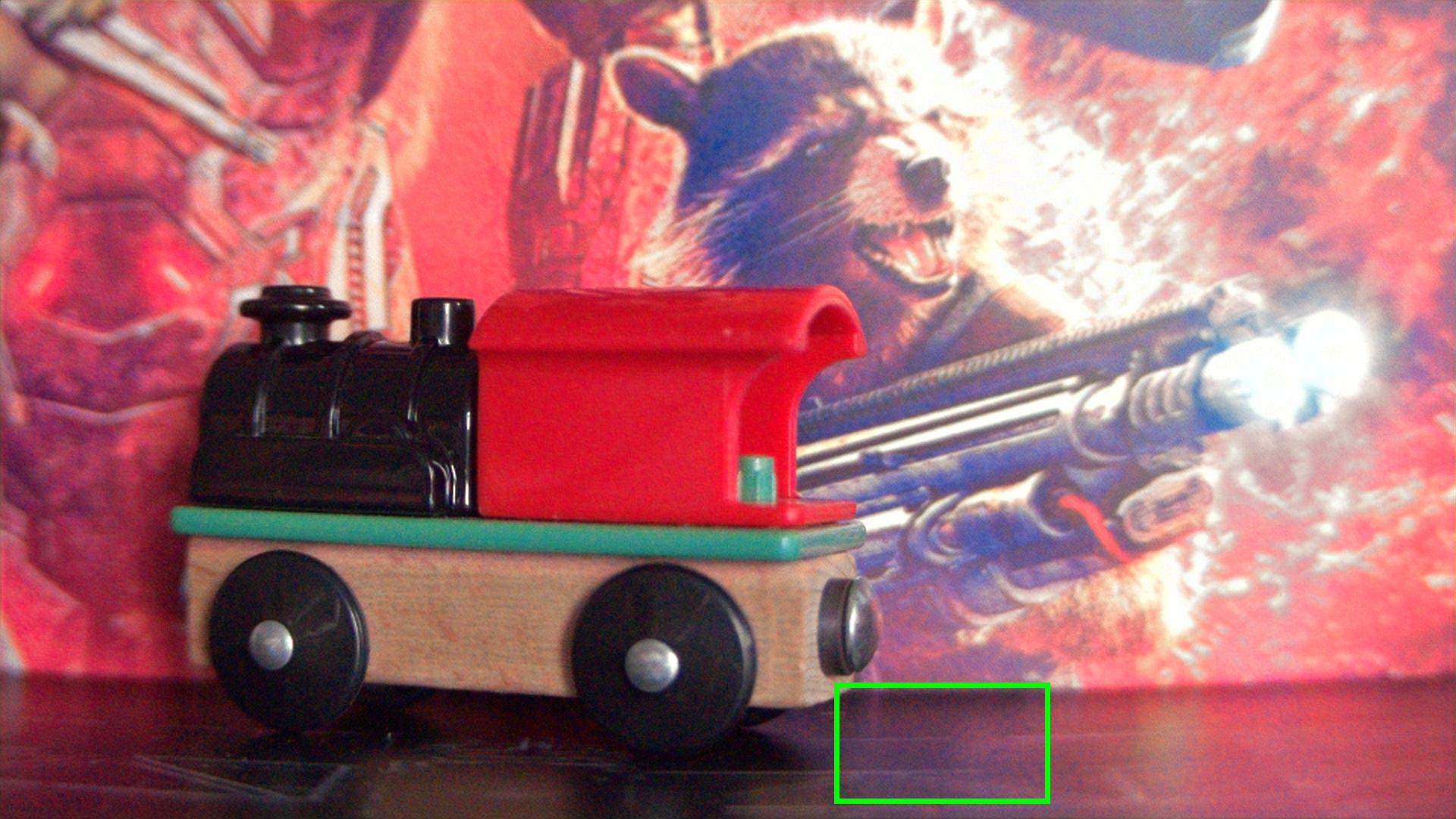}
		\caption{Noisy \\ 35.56 / 0.886}
		\label{fig:crvd_0_4_iso3200:noisy_rect}
	\end{subfigure}
	\begin{subfigure}{0.18\textwidth}
	    \captionsetup{justification=centering}
		\includegraphics[width=\textwidth]{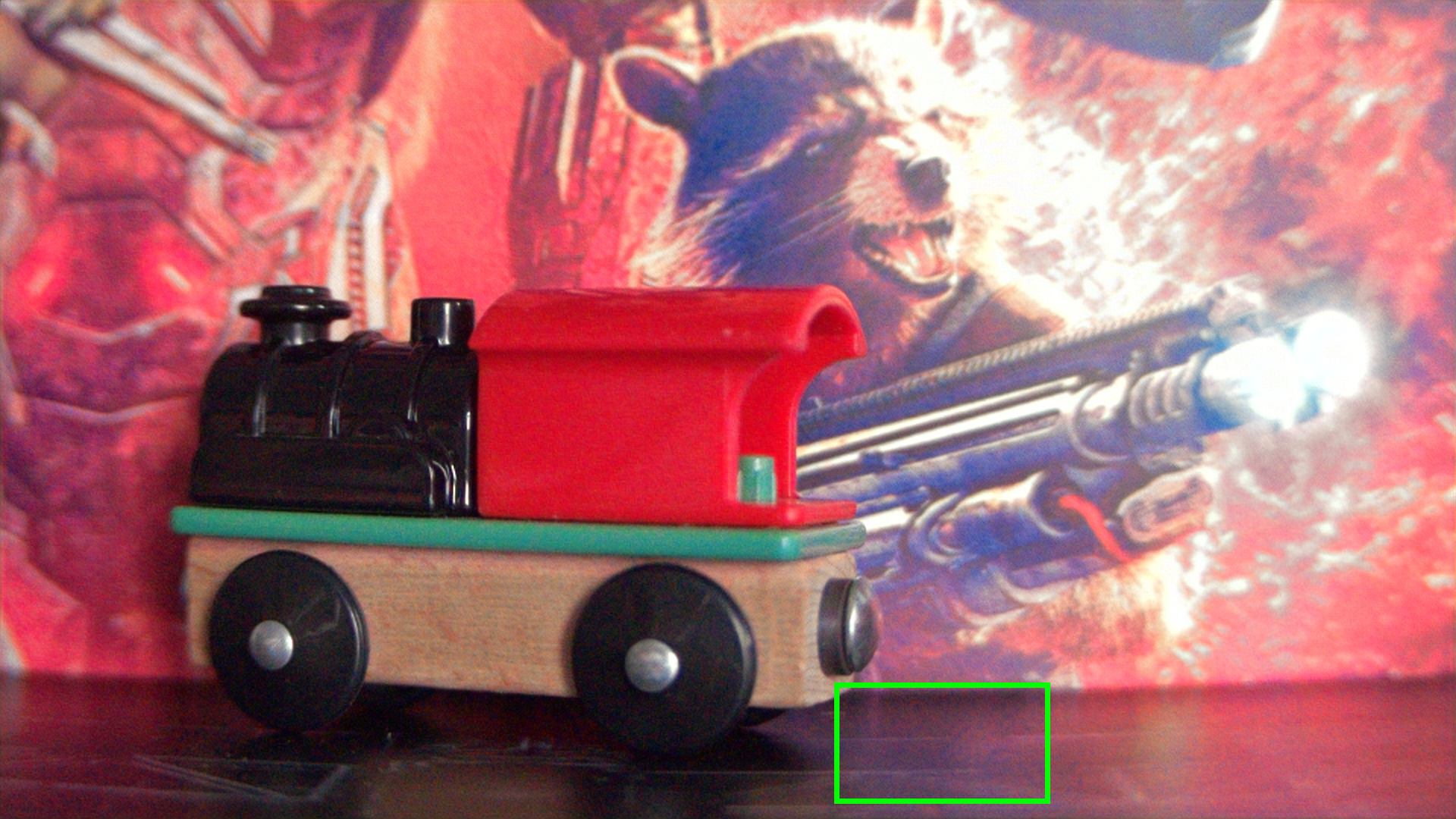}
		\caption{N2N \\ 35.54 / 0.886}
		\label{fig:crvd_0_4_iso3200:n2n_rect}
	\end{subfigure}
	\begin{subfigure}{0.18\textwidth}
	    \captionsetup{justification=centering}
		\includegraphics[width=\textwidth]{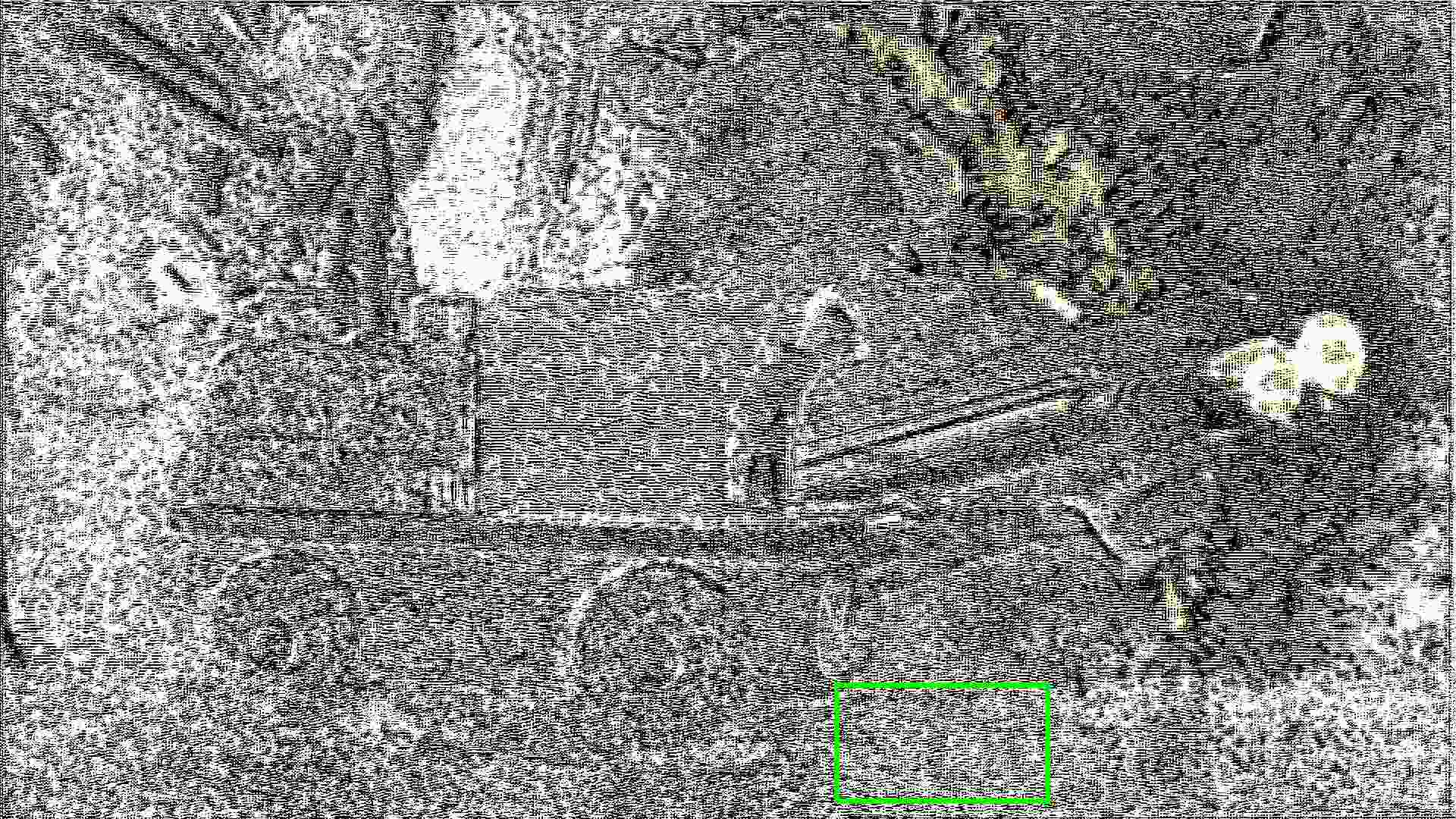}
		\caption{B2U \\ 4.66 / 0.012}
		\label{fig:crvd_0_4_iso3200:b2u_rect}
	\end{subfigure}
	\begin{subfigure}{0.18\textwidth}
	    \captionsetup{justification=centering}
		\includegraphics[width=\textwidth]{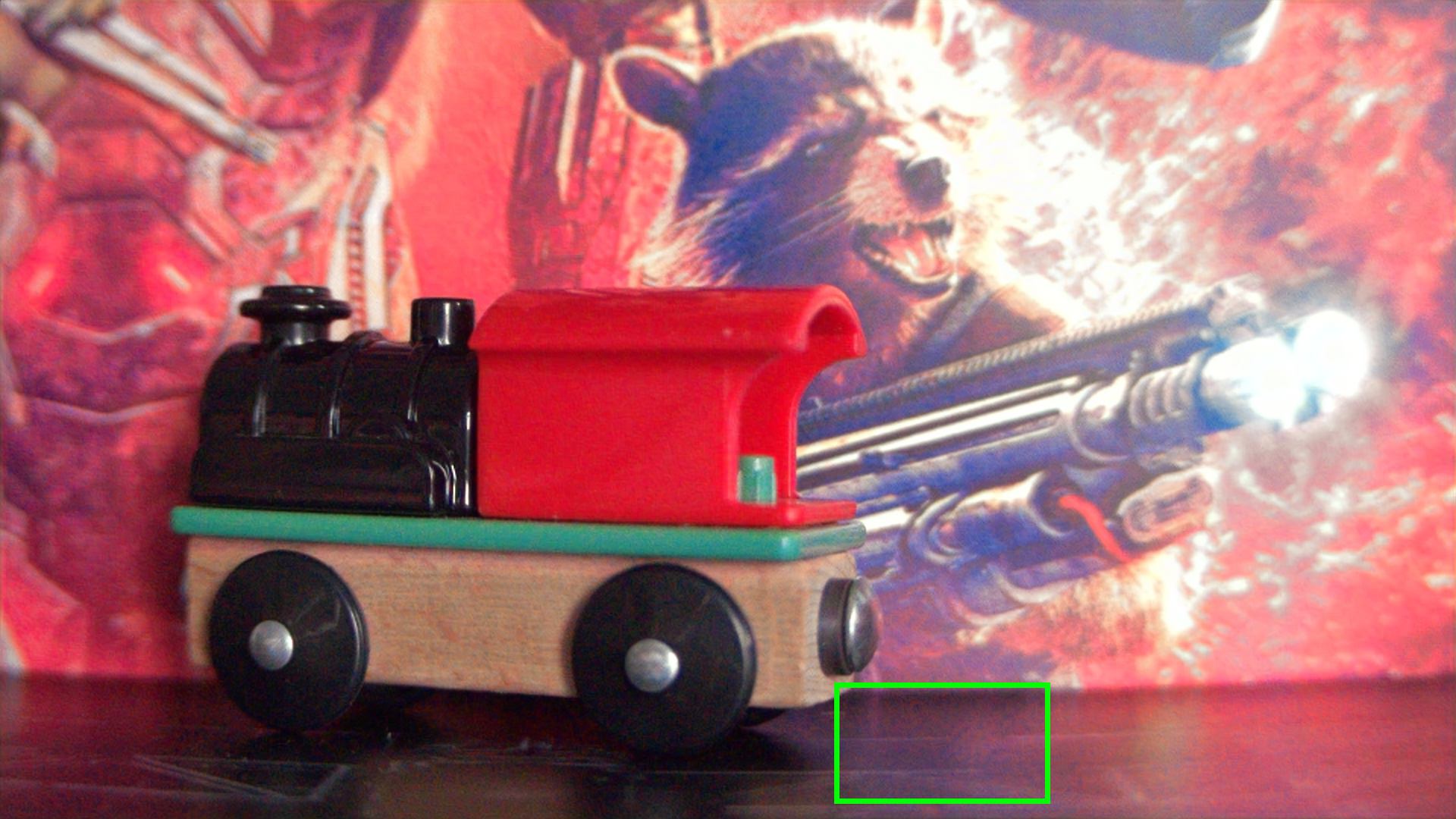}
		\caption{BM3D \\ 36.57 / 0.916}
		\label{fig:crvd_0_4_iso3200:bm3d_rect}
	\end{subfigure}
	\begin{subfigure}{0.18\textwidth}
	    \captionsetup{justification=centering}
		\includegraphics[width=\textwidth]{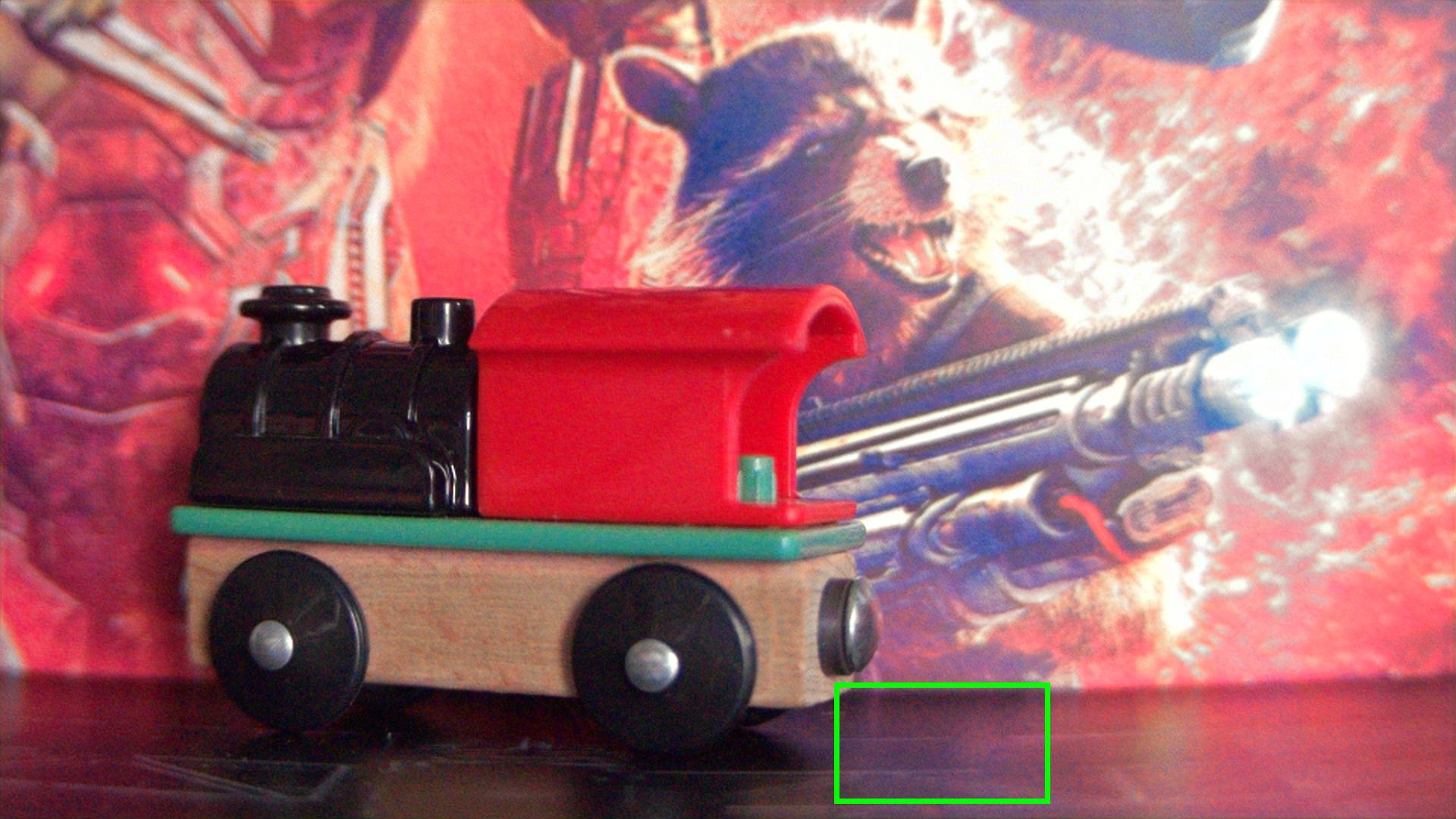}
		\caption{B-DnCNN \\ 35.67 / 0.888}
		\label{fig:crvd_0_4_iso3200:b_dncnn_rect}
	\end{subfigure}
	\begin{subfigure}{0.18\textwidth}
	    \captionsetup{justification=centering}
		\includegraphics[width=\textwidth]{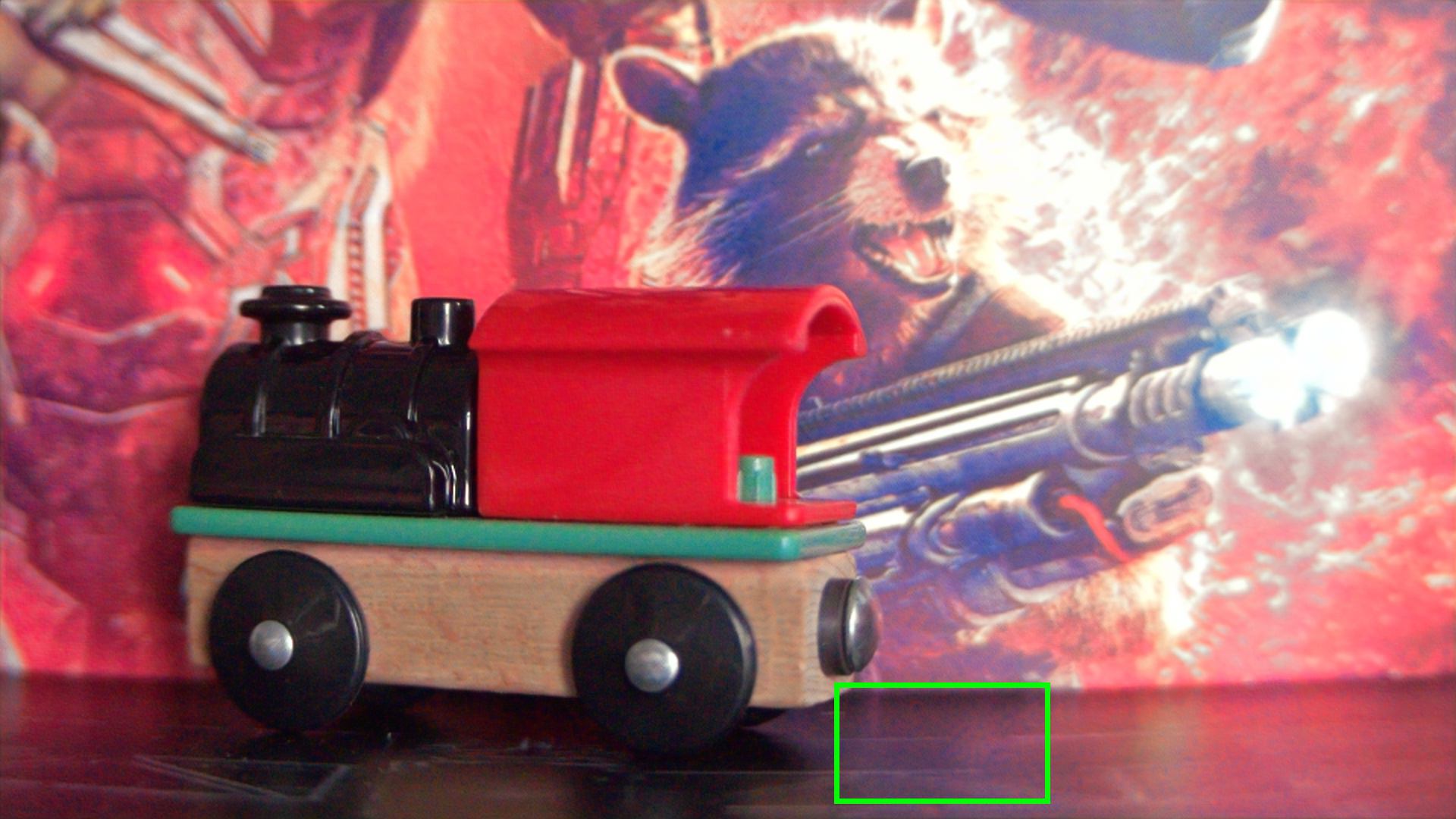}
		\caption{R2R \\ 37.65 / 0.941}
		\label{fig:crvd_0_4_iso3200:r2r_rect}
	\end{subfigure}
	\begin{subfigure}{0.18\textwidth}
	    \captionsetup{justification=centering}
		\includegraphics[width=\textwidth]{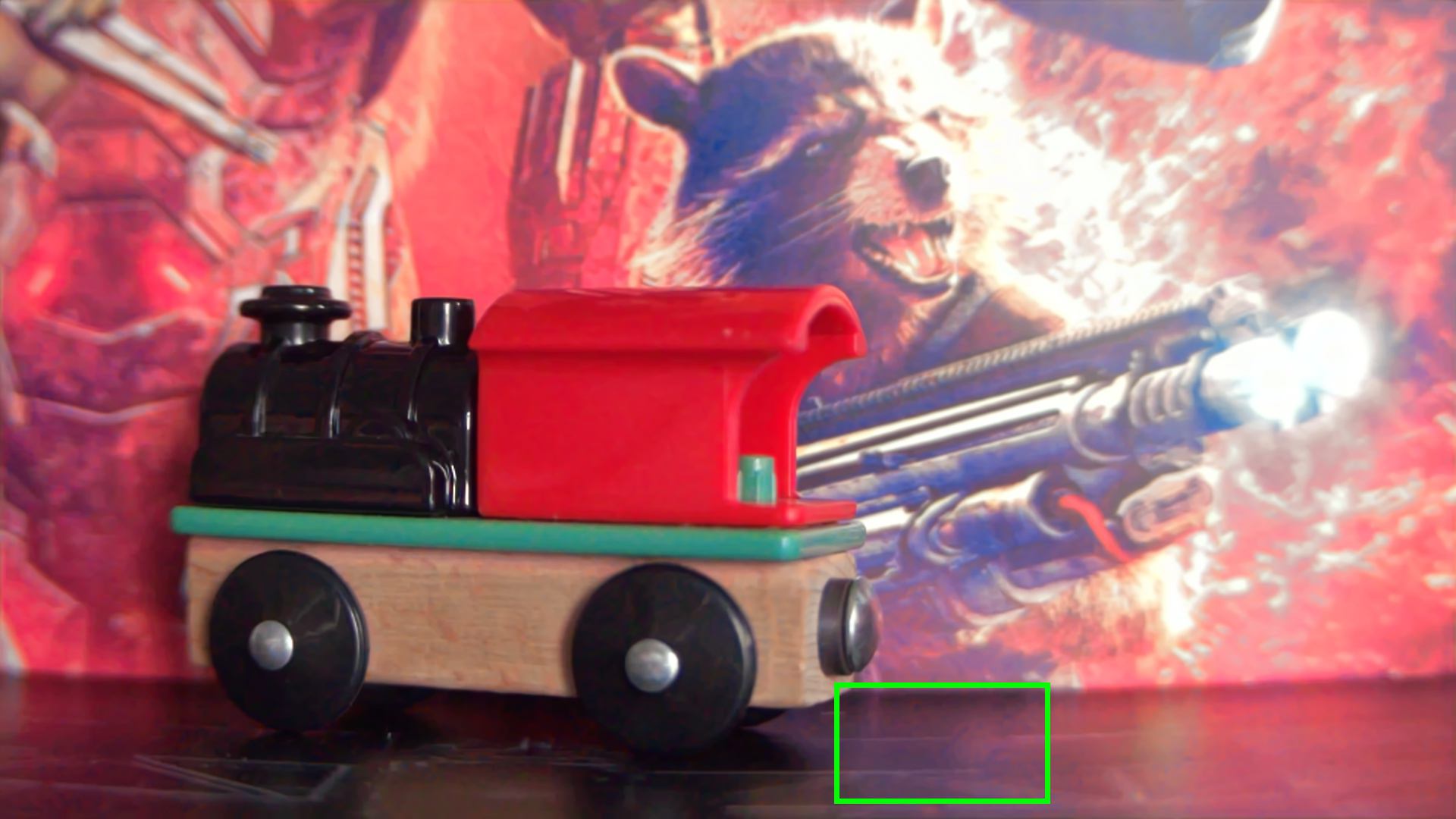}
		\caption{BM3D-O \\ 39.17 / 0.966}
		\label{fig:crvd_0_4_iso3200:bm3d_opt_rect}
	\end{subfigure}
	\begin{subfigure}{0.18\textwidth}
	    \captionsetup{justification=centering}
		\includegraphics[width=\textwidth]{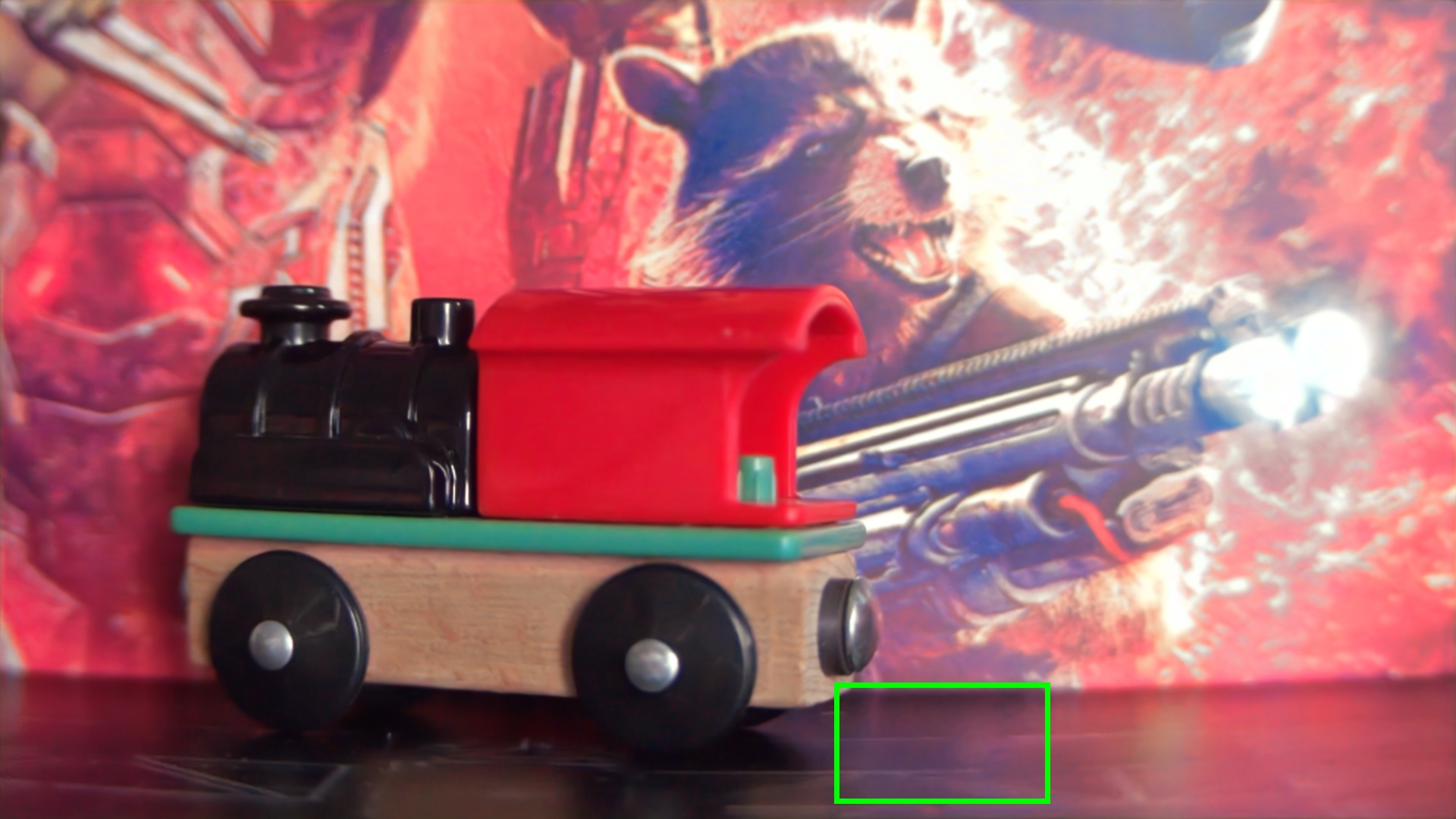}
		\caption{PC-UNet \\ 40.14 / 0.976}
		\label{fig:crvd_0_4_iso3200:pc_unet_rect}
	\end{subfigure}
	\begin{subfigure}{0.18\textwidth}
	    \captionsetup{justification=centering}
		\includegraphics[width=\textwidth]{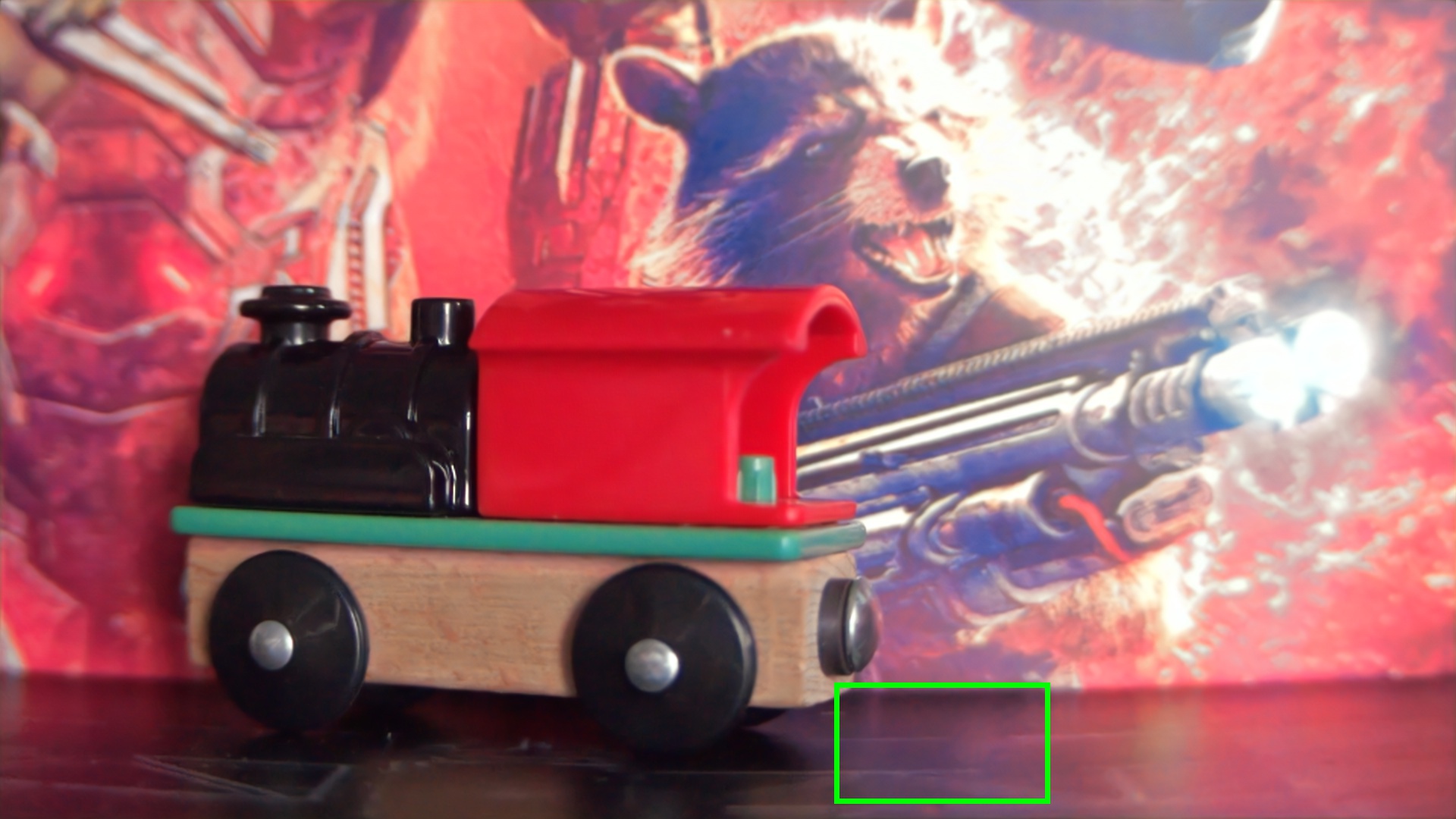}
		\caption{PC-DnCNN \\ 40.16 / 0.976}
		\label{fig:crvd_0_4_iso3200:pc_dncnn_rect}
	\end{subfigure}
	\begin{subfigure}{0.18\textwidth}
	    \captionsetup{justification=centering}
		\includegraphics[width=\textwidth]{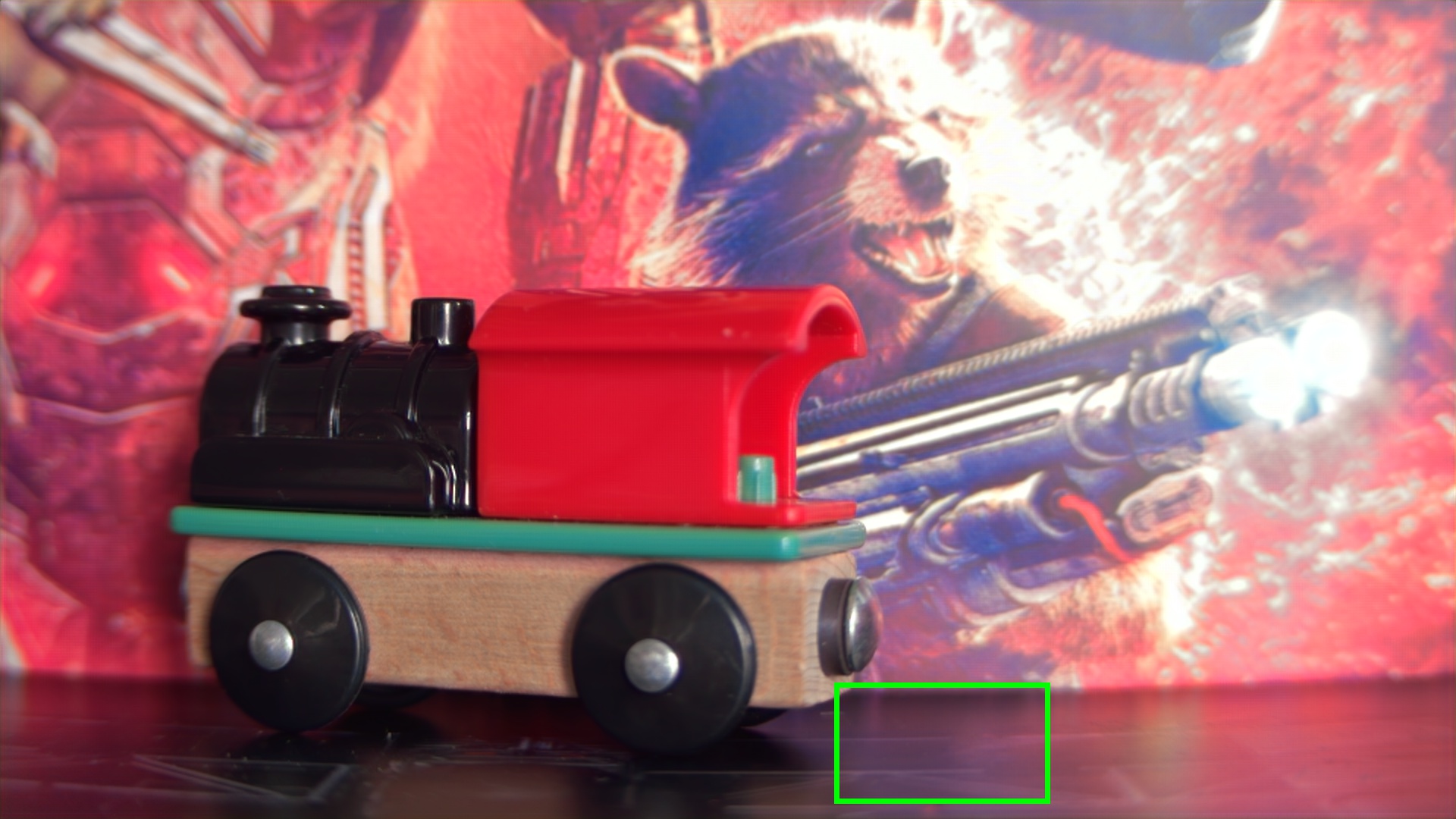}
		\caption{Clean \newline }
		\label{fig:crvd_0_4_iso3200:clean_rect}
	\end{subfigure}
	\begin{subfigure}{0.18\textwidth}
	    \captionsetup{justification=centering}
		\includegraphics[width=\textwidth]{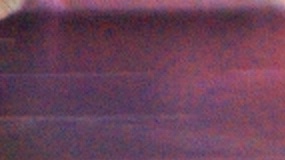}
		\caption{Noisy}
		\label{fig:crvd_0_4_iso3200:noisy_crop}
	\end{subfigure}
	\begin{subfigure}{0.18\textwidth}
	    \captionsetup{justification=centering}
		\includegraphics[width=\textwidth]{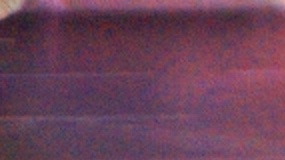}
		\caption{N2N}
		\label{fig:crvd_0_4_iso3200:n2n_crop}
	\end{subfigure}
	\begin{subfigure}{0.18\textwidth}
	    \captionsetup{justification=centering}
		\includegraphics[width=\textwidth]{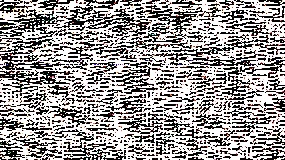}
		\caption{B2U}
		\label{fig:crvd_0_4_iso3200:b2u_crop}
	\end{subfigure}
	\begin{subfigure}{0.18\textwidth}
	    \captionsetup{justification=centering}
		\includegraphics[width=\textwidth]{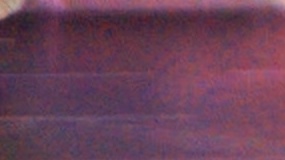}
		\caption{BM3D}
		\label{fig:crvd_0_4_iso3200:bm3d_crop}
	\end{subfigure}
	\begin{subfigure}{0.18\textwidth}
	    \captionsetup{justification=centering}
		\includegraphics[width=\textwidth]{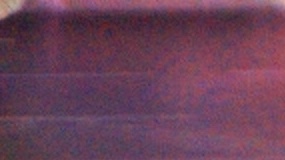}
		\caption{B-DnCNN}
		\label{fig:crvd_0_4_iso3200:b_dncnn_crop}
	\end{subfigure}
	\begin{subfigure}{0.18\textwidth}
	    \captionsetup{justification=centering}
		\includegraphics[width=\textwidth]{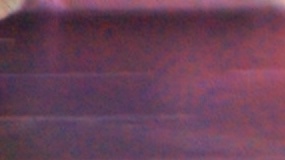}
		\caption{R2R}
		\label{fig:crvd_0_4_iso3200:r2r_crop}
	\end{subfigure}
	\begin{subfigure}{0.18\textwidth}
	    \captionsetup{justification=centering}
		\includegraphics[width=\textwidth]{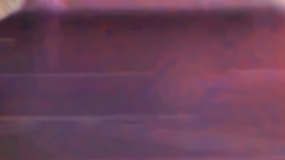}
		\caption{BM3D-O}
		\label{fig:crvd_0_4_iso3200:bm3d_opt_crop}
	\end{subfigure}
	\begin{subfigure}{0.18\textwidth}
	    \captionsetup{justification=centering}
		\includegraphics[width=\textwidth]{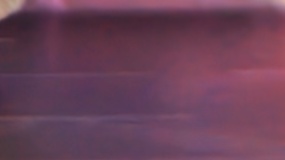}
		\caption{PC-UNet}
		\label{fig:crvd_0_4_iso3200:pc_unet_crop}
	\end{subfigure}
	\begin{subfigure}{0.18\textwidth}
	    \captionsetup{justification=centering}
		\includegraphics[width=\textwidth]{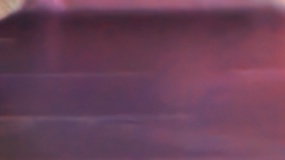}
		\caption{PC-DnCNN}
		\label{fig:crvd_0_4_iso3200:pc_dncnn_crop}
	\end{subfigure}
	\begin{subfigure}{0.18\textwidth}
	    \captionsetup{justification=centering}
		\includegraphics[width=\textwidth]{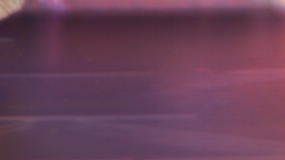}
		\caption{Clean}
		\label{fig:crvd_0_4_iso3200:clean_crop}
	\end{subfigure}
	\caption{Denoising examples with real-world noise. The first four rows show frame 6 of scene 2 captured with ISO 6400. The last four rows present frame 5 of scene 1 taken with ISO 3200. As can be seen, oracle BM3D leaves a noticeable amount of low-frequency noise unfiltered, while other algorithms, except ours (PC-Unet and PC-DnCNN), do not succeed in removing the noise.}
	\label{fig:crvd_1_5_iso6400_0_4_iso3200}
\end{figure*}

\end{document}